\documentclass[times,final]{elsarticle}
\usepackage{graphicx}
\usepackage{subfigure}
\usepackage{amssymb}
\usepackage{amstext}
\usepackage{amsmath}
\usepackage{amscd}
\usepackage{amsthm}
\usepackage{amsfonts}
\usepackage{enumerate}
\usepackage{graphicx}
\usepackage{latexsym}
\usepackage{mathrsfs}
\usepackage{mathtools}
\usepackage{manfnt}
\usepackage{empheq}
\usepackage{cases}
\usepackage[]{xcolor}
\usepackage{verbatim}
\usepackage{bm}
\usepackage{caption}
\usepackage{ragged2e} % for text justification
\usepackage[top=2.2cm, bottom=2.2cm, left=2.0cm, right=2.0cm]{geometry}
\usepackage[bookmarks=false]{hyperref}
\usepackage[capitalize,noabbrev]{cleveref}
\usepackage{tikz,pgfplots,siunitx}
\usetikzlibrary{calc,shapes}
\usetikzlibrary{positioning,arrows,fit,patterns}
\pgfplotsset{compat=newest}
\usetikzlibrary{external}

\usepgfplotslibrary{external}
\usepackage{booktabs}
\usepackage{colortbl}
\theoremstyle{plain}
\newtheorem{theorem}{Theorem}[section]
\newtheorem{example}{Example}[section]
\newtheorem{proposition}[theorem]{Proposition}
\newtheorem{lemma}[theorem]{Lemma}
\newtheorem{corollary}[theorem]{Corollary}
\theoremstyle{definition}
\newtheorem{definition}[theorem]{Definition}
\newtheorem{assumption}[theorem]{Assumption}
\theoremstyle{remark}
\newtheorem{remark}{Remark}[section]

\newcommand{\pressure}{{\mathrm{p}}}
\newcommand{\velocity}{\boldsymbol{\mathrm{v}}}
\newcommand{\ii}{\mathrm{i}}
\newcommand{\bx}{\boldsymbol{x}}

\newcommand{\bid}{\boldsymbol{\mathrm{I_d}}}
\newcommand{\bk}{\boldsymbol{\mathcal{K}}}
\newcommand{\bq}{{\mathcal{Q}}}
\newcommand{\br}{{\mathcal{R}}}
\newcommand{\E}{\boldsymbol{\mathbb{E}}}
\newcommand{\rad}{\boldsymbol{\epsilon}}
\newcommand{\prob}{\boldsymbol{\mathrm{P}}}
\newcommand{\Lp}{\mathrm{L}^2}
\newcommand{\forward}{\boldsymbol{\mathcal{F}}}
\newcommand{\Gscr}{\boldsymbol{\mathscr{G}}}
\newcommand{\Gcal}{\boldsymbol{\mathcal{G}}}
\newcommand{\no}{\boldsymbol{\mathscr{N}}}
\newcommand{\IDFT}{\boldsymbol{\mathrm{IDFT}}}
\newcommand{\DFT}{\boldsymbol{\mathrm{DFT}}}
\usepackage{csquotes}

\newcommand{\antonio}[1]{{\color{orange!90!black}[antonio: #1]}}

\newcommand{\MNO}{{\textit{s}}{\text{NO}}}
\newcommand{\NONeXt}{{\textit{s}}{\text{NO}}+\varepsilon \mathrm{I}}
\newcommand{\MFNO}{{\textit{s}}{\text{FNO}}}
\newcommand{\FNONeXt}{{\textit{s}}{\text{FNO}}+\varepsilon \mathrm{I}}

\newcommand{\multiplicative}{\text{sequential}}

\newcommand{\tikzcmark}{%
\tikz[scale=0.23] {
    \draw[line width=0.7,line cap=round] (0.25,0) to [bend left=10] (1,1);
    \draw[line width=0.8,line cap=round] (0,0.35) to [bend right=1] (0.23,0);
}}

\newcommand{\tikzxmark}{%
\tikz[scale=0.23] {
    \draw[line width=0.7,line cap=round] (0,0) to [bend left=6] (1,1);
    \draw[line width=0.7,line cap=round] (0.2,0.95) to [bend right=3] (0.8,0.05);
}}

\newcommand{\bourbaki}{%
\tikz[
  line cap=but,
  line join=round,
  x=.5em,
  very thick,
  y=1*(height("Z")-\pgflinewidth)*(1-sin(10)),
  rotate=-10,
  rounded corners=1.5pt,
]\draw (1, 0) -- (0, 0) -- (1, 1) -- (0, 1);
}
\newcommand{\myblue}   {blue!60!black}

\newcommand{\myred}    {red!70!black}

\newcommand{\mycolor}    {green}
\newlength{\modelwidth }
\newlength{\modelheight}
\newcommand{\modelfile}{}

\definecolor{newcolor}{rgb}{.8,.349,.1}

\journal{arXiv}

\begin{document} 
 \captionsetup{format=plain,labelfont=bf,font=small}

%\verso{J.A.L.B., T.F. \textit{et al.}}

\begin{frontmatter}

\title{Out-of-distributional risk bounds for neural operators with applications to the Helmholtz equation}

\author[1]{\small Jose Antonio Lara Benitez \corref{cor1}}
\cortext[cor1]{Corresponding author: 
  Email: \texttt{antonio.lara@rice.edu}}
  
\author[2]{\small Takashi Furuya \corref{cor2}}
\cortext[cor2]{Corresponding author: 
  Email: \texttt{takashi.furuya0101@gmail.com}}
  
\author[3]{ \small Florian Faucher}
\author[4]{ \small Anastasis Kratsios}
\author[5]{ \small Xavier Tricoche}
\author[1]{\small Maarten V. de Hoop}

\address[1]{ \small Rice University}
\address[2]{ \small Shimane University}
\address[3]{ \small Team Makutu TotalEnergies, Inria Bordeaux, University of Pau and Pays de l'Adour}
\address[4]{ \small McMaster University and the Vector Institute}
\address[5]{ \small Purdue University}

%\affil[1, 2, 4 ]{{\footnotesize Rice University,  Shimane University, Purdue University}}

%\affil[2]{{\footnotesize Shimane University}}

%\affil[3]{{\footnotesize Team Makutu, Inria Bordeaux, University of Pau and Pays de l'Adour.}}

%\affil[4]{{\footnotesize Purdue University}}

%\affil[a,b]{{\footnotesize Email: \texttt{antonio.lara@rice.edu}, \texttt{takashi.furuya0101@gmail.com}}}

%\affil[*]{{\footnotesize These two authors contributed equally to this work}}

\date{}

\begin{abstract}
\noindent Despite their remarkable success in approximating a wide range of operators defined by PDEs, existing \emph{neural operators} (NOs) do not necessarily perform well for all physics problems. We focus here on high-frequency waves to highlight possible shortcomings. To resolve these, we propose a subfamily of NOs enabling an enhanced empirical approximation of the nonlinear operator mapping wave speed to solution, or boundary values for the Helmholtz equation on a bounded domain. The latter operator is commonly referred to as the ''\emph{forward}’’ operator in the study of inverse problems. Our methodology draws inspiration from transformers and techniques such as stochastic depth. Our experiments reveal certain surprises in the generalization and the relevance of introducing stochastic depth. Our NOs show superior performance as compared with standard NOs, not only for testing within the training distribution but also for out-of-distribution scenarios. To delve into this observation, we offer an in-depth analysis of the Rademacher complexity associated with our modified models and prove an upper bound tied to their stochastic depth that existing NOs do not satisfy. Furthermore, we obtain a novel \emph{out-of-distribution} risk bound tailored to Gaussian measures on Banach spaces, again relating stochastic depth with the bound. We conclude by proposing a hypernetwork version of the subfamily of NOs as a surrogate model for the mentioned forward operator. 

\iffalse{ 

Despite their remarkable success in approximating certain operators defined by PDEs, existing \emph{neural operators} (NOs) perform poorly in a variety of problems of practical relevance. We document this shortcoming and propose a subfamily of NOs with an enhanced empirical approximation of the non-linear operator defined by the Helmholtz equation. Our methodology draws inspiration from transformer and techniques such as stochastic depth. The experiments reveal major surprises in the generalization and the impact of stochastic depth. The network demonstrates superior performance for wave propagation compared to traditional NOs, not only for data within the training distribution but also for \emph{out-of-distribution} scenarios. To delve into this observation, we offer an in-depth analysis of the Rademacher complexity associated with our modified models and prove an upper bound tied to their stochastic depth that existing NOs do not satisfy. Furthermore, we perform a novel out-of-distribution risk bound tailored for Gaussian measures on Banach spaces, again relating stochastic depth with the bound. We conclude by proposing a hypernetwork version of the subfamily as a surrogate model for the forward operator.
} 
\fi
\end{abstract}
\end{frontmatter}
%%%%%%%%%%%%%%%%%%%%%%%%%%%%%%%%%%%%%%%%%%%%%%%%%%%%%%%%%%%%%%%%%
\section{Introduction}\label{Introduction.}
     Data-driven approximation of operators is gaining momentum due to its potential to approximate operators over expensive numerical solvers at a fraction of the computational cost, particularly in the context of parametric partial differential equations (PDEs). This approach proves particularly advantageous in scenarios where constitutive laws are approximated, or only data are available. Once the model is fully-trained, the solution is, up to an approximation error, obtained by evaluating the neural network with restrictions on the input, i.e., the test data are drawn from the same or a sufficiently similar distribution as the training data.
Numerous architectures have been proposed in recent years, such as DeepONets \citep{Lu_2021, lu2022comprehensive}, PCA-Net \citep{SMAI-JCM_2021__7__121_0, hesthaven2018non}, PINNs \citep{RAISSI2019686, karniadakis2021physics}, and neural operators (NOs) \citep{li2020neural, kovachki2021neural, li2020fourier}. Among them, Fourier neural operators (FNOs) have gained widespread popularity. Indeed, they can efficiently compute the costly integral operator, they enjoy ''\textit{discretization invariance}''\footnote{In the sense of zero-shot super-resolution, that is, training in a coarse grid and testing in a finer grid.}, and are universal approximators, under the assumption of regularity and separability of spaces. FNOs or iterations thereof, \citep{brandstetter2022clifford, wen2022u, tapas2022Wavelet, cao2023lno}, have rapidly become the network of choice, finding applications in various domains \citep{pathak2022fourcastnet, guan2021fourier, yin2022learned, wen2022accelerating, li2022solving, grady2022towards, kurth2022fourcastnet}. 

FNOs have shown promising results in certain 2D PDE problems, e.g., \textit{incompressible Navier-Stokes equation}, \citep[Section 3.2]{kovachki2021universal}, and even some non-linear inverse problems, \citep{yang2021seismic, molinaro2023neural}. However, application to realistically complex large-scale problems remains an issue, despite some recent progress~\cite{grady2022towards}. NOs are the natural generalization of multilayer perceptron (MLPs) to functional spaces, and they share their limitations. 
For example, \citet{you2022learning} have shown that deep FNOs perform poorly on some nonlinear operators for PDEs, despite being theoretically universal \citep{kovachki2021universal}. These findings underscore the need for architectures that possess more desirable properties in implementation. Moreover, the increasing interest in enhancing or replacing traditional numerical methods has prompted a focus on understanding the generalization capabilities and training dynamics rather than solely relying on the approximation power of networks, e.g. \cite{marcati2023exponential, lanthaler2022nonlinear, lanthaler2022error, lanthaler2023nonlocal, deng2021convergence}. \par

In this paper we focus on \emph{the out-of-sample, or generalization}, performance of neural operators trained from finitely many noisy inputs. We consider neural operators of the form
\begin{equation}\label{Intro:main}
v_{\ell+1} = \sigma \circ ( W_{\ell} + \bk_\ell  + b_{\ell})\circ v_{\ell},
\end{equation} 
or our proposed network
\begin{equation} \label{Intro:proposed}
v_{\ell+1} = \big(\bid + \boldsymbol{\mathrm{X}}_{\ell}\, {f_\ell} \circ \mathrm{N} \big)\circ \big(\,\bid + \boldsymbol{\mathrm{X}}_{\ell}\,\sigma \circ ( \bk_\ell  + b_{\ell}) \circ \mathrm{N}\, \big)\circ v_\ell.
\end{equation} 
Here, $\bk v(x) = \int k(x,y) v(y) dy$ represents integral operators, where the kernel function $k_{i,j}$ is uniformly bounded for each point $x$ and $y$. This boundedness property allows us to establish theorems that hold regardless of the choice of basis expansions for $\bk$. $\boldsymbol{\mathrm{X}}_\ell$ are Bernoulli random variables, $\boldsymbol{\mathrm{X}}_\ell \sim \mathrm{Ber}(\mathrm{p}_\ell)$, acting as ''switches'', controlling the propagation of information within the network, and adding extra randomness in the training. The process of adding $\boldsymbol{\mathrm{X}}_\ell \sim \mathrm{Ber}(\mathrm{p}_\ell)$, such that $\mathrm{p}_\ell$ decreases with depth is known as \emph{stochastic depth} \citep{huang2016deep}. Finally, $f_\ell$ is a simple multilayer perceptron (MLP), $\sigma$ the activation, $\mathrm{N}$ a normalizer, and $\bid$ the identity operator, which we introduce formally in \cref{FNO+epsilon v2}. \par

The theory of generalization for neural operators is still in its early stages, with ongoing advancements in the field, as \citet{kim2022bounding}. However, these methods have primarily focused on finite-dimensional parameters, as they rely on established theorems within the statistical learning community. Nevertheless, the underlying theory can be extended to encompass a broader range of kernel functions, beyond those approximated by Fourier basis. In our work, we have extended the theory to a wider class of kernel functions and have avoided relying on the constraints of finite dimensionality. This allows for the consideration of alternative bases for expressing the integral operator, such as wavelet basis, spherical harmonics, and others. \par 

Additionally, the theory of generalization in neural networks encompasses their ability to handle perturbations in the underlying distribution, including out-of-distribution scenarios. While empirical and theoretical results for operator learning are relatively scarce and challenging to obtain, there are notable exceptions, such as the work by \citet[Sec 4.1.2]{de2023convergence} that focuses on learning linear operators from data. In our work, we make a further contribution to this area by investigating the robustness of the proposed network \eqref{Intro:proposed}, to changes in the input distribution. Empirically, we observe that the network exhibits robustness to such changes. Theoretically, we leverage properties from the theory of general Gaussian measures on Banach spaces and the duality of the Wasserstein $1$ distance to establish an upper bound on the network's robustness to a change of measure. It is important to note that the random variables in \eqref{Intro:proposed} play a significant role in controlling the bound, particularly as the depth of the networks increases. These findings shed light on the generalization capabilities of the networks and provide insights into their behavior beyond the training distribution. However, it is worth mentioning that our bounds rely on estimates of the Lipschitz constant, and those are not tight. Strictly speaking, further analysis is needed to fully understand the growth and provide a complete explanation of the observed out-of-distribution behavior.\par

Our proposed architecture modifications in \eqref{Intro:proposed} borrow ideas from transformers, in particular to the encoder part, whose layers can be described as
\begin{equation} \label{Intro:encoderTransformer}
v_{\ell+1} = \big(\bid + \boldsymbol{\mathrm{X}}_{\ell}\, {f_\ell} \circ \mathrm{N} \big)\circ \big(\,\bid + \boldsymbol{\mathrm{X}}_{\ell}\, \circ \mathrm{Attn} \, \big)\circ v_\ell, \quad \mathrm{Attn}(v_\ell) = \mathrm{softmax}\left( \mathrm{Const.}\,Q(v_\ell)K(v_\ell)^\top\right)V(v_\ell),
\end{equation} 
for $v_\ell\in \mathbb{R}^{n\times d}$, and $Q(v_\ell) = v_\ell W^{Q}$, $K(v_\ell) = v_\ell W^{K}$, $V(v_\ell) = v_\ell W^{V}$ for $W^{Q}, W^{K}, W^{V}\in \mathbb{R}^{d\times d}$. As we delve into the subject, we will discover how this approach grants us significant control over the complexity class within the family, while effectively bounding the out-of-distribution risk through stochastic depth for Gaussian measures. Additionally, it empowers us to leverage a proven network layout, which has consistently demonstrated promising empirical results across various domains. In recent years, there has been a shift towards the adoption of transformer-based architectures \citep{vaswani2017attention, devlin2018bert, dosovitskiy2020image, liu2021swin, Acciaio2022_GHT} throughout machine learning. These architectures, which include widely publicized models such as BERT and ChatGPT \citep{devlin2018bert,liu2023summary}, have shown remarkable success in various tasks, outperforming previous state-of-the-art models. \citet[Section 3.3]{kovachki2021neural} has identified a connection between transformers and neural operators, where self-attention can be viewed as a Monte Carlo approximation of a \textit{nonlinear integral operator}, showing that the underlying principles of these seemingly different architectures are linked. \citet{cao2021choose, kissas2022learning, li2022transformer} explored transformers for parametric PDEs. Despite the promising results for small-scale problems, using transformers to approximate operators is hindered by the inherent scalability issue of self-attention\footnote{In transformer applications, datasets are massive, but individual data samples are relatively small compared to those in PDE-related problems, particularly in $3$D cases. In areas like vision, attention is typically applied to image patches instead of at a pixel-wise level to reduce computational cost, e.g. \citep{dosovitskiy2020image}}. Attention operations have a cost of $\mathcal{O}(n^2)$, making them prohibitively expensive for realistic $3$D inputs. Incorporating workarounds like shifted windows in visual transformers, as seen in \citep{liu2021swin}, can be beneficial for certain applications. However, the absence of a solid theoretical foundation in these approaches makes it challenging to analyze the convergence of the architecture, particularly in scientific computing scenarios. In contrast, the \textit{convolutional integral operator} in FNOs can be efficiently estimated by the Fast Fourier Transform (FFT) with a computational cost of $\mathcal{O}(n \log n)$. Furthermore, the adaptive Fourier neural operator (AFNO) \citep{guibas2021adaptive} presents a promising approach to address the scalability limitations of transformers.  Nonetheless, current applications have been primarily limited to vision, and further research is needed to explore these architectures in scientific computing. 

Extensive empirical evidence has shown that design choices in transformers can yield significant improvements in the capacity of network families, training stability, generalization performance to in-distribution-data, and sometimes out-of-distribution, \citep{hendrycks2020pretrained}. This has resulted in a growing trend in various fields of machine learning to adopt \enquote{\emph{transformized}} architectures \citep{tolstikhin2021mlp, lee2021fnet, yu2022metaformer, martins2021infty, rao2021global, liu2022convnet}. The work of \citet{yu2022metaformer}, abstracts the self-attention of the transformer leading to \textit{metaformer} architecture. Here we take advantage of the abstracted layout of this approach to overcome limitations associated with traditional self-attention in terms of input's dimension. Further, this opens up possibilities for designing transformer-based models that can effectively tackle problems arising in scientific computing on an ad-hoc basis.  \par

\subsection*{Our Contributions}
\begingroup
  \setlength\itemsep{1.mm}
\begin{enumerate}[(a)]
\item We introduce modifications to neural operators to adopt a transformer-like architecture, drawing inspiration from works such as \citep{liu2022convnet, lee2021fnet,  yu2022metaformer}. The resulting network (\cref{{Proposed Method}}) is referred to as $\FNONeXt$ and $\NONeXt$, respectively, for experiments and theory, where the $\varepsilon$ indicates that ''minor'' changes are incorporated, and the $s$ stands for sequential, as we preserve the arrangement: non-local (integral operator layer as ''token mixer''), and local (MLP layer as ''feature mixing'') in transformers (contrasting with traditional NOs). 
\begin{comment}In Remark \ref{remark:PDOs}, it is observed that this arrangement has similarities with the computational approximation of pseudo-differential operators ($\Psi$DOs), if we use Fourier basis.\end{comment}
\item We construct a benchmark for the time-harmonic wave equation according to \citep{faucher2020full}. We observe that modifying FNOs towards $\FNONeXt$ leads to a smaller test loss in the Helmholtz equation (\cref{experiment: multiple_initialization:gelu}) \iffalse{In addition, we tested the architecture in the $2$D benchmarks in \citet{li2020fourier}, }\fi for data \emph{in-distribution}. \iffalse{ In the process, we gradually modify the architecture to understand the impact of the design changes, showing that our findings are consistent among multiple realizations of the training path and initialization(s) of the networks. Each step leading toward the \textit{transformer-look-like-version} of the Neural Operator has a gradual increase in the performance gap with (vanilla)-neural operators.}\fi
\begin{comment}
\item Building on the tools in \citep{li2018visualizing} we provide a visualization of the loss landscapes of FNOs, $\FNONeXt$, and the architectures in between. To provide a qualitative analysis of the loss-landscape and the impacts on the training trajectory and the flatness of the local minima. 
\end{comment}
\item  We provide an exhaustive empirical study of the robustness of the trained networks for perturbation in the data distribution. We show that the proposed architecture is able to generalize to \emph{out-of distribution} input, while earlier networks are unable to. Remarkably, the proposed network is able to obtain \textit{reasonable} wave propagation from an \emph{anisotropic} covariance operator, change in the input's range and roughness coefficient, despite being only trained on smooth Gaussian random fields with Whittle–Mat\'ern \emph{isotropic} covariance, and fixed wave speed range (\cref{section:experimental:OOD}). 

\item We propose a hypernetwork version of the architecture, as a surrogate model to effectively approximate the forward operator of the Helmholtz equation (\cref{section:StR}). That is, $(f, c) \to \pressure^{f} \,\vert_{\Sigma}$, where $\pressure^{f}\, \vert_\Sigma$ is the restriction of the wavefield at receivers location for a given source, $f$.

\item  We give theoretical guarantees supporting the out-of-distribution performance of the $\NONeXt$ and $\NONeXt$v2 \eqref{Intro:proposed} models in the case where the inputs are sampled from a centered Gaussian measure $\mu_X$ on various Banach spaces (\cref{section:theory_OOD}). We find that the out-of-sample generalization of both neural operator models is described by the metric entropy of the unit Cameron-Martin space associated with $\mu_X$.  The analysis extends the transport-theoretic tools for deriving risk-bounds introduced in \cite{hou2022instance} and merges it with small-ball estimates for Gaussian processes on Banach spaces, e.g.~\cite{li1999approximation}. 

\item We offer a novel analysis of the \emph{Rademacher complexity} of NOs and \emph{our proposed architecture} \eqref{Intro:proposed} (\cref{GEB}). For NO, our analysis is general in the sense that it applies independently of the discretization and of the choice of basis in the integral operator\footnote{In particular, Fourier basis corresponding to FNOs.}, contrasting with \citep{kim2022bounding}. In addition, our work not only extends the previous results to functional space but also provides a better bound on the Rademacher complexity with order $\mathcal{O}\left(1/n^{\frac{1}{\hat{d}+1}}\right)$ ($n$ is the number of training data, and $\hat{d}$ is the doubling dimension of $D \times D$, where $D$ is the spatial domain), whilst $\mathcal{O}(1)$ in \citep{kim2022bounding}.  For the \emph{Rademacher complexity} of \eqref{Intro:proposed} our analysis is tied to \emph{stochastic depth}. We show that stochastic depth controls the expected Rademacher complexity, irrespective of the number of layers. For instance, if $\boldsymbol{\mathrm{X}}_\ell\sim \mathrm{Ber}(p_\ell)$, and $p_\ell =\mathcal{O}(\ell^{-(1+\varepsilon)})$, where $\ell$ denotes the layer's number, and $\varepsilon>0$, the bound is uniform regardless of $\ell \to \infty$\footnote{similar conclusion is obtained in theoretical analysis of OOD.}. As a consequence, we show that the upper bound of the $\NONeXt$ can always be controlled with depth, while the upper bound of the other neural operators diverges.
 
\end{enumerate}
\endgroup

\iffalse{
\begin{center}
\begin{tikzpicture}
\node[draw, inner sep=0.2cm, rounded corners] (text) at (0,0) {%
\begin{minipage}{0.95\textwidth}
We expect the \emph{ideas in this paper} may influence the evolution of neural operators, and operator learning in general towards architectures with more desirable properties. These may include the preservation of universality, the potential to increase capacity without succumbing to overfitting, improved generalization within the distribution, and a degree of resilience to out-of-distribution inputs, among others.
\end{minipage}%
};
\end{tikzpicture}    
\end{center}
}\fi

\iffalse{
We expect the ideas in this paper may influence the evolution of neural operators towards architectures with more desirable properties, for instance, by preserving universality and reducing the test loss with the same parameter size as (traditional) neural operator counterparts, and some robustness to out-of-distribution input.}
\fi
%% getting a better bang for your buck.

%%%%%%%%%%%%%%%%%%%%%%%%%%%%%%%%%%%%%%%%%%%%%%%%%%%%%%%%%%%%%%%%%
%%%%%%%%%%%%%%%%%%%%%%% For testing format 

\section{Proposed networks: \enquote{Metaforming the neural operator}}\label{Proposed Method}
    In this section, we introduce the architecture known as $\NONeXt$ with stochastic depth. This architecture is designed to enhance the generalization performance and capabilities of neural operators. Here, we provide a comprehensive description of the layers that constitute $\NONeXt$, which are briefly outlined in \cref{Intro:proposed}. However, to understand the impact of different architectural changes, we gradually modify the NOs until obtaining the $\NONeXt$ with stochastic depth. Throughout the next sections, we provide both numerical evidence and theoretical reasoning to support our choices.
 
\paragraph{Neural Operator: standard structure}
We briefly review  NOs \citep{kovachki2021neural, kovachki2021universal}. Let $\bk_\ell$ be a linear integral operator (non-local), see Definition~\ref{IO},
and $W_{\ell}$ be the weight matrix (local). 
The standard layer structure is
\begin{equation}\label{eq:arch:NO}
v_{\ell+1} = \sigma \circ ( W_{\ell} + \bk_\ell  + b_{\ell})\circ v_{\ell},
\end{equation}
$(\ell = 1 ,\ldots, L)$ where $\sigma$ is an element-wise nonlinear activation function, and $b_{\ell}$ is a bias. For $\ell=1$, we have $v_1 = \br(a)$, i.e., the parameter $a$ is lifted by the map $\br$, and finally, the output is projected back to the corresponding space by $\bq$, forming the solution field, $u =  \bq (v_{L+1})$. We refer to Appendix~\ref{NO} for additional explanations.
    \paragraph{Sequential neural operators ($\MNO$)}\label{MNO}
Transformers \citep{yu2022metaformer, vaswani2017attention} adopt a compositional structure, wherein non-local and local layers are arranged sequentially instead of combining the operations within a single layer. The so-called token mixer (e.g. attention) precedes a MLP acting on feature space; see Figure~\ref{fig:MFNO}. 
This structure bears resemblance to the one described by \citet[Section 2.5.1]{kovachki2021universal} for $1$-layer NN, and FNOs (the authors of \citep{kovachki2021universal} observed that universality is preserved, so it can be expanded to MLP architecture with $M$-layers, see \cref{sNO: universal}). 

\noindent We introduce the $\multiplicative$ neural operator ($\MNO$). 
Let $f_\ell$ be a MLP with $M$-layers (local), \citep[Ch. 6]{goodfellow2016deep}. Then, 
\begin{subequations}\label{eq:MNO} 
\begin{empheq}[left={\empheqlbrace}]{align}
w'_\ell &= \sigma \circ ( \bk_\ell  + b_{\ell})\, \circ v_{\ell}, \label{eq:MFNO:fourierLayer}\\ 
v_{\ell+1} &=  {f_\ell} \circ w'_{\ell}, \label{eq:MFNO:MLP}& \end{empheq}
\end{subequations}
$(\ell = 1 ,\ldots, L)$. See, \cref{fig:MFNO}. If $\bk$ is convolutional, we find a significant improvement over the relative $\Lp$-norm compared to traditional FNOs for similar parameter count, see \cref{table:Models_Size} and \cref{fig:boxplot}.

\begin{figure}[!ht]
    \centering % Customize caption style
    \begin{tikzpicture}[scale=0.7]
    % Define the nodes
    \node[draw,trapezium,trapezium left angle=60,trapezium right angle=60,fill=blue!20,rotate=90,minimum width=.5cm,minimum height=.75cm,text width=1cm,text centered, rounded corners] (trap1) at (-3,0) {$\mathcal{R}$};
    \node[draw,trapezium,trapezium left angle=60,trapezium right angle=60,fill=blue!20,rotate=270, minimum width=.5cm,minimum height=.75cm,text width=1cm,text centered, rounded corners] (trap2) at (11.2,0) {$\mathcal{Q}$};
    %nonlocal
    \node[draw,rectangle,fill=orange!25,rotate=90,minimum width=0.5cm,minimum height=1.cm,text width=2.6cm,text centered, rounded corners] (rect1) at (0,0) {$\sigma \circ \left(\bk_1 + b_1\right)$};
    \node[draw,rectangle,fill=orange!25,rotate=90,minimum width=0.5cm,minimum height=1.cm,text width=2.6cm,text centered, rounded corners] (rect2) at (6,0) {$\sigma \circ \left(\bk_L + b_L\right)$};
    %MLPs 
    \node[draw,rectangle,fill=red!25,rotate=90,minimum width=1.cm,minimum height=0.5cm,text width=1cm,text centered, rounded corners] (rect_mlp1) at (2.5,0) {$f_1$};
    \node[draw,rectangle,fill=red!25,rotate=90,minimum width=1.cm,minimum height=0.5cm,text width=1cm,text centered, rounded corners] (rect_mlp2) at (8.5,0) {$f_L$};
    %circles
    \node[draw,circle,fill=gray!20,rotate=0,minimum width=0.75cm,text width=0.5cm,text centered] (circ0) at (-5,0) {$a$}; 
    \node[draw,circle,fill=gray!20,rotate=0,minimum width=0.75cm,text width=0.5cm,text centered] (circ1) at (13.2,0) {$u$};
    
    % Add a dashed square that covers rect1 and rect_mlp1
    \draw[dashed, line width=1.5pt, rounded corners, rotate=90] ($(rect1.north west) + (-0.2, 0.35)$) rectangle ($(rect_mlp1.south east) + (1.4,-0.2)$);
    \draw[dashed, line width=1.5pt, rounded corners, rotate=90] ($(rect2.north west) + (-0.2, 0.35)$) rectangle ($(rect_mlp2.south east) + (1.4,-0.2)$);
    % Draw the arrows
    \draw[ ->, line width=1.5pt] (circ0) -- (trap1);
    \draw[ ->, line width=1.5pt] (trap1) -- (rect1) node[midway,above]{$v_{1}$};
    \draw[ ->, line width=1.5pt] (rect1) -- (rect_mlp1) node[midway,above] {$w_{2}$};
    \draw[dashed, ->, line width=1.5pt] (rect_mlp1) -- (rect2) node[midway,above] {$v_{2}$};
    \draw[->, line width=1.5pt] (rect2) -- (rect_mlp2) node[midway,above] {$w_{L+1}$};
    \draw[->, line width=1.5pt] (rect_mlp2) -- (trap2) node[midway,above] {$v_{L+1}$};
    \draw[->, line width=1.5pt] (trap2) -- (circ1);
    \end{tikzpicture}
        \caption{$\MNO$ is called sequential, as the integral operator is followed by a MLP in a sequential manner. For comparison with the NO, see \cref{fig:NO}.}
        \label{fig:MFNO}
\end{figure}
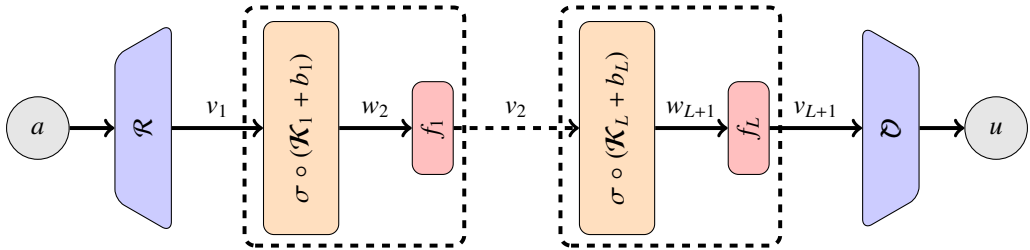 

    \paragraph{$\NONeXt$: $\MNO$ with the identity map--skip connection--}
We now incorporate the addition of the identity map on the $\MNO$ (in machine learning's jargon, this is referred to as a \emph{skip connection}). The use of the symbol $\varepsilon$ in the name is merely to signify that minor changes have been made to the $\MNO$ architecture.

Two variants can be considered: one without, and one with stochastic depth \citep{huang2016deep}, that will allow us to access deep versions of $\NONeXt$. For the sake of brevity, sometimes \emph{we may refer to $\NONeXt$ without stochastic depth as version $1$, and $\NONeXt$ with stochastic depth as version $2$}, in figures, or tables.

\paragraph{$\NONeXt$ without stochastic depth}\label{FNO+epsilon}
Incorporating skip connections, that is $\MNO + \boldsymbol{\mathrm{I}_{\mathrm{d}}}$, lead us to \cref{eq:FNO+epsilon}.
% In addition of the multiplicative structure, we incorporate the skip connection, and a layer normalization. 
The architecture can be seen as an instance of the metaformer \citep{yu2022metaformer}; whence, 
the token mixer is replaced by an integral operator, and the network is extended to functional space \footnote{This has not been done in the previously mentioned paper.}. Using a similar notation, we have
\begin{subequations} \label{eq:FNO+epsilon} \begin{empheq}[left={\empheqlbrace}]{align}
w'_\ell &= \big(\,\bid +\sigma \circ ( \bk_\ell  + b_{\ell}) \circ \mathrm{N}\, \big)\circ v_\ell, \\ 
v_{\ell+1} &= \big(\bid + {f_\ell} \circ \mathrm{N} \big)\circ w'_\ell,\quad
\end{empheq} 
\end{subequations}
$(\ell = 1 ,\ldots, L)$ where $\bid$ is the identity operator, and $\mathrm{N}$ is the layer normalization (or any other normalization). 

\iffalse{
\begin{figure*}[!ht]
    \centering
    \includegraphics[width=0.7\textwidth]{Figures/FnoNext.png}
    \caption{\small{\textbf{($\NONeXt$) version 1}. It is a modification based on the $\multiplicative$ structure in where we incorporate layer normalization and skip connection as in modern Transformers. For comparison with the NO, see \cref{fig:NO}.}}
    \label{fig:FNO+epsilon}
\end{figure*}
}
\fi

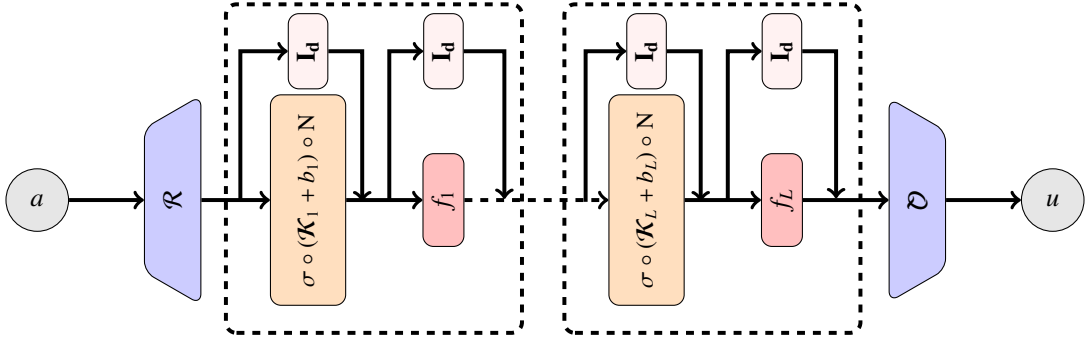
\begin{figure*}[!ht]
    \centering
    \begin{tikzpicture}[scale=0.9]
% Define the nodes
\node[draw,trapezium,trapezium left angle=60,trapezium right angle=60,fill=blue!20,rotate=90,minimum width=.5cm,minimum height=.75cm,text width=1cm,text centered, rounded corners] (trap1) at (-2,0) {$\mathcal{R}$};
\node[draw,trapezium,trapezium left angle=60,trapezium right angle=60,fill=blue!20,rotate=270, minimum width=.5cm,minimum height=.75cm,text width=1cm,text centered, rounded corners] (trap2) at (9,0) {$\mathcal{Q}$};
    
%nonlocal
\node[draw,rectangle,fill=orange!25,rotate=90,minimum width=2.8cm,minimum height=1.cm,text width=2.5cm,text centered, rounded corners] (rect1) at (0,0) {\small$\sigma \circ \left(\bk_1 + b_1\right)\circ \mathrm{N}$};
\node[draw,rectangle,fill=orange!25,rotate=90,minimum width=2.8cm,minimum height=1.cm,text width=2.5cm,text centered, rounded corners] (rect2) at (5,0) {\small $\sigma\circ \left(\bk_L + b_L\right)\circ \mathrm{N}$};

%identity 
\node[draw,rectangle,fill=pink!20,rotate=90,minimum width=1.cm,minimum height=0.5cm,text width=0.5cm,text centered, rounded corners] (Id1) at (0,2.2) {$\boldsymbol{\mathrm{I_d}}$};
\node[draw,rectangle,fill=pink!20,rotate=90,minimum width=1.cm,minimum height=0.5cm,text width=0.5cm,text centered, rounded corners] (Id2) at (2,2.2) {$\boldsymbol{\mathrm{I_d}}$};

\node[draw,rectangle,fill=pink!20,rotate=90,minimum width=1.cm,minimum height=0.5cm,text width=0.5cm,text centered, rounded corners] (Id3) at (5,2.2) {$\boldsymbol{\mathrm{I_d}}$};
\node[draw,rectangle,fill=pink!20,rotate=90,minimum width=1.cm,minimum height=0.5cm,text width=0.5cm,text centered, rounded corners] (Id4) at (7,2.2) {$\boldsymbol{\mathrm{I_d}}$};

%MLPs
\node[draw,rectangle,fill=red!25,rotate=90,minimum width=1.cm,minimum height=0.5cm,text width=1cm,text centered, rounded corners] (rect_mlp1) at (2,0) {$f_1$};
\node[draw,rectangle,fill=red!25,rotate=90,minimum width=1.cm,minimum height=0.5cm,text width=1cm,text centered, rounded corners] (rect_mlp2) at (7,0) {$f_L$};

%circles
\node[draw,circle,fill=gray!20,rotate=0,minimum width=0.75cm,text width=0.5cm,text centered] (circ0) at (-4,0) {$a$};
\node[draw,circle,fill=gray!20,rotate=0,minimum width=0.75cm,text width=0.5cm,text centered] (circ1) at (11,0) {$u$};

% Add a dashed square that covers rect1 and rect_mlp1
\draw[dashed, line width=1.5pt, rounded corners, rotate=90] ($(rect1.north west) + (-0.4, 0.65)$) rectangle ($(rect_mlp1.south east) + (2.2,-0.85)$);
\draw[dashed, line width=1.5pt, rounded corners, rotate=90] ($(rect2.north west) + (-0.4, 0.65)$) rectangle ($(rect_mlp2.south east) + (2.2,-0.85)$);

% Draw the arrows
\draw[ ->, line width=1.5pt] (circ0) -- (trap1);
\draw[ ->, line width=1.5pt] (trap1) -- (rect1);
\draw[ ->, line width=1.5pt] (rect1) -- (rect_mlp1);
\draw[dashed, ->, line width=1.5pt] (rect_mlp1) -- (rect2);
\draw[ ->, line width=1.5pt] (rect2) -- (rect_mlp2);
\draw[ ->, line width=1.5pt] (rect_mlp2) -- (trap2);
\draw[->, line width=1.5pt] (trap2) -- (circ1);
% Add the mid-point arrow
\draw[ ->, line width=1.5pt]  ($(trap1.east)!0.5!(rect1.west)+ (0, 0.1)$) |- node {} (Id1);
\draw[ ->, line width=1.5pt]  (Id1) -| ($(rect1.east)!0.4!(rect_mlp1.west)+ (0, -0.68)$);
\draw[ ->, line width=1.5pt]  ($(rect1.east)!0.6!(rect_mlp1.west)+ (0, -0.2)$) |- node {} (Id2);
\draw[ ->, line width=1.5pt]  (Id2) -| node {} ($(rect_mlp1.east)!0.3!(rect2.west)+ (0, 0)$);
% Add the 2nd mid-point arrow
\draw[ ->, line width=1.5pt]  ($(rect_mlp1.east)!0.7!(rect2.west)+ (0, 0.86)$) |- node {} (Id3);
\draw[ ->, line width=1.5pt]  (Id3) -| ($(rect2.east)!0.4!(rect_mlp2.west)+ (0, -0.68)$);
\draw[ ->, line width=1.5pt]  ($(rect2.east)!0.6!(rect_mlp2.west)+ (0, -0.2)$) |- node {} (Id4);
\draw[ ->, line width=1.5pt]  (Id4) -| node {} ($(rect_mlp2.east)!0.4!(trap2.west)+ (0, -0.9)$);
\end{tikzpicture}
    \caption{$\NONeXt$ without stochastic depth. It is a modification based on the $\multiplicative$ structure in where we incorporate layer normalization and skip connection as in transformers. For comparison with the NO, see \cref{fig:NO}.}
    \label{fig:FNO+epsilon}
\end{figure*}

If $\bk$ is a convolutional-type kernel, the architecture has similarities with the FNet introduced in \citet{lee2021fnet} though these connections have not been explored in the context of parametric PDEs. The addition of skip connection in the FNOs architecture has been previously investigated in the work of \citet{you2022learning}. However, the specific $\multiplicative$ structure used in here is not presented in the previous work. To provide a comprehensive analysis, we include the ResNet version of FNO in the ablation test (see \cref{experiment:ablation}) to evaluate its performance alongside the other described architectures. It is worth noting that similarities of the skip connection in the work of \citet{you2022learning} can also be drawn with $\FNONeXt$. For example, the skip connection can also be interpreted as unrolling Newton's method, see \citep{Kaltenbacher, NakamuraPotthast, Bakushinsky}. \par 

In comparing $\FNONeXt$ with NOs and $\MNO$s, we observe improvements in terms of loss and wavefield prediction across various settings (see \crefrange{OOD:table1}{OOD:table6}).

\paragraph{$\NONeXt$ with stochastic depth} \label{FNO+epsilon v2}
Despite the fact that a neural architecture is theoretically universal, in practice, the parameters are updated using gradient-based methods that cannot exhaustively search the parameter space. It is, therefore, necessary to consider the limitations of the optimization algorithm and the training data, both of which may render the model non-universal in practice. 

One possible approach to address this challenge is to enable the exploration of the optimization algorithm. \citet{huang2016deep} introduced the concept of \textit{stochastic depth}, which involves randomly dropping entire layers of the network using Bernoulli RVs\footnote{RVs refers to Random Variables.}. Practitioners have used this approach to facilitate the efficient training of large models. We conjecture that it also enables further exploration, which intuitively allows the algorithm to find better local minima (this procedure is in the spirit of an \textit{adaptive rejection sampling}). We adopt this technique in the final network design.
\begin{subequations} \label{eq:FNO+epsilon2} 
\begin{empheq}[left={\empheqlbrace}]{align}
w'_\ell &= \big(\,\bid + \boldsymbol{\mathrm{X}}_{\ell}\,\sigma \circ ( \bk_\ell  + b_{\ell}) \circ \mathrm{N}\, \big)\circ v_\ell, \\ 
v_{\ell+1} &= \big(\bid + \boldsymbol{\mathrm{X}}_{\ell}\, {f_\ell} \circ \mathrm{N} \big)\circ w'_\ell,\quad  
\end{empheq} 
\end{subequations}
$(\ell = 1 ,\ldots, L)$, $\boldsymbol{\mathrm{X}}_{\ell}$ is a Bernoulli RV, such that $\prob\{\boldsymbol{\mathrm{X}}_{\ell}  =  1 \} = p_\ell$, and $\prob\{\boldsymbol{\mathrm{X}}_{\ell}  = 0\} = 1-p_\ell$ for $p_\ell \in [0,1]$, and 
$p_1 = 1$, $p_{\ell+1}\le p_\ell$. $\mathrm{N}$ is the layer normalization (or any other normalization). In \cref{thrm:Risk_Bound_NONEXTv2} and \cref{GEB for multiplicative NO}, we shall show the relation of RVs in the generalization error bound (\textit{in-distribution} and \textit{out-of-distribution}).

%This has been previously used in Swin Transformers, \citep{liu2021swin}, and 
%ConvNeXt \citep{liu2022convnet}, among other architectures, with relative success in training very deep neural networks, 
%avoiding some of the usual issues associated with depth. 
%Lastly, following \citet{liu2022convnet}, the epochs in the training are extended.
    %\input{Architecture/4.tex}
%%%%%%%%%%%%%%%%%%%%%%%%%%%%%%%%%%%%%%%%%%%%%%%%%%%%%%%%%%%%%%%%%
%\section{Related Works}\label{Related Works}

\section{Parametric time-harmonic wave equations, forward operator and data generation}\label{modeling_data}
    Here we present a comprehensive overview of the coefficient to solution map associated with the Helmholtz equation, as well as the corresponding forward operator. Additionally, we outline the step-by-step procedure for generating the dataset and the guarantees in place to ensure: (a) independent realizations of the wave speed, and (b) sufficient regularity \footnote{nonnegative Sobolev spaces.} in accordance with the theory of neural operators.

% =======================================
\subsection{Time-harmonic wave equations}
% =======================================

We consider the propagation of time-harmonic acoustic 
waves for two dimensional domain $D \subset \mathbb{R}^2$.
The waves are given by the (scalar) pressure field $\pressure$ 
and (vector) particle velocity $\boldsymbol{v}$ solutions 
to \citep{Faucher2020adjoint,martin2021time}
\begin{subequations} \label{eq:Euler} \begin{empheq}[left={\empheqlbrace}]{align}
  -\ii \omega \rho(\bx) \, \velocity(\bx,\omega) \,- \, 
   \nabla \pressure(\bx,\omega) \,=\, 0 \, , \quad \text{in $D$} \,, \\ 
  -\dfrac{\ii\omega}{\kappa(\bx)} \pressure(x,\omega) \,+\, \nabla \cdot \velocity(\bx,\omega) \,=\, f(\bx,\omega)\, , \quad \text{in $D$} \,,
\end{empheq} 
\end{subequations}
where $f$ is the time-harmonic source of angular 
frequency $\omega$, $\rho$ is the density and $\kappa$ 
the bulk modulus. 
The boundary of the domain $\partial D=\Gamma_1 \,\cup \,\Gamma_2$ 
is separated into two, following a geophysical configuration: a 
free-surface condition is imposed at the surface $\Gamma_1$ (that is
the interface between the medium and the air), while absorbing boundary
conditions \citep{engquist1977absorbing} are imposed elsewhere 
(that is, to truncate the numerical domain), see Figure~\ref{fig:boundary}. 
These conditions correspond to
\begin{subequations} \label{eq:time-harmonic:bc}  \begin{align}
\pressure(\bx,\omega) \,=\, 0\,, & \qquad \text{on $\Gamma_1$ (Dirichlet boundary condition),} \\
\left( \partial_{\nu} - \dfrac{i \omega}{c(\bx)} \right) \pressure(\bx,\omega) = 0\,, & \qquad \text{on $\Gamma_2$ (absorbing boundary conditions).}
\end{align}\end{subequations}

Upon assuming constant density $\rho$, Problem~\eqref{eq:Euler} 
can be rewritten as the Helmholtz equation (see \citet[Remark~1]{faucher2020full}),
\begin{equation}
\label{eq:Helmholtz}
 -\left( \Delta \,+\, \dfrac{\omega^2}{c(\bx)^2} \right )\,  \pressure(\bx,\omega)
    \,=\, -\mathrm{i}\omega\rho \, f(\bx,\omega)\,,
\end{equation}
where $c$ is the wave speed,
\begin{equation}
  c(\bx) = \sqrt{\dfrac{\kappa(\bx)}{\rho(\bx)}} \, .
\end{equation}

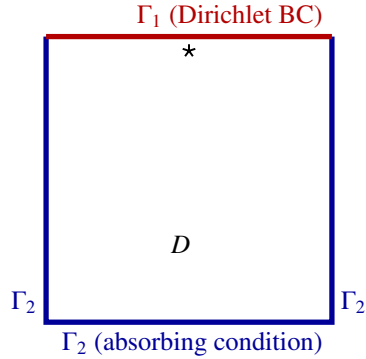
\begin{figure}[!ht] \centering
\begin{tikzpicture}[scale=1]
  \pgfmathsetmacro{\wid}{3.8} 
  \pgfmathsetmacro{\hei}{3.8}
  %% coordinates 
  \coordinate (r1)   at (  0.,  0.);  
  \coordinate (r2)   at (\wid,  0.);
  \coordinate (r3)   at (\wid,\hei); 
  \coordinate (r4)   at (  0.,\hei); 
  \draw[line width=2,\myblue]  (r4) to (r1) to (r2) to (r3);  
  \draw[line width=2,\myred]   (r3) to (r4);  
  \node[anchor=south east,\myred ,yshift=0mm,xshift=0.0em] at (r3) {$\Gamma_1$ (Dirichlet BC)};
  \node[anchor=north east,\myblue,yshift=0mm,xshift=0.0em] at (r2) {$\Gamma_2$ (absorbing condition)};
  \node[anchor=south east,\myblue,yshift=0mm,xshift=0.0em, rotate=0] at (r1) {$\Gamma_2$};
  \node[anchor=south west,\myblue,yshift=0mm,xshift=0.0em, rotate=0] at (r2) {$\Gamma_2$};
  \node[anchor=south east,black  ,yshift=8mm,xshift=-5em, rotate=0] at (r2) {$D$};  

  \draw[mark=star,mark size=2.50,mark options={color=black,line width=0.75}] plot[] 
        coordinates{(.5*\wid,0.95*\hei)};
%%\draw[line width=1,densely dashed](0.0, 0.85*\hei) to (\wid, 0.85*\hei);  
\end{tikzpicture}
    \caption{Illustration of domain $D$: 
    Dirichlet boundary condition is imposed on the the top (red line, $\Gamma_1$),
    while absorbing absorbing boundary conditions are imposed elsewhere
    (blue line, $\Gamma_2$).
    The source ($\star$) is typically positioned near surface.
    %  and the receivers (dashed line) slightly below.
    }
    \label{fig:boundary}
\end{figure}

\subsection{Wave speed to solution map, $\Gcal : c \mapsto \pressure$}

In the first experiment, the source $f$ is fixed, as well as the frequency $\omega$. The operator, $\Gcal$ is defined as a mapping from the wave speed model $c$ to the associated wavefield $\pressure$, \cref{figure:cp-GRF:experiment3}. 
That is, it gives the solution to the wave equation \cref{eq:Helmholtz} with boundary conditions \cref{eq:time-harmonic:bc} for a given physical model $c$ in the entire domain $D$. $c \mapsto \Gcal(c) = \pressure$. 
See \cref{figure:cp-GRF:experiment3} for an illustration of the operator when $c$ is a realization of a Gaussian random field.
%----------------------------------------
% visualization of the operator 
%--------------------------------------=

\setlength{\modelwidth} {4.10cm}
\setlength{\modelheight}{4.10cm}
% ---------------------------

\begin{figure}[ht!]\centering
\pgfmathsetmacro{\xmingb}{0}  \pgfmathsetmacro{\xmaxgb}{4.810}
\pgfmathsetmacro{\zmingb}{0}  \pgfmathsetmacro{\zmaxgb}{4.810}
\pgfmathsetmacro{\xminloc}{0} \pgfmathsetmacro{\xmaxloc}{4.810}
\pgfmathsetmacro{\zminloc}{0} \pgfmathsetmacro{\zmaxloc}{4.810}
\pgfmathsetmacro{\steph}{0.85}

\graphicspath{{Figures/dataset/GRF/}}

\begin{tikzpicture}

%\tikzexternaldisable

 ===============================================================================
\renewcommand{\modelfile}{wavespeed_06_scale1500to5000}
\pgfmathsetmacro{\cmin}{1.5} \pgfmathsetmacro{\cmax}{5}
% ===============================================================================
\hspace*{1em}\begin{axis}[xshift=1\modelwidth,
  width=\modelwidth, height=\modelheight,
  axis on top, separate axis lines,
  xmin=\xminloc, xmax=\xmaxloc, xlabel={$x$   (\si{\kilo\meter})},
  ymin=\zminloc, ymax=\zmaxloc, ylabel={depth (\si{\kilo\meter})}, y dir=reverse,
  xtick={}, xticklabels={,,},xlabel={\tiny wave speed model $c$ (\si{\km\per\second})},
  xticklabel pos=right,
  ytick={}, yticklabels={,,},ylabel={},
  xmajorticks=false,
  ymajorticks=false,
  x label style={xshift=-0.0cm, yshift=-0.05cm}, 
  y label style={xshift= 0.0cm, yshift=-0.00cm},
  colormap/jet,colorbar,colorbar style={title={},
               title style={yshift=0mm, xshift=3mm},
  width=.10cm, xshift=-0.5em},
  point meta min=\cmin,point meta max=\cmax,
  % style of the axis and size -------------------------------------
  label style={font=\tiny},
  tick label style={font=\tiny},
  legend style={font=\tiny\selectfont},
]
\addplot [forget plot]  graphics [xmin=\xmingb,xmax=\xmaxgb,ymin=\zmingb,ymax=\zmaxgb] {{\modelfile}.png};
\end{axis} 
 ===============================================================================

\draw[line width=2,black,->]  (22.8em,1.3em) to node[draw,rectangle,black,anchor=south,text width=9em,align=center,line width=1pt,yshift=0.50em]{from $c$ to $\pressure$ solving \cref{eq:Helmholtz}} (29.4em,1.3 em);

\pgfmathsetmacro{\cmin}{-1} \pgfmathsetmacro{\cmax}{1}

% ===============================================================================
\renewcommand{\modelfile}{15Hz_cp6_wavefield_scale1e-1_real_ref}
% ===============================================================================

\begin{axis}[xshift=2.75\modelwidth,
  width=\modelwidth, height=\modelheight,
  axis on top, separate axis lines,
  xmin=\xminloc, xmax=\xmaxloc, xlabel={$x$   (\si{\kilo\meter})},
  ymin=\zminloc, ymax=\zmaxloc, ylabel={depth (\si{\kilo\meter})}, y dir=reverse,
  xmajorticks=false,
  ymajorticks=false,
  xticklabel pos=right,
  xtick={}, xticklabels={,,},xlabel={Real part of $\pressure$},
  ytick={}, yticklabels={,,},ylabel={},
  x label style={xshift=-0.0cm, yshift=-0.00cm}, 
  y label style={xshift= 0.0cm, yshift=-0.00cm},
  % colormap/jet,colorbar,colorbar style={title={\tiny $p$},
  %              title style={yshift=-2mm, xshift=0mm},
  % width=.10cm, xshift=-0.70em},
  % point meta min=\cmin,point meta max=\cmax,
  % % style of the axis and size -------------------------------------
  label style={font=\tiny},
  tick label style={font=\tiny},
  legend style={font=\tiny\selectfont},
]
\addplot [forget plot] graphics [xmin=\xmingb,xmax=\xmaxgb,ymin=\zmingb,ymax=\zmaxgb] {{\modelfile}.png};
\end{axis}

% ===============================================================================
\renewcommand{\modelfile}{15Hz_cp6_wavefield_scale1e-1_imag_ref}
% ===============================================================================

\begin{axis}[xshift=3.4\modelwidth,
  width=\modelwidth, height=\modelheight,
  axis on top, separate axis lines,
  xmin=\xminloc, xmax=\xmaxloc, xlabel={$x$   (\si{\kilo\meter})},
  ymin=\zminloc, ymax=\zmaxloc, ylabel={depth (\si{\kilo\meter})}, y dir=reverse,
  xmajorticks=false,
  ymajorticks=false,
  xticklabel pos=right,
  xtick={}, xticklabels={,,},xlabel={Imaginary part of $\pressure$},
  ytick={}, yticklabels={,,},ylabel={},
  x label style={xshift=-0.0cm, yshift=-0.00cm}, 
  y label style={xshift= 0.0cm, yshift=-0.00cm},
  colormap/jet,colorbar,colorbar style={title={},
               title style={yshift=-2mm, xshift=0mm},
  width=.10cm, xshift=-.50em},
  point meta min=\cmin,point meta max=\cmax,
  colormap={gb}{color=(blue) color=(white) color=(red)},
  label style={font=\tiny},
  tick label style={font=\tiny},
  legend style={font=\tiny\selectfont},
]
%\tikzexternaldisable
\addplot [forget plot] graphics [xmin=\xmingb,xmax=\xmaxgb,ymin=\zmingb,ymax=\zmaxgb] {{\modelfile}.png};
%\tikzexternaldisable
\end{axis}
%\tikzexternalenable
\end{tikzpicture}
\caption{Illustration of the full-wave dataset for experiment
         that considers a computational domain of size 
         \num{1.27}$\times$\num{1.27}\si{\km\squared} with
         a source near surface.
         The wave speed and pressure field are represented 
         on a Cartesian grid of size \num{64}$\times$\num{64}
         with a grid step of \num{20}\si{\meter}. 
         The complete dataset corresponds to \num{50000} couples
         made up of a wave speed model and associated acoustic 
         wave.}
         \label{figure:cp-GRF:experiment3}
\end{figure}

% ===========================================================
% \subsection{Wave dataset: $c\mapsto \pressure$.} \label{dataset:Helmholtz}
% ===========================================================

The dataset corresponds to $N$ couples of wave speed and pressure field, denoted as, $(c_k, \pressure_k)_{k  = 1,\ldots,N}$. The pressure field, $\pressure_k$ is obtained by solving \eqref{eq:Helmholtz} with the corresponding wave speed $c_k$. We use the \textit{hybridizable discontinuous Galerkin method} 
(HDG, \cite{Faucher2020adjoint}) and  the (open-source) software \texttt{hawen} \citep{faucher2021hawen}, to obtain $\pressure_k$. The source $f$ in \eqref{eq:Helmholtz} is a fixed point-source, and the frequency is set to $15$ Hz. 
We have the following configuration:
\begin{equation}\label{experiment2}
\text{Experiment with $\Gcal$~} \,\left\lbrace\quad\begin{aligned}
 & \text{2D domain of size \num{1.27}$\times$\num{1.27}\si{\km\squared}} \\
 & \text{\num{50000} GRF wave speeds generated, imposing \,
         $\num{1.5}\si{\km\per\second}\,\leq c(x) \leq \num{5}\si{\km\per\second}$} \\
 & \text{The data are $\pressure$ that solve 
         \cref{eq:Euler} at frequency $\omega/(2\pi)=$  \num{15} \si{\Hz}}.
\end{aligned}\right.\end{equation}

To ensure a statistical learning framework, we generate independent identically distributed realizations of a Gaussian random field (GRF) as our wave speed. The process is described in \cref{subsection:GRF-wavespeed}.

%-------------------------------
\subsection{Forward operator $\forward^{f}_\omega : (c,f,\omega) \mapsto \{ \pressure(\bx_j,\omega,f) \}_{j=1,\ldots,n_{\mathrm{rcv}}}$}
\label{section:Forward}
In the following, the term \emph{forward operator} refers to the forward operator in the context of the study of the inverse %boundary value 
problem for the Helmholtz equation \cite{beretta2016inverse} (which maps parameter and source to the data) $\forward^{f}_\omega$ at frequency $\omega$ for a source $f$ such that, $\forward^{f}_\omega (c) = \pressure_\omega^f \, \vert_\Sigma$. The model parameter is the wave speed $c$ from \eqref{eq:Helmholtz}, and $\Sigma$ corresponds to a discrete set of receiver locations. That is, $    \pressure^f_\omega\vert_ \Sigma = \bigg \{\pressure^f_\omega(x_1), \ldots \pressure^f_\omega(x_{n_{\mathrm{rcv}}}) \bigg\}$, where $x_i$ is the position of the $i^{\mathrm{th}}$ receiver for a total of $n_{\mathrm{rcv}}$.  
For notation, we introduce the \emph{restriction operator} $\mathcal{R}$, which reduces the fields to the set of receivers positions, $\Sigma$, such that $\mathcal{R}(\pressure) = \pressure \vert_\Sigma$.

The dataset is composed of $N_{\mathrm{src}}$ sources, denoted as $f_\ell$ and consists of $N$ pairs of wave speed and restricted pressure field, that is $\left(c_k, \mathcal{R} \left( \pressure^{f_\ell}_k \right) \right)_{k  = 1,\ldots,N; \,\,
l = 1,\ldots,N_{\mathrm{src}}}$.
The restricted pressure field, $\mathcal{R} \left( \pressure^{f_\ell}_k \right)$, is obtained by solving \eqref{eq:Helmholtz} with the corresponding wave speed $c_k$ and source $f_\ell$, then restricted at the set $\Sigma$. 
% We employ the \textit{hybridizable discontinuous Galerkin method} 
% (HDG, \cite{Faucher2020adjoint}) and the (open-source) software \texttt{hawen} \citep{faucher2021hawen} to obtain $\mathcal{R} \left( \pressure^{f_\ell}_k \right)$. The sources $f_\ell$ in \eqref{eq:Helmholtz} are point-sources, and the frequency is set at $15$ Hz.
Similar to the experiment with the full modeling operator,
the wave speeds are independent identically distributed 
realizations of a GRF (see \cref{subsection:GRF-wavespeed}). 
The data set is illustrated in \cref{figure:cp-StR:experiment}.

\begin{equation}\label{experiment_forward}
\text{Experiment with $\forward$~} \,\left\lbrace\quad\begin{aligned}
 & \text{2D domain of size \num{1.27}$\times$\num{1.27}\si{\km\squared}} \\
 & \text{\num{50000} GRF wave speeds generated, imposing \,
         $\num{1.5}\si{\km\per\second}\,\leq c(x) \leq \num{5}\si{\km\per\second}$} \\
 & \text{\num{64} point-sources, located at a fixed depth of \num{10} \si{\m}, and \num{20} \si{\m} apart along the width} \\
 & \text{The data are $\mathcal{R} \left( \pressure^{f_\ell}_k \right)$ that solve 
         \cref{eq:Euler} at frequency $\omega/(2\pi)=$  \num{15} \si{\Hz}}.\\
 & \text{The line of receivers $\Sigma$ is located at a fixed depth \num{10} \si{\m}, and \num{10} \si{\m} apart along the width}
\end{aligned}\right.\end{equation}

\setlength{\modelwidth} {4.1cm}
\setlength{\modelheight}{4.1cm}
% ---------------------------

\begin{figure}[ht!]\centering
\pgfmathsetmacro{\xmingb}{0}  \pgfmathsetmacro{\xmaxgb}{1.27}
\pgfmathsetmacro{\zmingb}{0}  \pgfmathsetmacro{\zmaxgb}{1.27}
\pgfmathsetmacro{\xminloc}{0} \pgfmathsetmacro{\xmaxloc}{1.27}
\pgfmathsetmacro{\zminloc}{0} \pgfmathsetmacro{\zmaxloc}{1.27}
\pgfmathsetmacro{\steph}{0.85}

\graphicspath{{Figures/StR/}}

\begin{tikzpicture}

 ===============================================================================
\renewcommand{\modelfile}{wavespeed/cp_017_wavespeed}
\pgfmathsetmacro{\cmin}{1.5} \pgfmathsetmacro{\cmax}{5}
% ===============================================================================
\begin{axis}[xshift=0.1\modelwidth,
  width=\modelwidth, height=\modelheight,
  axis on top, separate axis lines,
  xmin=\xminloc, xmax=\xmaxloc, xlabel={$x$   (\si{\kilo\meter})},
  ymin=\zminloc, ymax=\zmaxloc, ylabel={depth (\si{\kilo\meter})}, 
  y dir=reverse,
  xtick={0.2,1.0},
  ytick={0.1,1.2},
  xticklabel pos=right,
  % xmajorticks=false,
  % ymajorticks=false,
  x label style={xshift=-0.0cm, yshift=-0.40cm}, 
  y label style={xshift= 0.0cm, yshift=-0.40cm},
  colormap/jet,colorbar,colorbar style={title={(\si{\km\per\second})},
               title style={yshift=0mm, xshift=7mm},
  width=.10cm, xshift=-0.5em},
  point meta min=\cmin,point meta max=\cmax,
  % style of the axis and size -------------------------------------
  label style={font=\small},
  tick label style={font=\small},
  legend style={font=\small\selectfont},
]
\addplot [forget plot]  graphics [xmin=\xmingb,xmax=\xmaxgb,ymin=\zmingb,ymax=\zmaxgb] {{\modelfile}.png};
%\node[anchor=south east, orange!60!black, yshift=0mm,xshift=1.0em, rotate=0] at (\xminloc, 0.8) {$\nabla$};
%\node[anchor=south east, orange!60!black, yshift=0mm,xshift=3.5em, rotate=0] at (\xminloc, 0.8) {$\nabla$};
%\draw[black, dashed, line width=2, dash pattern=on 1pt off 3pt]  (\xminloc, 0.8) to (\xmaxloc, 0.8); 
\end{axis} 
 ===============================================================================

\draw[line width=2,black,->]  (11.50em,1.3em) to node[draw,rectangle,black,anchor=south,text width=7em,align=center,line width=1pt,yshift=0.50em]{\small from $c$ to Source-to-Receiver operator} (16.4em,1.3 em);

\pgfmathsetmacro{\cmin}{-1} \pgfmathsetmacro{\cmax}{1}

% ===============================================================================
\renewcommand{\modelfile}{wavefield/cp_017_pressure_hdg_real}
% ===============================================================================

\pgfmathsetmacro{\xmingb}{1}  \pgfmathsetmacro{\xmaxgb}{128}
\pgfmathsetmacro{\zmingb}{1}  \pgfmathsetmacro{\zmaxgb}{64}
\pgfmathsetmacro{\xminloc}{1} \pgfmathsetmacro{\xmaxloc}{128}
\pgfmathsetmacro{\zminloc}{1} \pgfmathsetmacro{\zmaxloc}{64}

\begin{axis}[xshift=1.7\modelwidth,
  width=1.30*\modelwidth, height=\modelheight,
  axis on top, separate axis lines,
  xmin=\xminloc, xmax=\xmaxloc, xlabel={receiver index},
  ymin=\zminloc, ymax=\zmaxloc, ylabel={source index}, 
  xticklabel pos=right,
  xtick={10,120},
  ytick={5,60},
  x label style={xshift=-0.0cm, yshift=-0.40cm}, 
  y label style={xshift=-0.0cm, yshift=-0.40cm},
  y dir=reverse,
  % xmajorticks=false,
  % ymajorticks=false,
  % xtick={}, xticklabels={,,},
  % ytick={}, yticklabels={,,},
  % xlabel={Real part of $\pressure\,\vert_\Sigma$},
  % colormap/jet,colorbar,colorbar style={title={\tiny $p$},
  %              title style={yshift=-2mm, xshift=0mm},
  % width=.10cm, xshift=-0.70em},
  % point meta min=\cmin,point meta max=\cmax,
  % % style of the axis and size -------------------------------------
  label style={font=\small},
  tick label style={font=\small},
  legend style={font=\small\selectfont},
]
\addplot [forget plot] graphics [xmin=\xmingb,xmax=\xmaxgb,ymin=\zmingb,ymax=\zmaxgb] {{\modelfile}.png};
% Add a horizontal dashed line
%\draw [dashed, ultra thick] (\xminloc, 0.6) -- (\xmaxloc, 0.6);

\end{axis}

% ===============================================================================
\renewcommand{\modelfile}{wavefield/cp_017_pressure_hdg_imag}
% ===============================================================================

\begin{axis}[xshift=2.80\modelwidth,
  width=1.30*\modelwidth, height=\modelheight,
  axis on top, separate axis lines,
  xmin=\xminloc, xmax=\xmaxloc, xlabel={receiver index},
  ymin=\zminloc, ymax=\zmaxloc, ylabel={source index}, 
  xticklabel pos=right,
  xtick={10,120},
  ytick={5,60},
  x label style={xshift=-0.0cm, yshift=-0.40cm}, 
  y label style={xshift=-0.0cm, yshift=-0.40cm},
  y dir=reverse,  
  % xmajorticks=false,
  % ymajorticks=false,
  % xtick={}, xticklabels={,,},
  % ytick={}, yticklabels={,,},
  % xlabel={Real part of $\pressure\,\vert_\Sigma$},
   colormap/jet,colorbar,colorbar style={title={\tiny $p$},
                title style={yshift=-2mm, xshift=0mm},
   width=.10cm, xshift=-0.70em},
   point meta min=\cmin,point meta max=\cmax,
  % % style of the axis and size -------------------------------------
  colormap={gb}{color=(blue) color=(white) color=(red)},
  label style={font=\small},
  tick label style={font=\small},
  legend style={font=\small\selectfont},
]
\addplot [forget plot] graphics [xmin=\xmingb,xmax=\xmaxgb,ymin=\zmingb,ymax=\zmaxgb] {{\modelfile}.png};

%\draw [dashed, ultra thick] (\xminloc, 0.6) -- (\xmaxloc, 0.6);
\end{axis}

\node[anchor=north] at (1.70,-0.2) {\small GRF wave speed model};
\node[anchor=north] at (8.8,-0.2) 
{\small $\mathrm{Re}\big(\pressure\vert_\Sigma\big)$};
\node[anchor=north] at (13.4,-0.2) 
{\small $\mathrm{Im}\big(\pressure\vert_\Sigma\big)$};

\end{tikzpicture}
\caption{Illustration of forward operator experiment
         that considers a computational domain of size 
         \num{1.27}$\times$\num{1.27}\si{\km\squared} with $64$ source near surface, and $128$ receivers located slightly beneath the sources' location. The illustration of the wave field represent the "matrix" response, each row corresponds to a source, and each column to the pressure field registered by the receivers' line.}
         \label{figure:cp-StR:experiment}
\end{figure}
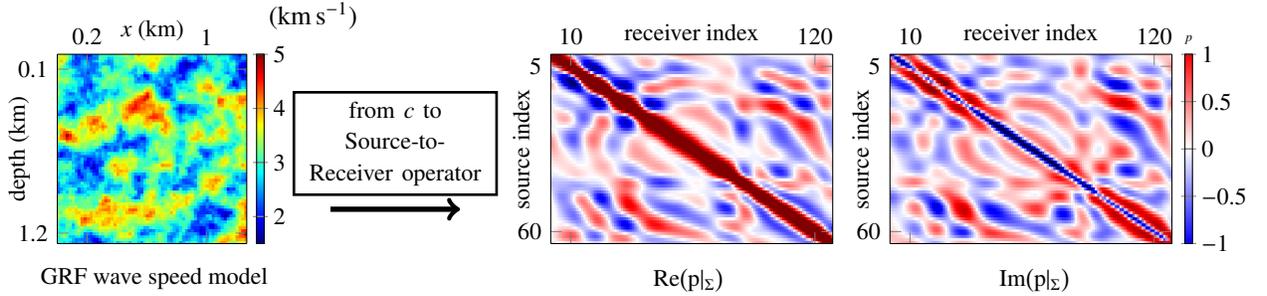

% ===========================================================
\subsection{Wave speeds as Gaussian random fields (GRF). Whittle–Mat\'ern field}
\label{subsection:GRF-wavespeed}
% ===========================================================

The wave speed is obtained as the composition of linear transformation and an independent realizations of GRF with the Whittle–Mat\'ern covariance kernel $\mathrm{C}_\nu$ 
\citep{ghosal2017fundamentals, bogachev2015gaussian, lord2014introduction}. The linear transformation $\mathrm{T}$, is a linked function to ensure that the wave speed is nonnegative, $\mathrm{T}\circ Z \ge 0$ and $Z \sim \mathrm{GRF}$. A most sophisticated version of this idea is presented in \citet{abraham2020statistical} for the \emph{conductivity} in the Calder\'on problem, the conductivity is also restricted to be nonnegative.

An introduction of Gaussian random fields is presented in \cref{appendix:grf}. We briefly discuss the Whittle–Mat\'ern kernel, (\cite{bolin2020numerical} and \cite[Sec. 2.2.3]{chada2018parameterizations}).
A real-valued Gaussian random field $Z$ defined on a spatial domain $D \subset \mathbb{R}^d$ is a  Whittle–Mat\'ern field if its covariance function $C : D\times D \to \mathbb{R}$ is given by
\begin{equation}\label{Whittle–Mat}
    \mathrm{C}_\nu(r)  = s^2\, \dfrac{2^{1-\nu}}{\Gamma(\nu)} \left( \sqrt{2 \nu}\, \dfrac{r}{\mathfrak{a}}\right)^\nu \mathrm{K}_\nu \left( \sqrt{2 \nu}\,  \dfrac{r}{\mathfrak{a}} \right) \, .
\end{equation}

% , where $d$ is the distance between two points,  
\noindent $s$ is the variance of the process, $\Gamma$ is the gamma function \citep{artin2015gamma}, $\mathrm{K}_\nu$ is the modified Bessel 
function of the second kind \citep{bowman2012introduction, arfken1999mathematical}, 
and $\mathfrak{a}$, $\nu$ are positive parameters. $\nu$ is known as the \textit{smoothness} of the random field.
$r$ is the distance, that is, given  
$\boldsymbol{x}=(x_1, \ldots, x_d) \in \mathbb{R}^d$, and $\boldsymbol{x}'=(x_1', \ldots, x_d')\in \mathbb{R}^d$, 
is defined as 
\begin{equation} \label{eq:lambda_cov}
r
 \, := \, \sqrt{\sum_{i = 1}^d \left(\dfrac{x_i - x'_i}{\lambda_i}\right)^2\,} \, ,
\end{equation}
where the vector coefficient $\boldsymbol{\lambda} = (\lambda_1, \ldots \lambda_d)$ 
defines the correlation length along two points in $\mathbb{R}^d$. \par

%%%%%%%%%%%%%%%%%%%%%%%%%%%%%%%%%%

The regularity of the Whittle-Matérn field and its generalizations can be readily obtained by viewing the field as a stochastic partial differential equation (SPDE for short). That is, a GRF $Z$  with covariance function of  Whittle-Matérn solves the fractional SPDE,
\begin{equation} \label{eq:Matern} 
  \boldsymbol{L}^\beta   {Z}  = \boldsymbol{d{\mathscr{W}}} \, , \quad \text{in $D$} \,, \quad Z \,=\, 0\, , \quad \text{on $\partial D$} \,,
\end{equation}

\noindent for $4 \beta = d+ 2 \nu$ and $d \in \{1, 2, 3\}$ the spatial dimension, $\boldsymbol{d{\mathscr{W}}}$ is Gaussian white noise, and  
$\boldsymbol{L}^\beta $ is a second-order elliptic differential operator.\par 
In the case of $\boldsymbol{L} = - \Delta + \kappa^2\,$\footnote{where $\Delta$ denotes the Laplacian.}, each realization of Whittle–Mat\'ern field, defined in (\ref{Whittle–Mat}) ($\lambda_i=1$), coincide with the solution of (\ref{eq:Matern}), and belongs to the Sobolev space $H^{2\beta - d/2 - \varepsilon}(D)$ ($\prob$-a.s. \footnote{almost surely with respect to the probability measure $\prob$.}) \citep[Remark 2.4]{bolin2020numerical}. For positive non-integer order Sobolev spaces, we refer to (\ref{positive-Sobolev space}).  In the more general case, when $\boldsymbol{L} = \nabla \cdot (\sigma \nabla) + \kappa^2$, which is referred to as the generalized Whittle-Matérn field, the regularity of its solution can be established under mild assumptions on $\sigma$, $\kappa$, and the boundary $\partial D$, see \citep[Lemma 4.2]{cox2020regularity}.

%%%%%%%%%%%%%%%%%%%%%%%%%%%%%%
\begin{comment}
The regularity of the Whittle-Matérn field and its generalizations can be readily obtained by viewing the field as a stochastic partial differential equation (SPDE for short).
\begin{subequations} \label{eq:Matern} 
\begin{empheq}[left={\empheqlbrace}]{align}\,\;
  \boldsymbol{L}^\beta   {Z}  = \boldsymbol{d{\mathscr{W}}} \, , \quad \text{in $D$} \,, \\ 
  Z \,=\, 0\, , \quad \text{on $\partial D$} \,,
\end{empheq} 
\end{subequations}

\noindent for $4 \beta = d+ 2 \nu$ and $d \in \{1, 2, 3\}$ the spatial dimension, $\boldsymbol{d{\mathscr{W}}}$ is Gaussian white noise, and  
$\boldsymbol{L}^\beta $ is a second-order elliptic differential operator, $\boldsymbol{L}  = \nabla \cdot  (\sigma \nabla ) + \kappa^2$.  In the classical case of Whittle–Mat\'ern,  $\boldsymbol{L} = - \Delta + \kappa^2$, from \citet[Remark 2.4]{bolin2020numerical} we know that each realization of Whittle–Mat\'ern field lies on a Sobolev space $H^{2\beta - d/2 - \varepsilon}(D)$ ($\prob$-a.s. \footnote{almost surely.}). In the more general case where $\boldsymbol{L}$ is a (general) second-order elliptic operator, similar conclusions can be drawn under suitable assumptions on $\sigma$, $\kappa$, and the boundary $\partial D$. This result is discussed in \citet[Lemma 4.2]{cox2020regularity}. 
\end{comment}
\par 

From \eqref{eq:Matern} and the associated SPDE, we know that the Whittle–Mat\'ern field lies in a "\emph{nicer}" space than Gaussian white noise. The use of GRF is 
motivated by the following reasons. (a) In our construction the wave speed is generated as $c = \mathrm{T} \circ Z^{\small \beta}$, in where $\mathrm{T}$ is an affine transformation and $Z^{\small \beta}$ lies in $H^{2\beta - d/2 - \varepsilon}(D)$ a.s., specifically $d =2$, $\beta = 1$ so $Z^{\small \beta}\in {H}^{1 - \varepsilon}(D)$ a.s., and for all $\varepsilon \in [0,1]$, the wave speed lies on non-negative Sobolev spaces. The field satisfies the conditions in \cite[Theorem 2.5]{kovachki2021universal}. If we would have $c \in {\mathrm{L}}^\infty(D)$, the operator $\Gcal$ would not be covered by the universality in \citep{kovachki2021universal} \footnote{${\mathrm{L}}^\infty$ is not a separable Banach space.}. (b) GRF samples are easily generated, and for the case of Whittle–Mat\'ern field, the variance, smoothness, and correlation length are easy to control; this observation plays a crucial role in \cref{{section:experimental:OOD}} to test the \emph{out-of-distribution behavior}. (c) This distribution is independent of the grid resolution. (d) GRF are often used in Bayesian statistics as prior probability measures with covariance kernels related to the Laplace operator (\citep[Section 2.1]{nickl2020polynomial}, \citep{dashti2017bayesian} and \cite{chada2018parameterizations}).

\begin{remark}\label{exp:GRF_param} 
    \textbf{The parameters in the experiments} are the following: $d = 2$, $s=1$, $\lambda_1 = \lambda_2 = 0.1$, and smoothness coefficient $\nu = 1$. For the implementation of Gaussian fields, see  \citep{dietrich1997fast, lord2014introduction, Kumar:2019:Online}, and particularly \citep{bolin2020numerical}.
\end{remark}

\section{Training and testing in-distribution for $\Gcal$}\label{Experiments}
%%%%location of the code 

% \begin{remark}
%     \textbf{The official code} is publicly available at \nocite{official_repo:2023:Online}\cite{official_repo:2023:Online}, and the dataset is located at \nocite{official_dataset:2023:Online}\cite{official_dataset:2023:Online}.
% \end{remark}
\begin{comment}
\begin{center}
\tikzset{external/export next=false}
\begin{tikzpicture}[remember picture]
\node[draw, inner sep=0.2cm] (text) at (0,0) {%
  \begin{minipage}{0.7\textwidth}
    \textbf{The official code} is publicly available at \nocite{official_repo:2023:Online}\cite{official_repo:2023:Online}, and the dataset is located at \nocite{official_dataset:2023:Online}\cite{official_dataset:2023:Online}.
  \end{minipage}%
};
\end{tikzpicture}
\end{center}
\end{comment}
    In this section, our focus lies on training the architectures to accurately predict the coefficient to solution map $\Gcal$ for the Helmholtz equation at a frequency of $15$ Hz, \eqref{eq:Helmholtz}. For the sake of completeness, more experiments with different frequencies and domain's configuration are presented in \cref{appendix:more_experiments}. Throughout this section and \cref{appendix:more_experiments}, all the models are tested with in-distribution data. However, we significantly increase the test set compared to traditional applications of deep learning \footnote{we deliberately avoid using the traditional 80:20 split of training and test data.}. We choose a test set of the same size than our training. This choice enables us to obtain more reliable estimates of the neural operators' generalization capabilities specifically for in-distribution data. A detailed analysis of the generalization to in-distribution data is presented in \cref{GEB}. 

\begin{remark}
    \textbf{The code is publicly available} at \nocite{official_repo:2023:Online}\cite{official_repo:2023:Online}, and the dataset is located at \nocite{official_dataset:2023:Online}\cite{official_dataset:2023:Online}.
\end{remark}

\begin{remark}
    \bourbaki\footnote{Bourbaki dangerous bend \citep{Bourbaki:symbol}.}In our experiments, we adhere to specific constraints. When adjusting the parameters of the networks, the increase in the parameter count is typically negligible, adding around 100 additional parameters to maintain comparability with the base neural operator. If we increase the number of layers, it is based on mathematical considerations, particularly when incorporating stochastic depth. \emph{We consciously refrain from increasing the training epochs or the size of the training dataset}. Our emphasis is on making fundamental changes to the network architecture rather than compensating for these alterations by merely expanding the model's capacity, dataset size, or training time.  
\end{remark}

\begin{comment}
\begin{center}
\begin{tikzpicture}
\node[draw,  inner sep=0.2cm,] (text) at (0,0) {%
\begin{minipage}{1\textwidth}
  \bourbaki \citep{Bourbaki:symbol}
    In our experiments, we adhere to specific constraints. When adjusting the parameters of the networks, the increase is typically negligible, adding around 100 additional parameters to maintain comparability with the base neural operator. If we increase the number of layers, it is based on mathematical considerations, particularly when incorporating stochastic depth. \par 
   We consciously refrain from increasing the training epochs or the size of the training dataset. Our emphasis is on making fundamental changes to the network architecture rather than compensating for these alterations by merely expanding the model's capacity, dataset size, or training time.  
  \end{minipage}
};
\end{tikzpicture}
\end{center}
\end{comment}

\subsection{Neural network ''prediction'' of the wavefield}
% -----------------------------------------
\setlength{\modelwidth} {4.10cm}
\setlength{\modelheight}{4.10cm}
% -----------------------------------------
\begin{figure}[!ht] \centering
   \input{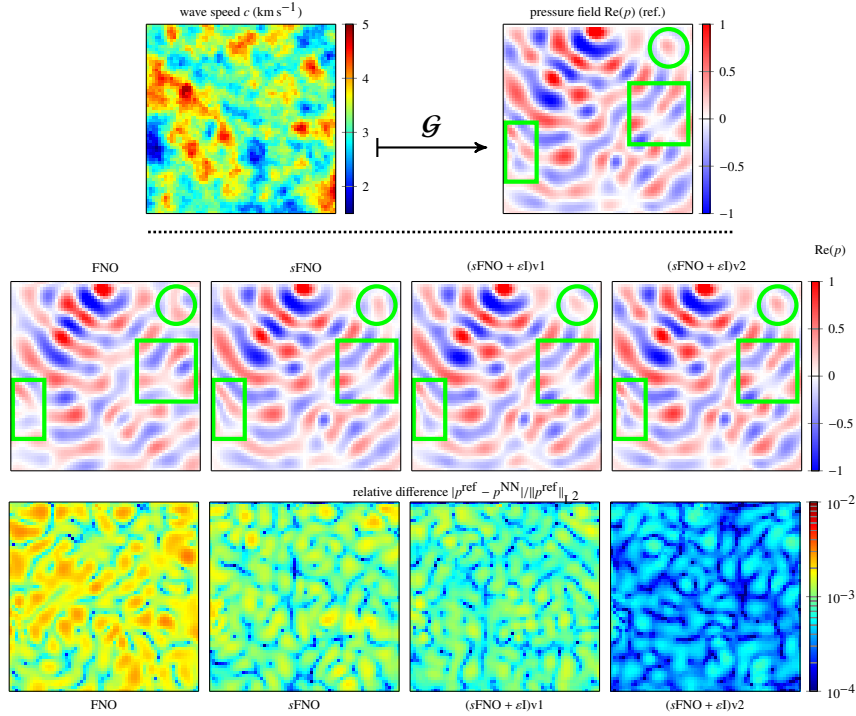}
    \vspace*{-1em}
    \caption{\small{Comparison of the reconstructed wavefields obtained  
    with the different architectures (middle row) and relative error with the reference solution (bottom row). The circles and rectangles serve as a visual aid to highlight the distinction in the propagation of waves.
    The dataset corresponds to wave propagation from Gaussian Random Field realizations of 
    wave speed in a domain of size \num{1.27}$\times$\num{1.27}\si{\km\squared}, with reference wavefield obtained by solving 
    the wave PDEs with software \texttt{hawen} \citep{faucher2021hawen} (top row). }}
    \label{fig:reconstruction_w_annotations}
\end{figure}

Upon the previous constraints, we conducted training on the wave dataset as described in the previous section for all the neural networks outlined in \cref{{Proposed Method}}. The results shown in \cref{fig:reconstruction_w_annotations} clearly demonstrate that each architecture leads to a superior reconstruction of the wave field. The figure displays only the real part of the wave field. For the approximation of both the real and imaginary parts of the pressure field, we refer to \cref{Appendux:15Hz_experiment}.

%--------------------------------------------------
\subsection{Hyperparameters of the neural networks}
%--------------------------------------------------
The summary of parameters used in the training is presented in \cref{table:hyperparam:experiment2}.

\begin{table}[ht!]
\centering
\small
\begin{tabular}{llllll}
\toprule
Model & FNO & $\MFNO$ & ($\FNONeXt$)v1 & ($\FNONeXt$)v2 \\ 
\midrule
Fourier modes: 12          & \tikzcmark           & \tikzcmark            & \tikzcmark                   & \tikzcmark    \\ 
Layers: 4         & \tikzcmark            & \tikzcmark             & \tikzcmark                    & [3, 3, 9, 3]         \\ 
Features: 36      & \tikzcmark           & \tikzcmark            & \tikzcmark                   & [36, 36, 36, 36]     \\ 
GeLU & \tikzcmark         & \tikzcmark          & \tikzcmark                 & \tikzcmark                 \\ 
Positional Encoder  $[0,1]^2$ & \tikzcmark &\tikzcmark & \tikzcmark & \tikzcmark \\ 
Lifting       &  $3\mapsto 18\mapsto 36$ &  $3\mapsto 18\mapsto 36$ & $3\mapsto 18\mapsto 36$ &  $3\mapsto 18\mapsto 36$ \\ 
Proj.     & $36\mapsto 2$ & $36\mapsto 2$ &  $36\mapsto 2$ &  $36\mapsto 2$ \\ 
Dropout        & \tikzxmark           & \tikzxmark            & \tikzxmark                   & \tikzxmark                   \\ 
DropPath      & \tikzxmark           & \tikzxmark            & \tikzxmark                   & 0.3                  \\ 
\bottomrule
\end{tabular} 
\caption{Architectures' parameters. The networks recovered the real, and imaginary part of the pressure field, i.e., the output is a vector field in $\mathbb{R}^2$ which can be associated with $\mathbb{C}$, and the projection operator is simplified by a linear layer instead of a MLP to speed up the training process. The only architecture that differs is ($\FNONeXt$) version $2$ (with stochastic depth).}
\label{table:hyperparam:experiment2}
\end{table} 

 The Fourier modes represent the truncated Fourier modes in the approximation of the integral kernel per layer $\bk_\ell$ as described in \cite{li2020fourier}. The number of layers represents the compositions of equations of the form \crefrange{eq:arch:NO}{eq:FNO+epsilon2}, The positional encoder means that the wave speed, $c$, is input in the neural operators as a couple $\{\left(c(x_i, y_k), \mathrm{T}(x_i),\mathrm{T}(y_k) \right) \}_{i,k=1}^{n}$. Here, $\mathrm{T}$ denotes an affine transformation applied to each grid realization to move the grid to the interval $[0,1]\times [0,1]$ usually for training stability. \par
 
The feature space refers to the range of the lifting operator (as explained after \cref{eq:arch:NO}), denoted as $\br$\footnote{Lifting map, following notation in \cite{kovachki2021universal}}. It maps $\mathbb{R}^3$ to $\mathbb{R}^{36}$, with $\left(c(x_i, y_k), \mathrm{T}(x_i),\mathrm{T}(y_k) \right)$ being transformed to $\mapsto v_1(z_1, \ldots, z_{36}) = \br\left(c(x_i, y_k), \mathrm{T}(x_i),\mathrm{T}(y_k) \right) \in \mathbb{R}^{36}$. It is implemented using a $2$-layers MLP with weight matrices,  $W_1^{\br} \in \mathbb{R}^{18\times 3}$ and $W_2^{\br} \in \mathbb{R}^{36\times 18}$, and bias $b_1^{\br}\in \mathbb{R}^{18}$, $b_2^{\br}\in \mathbb{R}^{18}$.  \par 

The projection, $\bq$ \footnote{projection map, following notation in \cite{kovachki2021universal}}, maps $v_L(z_1, \ldots, z_{36})\in \mathbb{R}^{36}$ to $c(x_i, y_k)\in \mathbb{R}^2$, with a linear affine transformation such that $W^{\bq}\in \mathbb{R}^{2\times 36}$ and  $b^{\bq}\in \mathbb{R}^{2}$. We associate $\mathbb{R}^2$ with $\mathbb{C}$ to recover the imaginary and real part of the solution\footnote{pressure field.}. \par 

In our experiments, we do not implement dropout. For stochastic depth (also known as drop path), the random variables have a linear decay. The probability is set as follows in the experiments, $\prob\{\boldsymbol{\mathrm{X}}_1 = 1\} = 1$ for the first layer, and $\prob\{\boldsymbol{\mathrm{X}}_L = 1\} = 0.7=1-0.3$ for the last layer, \citep{stochasticdepth:implemen}. For the layers in between, a survival probability is assigned using linear interpolation. \par
 
The architectures used in our study have similar parameter count, except for $\FNONeXt$ with stochastic depth (v2). This architecture consists of four stages, each containing a different number of blocks. Specifically, the number of blocks in each stage is $k\in [3, 3, 9, 3]$, and the blocks follow \cref{eq:FNO+epsilon2}. This results in a total of $22$ layers, with each layer truncated to $12$ principal modes in the Fourier expansion of $\bk_\ell$, and the feature spaces of dimension $36$ \footnote{that is for each $(x_i,y_k)$ we have $\br\left(c(x_i, y_k), \mathrm{T}(x_i),\mathrm{T}(y_k) \right) \in \mathbb{R}^{36}$}. The parameter for the other networks, namely FNO, $\MFNO$, and $\FNONeXt$ without stochastic depth (v1), are essentially the same. 

\paragraph{Number of parameters of the neural operators}\label{table:Models_Size}
As mentioned earlier, both $\MFNO$ and $\FNONeXt$ have a similar ''size'' to FNO when stochastic depth is not considered. However, the significant difference arises when stochastic depth is incorporated, resulting in a much deeper neural network. In all the networks, the lifting and projection sections have a parameter count of $756$ and $685$, respectively. \par

The main part of the networks, which encompasses the ''operator'' layers described in \crefrange{eq:arch:NO}{eq:FNO+epsilon2}, are divided into two categories: layers without stochastic depth, and  layers with stochastic depth. In the former type (\crefrange{eq:arch:NO}{eq:FNO+epsilon}), the parameters are fixed at $1.5$ million for all the layers, while in the case of $\FNONeXt$v2 (\cref{eq:FNO+epsilon2}), the parameter count increases to $8.1$ million.

\begin{comment} For more detailed information, please refer to \cref{table:Models_Size}.

\begin{table}[ht!]
\centering
\small
\captionsetup{format=plain,labelfont=bf,font=small} % Customize caption style
\begin{tabular}{@{}lllll@{}}
\toprule
\# Param & FNO (Baseline) & $\MFNO$ & $(\FNONeXt)$v1 & $(\FNONeXt)$v2 \\
\midrule
Lifting & 756 & 756 & 756 & 756 \\
Proj & 685 & 685 & 685 & 685 \\
Operator NN & 1.5M & 1.5M & 1.5M & 8.1M \\
\bottomrule
\end{tabular}
\caption{Summary of parameters of each network, following the specifications in \cref{table:hyperparam:experiment2}}
\label{table:Models_Size}
\end{table}
\end{comment}
%------------------------------------
\subsection{Training of the experiment}\label{experiment:training}
%------------------------------------
For all the architectures, following the tradition of transformers, we employ the \textit{AdamW} optimizer \citep{loshchilov2017decoupled} with an initial learning rate of $10^{-3}$. We utilize a linear step scheduler with parameters: step size = $40$, and a multiplicative factor of learning rate decay of $\gamma = 0.5$. \par 

The number of epochs is set to $100$ ($300$ epochs yielded the best results for $\FNONeXt$ with stochastic depth, but this is not documented here as we try to keep the same parameters across networks). In all architectures, we apply a small $\Lp$ weight regularizer with a parameter of $10^{-5}$. The training process is conducted using $25,000$ out of $50,000$ generated samples \cref{experiment2}, while $5,000$ samples are used for validation, and $20,000$ for testing. Our \emph{testing dataset is substantially larger} than what is typically encountered in the machine learning literature. This choice reflects our objective of showcasing the networks' generalization capabilities.

%--------------------------------------------------
\subsection{Multiple random initialization} \label{experiment: multiple_initialization:gelu}
%--------------------------------------------------
To ensure the consistency of our results, we train each network using six different random initializations of the parameters and in consequence, different trajectories of the optimization algorithm. The trend is consistently observed across all initializations, as depicted in \cref{fig:boxplot}. The values of the relative $\Lp$-loss among multiple training paths can be found in \cref{table:experiment1:l2error}. 

% -----------------------------------------
\noindent
\begin{minipage}[c]{0.45\textwidth}
\small
\renewcommand{\arraystretch}{0.8}
    \begin{tabular}{@{}llll@{}}
\toprule
 {FNO} & {$\MFNO$} & {($\FNONeXt$)v1} & {($\FNONeXt$)v2} \\ 
\midrule
 $0.174050$     & $0.119564$      & $0.097434$                       & $0.046988$                       \\ 
 $0.180532$     & $0.115850$      & $0.089249$                       & $0.042121$                       \\ 
 $0.145947 $     & $0.110553$      & $0.096739$                       & $0.041300$                       \\ 
 $0.153028$     & $0.102238$      & $0.097211$                       & $0.045696$                       \\ 
 $0.144907$     & $0.102998$      & $0.102930$                       & $0.049157$                       \\ 
 $0.172738$     & $0.103829$      & $0.092119$                       & $0.037969$                       \\ 
\bottomrule
\end{tabular}
\captionof{table}{Magnitude of the relative $\Lp$-norm. Multiple realizations of the trained networks with different seeds. Each row represents a different realization, and the values correspond to the test loss after training. The visualization of the table is presented in \cref{fig:boxplot}.}
\label{table:experiment1:l2error}
\end{minipage}%
\hfill
\begin{minipage}[c]{0.5\textwidth}
    \centering
        \includegraphics[width=1\linewidth]{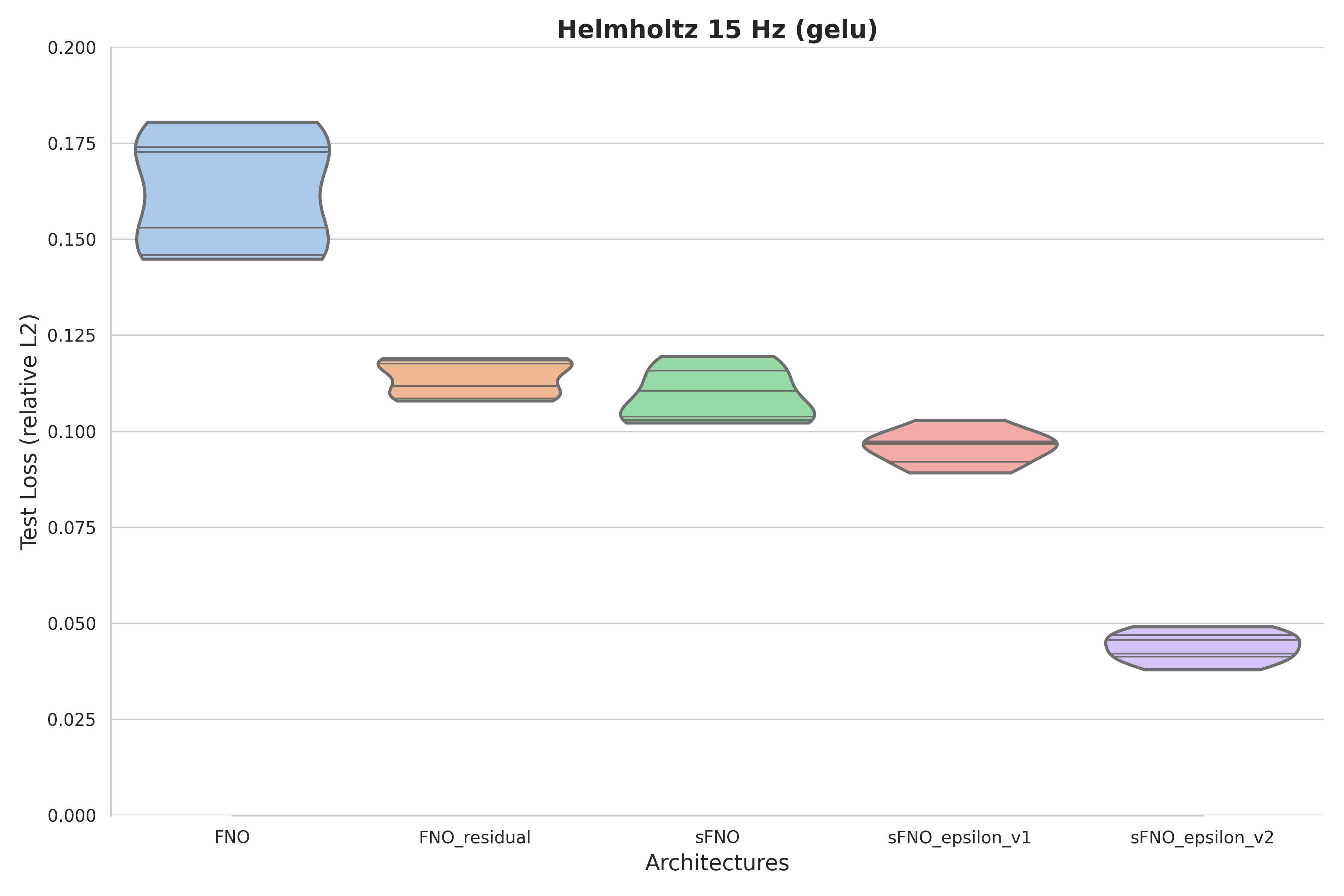}
        \captionof{figure}{Violin plot \citep{hintze1998violin} of the test-loss in Experiment $15$ Hz of \eqref{eq:Helmholtz}. Each architecture is trained $6$ times, the rel. $\Lp$-loss, $|\Gcal^{\mathrm{ref}} - \Gcal^{\mathrm{approx}}|_{\Lp}/ |\Gcal^{\mathrm{ref}}|_{\Lp}$, on the test set.}
        \label{fig:boxplot}
\end{minipage}

\subsection{Visualization of the loss landscape} \label{section:loss landscape}

The observed differences in the performance of the four considered architectures prompted us to study their respective learning landscapes in search of structural characteristics that could explain the results. To that end, we sampled the training loss in a two-dimensional domain spanned by the first two principal components of the learning trajectory~\citep{li2018visualizing}. By construction, this planar domain best captures the portion of the landscape visited during the training of each model and, therefore, may offer valuable insight into the training convergence. 

Corresponding results are shown in Figure~\ref{fig:loss landscapes}.

\begin{figure}[hbt]
\includegraphics[width=\linewidth]{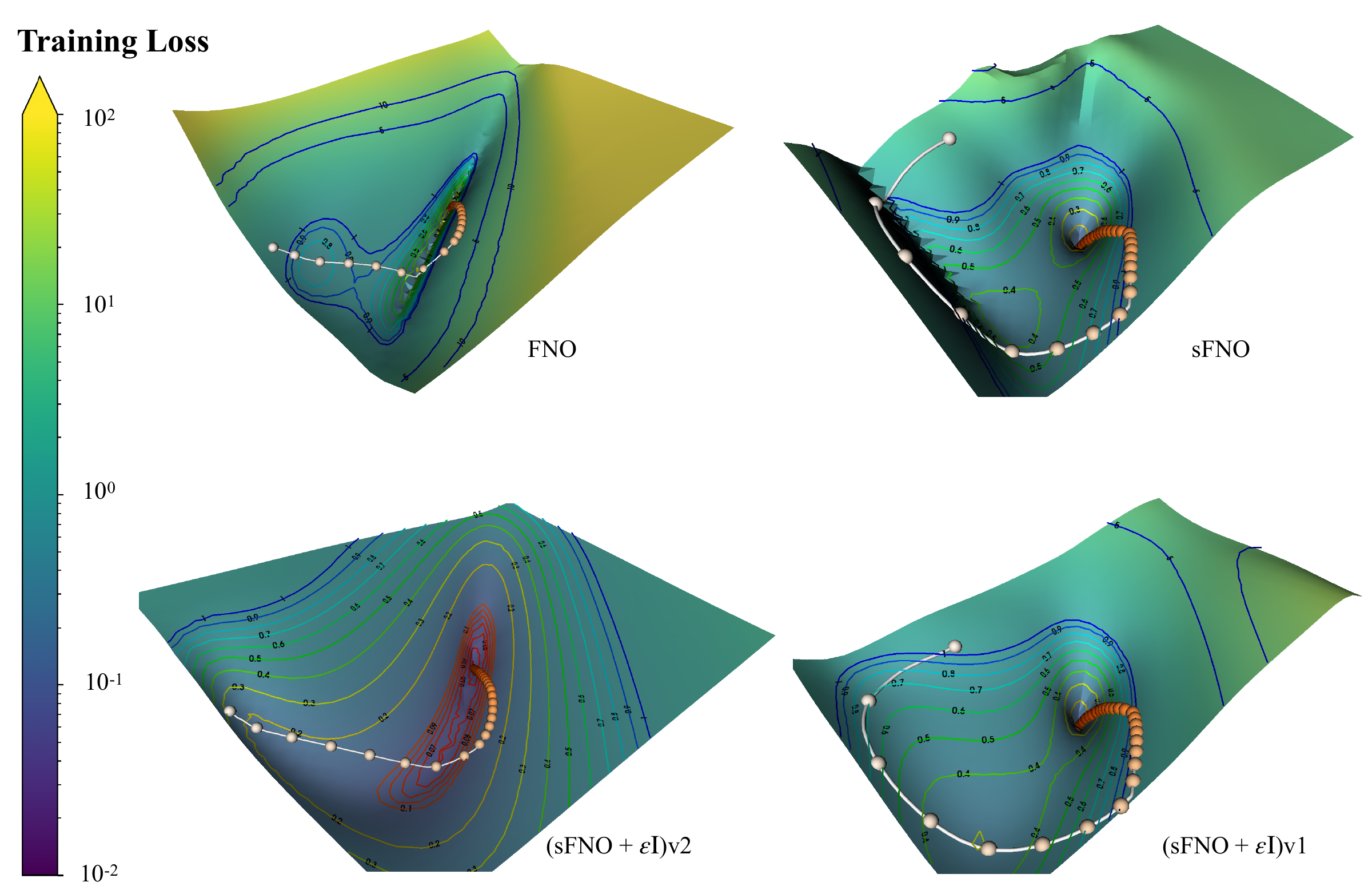}
\caption{Learning landscapes of the four considered models. The loss is visualized in logarithmic scale. Level sets reveal significant differences in topologies.\label{fig:loss landscapes}}
\end{figure}

As can be seen, the landscapes fall into three major categories. The FNO landscape is characterized by the presence of a shallow and irregular crease-like structure that runs across the domain. The $\MFNO$ and $\FNONeXt$ landscapes share remarkable similarities, which is consistent with the similar loss values shown in Table~\ref{table:experiment1:l2error}. Both possess a well-delineated and deeper convergence basin. Finally, the $\FNONeXt$ v2 landscape exhibits a crease-like structure similar to the one seen in the FNO landscape, but its topology is much simpler, and the central anisotropic basin is the deepest of all considered models. We refer the reader to \cref{appendix:loss landscape visualization} for more visualization.

 The level sets are plotted with a spectral (rainbow) color map that contrasts with the underlying landscape color scale (shown). The associated values are visible. The individual points along the trajectory in each landscape show every 10th training epoch, and the orange color saturation encodes the epoch. The increased geometric complexity along the diagonal crease present in the FNO and $\FNONeXt$ v2 landscapes was handled with a refined sampling of the training loss in the corresponding area. The principal components that span the two-dimensional sampling domain were computed by splitting real and imaginary parts of the layers' complex weights to form the large column vector representations of each model in the covariance matrix.

%--------------------------------------------------
\subsection{Ablation study} \label{experiment:ablation}
%--------------------------------------------------
\noindent
\begin{minipage}[c]{0.42\textwidth}
We have already conducted a study of ablation to some extent by the design of the networks. For example, when the skip connection is removed, $\FNONeXt$ without stochastic depth (V1) reduces to $\MFNO$. Similarly, when we set $\prob\{\boldsymbol{\mathrm{X}}_\ell = 1\} = 1$ for all layers $\ell$, then ($\FNONeXt$)v2 reduces to ($\FNONeXt$)v1.

In the following, we explore the changes in activation functions, with a particular focus on the identity activation, $\sigma(x) = x$ for the Fourier layers. Additionally, we investigate the behavior of the residual version of FNO as described in \citet{you2022learning}. The parameters are prescribed in \cref{table:hyperparam:experiment2}.

We adopt a strategy similar to \cref{experiment: multiple_initialization:gelu}, training each network multiple times with different random seeds to ensure the consistency of our empirical findings.
 
\end{minipage}
\hfill
\begin{minipage}[c]{0.55\textwidth}
    \centering
        \includegraphics[width=1\linewidth]{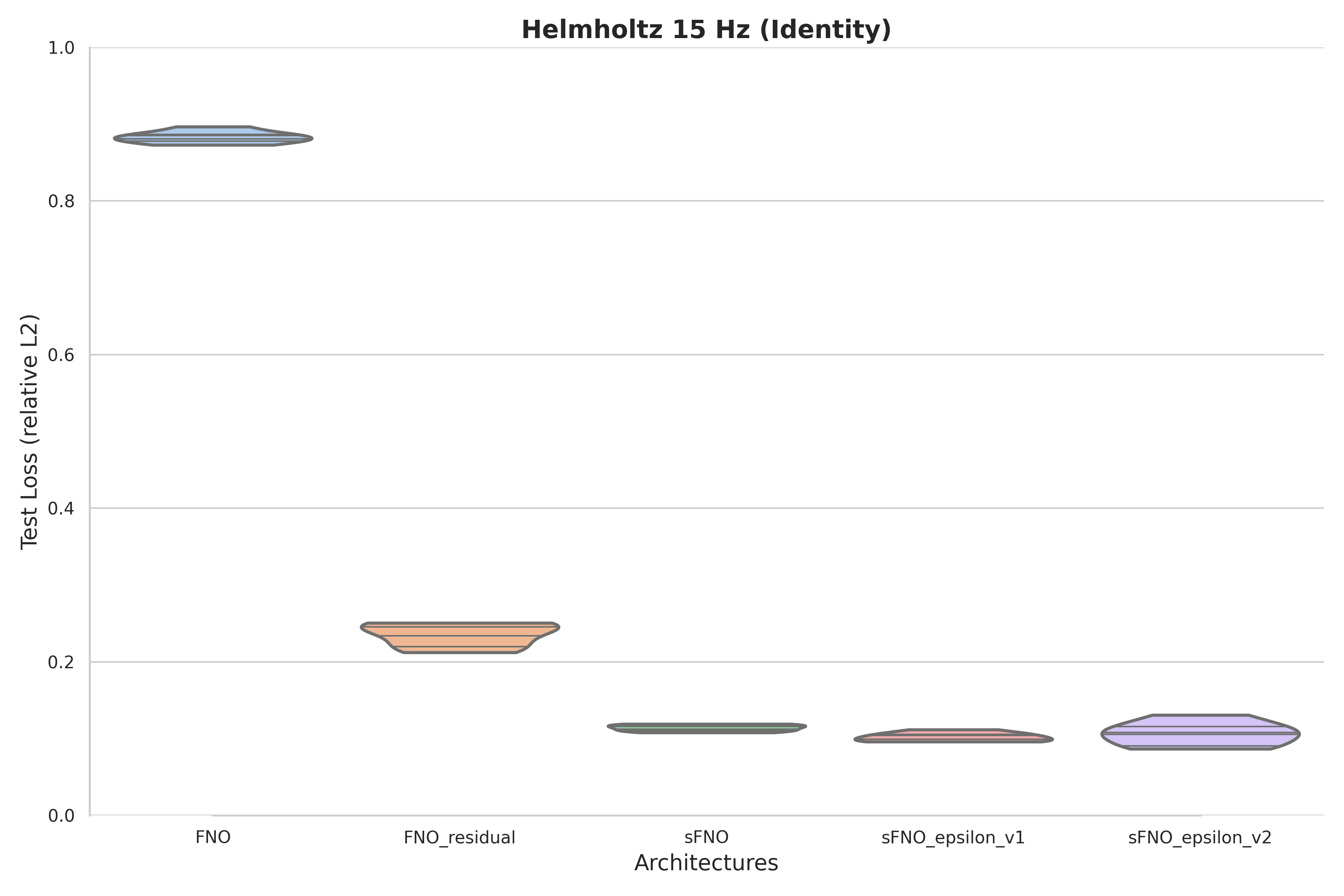}
        \captionof{figure}{Test-loss with no activation $15$ Hz of \eqref{eq:Helmholtz}. Each architecture is trained $6$ times, the rel. $\Lp$-loss on the test set.}
        \label{fig:boxplot_Id}
\end{minipage}
\par

\vspace{3 pt}
In the case where no activation is used in the Fourier layers, we observe that $\MNO$ achieves a lower relative $\Lp$ loss compared to FNO supported by results in \cref{fig:boxplot_Id} and \cref{table:boxplot_Id}. Notably, even when FNO is trained with a non-linear activation function (as proposed in \citet{li2020fourier}), $\MNO$ consistently exhibits significantly smaller test loss. This distinction can be observed by comparing the first violin plot in \cref{fig:boxplot} (detailed values are presented in \cref{table:experiment1:l2error}), \cref{fig:boxplot_leaky_relu}, and \cref{fig:boxplot_relu} representing FNO with GeLU, leaky-ReLU, and ReLU, activation functions respectively, with the third violin plot in \cref{fig:boxplot_Id} (see values in \cref{table:boxplot_Id}). Additionally, the second violin plot in \cref{fig:boxplot_Id}  also presents the residual implementation of FNO. There is a significant improvement observed over FNO.

\begin{table}[ht!]
    \centering
\small
    \begin{tabular}{@{}lllll@{}}
\toprule
 FNO & FNO residual& {$\MFNO$} & {($\FNONeXt$)v1} & {($\FNONeXt$)v2} \\ 
\midrule
 $0.873029$ &  $0.219886$  & $0.107912$      & $0.096015$                       & $0.130741$                       \\ 
 $0.880681$ &  $0.250052$  & $0.110204$      & $0.096212$                       & $0.115885$                       \\ 
 $0.885039$ &  $0.212233$  & $0.119039$      & $0.099064$                       & $0.105632$                       \\ 
 $0.896779$ &  $0.245763$  & $0.118050$      & $0.105330$                       & $0.086690$                       \\ 
 $0.878160$ &  $0.233708$  & $0.112514$      & $0.111709$                       & $0.108062$                       \\ 
 $0.886912$ &  $0.250634$  & $0.115919$      & $0.104147$                       & $0.090379$                       \\ 
\bottomrule
\end{tabular}
    \caption{Test-loss with no activation $15$ Hz of \eqref{eq:Helmholtz}. Each architecture is trained $6$ times, the rel. $\Lp$-loss on the test set.}
    \label{table:boxplot_Id}
\end{table}

\noindent
\begin{minipage}[c]{0.45\textwidth}
    \centering
        \includegraphics[width=1\linewidth]{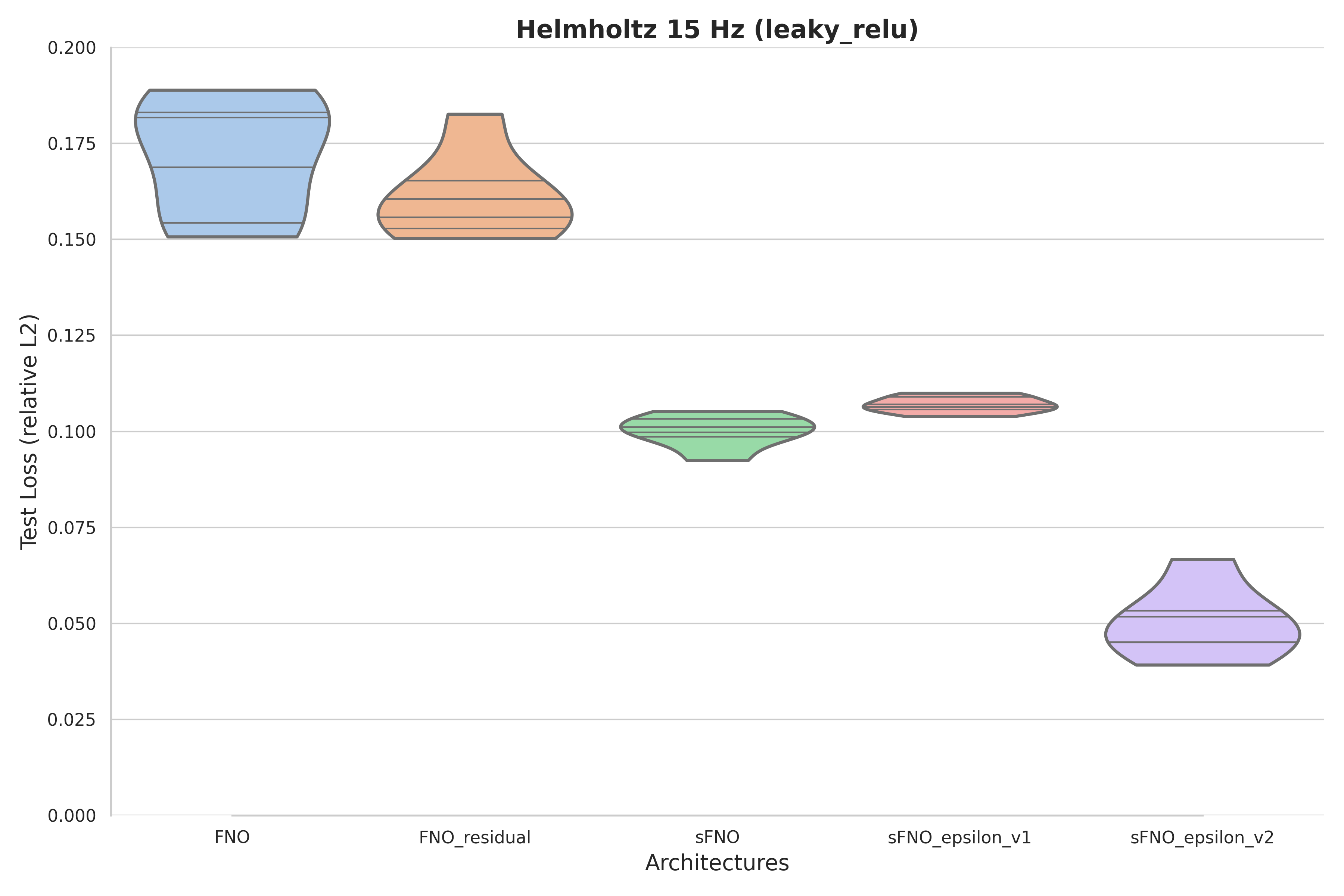}
        \captionof{figure}{Test-loss with Leaky-ReLU $15$ Hz of \eqref{eq:Helmholtz}. Each architecture is trained $6$ times, the rel. $\Lp$-loss on the test set.}
        \label{fig:boxplot_leaky_relu}
\end{minipage}
\hfill
\begin{minipage}[c]{0.45\textwidth}
    \centering
        \includegraphics[width=1\linewidth]{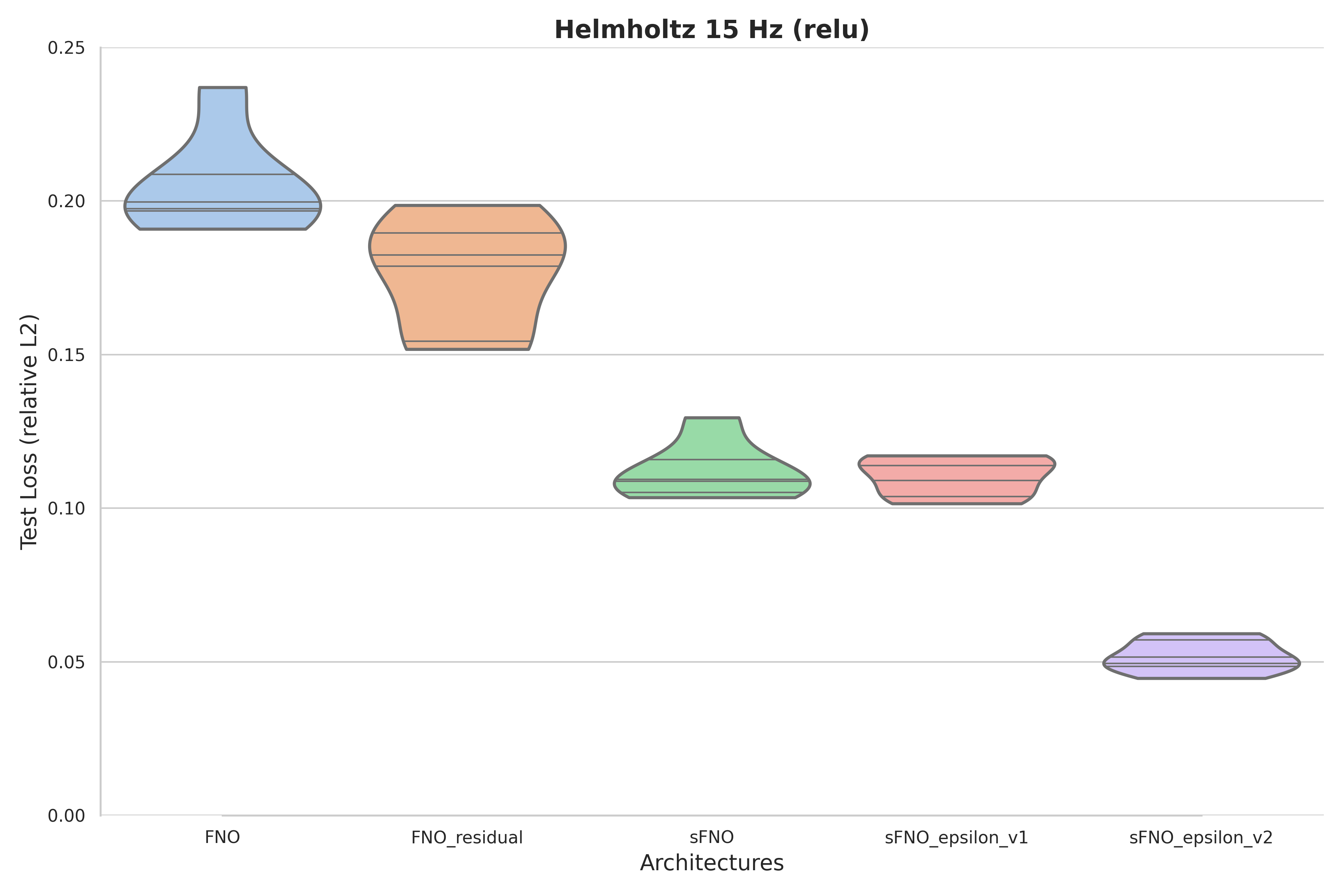}
        \captionof{figure}{Test-loss with ReLU $15$ Hz of \eqref{eq:Helmholtz}. Each architecture is trained $6$ times, the rel. $\Lp$-loss on the test set.}
        \label{fig:boxplot_relu}
\end{minipage}
\par 
The results obtained from \crefrange{fig:boxplot}{fig:boxplot_relu} demonstrate that the residual architecture aligns with the findings of \citet{you2022learning}. In every case, we observe a noticeable improvement in the relative $\Lp$-loss.

Among the different architectures, $\MNO$ achieves the most significant improvement compared to the previous architecture. Within the activation functions, Leaky-ReLU and ReLU exhibit the most significant change when transitioning from the architecture FNO to $\MNO$. In contrast, when the identity activation is used in the Fourier layers, $\FNONeXt$ with stochastic depth does not show a noticeable improvement compared to its counterpart ($\FNONeXt$)v1. However, in all cases where a nonlinear activation is employed, ($\FNONeXt$)v2 consistently outperforms other architectures without any sign of overfitting. Notably, for ReLU and Leaky-ReLU activations, the potential benefits of the skip connection are difficult to observe, when compared to $\MNO$.

    %input{Experiments/2.tex} %%Xavier

\section{Testing out-of-distribution analysis (OOD) for $\Gcal$}\label{section:experimental:OOD}
    In this section, we study the out-of-distribution (OOD for short) behavior for all the architectures. Specifically, we investigate how the models perform when faced with perturbations in the covariance operators of the Gaussian fields used for training. Our findings demonstrate that the $\NONeXt$ architecture with stochastic depth shows resilience to these perturbations. However, despite these encouraging results, the theoretical understanding of the impact of Bernoulli's random variable on the generalization ability of the neural operators in the context of OOD is still in its early stages. To provide further insights, we present a theoretical analysis in \cref{section:theory_OOD} for Gaussian measures.  \par

We recall from \cref{subsection:GRF-wavespeed}, and particularly \eqref{Whittle–Mat} that the Whittle-Matérn fields have three essential parameters: variance $s$, smoothness $\nu$, and correlation range $\boldsymbol{\lambda}$. As mentioned in \cref{Experiments}, the neural operators were trained using Gaussian random fields (GRF) with an isotropic Whittle-Mat\'{e}rn covariance operator such that the wave speed $c$ varied between 1500 and 5000, $\boldsymbol{\lambda} = (1,1)$, and the smoothness coefficient is $\nu= 1$. \emph{Throughout this section the models obtained in \cref{Experiments} are not retrained}. We refer to the settings in \ref{experiment2}, and \cref{exp:GRF_param} for details. 

%-----------------------------------------
\subsection{OOD experiments with different correlation and affine transformation}
%-----------------------------------------
We investigate the effect of changing the correlation parameter $\boldsymbol{\lambda}$ and the range on which the wave speeds vary. By adjusting $\boldsymbol{\lambda}$, it changes the  correlation range of the field. The scenario where $\lambda_1 \neq \lambda_2$ in equation \eqref{Whittle–Mat} is particularly interesting as it introduces non-euclidean distances and leads to the generation of anisotropic fields. The range of the wave speeds are adjusted using a different affine transformation denoted as $\mathrm{T}$, as explained in \cref{subsection:GRF-wavespeed}. Here, we keep the smoothness coefficient fixed, ensuring that the wave speeds remain within the same Sobolev space as the training data. Then, the new realizations of the wave speeds are given by $c' = {\mathrm{T}}' \circ Z^{(\lambda_1, \lambda_2)}$, in which $\mathrm{T}'$ changes the wave speed interval, and $(\lambda_1, \lambda_2)$ the correlation of points in the domain.  \par
To generate new samples of the wave speed $c$, we sample GRF following the parameters described in \cref{table:ood}. For each family, we generate $100$ samples and we obtain the corresponding solution of Helmholtz using the software \texttt{hawen}. The smoothness of $c$, domain's configuration $D$, source position, and frequency $\omega$ are fixed following \eqref{experiment2}.

\begin{comment}
the software \texttt{hawen} was utilized to derive the corresponding wave fields. As seen from \cref{table:ood}, in each family case, we made adjustments to either the mean or the parameter $\boldsymbol{\lambda}$ to introduce different ranges in the realizations. A total of six families, as listed in \cref{table:ood}, were generated. The domain's size, source positions, and frequency were kept fixed, consistent with the trained data. Each family comprised 100 samples of the $(c_i, p_i)$ pairs following the new distribution. It is important to note that the neural networks are not retrained for the evaluation on these new samples.
\end{comment}

\begin{center}
\begin{tabular}{@{}llll@{}}
\toprule
GRF model       & $\lambda_1$ & $\lambda_2$  & wave speed interval \\
\midrule
\textbf{Training} (baseline)  &  0.10 & 0.10 &  $[1500,5000]$ \\
OOD family 1     &  0.20 & 0.20 &  $[1500,5000]$\\
OOD family 2     &  0.10 & 0.20 &  $[1500,5000]$\\
OOD family 3     &  0.20 & 0.20 &  $[2000,3500]$ \\
OOD family 4     &  0.10 & 0.20 &  $[2000,3500]$ \\
OOD family 5     &  0.10 & 0.30 &  $[2000,6000]$ \\
OOD family 6     &  0.25 & 0.75 &  $[2000,6000]$ \\
\bottomrule
\end{tabular} 
\captionof{table}{Parameters for the experiments out-of-distribution. $\boldsymbol{\lambda}= (\lambda_1, \lambda_2)$ is defined in \cref{eq:lambda_cov}. The parameter $\nu$ is fixed to $1$.} 
\label{table:ood} 
\end{center}
%-----------------------------------------
\paragraph{Empirical analysis of OOD for each family} \label{subsection:OOD}
%-----------------------------------------
For the experiment, we selected three out of the six previously trained models (specifically, the first three models in \cref{fig:boxplot}) that utilized the GeLU activation function, see \cref{experiment: multiple_initialization:gelu}. We recall that we have obtained an estimation of the expected error within the distribution by evaluating the empirical loss in a test data set of the same size as the training data set\footnote{We trained the models using a dataset of $25,000$ samples and evaluated their performance on a separate test dataset also consisting of $25,000$ samples.}, for more details we refer to the training baseline in \cref{table:experiment1:l2error}, and \cref{experiment:training}. \par

By sampling multiple realizations of new random fields according to the families outlined in \cref{table:ood} we are able to estimate the expected error of the trained network with respect to these new probability distributions, and in consequence the robustness of the networks towards these changes. This enables us to assess the models' performance on the new samples and evaluate its generalization capabilities beyond the in-distribution data. \par

The empirical results for all architectures are presented in \crefrange{OOD:table1}{OOD:table6}. 
\vspace{7pt}

\noindent
\begin{minipage}[c]{0.47\textwidth}
\small
\renewcommand{\arraystretch}{0.8}
\begin{tabular}{@{}lllll@{}}
\toprule
OOD 1 & FNO & $\MNO$ & ($\FNONeXt$)v1 & ($\FNONeXt$) v2 \\ 
\midrule
 model 1 & 0.6689 & 0.6025 & 0.5341 & 0.2502 \\ 
 model 2 & 0.6842 & 0.5437 & 0.5451 & 0.2347 \\ 
 model 3 & 0.6817 & 0.5837 & 0.5404 & 0.2428 \\ 
\bottomrule
\end{tabular}
\captionof{table}{Relative test loss of three networks tested with the probability defined by family 1.}
\label{OOD:table1}
\end{minipage}
\hfill
\begin{minipage}[c]{0.47\textwidth}
\small
\renewcommand{\arraystretch}{0.8}
\begin{tabular}{@{}lllll@{}}
\toprule
OOD 2 & FNO & $\MNO$ & ($\FNONeXt$)v1 & ($\FNONeXt$) v2 \\ 
\midrule
 model 1 & 0.6602 & 0.6106 & 0.5438 & 0.2340  \\
 model 2 & 0.6715 & 0.5644 & 0.5561 & 0.2239  \\
 model 3 & 0.6726 & 0.5959 & 0.5509 & 0.2407  \\
\bottomrule
\end{tabular}
\captionof{table}{Relative test loss of three networks tested with the probability defined by family 2.}
\label{OOD:table2}
\end{minipage}

%------------------

\vspace{5pt}

\noindent
\begin{minipage}[c]{0.47\textwidth}
\small
\renewcommand{\arraystretch}{0.8}
\begin{tabular}{@{}lllll@{}}
\toprule
OOD 3 & FNO & $\MNO$ & ($\FNONeXt$)v1 & ($\FNONeXt$) v2 \\ 
\midrule
 model 1 & 0.5116 & 0.4368 & 0.3645 & 0.1324\\
 model 2 & 0.4757 & 0.3490 & 0.3678 & 0.1220 \\
 model 3 & 0.5001 & 0.4061 & 0.3490 & 0.1368 \\
\bottomrule
\end{tabular}
\captionof{table}{Relative test loss of three networks tested with the probability defined by family 3.}
\label{OOD:table3}
\end{minipage}
\hfill
\begin{minipage}[c]{0.47\textwidth}
\small
\renewcommand{\arraystretch}{0.8}
\begin{tabular}{@{}lllll@{}}
\toprule
OOD 4 & FNO & $\MNO$ & ($\FNONeXt$)v1 & ($\FNONeXt$) v2 \\ 
\midrule
 model 1 & 0.5249 & 0.4685 & 0.3845 & 0.1335 \\
 model 2 & 0.4992 & 0.3798 & 0.3869 & 0.1249 \\
 model 3 & 0.5146 & 0.4264 & 0.3713 & 0.1376 \\
\bottomrule
\end{tabular}
\captionof{table}{Relative test loss of three networks tested with the probability defined by family 4.}
\label{OOD:table4}
\end{minipage}

%------------------

\vspace{5pt}

\noindent
\begin{minipage}[c]{0.47\textwidth}
\small
\renewcommand{\arraystretch}{0.8}
\begin{tabular}{@{}lllll@{}}
\toprule
OOD 5 & FNO & $\MNO$ & ($\FNONeXt$)v1 & ($\FNONeXt$) v2 \\ 
\midrule
 model 1 & 0.9248 & 0.8698 & 0.8827 & 0.3899 \\
 model 2 & 0.9471 & 0.8379 & 0.8209 & 0.3910 \\
 model 3 & 1.0488 & 0.9269 & 0.8130 & 0.4188 \\
\bottomrule
\end{tabular}
\captionof{table}{Relative test loss of three networks tested with the probability defined by family 5.}
\label{OOD:table5}
\end{minipage}
\hfill
\begin{minipage}[c]{0.47\textwidth}
\small
\renewcommand{\arraystretch}{0.8}
\begin{tabular}{@{}lllll@{}}
\toprule
OOD 6 & FNO & $\MNO$ & ($\FNONeXt$)v1 & ($\FNONeXt$) v2 \\ 
\midrule
 model 1 & 0.9707 & 0.8903 & 0.9606 & 0.4426 \\
 model 2 & 1.0087 & 0.8851 & 0.8576 & 0.4585 \\
 model 3 & 1.1578 & 0.9831 & 0.8712 & 0.4864 \\
\bottomrule
\end{tabular}
\captionof{table}{Relative test loss of three networks tested with the probability defined by family 6.}
\label{OOD:table6}
\end{minipage}

\vspace{7pt}
The families presented in \crefrange{OOD:table4}{OOD:table6} exhibit anisotropy due to the difference in the values of $\lambda_1$ and $\lambda_2$. When considering the relative $\Lp$ loss as a reference, it is evident that the trained FNOs perform significantly worse compared to other neural operators. This indicates that FNOs may struggle with generalizing to new distributions. However, we observe that the $\NONeXt$ architecture coupled with stochastic depth demonstrates notable robustness when faced with changes in distribution across all the families. In particular, we notice for the experiments in \cref{OOD:table6} where both $\mathrm{T}'$ and $\boldsymbol{\lambda}$ are changed, the $\FNONeXt$v2 exhibits superior adaptability compared to other architectures, resulting in test losses that are half the values of any other neural operator.

%-----------------------------------------
\paragraph{OOD wave field ''prediction'' by the neural networks}
%-----------------------------------------

In \cref{fig:ood6}, we present the wave field predictions of the trained networks from family $6$. The figure showcases two samples from family $6$, illustrating the shortcomings of the FNO in accurately reproducing the desired behavior. In particular, we emphasize the discrepancy within the green rectangle, which indicates a notable deviation between the predicted wave field and the ground truth. This discrepancy further highlights the limitations of the FNO in capturing the complex dynamics of the wave propagation from different distributions. Among the models considered, it is observed that only the $\FNONeXt$v2 model is capable of accurately predicting admissible wave propagation in the family $6$.   
%From the solution prediction, $\FNONeXt$v2 is the only model able to predict an admissible wave propagation. 
\setlength{\modelwidth} {3.7cm}
\setlength{\modelheight}{3.7cm}
\begin{figure}[!ht] \centering
   \input{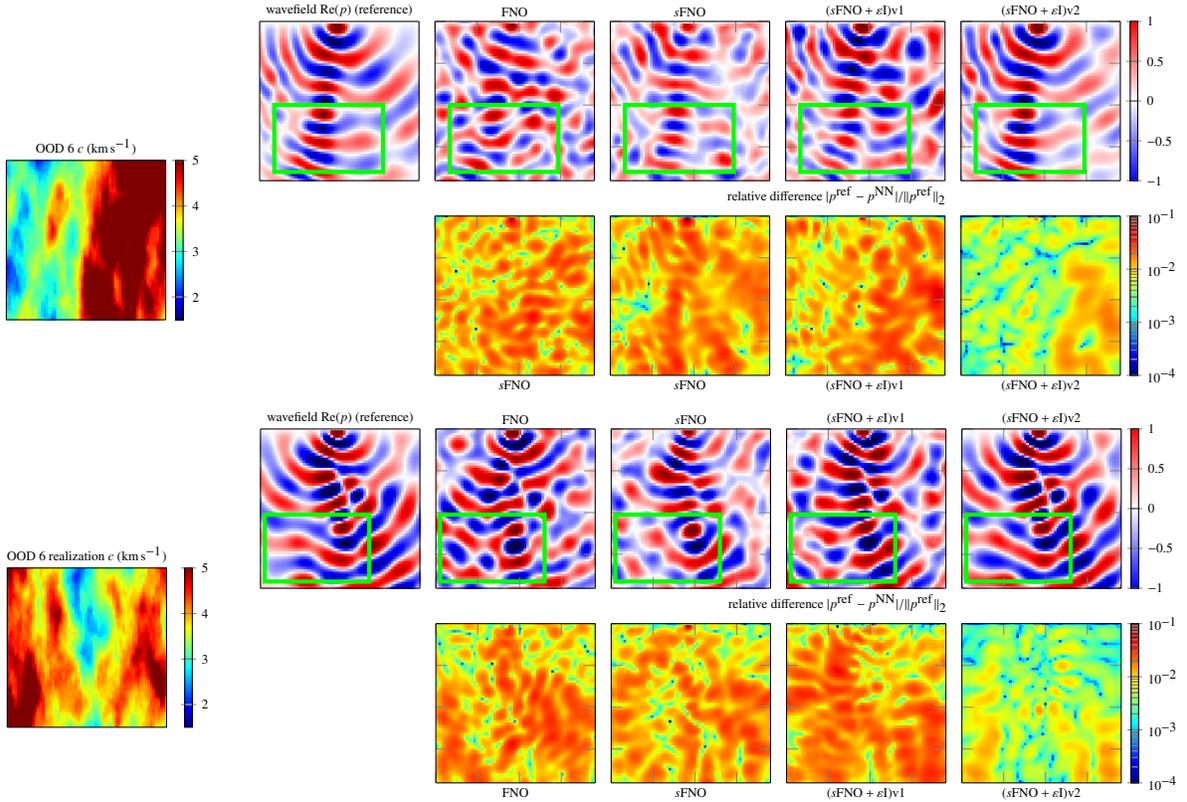}
    \caption{OOD (family 6). Real part of the wave field of OOD family $6$. Anisotropic case, see table \cref{table:ood} and \cref{OOD:6}. The green square positioned on the image serves as a visual aid to help identify and compare the differences in the reconstructed fields.}   
    \label{fig:ood6}
\end{figure}

     %\subsection{OOD moving the smoothness coefficient of the Whittle–Matérn field.}
\subsection{OOD experiments changing the smoothness of the field}
Here, we change the smoothness of the wave speed by modifying the parameter $\nu$. We recall from \eqref{eq:Matern} that the regularity of the field, $\beta$ is directly connected with the dimension $d$ of the domain (in our case $d=2$), and the coefficient $\nu$ (in our training $\nu=1$). Thus, by changing $\nu$, we generate Gaussian random fields of different Sobolev regularity, than those using in the training dataset. Our experiments are divided into two categories. (a) We first keep all but $\nu$ parameters fixed, as described in \cref{exp:GRF_param}, that is, we only change the Sobolev class of the wave speed without altering any other factor (e.g. if the field is isotropic or anistropic). (b) Finally, we move the rest of the parameters, by following the description of the family $6$ in the \cref{table:ood}, the ''\emph{hardest}'' family in terms of solution field prediction and test loss, see \cref{OOD:table6}. 

\begin{center}
    \begin{tabular}{@{}lllll@{}}
\toprule
GRF model       & $\lambda_x$ & $\lambda_y$  & wavespeed interval & $\nu$ \\
\midrule
OOD family 7     &  0.10 & 0.10 &  $[1500,5000]$ & 0.5\\
OOD family 8     &  0.10 & 0.10 &  $[1500,5000]$ &3.5\\
OOD family 9     &  0.25 & 0.75 &  $[2000,6000]$ &0.5\\
OOD family 10    &  0.25 & 0.75 &  $[2000,6000]$ & 3.5\\

\bottomrule
\end{tabular}
\captionof{table}{Parameters for the experiments out-of-distribution. $\boldsymbol{\lambda}= (\lambda_1, \lambda_2)$ is defined in \cref{eq:lambda_cov}. The parameter $\nu$ is changing.} 
\label{table:ood_2}
\end{center}

\begin{comment}
\noindent
\begin{minipage}{0.5\textwidth}
By selecting the random field as described in \cref{subsection:GRF-wavespeed}, we gain control over crucial aspects such as variance, smoothness, and correlation range. In the experiments conducted in the previous section (\cref{subsection:OOD}), we maintained the smoothness parameter constant, i.e., $\nu = 1$. However, to provide a exhaustive analysis, we will now examine the network's resilience to variations in the smoothness coefficient similarly, as in the previous subsection \textbf{we do not re-train the networks}. We move the smoothness coefficient of the the baseline, and family $6$ from the previous section, while keeping the rest of parameters fixed, wave speed range and correlation parameter, $\boldsymbol{\lambda}$, as outlined in \cref{table:ood_2}.
\end{minipage}
\hfill
\begin{minipage}[c]{0.45\textwidth}
\small
\begin{tabular}{@{}lllll@{}}
\toprule
GRF model       & $\lambda_x$ & $\lambda_y$  & wavespeed interval & $\nu$ \\
\midrule
OOD family 7     &  0.10 & 0.10 &  $[1500,5000]$ & 0.5\\
OOD family 8     &  0.10 & 0.10 &  $[1500,5000]$ &3.5\\
OOD family 9     &  0.25 & 0.75 &  $[2000,6000]$ &0.5\\
OOD family 10    &  0.25 & 0.75 &  $[2000,6000]$ & 3.5\\

\bottomrule
\end{tabular}
\captionof{table}{Parameters for the experiments out-of-distribution. $\boldsymbol{\lambda}= (\lambda_x, \lambda_y)$ is defined in \cref{eq:lambda_cov}. The parameter $\nu$ is changing.} 
\label{table:ood_2} 
\end{minipage}
\vspace{5pt}
\end{comment}

\paragraph{Empirical analysis of OOD for each family} We follow a similar procedure as in the previous subsection. We select three out of the six trained models shown in \cref{fig:boxplot}) and evaluate their performance against $100$ realizations of the wave speed for each of the new families described in \cref{table:ood_2}. We notice that the first two families preserves the rest of the parameters as in our training baseline, while in the last two all the parameters are changed. \par

In the selection of $\nu = 0.5$ and $\nu = 3$, the random fields in families 7 and 9 have Sobolev regularity $H^{1/2 - \varepsilon}(D)$, while the random fields $Z$ in families 8 and 10 have Sobolev regularity $H^{3+1/2 - \varepsilon}(D)$, almost surely. This is in contrast to the training data set, which lies almost surely in $H^{1 - \varepsilon}(D)$ for any $\varepsilon > 0$. For families 9 and 10, the wave speeds $c' = {\mathrm{T}}' \circ Z^{(\lambda_1, \lambda_2)}$ are different on each aspect than the training set. The affine maps $\mathrm{T}'$ are different from those used in the training. The correlation range $\boldsymbol{\lambda} = (0.25, 0.75)$ differs from the training correlation range, introducing anisotropy into the fields. Furthermore, the regularity of each realization in these families also varies from the baseline. \par

The empirical results for all architectures are presented in \crefrange{OOD:table7}{OOD:table10}. 
\vspace{7pt}

\noindent
\begin{minipage}[c]{0.47\textwidth}
\small
\renewcommand{\arraystretch}{0.8}
\begin{tabular}{@{}lllll@{}}
\toprule
OOD 7 & FNO & $\MNO$ & ($\FNONeXt$)v1 & ($\FNONeXt$) v2 \\ 
\midrule
 model 1 & 0.3257 & 0.3037 & 0.2889 & 0.1814 \\ 
 model 2 & 0.3244 & 0.3207 & 0.2905 & 0.1748 \\ 
 model 3 & 0.3261 & 0.3024 & 0.2921 & 0.1845 \\ 
\bottomrule
\end{tabular}
\captionof{table}{Relative test loss of three networks tested with the probability defined by family 7.}
\label{OOD:table7}
\end{minipage}
\hfill
\begin{minipage}[c]{0.47\textwidth}
\small
\renewcommand{\arraystretch}{0.8}
\begin{tabular}{@{}lllll@{}}
\toprule
OOD 8 & FNO & $\MNO$ & ($\FNONeXt$)v1 & ($\FNONeXt$) v2 \\ 
\midrule
 model 1 & 0.5508 & 0.4836 & 0.4621 & 0.2547  \\
 model 2 & 0.5527 & 0.5001 & 0.4706 & 0.2137  \\
 model 3 & 0.5547 & 0.4771 & 0.4604 & 0.2235  \\
\bottomrule
\end{tabular}
\captionof{table}{Relative test loss of three networks tested with the probability defined by family 8.}
\label{OOD:table8}
\end{minipage}

%------------------

\vspace{5pt}

\noindent
\begin{minipage}[c]{0.47\textwidth}
\small
\renewcommand{\arraystretch}{0.8}
\begin{tabular}{@{}lllll@{}}
\toprule
OOD 9 & FNO & $\MNO$ & ($\FNONeXt$)v1 & ($\FNONeXt$) v2 \\ 
\midrule
 model 1 & 0.9328 & 0.7053 & 0.6249 & 0.4419\\
 model 2 & 0.9209 & 0.8811 & 0.7141 & 0.3303 \\
 model 3 & 0.8794 & 0.7270 & 0.6231 & 0.3167 \\
\bottomrule
\end{tabular}
\captionof{table}{Relative test loss of three networks tested with the probability defined by family 9.}
\label{OOD:table9}
\end{minipage}
\hfill
\begin{minipage}[c]{0.47\textwidth}
\small
\renewcommand{\arraystretch}{0.8}
\begin{tabular}{@{}lllll@{}}
\toprule
OOD 10 & FNO & $\MNO$ & ($\FNONeXt$)v1 & ($\FNONeXt$) v2 \\ 
\midrule
 model 1 & 1.1806 & 0.8269 & 0.8217 & 0.5586 \\
 model 2 & 1.1391 & 0.9973 & 0.8727 & 0.4133 \\
 model 3 & 1.1049 & 0.8957 & 0.7825 & 0.4096 \\
\bottomrule
\end{tabular}
\captionof{table}{Relative test loss of three networks tested with the probability defined by family 10.}
\label{OOD:table10}
\end{minipage}

\vspace{7pt}
%-----------------------------------------
\paragraph{OOD wave field ''prediction'' by the neural networks (different smoothness)}
%-----------------------------------------

We showcase the performance of the trained networks by presenting wave field predictions for family $9$ and $10$, which correspond to rough and smooth anisotropic fields, respectively. These families pose a greater challenge for the neural operators, as evidenced by the test loss values shown in \crefrange{OOD:table9}{OOD:table10}. In the figures, we highlight in green some of the main discrepancies between the predicted wave fields of the architectures, and the reconstruction by \emph{numerical methods}. These discrepancies serve to illustrate the limitations and areas where the models may fall short in accurately reproducing the desired behavior.

\begin{comment}
We demonstrate the performance of the trained networks by presenting wave field prediction on family $9$ and $10$, corresponding to the rough, and smooth anisotropic fields, respectively. As seen in \crefrange{OOD:table0}{OOD:table10} for the test loss value, these families are the most complicated for the neural operators. Similarly as in the previous cases the figures highlights samples discrepancy of the modesl with respect to the HDG reconstruction. 
\end{comment}
\setlength{\modelwidth} {3.7cm}
\setlength{\modelheight}{3.7cm}
\begin{figure}[!htb] \centering
   \input{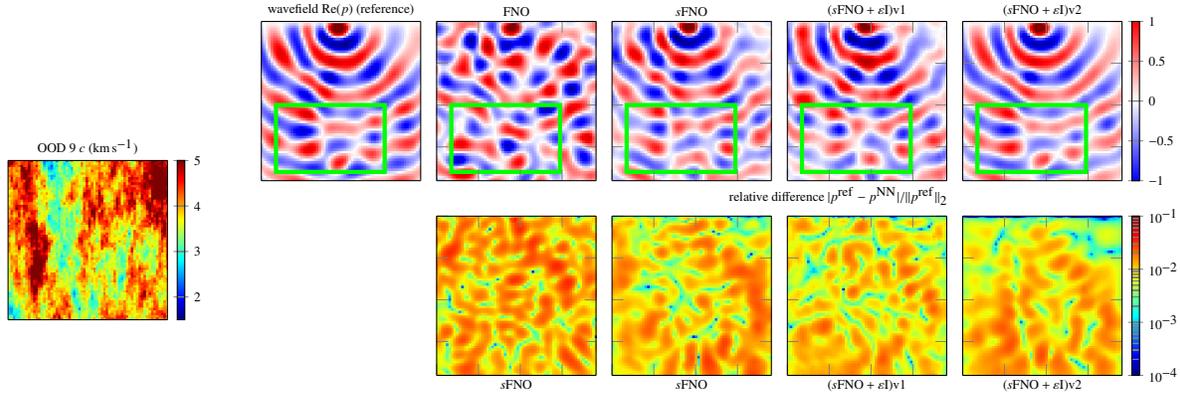}
    \caption{OOD (family 9). Real part of the wave field of OOD family $9$. Anisotropic case, with $\nu = 0.5$ see table \cref{table:ood_2}. The green square positioned on the image serves as a visual aid to help identify and compare the differences in the reconstructed fields.}   
    \label{fig:ood9}
\end{figure}

\begin{figure}[!htb] \centering
   \input{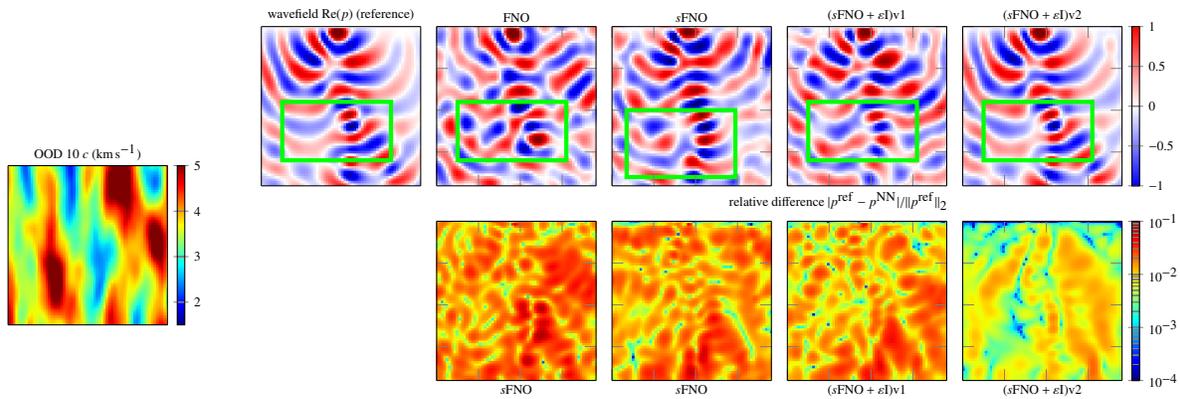}
    \caption{OOD (family 10). Real part of the wave field of OOD family $10$. Anisotropic case, with $\nu = 3.5$ see table \cref{table:ood_2}. The green square positioned on the image serves as a visual aid to help identify and compare the differences in the reconstructed fields.}   
    \label{fig:ood10}
\end{figure}

\begin{remark}
    We finally notice that the $\FNONeXt$ network has promising results with respect to the BP 2004 \citep{billette20052004} model. See \cref{BP_model}. It is worth noting that these findings go beyond the scope of the current theoretical framework described in \cref{section:theory_OOD}.
\end{remark}

\section{Hyperneural operator as a surrogate model of the forward operator: $\forward^{f}: (c,f) \mapsto \{ \pressure(\bx_j,f) \}_{j=1,\ldots,n_{\mathrm{rcv}}}$} \label{section:StR}
    We propose a hyperneural operator as a surrogate model for the forward operator associated with the inverse boundary value problem for the Helmholtz equation, as discussed in \cref{section:Forward}. Our experiments are based on two key assumptions that persist throughout this work: (1) the sources are point sources, and (2) the output is a fixed-size vector, representing measurements of $\pressure$ at the receiver locations. These assumptions align with the typical practical considerations of seismic wave propagation in an acoustic medium. However, in \cref{{eq:NN_2}}, we provide a potential relaxation of the first assumption to accommodate more general sources.\par

Notice that the direct application of a neural operator or any other derived architecture is difficult for the following reasons.
\begin{enumerate}[(a)]
    \item  The representation of the forward operator using neural networks faces challenges due to the distinct computational properties of point sources and wave speed. Wave speed can be discretized as a matrix $(\bx, c(\bx))$, while point sources are defined by their spatial position $\delta_{\bx} \leftrightarrow \bx$. To approximate the forward operator, a neural network must handle inputs of different natures (a point and a matrix) and generate an output with fixed discretization, based on the receiver positions. \par 

    In our experiments, the wave speed is discretized with a spacing of approximately 20 meters. However, point sources may not align precisely with the grid points of the wave speed field $c$. For example, the support of $\delta_{\bx}$, where $\bx = (x_1, x_2)$, may not necessarily be a multiple of 20. These discrepancies require careful consideration in designing the neural network architecture. The networks need to exhibit discretization invariance for both the wave speed parameter and the position of sources while ensuring that the output is discretized based on the receiver locations. 
    \item From a theoretical perspective, the forward operator is as a map from a function space to a linear bounded operator (the data operator). By construction, neural operators only deal with maps from functions to functions, not from functions to operators (some interesting alternatives are proposed in \cite{molinaro2023neural} and \cite{de2022deep}). See \citet[Sec 2.1]{beretta2016inverse} for the description of the forward operator in the time-harmonic case.
\end{enumerate}

\noindent Given the previous difficulties, we proposed a hypernetwork solution, partially inspired by the empirical work of \citep{zhmoginov2022hypertransformer} and the theoretical results on hypernetworks of \citep{Acciaio2022_GHT}, subsequently improved in \cite{galimberti2022designing}. See \cref{fig:HyperNetwork}. 

\begin{remark}
Although our primary focus has been on the experimental implementation of the forward problem, we will consider the inverse problem in the future. Bayesian statistical approaches of the inverse problems such as Markov chain Monte Carlo \citep{stuart2010inverse} and ensemble Kalman filter \citep{iglesias2016regularizing, iglesias2013ensemble} are commonly employed in inversion. However, these methods rely solely on the forward operators, but the computational challenge arises from multiple forward models. Our approach provides a \textbf{surrogate forward operators}, that, once trained, enables straightforward and efficient computation of multiple forward models. Therefore, we anticipate that by combining Bayesian statistical approaches with our method, we will be able to solve \emph{Bayesian statistical inverse problems}.
\end{remark}

\paragraph{Architecture}
\renewcommand{\mycolor}    {red!60!black}
\begin{figure}[!ht]
\centering
\begin{tikzpicture}
  \pgfmathsetmacro{\first}{0.} 
  \pgfmathsetmacro{\second}{-2.3}
  \pgfmathsetmacro{\colfirst}{-4.5}
  \pgfmathsetmacro{\colsec}{2.5}

  %coordinates
  \coordinate (r1)   at (\colfirst/2, \second); 
  \coordinate (r2)   at (\colfirst/2,  \second-1.3);
  \coordinate (r3)   at ( \colsec, \second-1.3);
  \coordinate (r4)   at (\colsec, \second );

  \coordinate (lc1)   at (\colsec-3.9, \first );
  \coordinate (lc2)   at (\colsec-3.9, \first +1.5);
  \coordinate (lc3)   at (\colsec-1.4, \first +1.5);
  \coordinate (lc4)   at (\colsec-1.4, \first);
%%%%%%%%%%%%%%%%%%%%%%%%%%%%%%%%%%%%%%%%%%%%%%%%%%%%%%%%%
  \coordinate (l2c1)   at (\colsec-1.1, \first );
  \coordinate (l2c2)   at (\colsec-1.1, \first +1.5);
  \coordinate (l2c3)   at (\colsec+1.1, \first+1.5 );
  \coordinate (l2c4)   at (\colsec+1.1, \first);

% Add a dashed square that covers rect1 and rect_mlp1

%circles
\node[draw,circle,fill=gray!20,rotate=0,minimum width=0.75cm,text width=0.75cm,text centered] (circ0) at (\colfirst,\first) {$f$};
\node[draw,circle,fill=gray!20,rotate=0,minimum width=0.75cm,text width=0.75cm,text centered] (circ1) at (\colfirst,\second) {$c$};
\node[draw,circle,fill=gray!20,rotate=0,minimum width=0.75cm,text width=0.75cm,text centered] (circ2) at (\colsec+2.5,\first) {$\mathcal{R} \left( \pressure^{f} \right)$};

% Add a dashed square that covers rect1 and rect_mlp1
\node[draw,rectangle,fill=pink!40,rotate=90,minimum width=1.cm,minimum height=1cm,text width=1.7cm,text centered, rounded corners] (feature_extr) at (\colfirst/2,\first) {$\mathrm{MLP}$}; 
\node[draw,rectangle,fill=blue!20,rotate=90,minimum width=1.cm,minimum height=1cm,text width=1.7cm,text centered, rounded corners] (L1) at (0, \first) {$\boldsymbol{\mathcal{L}_1\left(\cdot, {\color{\mycolor}{\theta_1^c}}\right)}$};
\node[draw,rectangle,fill=blue!20,rotate=90,minimum width=1.cm,minimum height=1cm,text width=1.7cm,text centered, rounded corners] (L2) at (\colsec,\first) {$\boldsymbol{\mathcal{L}_2\left(\cdot, {\color{\mycolor}{\theta_2^c}}\right)}$};

\node[draw,rectangle,fill=orange!25,rotate=0,minimum width=1.cm,minimum height=1cm,text width=1.7cm,text centered, rounded corners] (NO1) at (0, \second) {$\mathscr{E}_1 \circ  (\NONeXt)$};
\node[draw,rectangle,fill=orange!25,rotate=0,minimum width=1.cm,minimum height=1cm,text width=1.7cm,text centered, rounded corners] (NO2) at (\colsec, \second) {$\mathscr{E}_2 \circ (\NONeXt)$};

% Add a dashed square that covers rect1 and rect_mlp1
\draw[dashed, line width=1.5pt, rounded corners, rotate=90] ($(feature_extr.north west) + (-0.3, 0.4)$) rectangle ($(L2.south east) + (1.2,-0.85)$);

\draw[\mycolor ,->, line width=1.5pt] (circ1) -- (NO1);
\draw[\mycolor,line width=1.5pt] (r1) -- (r2);
\draw[\mycolor, line width=1.5pt] (r2) -- (r3);
\draw[\mycolor, ->, line width=1.5pt] (r3) -- (NO2);
\draw[ ->, line width=1.5pt] (circ0) -- (feature_extr);
\draw[ ->, line width=1.5pt] (feature_extr) -- (L1);
%%%%%%%%%%%%%%%%%%%%%%%%%%%%%%%%%%%%%%%
\draw[ ->, line width=1.5pt] (L1) -- (L2);
\draw[ ->, line width=1.5pt] (lc3) -- (lc4);
\draw[line width=1.5pt] (lc2) -- (lc3);
\draw[line width=1.5pt] (lc1) -- (lc2);

%%%%%%%%%%%%%%%%%%%%%%%%%%%%%%%%%%%%%%%
\draw[ line width=1.5pt] (l2c1) -- (l2c2);
\draw[ ->, line width=1.5pt] (l2c3) -- (l2c4);
\draw[line width=1.5pt] (l2c2) -- (l2c3);
\draw[line width=1.5pt] (l2c1) -- (l2c2);

\draw[\mycolor ,  ->,line width=1.5pt] (NO1) -- (L1) node[pos=0.25, right] {$\boldsymbol{\theta_1^c}$};
\draw[\mycolor ,  ->, line width=1.5pt] (NO2) -- (L2)node[pos=0.25, right] {$\boldsymbol{\theta_2^c}$};

\draw[line width=1.5pt] (l2c1) -- (l2c2);
\draw[ ->, line width=1.5pt] (L2) -- (circ2);

\end{tikzpicture}
\caption{Hypernetwork surrogate of the forward operator used in the experiments. We call the network inside the dashed rectangle a \textit{metanetwork}, and the bottom network a \textit{hyperneural operator}.}
\label{fig:HyperNetwork}
\end{figure}
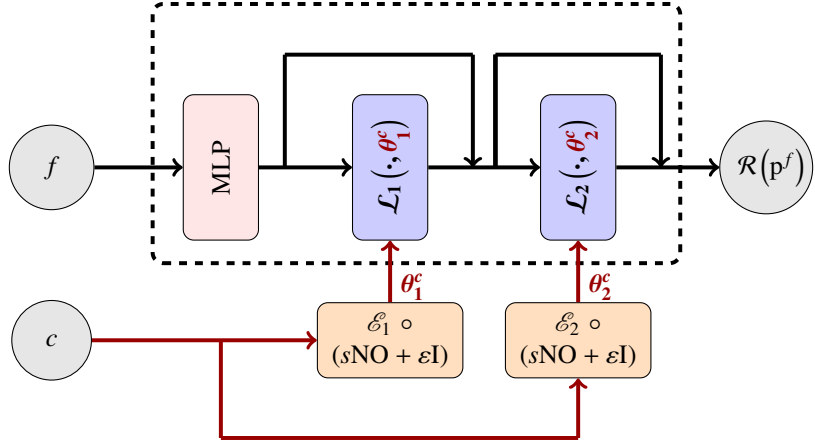

The layers $\mathcal{L}_k$  for $k=1, 2$ in the \cref{fig:HyperNetwork} are simple layers of Euclidean neural networks, that is $\mathcal{L}_k(\bx) = \sigma \circ \left( W_k^c + b_k^c\right) \circ \bx$. So that, 
$[W_k^c, b_k^c] = \mathscr{E}_k \circ \boldsymbol{\mathcal{G}}(c)$, where $\mathscr{E}_k$ is an encoder sending the values of $\boldsymbol{\mathcal{G}}(c)$ to a fixed parameter size, determining the capacity of the metanetwork (dashed rectangle from \cref{fig:HyperNetwork}). A more general setting can be considered from the layers $\mathcal{L}$, depending on the nature of the sources (point-sources type or more general sources). However, given the simplicity of the Source-to-Receiver map, and the imposed discretization in the output, we can associate the point-source with its support $\bx = (x_1, x_2)$ and the output is discretized by the number of receivers, corresponding to the columns of the of response matrix in \cref{figure:cp-StR:experiment}. The main difficulty of the approximation is coming from the nonlinear dependency of the Helmholtz equation with respect to the wave speed. 
\par 
We have the following association, if we call the metanetwork $\mathcal{NN}$ and the hyperneural operator $\mathcal{G}$, then for a point-source $\delta^\omega_{\bx}$ we have 
\begin{equation}\label{eq:hyperneuraloperator}
    \mathcal{NN}(\delta_{\bx}^\omega, \Theta(c)) = \mathcal{NN}(\delta_{\bx}^\omega, \mathscr{E}^{\mathrm{hyper}} \circ\boldsymbol{\mathcal{G}}(c)) \approx \mathcal{R} \left( \pressure^{\delta_{\bx}, \omega} \right).
\end{equation}
where $\mathcal{R}$ is a restriction operator which reduces the fields to the set of receivers positions \cref{section:Forward}, and $\Theta(c) =  \mathscr{E}^{\mathrm{hyper}} \circ\boldsymbol{\mathcal{G}}(c) \in \mathbb{C}^m$. In experiments, we see that a layer-wise form as \cref{fig:HyperNetwork} is more stable in the presence of the optimization algorithm. Similar conclusions were drawn in \cite{zhmoginov2022hypertransformer}. 
\par 
From \cref{eq:hyperneuraloperator}, it is evident that the point source is independent of the discretization used for $c$, and multiple sources can be implemented efficiently, by increasing the vector inputs, indicating the source's position. This means that the support of the point source can be finer than the discretization of $c$. On the other hand, the encoder, $\mathscr{E}^{\mathrm{hyper}}$, is used in a similar manner as in \emph{DeepOnet} described in \citet{lanthaler2022error}. Its purpose is to map the range of $\Gcal$ (the functional space) to a finite-dimensional space $\mathbb{C}^m$, which contains the parameters of the neural network $\mathcal{NN}$. Namely, $\mathscr{E}\circ \Gcal: H \to \mathbb{C}^m$, where $H$ represents the functional space where each realization of $c$ lies. In our case, $H$ can be identified with $H^{2\beta - d/2 - \varepsilon}(D)$, as the wave speed are realizations of the Whittle–Matérn field (see \cref{remark:grf}). Finally, the dimension $m$ in $\mathbb{C}^m$ depends on the capacity chosen for the metanet family, $\mathcal{NN}$. In our experiments, we restrict it to a small two-layers network. \par 

If $f$ are not point-sources, $\mathcal{NN}$ can be expressed firstly by one layer, followed by a second encoder $\mathscr{E}^{\mathrm{meta}}$ as the output is always discrete given the position of the receivers. That is, 

\begin{equation}\label{eq:NN_1}
    \mathcal{NN}(f, \Theta(c)) =\mathscr{E}^{\mathrm{meta}} \circ \mathrm{MLP}^{\theta_2} \circ \, \IDFT \left(  G_{k, \ell}^{\theta_1} (\xi) \, \DFT (f) \right) \approx \mathcal{R} \left( \pressure^{f, \omega} \right) , %^{\mathrm{approx}},
\end{equation}

\noindent and $\Theta(c) = [\theta_1, \theta_2] =  \mathscr{E}^{\mathrm{hyper}} \circ\boldsymbol{\mathcal{G}}(c)$. Rather than \eqref{eq:NN_1}, more general operator layers, neural operators, or DeepOnet networks can be used as a metanetwork. However, \cref{eq:Helmholtz} is linear with respect to $f$ for a fixed $c$ and $\omega$.  \par 

The most general form of the $\mathcal{NN}$ is

\begin{equation}\label{eq:NN_2}
    \mathcal{NN}(f, \Theta(c)) =\mathscr{E}^{\mathrm{meta}} \circ \mathcal{G}^{\mathrm{meta}}\bigg(f , \mathscr{E}^{\mathrm{hyper}} \circ\boldsymbol{\mathcal{G}}(c) \bigg) \approx \mathcal{R} \left( \pressure^{f, \omega} \right). %^{\mathrm{approx}}.
\end{equation}

\noindent for $\Theta(c) =  \mathscr{E}^{\mathrm{hyper}} \circ\boldsymbol{\mathcal{G}}(c)$, and $\mathcal{NN} = \mathscr{E}^{\mathrm{meta}} \circ \mathcal{G}^{\mathrm{meta}}$, composition of an encoder sending the values to the position of the receiver, and an operator network $\mathcal{G}^{\mathrm{meta}}$. Notice that $\mathscr{E}^{\mathrm{meta}}$ is playing a similar role to the restriction operator $\mathcal{R}$. Moreover, $c \mapsto \mathscr{E}^{\mathrm{meta}} \circ \mathcal{G}^{\mathrm{meta}}(\cdot , \mathscr{E}^{\mathrm{hyper}} \circ\boldsymbol{\mathcal{G}}(c))$ can be realized as an observational operator.

\paragraph{''Prediction'' of the "matrix" response for the forward operator}
The wave field reconstruction at  the receiver position, by probing multiple point sources is presented in \cref{figure:cp-StR:reconstruction}. The rows correspond to the multiple point sources, and the columns to the pressure field detected at multiple positions of the domain. In the top left side, we appreciate the wave speed, and bottom left side, the error of the approximation. The dataset of the experiment is described in \cref{section:Forward}, and the network is a special case of \eqref{eq:NN_2}, exactly described in \cref{fig:HyperNetwork}. 
%%%%%%%%%%%%%%%%%%%%%%%%%%%
% Reconstruction
%%%%%%%%%%%%%%%%%%%%%%%%%%%%%%%%

\setlength{\modelwidth} {6cm}
\setlength{\modelheight}{3.8cm}

% ---------------------------
\begin{figure}[ht!]\centering
\input{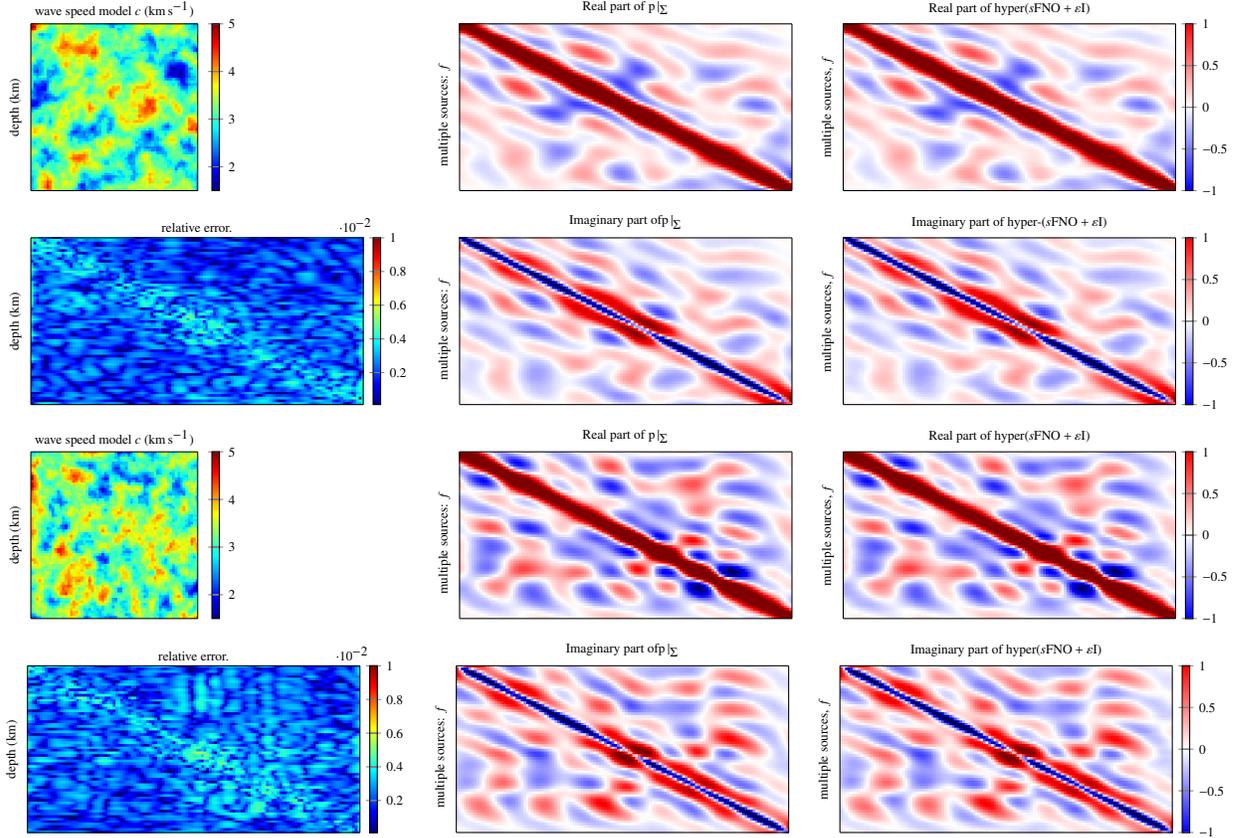}
\caption{Forward operator. Approximation of the forward operator by hyperneural operators.}
         \label{figure:cp-StR:reconstruction}
\end{figure}

%------------------------------------
\paragraph{Details of the experiment}\label{experiment:experiment_StR}
%------------------------------------
We employ the \textit{AdamW} optimizer \citep{loshchilov2017decoupled} with an initial learning rate of $10^{-3}$. We utilize a linear step scheduler with parameters: step size = $40$, and a multiplicative factor of learning rate decay of $\gamma = 0.5$. The number of epochs is set to $100$. In all architectures, we apply a small $\ell_2$ weight regularizer with a parameter of $10^{-5}$. Given that we already restricted the training to the empirical analysis of the architectures, the training process is conducted using $40,000$ out of $50,000$ generated samples, while $5,000$ samples are used for validation and $5,000$ for testing. \par 
The relative $\Lp$-error is $3\times 10^{-2}$. In the implementation, $ \mathscr{E} \circ\boldsymbol{\mathcal{G}}(c) \in \mathbb{R}^{2m}$ and the complex-valued product is defined independently. Also, we did not split the learning rate from the metanetwork and hypernetwork, nor did we incorporate a more robust feature-extractor, as \citep{zhmoginov2022hypertransformer}. Provided that the sources are point sources, the complexity of the task is encoded in the high capacity  of the neural operator defining the metanetwork. The latter is a residual network with $2$-layers, and leaky ReLU activation.

    \section{Out-of-distribution under Gaussian sampling} \label{section:theory_OOD}
       %\section{Out of Distribution Under Gaussian Sampling}
%\label{s:OOD}

%{\color{blue} Let us decide which Banach space to take later.}
While \cref{section:experimental:OOD} presents the empirical out-of-distribution performance of our network design, specifically in the context of time-harmonic waves, we consider here an analysis of the out-of-distribution phenomenon under centered Gaussian measures for Banach spaces. We recall from \cref{subsection:GRF-wavespeed} that the Whittle–Mat\'ern field belongs to the spaces $H^{2\beta - d/2 - \varepsilon}(D)$ a.s for all $\varepsilon>0$. The Whittle–Mat\'ern field generates a centered Gaussian measure on the Hilbert spaces to which it belongs, see e.g. \citep[Proposition 2.18]{da2014stochastic}. This property holds under mild assumptions of the negative fractional power $L^{-2\beta}$, as described in \citep{cox2020regularity}. Here, $L$ represents the second-order elliptic operator presented in \cref{eq:Matern}. \par 

We introduce the general framework for analyzing the out-of-distribution risk. (a) We defined the centered Gaussian measures on Banach spaces, (b) the Cameron-Martin spaces, and (c) the Wasserstein distance. Building upon these concepts and the powerful tools provided by Gaussian measures, we establish upper-bounds for the out-of-distribution risks associated with each of the architectures discussed in this paper. These bounds are expressed in terms of the Lipschitz norms of the neural operators. This theoretical foundation allows us to gain insights into the behavior of the neural operator family when confronted with data distributions that differ from the training distribution, as demonstrated in the experimental results in \cref{section:experimental:OOD}.\par

% By establishing this theoretical foundation, we enhance our understanding of the neural operator family in out-of-distribution scenarios, with particular focus on centered Gaussian sampling and potential distributional shifts,

\subsection{Preliminaries}
Our main theoretical results supporting the out-of-distribution performance of our neural operators require some background notions from optimal transport and the theory of Gaussian measures on Banach spaces, which we now review before stating our main learning-theoretic result.
\label{s:OOD_Preliminaries}
\paragraph{The order one Wasserstein distance}
In what follows, we make use of the {\it Wasserstein distance} of the order one between any two probability measures $\mu$ and $\nu$, denoted by $\mathcal{W}_{1}(\mu, \nu)$.
By the {\it Kantorovich–Rubinstein} duality, see \citep[Theorem 5.10]{villani2009optimal}, $\mathcal{W}_{1}(\mu, \nu)$ has the form
\[
\mathcal{W}_{1}(\mu, \nu)
=
\sup_{\underset{\left\| f
\right\|_{\mathrm{Lip}}\leq 1}{f \in \mathscr{F}}}
%% EASY TO SEE WE DON"T NEED ABSOLUTE VALUES
% \left| 
\E_{(a,u)\sim \mu}\left[f(a,u)\right]
-\E_{(a,u)\sim \nu}\left[f(a,u) \right]
% \right|
,
\]
where $\mathscr{F}$ is a class of the Lipschitz continuous operators mapping from $\mathcal{X}\times \mathcal{Y}$ to $\mathbb{R}$, and $\left\| \cdot \right\|_{\mathrm{Lip}}$ is the Lipschitz norm\footnote{\ $\|\cdot\|_{\mathrm{Lip}}$ is simply the $W^{1,\infty}$ norm. See \cref{appendix:Sobolev_spaces} and the reference therein.} defined by 
$$
\left\| f \right\|_{\mathrm{Lip}}
:=
\sup_{(a,u)}\, \vert f(a,u)\vert
+
\sup_{\substack{(a,u)\neq (b,v)}}
\frac{\vert f(a,u) - f(b,v) \vert}
{ \left\| (a,u) - (b,v) \right\|_{\mathcal{X}\times \mathcal{Y}}}
\ge 
\mathrm{Lip}(f),
$$
where $\mathrm{Lip}(f):=\sup_{\substack{(a,u)\neq (b,v)}}
\frac{\vert f(a,u) - f(b,v) \vert}
{ \left\| (a,u) - (b,v) \right\|_{\mathcal{X}\times \mathcal{Y}}}$.

\paragraph{Centered Gaussian measures on Banach spaces}
Let $\mathcal{X}$ be a separable Banach space of functions from $D$ to $\mathbb{R}^{d_a}$ and $\mathcal{Y}$ be a separable Banach space of functions from $D$ to $\mathbb{R}^{d_u}$.  Recall that $\mathcal{X}\times \mathcal{Y}$ is also a separable Banach space, when normed by
\[
        \|(x,y)\|_{\mathcal{X}\times \mathcal{Y}}
    :=
        (
            \|x\|_{\mathcal{X}}^2
        +
            \|y\|_{\mathcal{Y}}^2
        )^{1/2}
    .
\]

%{\color{red}
%[I moved definition of Gaussian, small ball function, and entropy to here. Is it fine with you ? Takashi]}

Let us briefly recall the definition of a Gaussian measure on a Banach space.

\begin{definition}[Gaussian measure]
\label{defn:Gaussian_Measure}
A measure $\mu_{X} \in \mathcal{P}_1(\mathcal{X})$ is said to be centered and Gaussian if, for every continuous linear functional $E \in \mathcal{X}^{\star}$ the measure $E_{\#}\mu_X$ is a zero-mean Gaussian on $\mathbb{R}$. 
The \textit{weak variance} 
% \antonio{what is the diff with the usual way of defining the cov by functionals and not sup} 
$\Sigma$ of $\mu_X$ is defined to be
\[
        \Sigma
    =
        \sup_{E\in \mathcal{X}^{\star},\,\|E\|\le 1}\,
        \mathbb{E}_{a \sim \mu_X}[E^2(a)]^{1/2}
    .
\]
\end{definition}

Associated to every centered Gaussian measure, we may define a \textit{small ball function} $\psi:(0,\infty)\rightarrow \mathbb{R}$ as
\[
        \psi(\eta)
    :=
        -\log\big(
            \mu_{X}
            (B(0,\eta))
        \big)
    ,
\]
for every $\eta>0$.
%{\color{blue}
There is a reproducing kernel Hilbert space $\mathcal{H}_{\mu}$ naturally associated to $\mu_X$ which is the completion of the range of the map $S:\mathcal{X}^{\star}\rightarrow\mathcal{X}$ sending any $E\in \mathcal{X}^{\star}$ to the Bochner integral\footnote{Cf. \cref{Bochner}.} $S(E):= \int_{a \in \mathcal{X}}\,E(a)\cdot a\,\mu_X(da)$ with respect to the inner product $\langle \cdot,\cdot \rangle_{\mu_X}$, defined for any $E,F\in \mathcal{X}^{\star}$ by $\langle S(E),S(F) \rangle_{\mu}:= \int_{a \in \mathcal{X}}\, E(a)F(a)\,\mu_X(da)$.  We denote the induced norm on $\mathcal{H}_{\mu_X}$ by $\|\cdot\|_{\mathcal{H}_{\mu_X}}$; which is induced by an inner product, see \cite{kuelbs1976strong}.  In fact, $\mathcal{H}_{\mu_X}$ is a reproducing kernel Hilbert space with a relatively compact unit ball, called the \textit{Cameron-Martin space} associated with the centered Gaussian measure $\mu_X$.  In fact, the Cameron-Martin space $\mathcal{H}_{\mu_X}$ characterizes $\mu_X$, see \citep[Chapter 8]{LedouxTalagrandBookRepring2} for details (we will briefly review the Cameron-Martin space in Section~\ref{Cameron-Martin space-app}).
\hfill\\
Since the closed unit ball $\overline{B_{\mathcal{H}_{\mu_X}}(0,1)}$ of $\mathcal{H}_{\mu_X}$ is compact (see\footnote{And the remark following its proof at the bottom of page 209.} \citep[Lemma 8.4]{LedouxTalagrandBookRepring2}), then its \textit{metric entropy} $H_{\mu_X}(\varepsilon):=\log(N_{\mu_X}(\varepsilon))$ are finite; where $N_{\mu_X}(\varepsilon):=\min\{n\in \mathbb{N}_+:\,\exists x_1,\dots,x_n\in \overline{B_{\mathcal{H}_{\mu_X}}(0,1)}\, \text{s.t. } \forall x  \in \overline{B_{\mathcal{H}_{\mu_X}}(0,1)}\, \exists i \in [n] \ \text{s.t. }  \|x-x_i\|_{\mathcal{H}_{\mu_X}}<\varepsilon\}$ is the \textit{covering number}\footnote{See \citep[Chapter 27]{shalev2014understanding} for details in the context of learning theory or \cite{CarlOriginalPaper1981} in the context of approximation theory.} of $\overline{B_{\mathcal{H}_{\mu_X}}(0,1)}$.  The key connection between the \textit{small ball function} $\psi$, a probabilistic notion, and entropy numbers, a constructive approximation theoretic tool, is that estimates on the growth of one imply estimates on the growth of the other.

\subsection{Out-of-distributional generalization}
\label{s:OOD_RiskBounds}

Consider an ``unknown'' $L^{\star}$-Lipschitz (non-linear forward) operator $\Gcal^{\star}:\mathcal{X}\rightarrow\mathcal{Y}$, a sampling distribution $\mu_X\in \mathcal{P}_1(\mathcal{X})$, that is $\mathbb{E}_{X\sim \mu_X}[\|X\|_\mathcal{X}]<\infty$, and a sequence of i.i.d.\ samples $(a_n)_{n=1}^{\infty}$ defined on a common measurable space $(\Omega,\mathcal{A})$, where $a_1$ 
has law $\mu_X$ and where $L^{\star}\ge 0$.  We also consider an \textit{out-of-distributional} sampling measure $\tilde{\mu}_X$ in $\mathcal{P}_1(\mathcal{X})$.
We consider a \textit{common irreducible measurement noises} $\epsilon$ taking values in $\mathcal{Y}$, and quantifying hardware and sampling limitations, defined on $(\Omega,\mathcal{A})$ and independent from $\{a_n\}_{n=1}^{\infty}$ with law $\mu_{\epsilon}\in \mathcal{P}_1(\mathcal{Y})$.

The data-generating and out-of-distribution, laws defined $\mu_{OOD:X}$ and $\mu$ are respectively defined by
\begin{equation}
\label{eq:True_Measure}
        \mu
    :=
        (\bid\times \Gcal^{\star})_{\#}\mu_X 
        \star
        \mu_{\epsilon}
\mbox{ and }
        \mu_{OOD}
    :=
        (\bid\times \Gcal^{\star})_{\#}\mu_{OOD:X}
        \star
        \mu_{\epsilon}
,
\end{equation}
where $\star$ is the convolution operation and $\bid$ is the identity map on $\mathcal{X}$.  The out-of-distributional measure $\mu_{OOD}$ and data-generating measures $\mu$ are \textit{coupled} via the following condition: there is a $\varepsilon\ge 0$ such that 
\begin{equation}
\label{eq:coupling_condition}
\tag{Coupling}
        \mathcal{W}_1(\mu_{OOD},\mu)
    \le 
        \varepsilon
.
\end{equation}
A key advantage of coupling $\mu$ and $\mu_{OOD}$ using the $\mathcal{W}_1$ distance, over other notions, esp.\ $f$-divergences, is that the data-generating and out-of-distribution laws can be \emph{mutually singular}\footnote{For example, if $\nu$ is the standard Gaussian measure on $\mathbb{R}$ then any finitely supported measure $\sum_{n=1}^N\,w_n\delta_{x_n}$ is singular with respect to $\nu$ and vice versa.} but still remain comparable; this is, of course, not possible with classical divergences.  

% We also emphasize that, in coupling $\mu_{OOD}$ to $\mu$ via~\eqref{eq:coupling_condition}, and not $\mu_{OOD:X}$ directly to $\mu_X$ via $\mathcal{W}_1(\mu_{OOD:X},\mu_X)\le \varepsilon$ instead, allows us to work with perturbations of the forward operator $\Gcal^{\star}$.  In experiments, we typically set $\Gcal^{\star}=\tilde{\Gcal}^{\star}$, however we emphasize that our approach readily allows for simultaneous perturbations of the sampling measure and the forward operator we are learning from.

When training input-output pairs are generating by sampling $\mu$, by which we mean that we have access to the following (random) \textit{empirical} measure \begin{equation}
\label{eq:Empirial_measure}
        \mu^N
    :=
        (\bid\times \Gcal^{\star})_{\#}\mu_X^N 
        \star
        \mu_{\epsilon},
\end{equation}
where the empirical (random) probability measure $\mu_X^N$ is defined by $\mu_X^N=\frac1{N}\,\sum_{n=1}^N\,\delta_{a_n}$.

\medskip

We now state our main out-of-distribution bound, which operates under the following conditions.
\begin{assumption}[Regularity of the Cameron-Martin space]
\label{ass:Small_Ball_Regularity}
Suppose further that $\mu_X$ is a center Gaussian measure on $\mathcal{X}$ with weak variance $\Sigma$ and that the 
%%%
small ball function $\psi$ satisfies:
\begin{enumerate}
    \item[(i)] There exists a constant $c>0$ such that $\psi(\eta)\le c\,\psi(2\eta)$ for every $\eta$ small enough\footnote{I.e.: There exists some $\eta_0>0$ such that (i) holds whenever $0<\eta\le \eta_0$.},
    \item[(ii)] For every $\alpha>0$ and each positive integer $N$, $N^{-\alpha} = o\big( \psi^{-1}\left(\log(N)\right) \big)$.
\end{enumerate}
\end{assumption}

\begin{theorem}[{Out-of-distributional generalization bounds for the NO and $\NONeXt$v2 hypothesis classes}]
\label{thrm:Risk_Bound_NONEXTv2}
Suppose that either of Assumption~\ref{Assumption2} or Assumption~\ref{Assumption3}, that the small ball function $\psi$ satisfies Assumption~\ref{ass:Small_Ball_Regularity}, and that there is an $\varepsilon\ge 0$ such that the coupling condition~\eqref{eq:coupling_condition} holds.
Then there exists a constant $C_{\mu}$, depending only on $\mu_X$, such that: for every $0<\delta\le 1$
\begin{equation}
\label{eq:Concentration_Infinite__LemmaAppliedOmegawise}
    \sup_{\Gcal\in \Gscr}
    \,
            \E_{(a,u)\sim \mu_{OOD}}\left[\ell(\Gcal(a),u) \right]
            -
            \bar{L}\,\E_{(a,u)\sim \mu^N}\left[\ell(\Gcal(a),u) \right] % TB Empirical Risk
    \le 
        \bar{L}
        \,
            \left(
                    \varepsilon %OOD Shift
                +
                    %%% Empirical Concentration Rate
                        C_{\mu}\psi^{-1}\big(\log(N)\big)
                    +
                        \frac{\Sigma \sqrt{-2\log(\delta)}}{\sqrt{N}}
                    %%%
            \right),
\end{equation}
holds with probability at-least $1-\delta$; where $\bar{L}:=L_{\ell}\max\{1,L^{\star}\}\max\{1,L_{\Gscr}\}$; where $L_{\Gscr}\ge 0$ depends on which if Assumption~\ref{Assumption2} or Assumption~\ref{Assumption3} hold, and is respectively given by: 
\begin{enumerate}
    \item[(i)] If Assumption~\ref{Assumption2} holds and $\Gscr={\no}$ (defined in (\ref{neuraloperatorclass})), then:
    \begin{align}
            L_{\no}
        \le
            (C_w + C_k)^{L+1}C_{\sigma}^{L}, 
            % \ \Gcal \in \no
    \label{lip-esti-for-NO}
    \end{align}
\item[(ii)] If Assumption~\ref{Assumption3} holds and $\Gscr=\widetilde{\no}$ (defined in (\ref{neuraloperatortildeclass})), then:
    \begin{align}
        L_{\widetilde{\no}}
    \le
        \left[\prod_{\ell=0}^{L}(\boldsymbol{\mathrm{Z}}_{L}+\boldsymbol{\mathrm{X}}_{L}C_{w}^{M+1}C_{\sigma}^{M})(\boldsymbol{\mathrm{Z}}_{L}+\boldsymbol{\mathrm{X}}_{L}C_k C_{\sigma})
        \right], 
        % \ \Gcal \in \widetilde{\no}.
        \label{lip-esti-tilde-for-NO}
    \end{align}
\end{enumerate}
Furthermore, if the metric entropy $H_{\mu}$ of the unit ball in the Cameron-Martin space associated with the sampling measure $\mu_X$ satisfies $H_{\mu}(r)\in \Theta\big( 
\frac{
\log(1/r)^{2\beta/(2+\alpha)}
}{
r^{2\alpha/(2+\alpha)}
}
\big)$ then the right-hand side of~\eqref{eq:Concentration_Infinite__LemmaAppliedOmegawise}
\begin{equation}
\label{eq:Concentration_Infinite__LemmaAppliedOmegawise__Clarified}
    \sup_{\Gcal\in \Gscr}
    \,
            \E_{(a,u)\sim \mu_{OOD}}\left[\ell(\Gcal(a),u) \right]
            -
            \bar{L}\,\E_{(a,u)\sim \mu^N}\left[\ell(\Gcal(a),u) \right] % TB Empirical Risk
    \le 
        \bar{L}
        \,
            \left(
                    \varepsilon %OOD Shift
                +
                    %%% Empirical Concentration Rate
                        C_{\mu}C\,
                        \Psi\big(\log(N)\big)
                    +
                        \frac{\Sigma \sqrt{-2\log(\delta)}}{\sqrt{N}}
                    %%%
            \right),
\end{equation}
where $\Psi$ is the inverse%%
\footnote{For example, if $\beta=\alpha=1$ then $\Psi(\eta)=W(\eta)/\eta$, where $W$ is the Lambert W function.}%%
 of the map $\eta\mapsto \frac{\log(1/\eta)^{\beta}}{\eta^{\alpha}}$ and $C>0$ is an absolute constant.
\end{theorem}

The case of $\Gscr={\no}$ corresponds to the family of standard Neural Operators, while the case of $\Gscr=\widetilde{\no}$ corresponds to the family of proposed Neural Operator $\NONeXt$.
$L_{\Gscr}$ represents the upper bound of Lipschitz norms for hypothesis classes $\Gscr$, and $L_{\no}$ and $L_{\widetilde{\no}}$ are estimated by Lemma~\ref{lip-norm-NO}.
Analogous observations can be made as in Remark~\ref{GEB Comparison} regarding the upper bound of Lipschitz norms.
Specifically, if $(C_{w}+C_{k})C_{\sigma} > 1$, then the upper bound in (\ref{lip-esti-for-NO}) diverges with depth $L$.
On the other hands, if $\boldsymbol{\mathrm{Z}}_{\ell}=1$ and $\boldsymbol{\mathrm{X}}_{\ell}$ follows a Bernoulli distribution (which corresponds to ($\NONeXt$v2) with an appropriate choice of $p_{\ell}$, then upper bound in (\ref{lip-esti-tilde-for-NO}) remain bounded as $L\to \infty$.

We now show that the conditions of Theorem~\ref{thrm:Risk_Bound_NONEXTv2}, namely the regularity of the Cameron-Martin space associated with the data-generating measure $\mu$ are easily satisfied.  We consider two examples, one of a Brownian sheet and a fractional Brownian sheet on different hypercubes with respect to different norms on their associated function spaces.  

Table~\ref{tab:rates} reports the rates implied by Theorem~\ref{thrm:Risk_Bound_NONEXTv2} in the case of a Brownian sheet on $[0,1]^2$ and $[0,1]^d$ with respect to the $L^2$ and uniform norms.  More generally, we report the rates implied by the result for centered Gaussian measures $\mu$ a general Banach space, when we have access to tight asymptotics on the covering number of the unit ball in the \textit{Cameron-Martin}\footnote{See \citep[Chapter 8]{LedouxTalagrandBookRepring2} for details.} RHKS associated to $\mu$.

\begin{table}[ht!]
    \centering
    \caption{Rates for Different Sampling Measures and Banach Spaces.}
    \label{tab:rates}
   % \resizebox{\columnwidth}{!}{%
        \begin{tabular}{@{}llll@{}}
        \toprule
       	Space & Covariance function &  Entropy estimate & Small ball asymptotics ($\psi(\eta)$)
			\\
			\midrule
			$L^2([0,1]^2)$ & $\min\{s_1,t_1\}\min\{s_2,t_2\}$ & - & $\Theta\big(\frac{\log(1/\eta^2)^2}{\eta^2}\big)$\\
            $C([0,1]^d)$ & $\frac{\alpha}{2^d}\,\prod_{i=1}^d\,s_i^{h_i}+t_i^{h_i}-|s_i-t_i|^{h_i}$ & - & $\Theta\big(\frac1{\eta^{2/h}}\big)$\\
            General & General & $\Theta\big(\frac{\log(1/r)^{2\beta/(2+\alpha)}}{r^{2\alpha/(2+\alpha)}}\big)$ & $\Theta\big(\frac{\log(1/\eta)^{\beta}}{\eta^{\alpha}}\big)$\\
			\bottomrule
    \end{tabular}
   % }%
\caption*{The ``entropy estimates'' the required condition on the \textit{metric entropy} of the unit ball in the Cameron-Martin space associated to the centered Gaussian ``sampling'' measure $\mu_X$. Here $h:=\min_{i=1,\dots,d}\,h_i$ is the minimal ``regularity'' of the Brownian sheet of $C([0,1]^d)$ in all directions.}
\end{table}

\iffalse{ 
\begin{comment}
\begin{proof}[{Proof of Theorem~\ref{thrm:Risk_Bound_NONEXTv2}}]
Using the same argument in the proof of Corollaries~\ref{GEB for additive NO}, we can show that Lipschitz constants $L_{\no}$ and $L_{\widetilde{\no}}$ for hypothesis classes $\no$ and $\widetilde{\no}$ are bounded above by (\ref{lip-esti-for-NO}) and (\ref{lip-esti-tilde-for-NO}), respectively.  The result then follows from Lemma~\ref{lem:COD_Gaussianm}.
\end{proof}
\end{comment}
}\fi

%\newpage
%\subsubsection{Leveraging Small-Ball Estimates}
%{\color{blue}{Choose the space you prefer to specialize our bound.}}

\medskip

%When the spatial domain $D=[0,1]^d$, we can estimate a small ball function $\psi$ as follows:

\begin{lemma}[Estimates on small ball functions for Gaussian sheets in uniform topology {\citep[Theorem 2.1]{mason2001small}}]
\label{lem:GaussianSHeets}
Let $D=[0,1]^d$ for a positive integer $d$, fix ``Hurst parameters'' $0<h_1,\dots,h_d<2$, a parameter $0<\alpha<2$, and let $\mu$ be the continuous centered Gaussian measure on the Banach space $C_0([0,1]^d,\mathbb{R})$ equipped with the supremum norm and with covariance function
\[
        \mathbb{E}\big[
            X_{s_1,\dots,s_d}
            X_{t_1,\dots,t_d}
        \big]
    =
        \frac{\alpha}{2^d}\,
        \prod_{i=1}^d\,
            s_i^{h_i}
            +
            t_i^{h_i}
            -
            |s_i-t_i|^{h_i}
    .
\]
Then, Assumptions (i)-(ii) on the small ball function $\psi$, in Lemma~\ref{lem:COD_Gaussianm}, hold and there exists a constant $0<C_1\le C_2$, depending only on $d$, $a$, and on $\alpha$, such that
\[
        C_1\,
        \frac{1}{\eta^{2/h}}
    \le 
        \psi(\eta)
    \le 
        C_2\,
        \frac{1}{\eta^{2/h}}
\]
for $\eta$ small enough\footnote{That is, there is some $\eta_0>0$ such that the condition holds for every $0<\eta\le \eta_0$.}, where $h:=\min_{i=1,\dots,d}\, h_i$.  
\end{lemma}

\begin{example}[{Estimate on the standard Brownian sheet on $[0,1]^2$ \citep[Equation (5.37)]{kuelbs1993metric}}]
Let $1\le p\le 2$.  Let $D=[0,1]^2$ and consider the centered continuous Gaussian process $X:=(X_{s,t})_{0\le s,t\le 1}$ in $L^2(D)$ with covariance function
\[
        \mathbb{E}[X_{s_1,t_1}X_{s_2,t_2}]
    =
        \min\{s_1,t_1\}\min\{s_2,t_2\}
    .
\]
Then, Assumptions (i)-(ii) on the small ball function $\psi$, in Lemma~\ref{lem:COD_Gaussianm}, hold and there exists a constant $0<C_1\le C_2$ such that
\[
        C_1\,
        \frac{\log(1/\eta)^2}{\eta^2}
        % \log(N)^{-1/2}\log\log(N) % <- Covering Numbers via: \citep[Corollary 1.3]{kuelbs1993metric}
    \le 
        \psi(\eta)
    \le
        C_2\,
        \frac{\log(1/\eta^2)^2}{\eta^2}
        % \log(N)^{-1/2}\log\log(N) % <- Covering Numbers via: \citep[Corollary 1.3]{kuelbs1993metric}
    .
\]
\end{example}

We now derive Theorem~\ref{thrm:Risk_Bound_NONEXTv2} via a sequence of lemmata.
\subsection{{Derivation of Theorem~\ref{thrm:Risk_Bound_NONEXTv2}}}
\label{s:OOD_Derivation}
The proof of Theorem~\ref{thrm:Risk_Bound_NONEXTv2} extends the transport-theoretic approach to deriving generalization bounds of \cite{hou2022instance}, to the infinite-dimensional setting, by incorporating elements of the geometry of Cameron-Martin spaces.  
We begin, with the following ``change-of-measure lemma'' which bounds the gap between the risks from data samples from any two distinct, arbitrary, probability measures $\mu$ and $\nu$ in $\mathcal{P}_1(\mathcal{X}\times \mathcal{Y})$.
\begin{example}[Change of distribution]
\label{lem:COD}
Consider an in-distribution measure $\mu$ and and out-of-distribution probability measure $\mathbb{Q}$, where $\mu,\mathbb{Q},\nu\in \mathcal{P}_1\big(\mathcal{X}\times \mathcal{Y}\big)$.  
Let $\Gscr$ be a family of $L$-Lipschitz functions from $\mathcal{X}$ to $\mathcal{Y}$ , and suppose that 
\[
    \mathcal{W}_1(\mathbb{Q},\mu) \le \varepsilon,
\]
% \antonio{shouldn't be the Wasserstein distance, $\mathcal{W}_1$ defined on the Appendix?}
for some $\varepsilon>0$.  
For any $L_{\ell}$-Lipschitz loss function $\ell:\mathcal{Y}\times \mathcal{Y}\rightarrow [0,\infty)$ we have
\[
            \E_{(a,u)\sim \mathbb{Q}}\left[\ell(\Gcal(a),u) \right]
    \le 
        \operatorname{L}_{\ell}\max\{1,L\}\,
            \big(
                    \varepsilon %OOD Shift
                +
                    \mathcal{W}_1(\nu,\mu) % TB Empiriical Measure
                +
                    \E_{(a,u)\sim \mathbb{\nu}}\left[\ell(\Gcal(a),u) \right] % TB Empirical Risk
            \big)
    ,
\]
for each $\Gcal \in \Gscr$.
\end{example}
\begin{proof}[{Proof of Lemma~\ref{lem:COD}}]
Let $1_{\mathcal{Y}}$ denote the identity map on $\mathcal{Y}$.  
Fix $\Gcal \in \Gscr$ and consider the map 
\[
        f 
    := 
        \ell\circ (\Gcal \times 1_{\mathcal{Y}}).
\]
If $\Gcal$ is constant, we are done. 
Therefore, assume that $\Gcal$ is non-constant; whence, $\operatorname{Lip}(\Gcal)>0$. 
Therefore, the map $\tilde{f}:\mathcal{X}\times \mathcal{Y}\rightarrow [0,\infty)$ given by
\[
        \tilde{f}
    :=
        \frac1{\operatorname{Lip}(f)}\,
        f,
\]
is $1$-Lipschitz. 
The Kanotorovich-Rubinstein duality \citep[Theorem 5.10]{villani2009optimal} implies that
\begin{equation}
\label{eq:PROOF_lem:OOD_KRDuality_PRE}
            \E_{(a,u)\sim \mathbb{Q}}\left[\tilde{f}(a,u) \right]
        -
            \E_{(a,u)\sim \nu}\left[\tilde{f}(a,u) \right]
    \le 
            \mathcal{W}_1(\mathbb{Q},\nu)
.
\end{equation}
The triangle inequality and the assumption that $\mathcal{W}_1(\mathbb{Q},\mu)\le \varepsilon$ imply that the right-hand side of~\eqref{eq:PROOF_lem:OOD_KRDuality} may be further bounded by
\begin{equation}
\label{eq:PROOF_lem:OOD_KRDuality}
            \E_{(a,u)\sim \mathbb{Q}}\left[\tilde{f}(a,u) \right]
        -
            \E_{(a,u)\sim \nu}\left[\tilde{f}(a,u) \right]
    \le 
            \mathcal{W}_1(\mathbb{Q},\mu)
            +
            \mathcal{W}_1(\nu,\mu)
    \le 
            \varepsilon
            +
            \mathcal{W}_1(\nu,\mu)
,
\end{equation}
Multiplying across~\eqref{eq:PROOF_lem:OOD_KRDuality} by $\operatorname{Lip}(f)$, using the linearity of integration, and re-arranging yields 
\begin{equation}
\label{eq:lem:OOD_KRDuality_2}
            \E_{(a,u)\sim \mathbb{Q}}\left[f(a,u) \right]
    \le 
        \operatorname{Lip}(f)\big(
            \varepsilon
        +
            \E_{(a,u)\sim \nu}\left[f(a,u) \right]
        \big)
.
\end{equation}
It remains to compute the Lipschitz constant of $F$.  Let $(a_1,u_1),(a_2,u_2)\in \mathcal{X}\times \mathcal{Y}$ and note that
\allowdisplaybreaks
\begin{align}
    \label{eq:PROOF_lem:BOUDN_LipConst__BEGIN}
    \big|
            f\big(a_1,u_1\big)
        -
            f\big(a_2,u_2\big)
    \big|
        \le &
    L_{\ell}\,
        \big(
            \|\Gcal(a_1)-\Gcal(a_2)\|^2_{\mathcal{Y}}
        +
            \|u_1-u_2\|^2_{\mathcal{Y}}
        \big)^{1/2}
\\
\nonumber
            \le &
    L_{\ell}\,
        \big(
            \operatorname{Lip}(\Gcal)^2\|a_1-a_2\|^2_{\mathcal{X}}
        +
            1\,\|u_1-u_2\|^2_{\mathcal{Y}}
        \big)^{1/2}
\\
\nonumber
            \le &
    L_{\ell}\,
        \big(
            \max\{\operatorname{Lip}(\Gcal)^2,1\}\,
                \|a_1-a_2\|^2_{\mathcal{X}}
        +
            \max\{\operatorname{Lip}(\Gcal)^2,1\}\, 
            \|u_1-u_2\|^2_{\mathcal{Y}}
        \big)^{1/2}
\\
\nonumber
            = &
    L_{\ell}\,
        \max\{\operatorname{Lip}(\Gcal),1\}
        \big(
            \|a_1-a_2\|^2_{\mathcal{X}}
        +
            \|u_1-u_2\|^2_{\mathcal{Y}}
        \big)^{1/2}
\\
\label{eq:PROOF_lem:BOUDN_LipConst__END}
            := &
    L_{\ell}\,
        \max\{\operatorname{Lip}(\Gcal),1\}
            \|(a_1,u_1)-(a_2,u_2)\|^2_{\mathcal{X}\times \mathcal{Y}}
,
\end{align}
where the right-hand side of~\eqref{eq:PROOF_lem:BOUDN_LipConst__BEGIN} follows from definition of the $2$-product metric on $\mathcal{X}\times \mathcal{Y}$. 
Incorporating the estimate of $\operatorname{Lip}(f)$ computed in~\eqref{eq:PROOF_lem:BOUDN_LipConst__BEGIN}-\eqref{eq:PROOF_lem:BOUDN_LipConst__END} into~\eqref{eq:lem:OOD_KRDuality_2} completes the proof.
\end{proof}

Next, we incorporate the structure of $\mu$ and $\mu^N$ into Lemma~\ref{lem:COD}, in place of the arbitrary measures $\mu$ and $\nu$, respectively.
\begin{lemma}[Structured change-of-measure]
\label{lem:structuring}
%Consider the setting of Lemma~\ref{lem:COD} and 
Assume that $\mu$ and $\mu^N$ are respectively given by~\eqref{eq:True_Measure} and~\eqref{eq:Empirial_measure}. Then, we have that  
%Then: for every $L$-Lipschitz function $\Gcal :\mathcal{X}\rightarrow \mathcal{Y}$ we have that
\[
        \mathcal{W}_1\Big(
            \mu
        ,
            \mu^N
        \Big)
    \le 
        \max\{1,L^{\star}\}
        \,
        \mathcal{W}_1(\mu_X,\mu_X^N)
.
\]
\end{lemma}
\begin{proof}[{Proof of Lemma~\ref{lem:structuring}}]
Arguing as in \citep[Lemma 5.2]{santambriosiTransport}, we see that
\allowdisplaybreaks
\begin{align}
\label{eq:PROOF_lem:structuring__Deconvolving}
        \mathcal{W}_1\Big(
            \mu
        ,
            \mu^N
        \Big)
    \le &
        \mathcal{W}_1\Big(
                (\bid\times \Gcal^{\star})_{\#}\mu_X
            ,
                (\bid\times \Gcal^{\star})_{\#}\mu_X^N 
            \Big)
.
\end{align}
Arguing analogously to~\ref{eq:PROOF_lem:BOUDN_LipConst__BEGIN}-~\ref{eq:PROOF_lem:BOUDN_LipConst__END} we find that $\bid\times \Gcal^{\star}$ is $\max\{1,L^{\star}\}$-Lipschitz.  Therefore, the Kantorovich-Rubinstein  duality \citep[Theorem 5.10]{villani2009optimal} and the estimate~\eqref{eq:PROOF_lem:structuring__Deconvolving} imply that
\begin{align*}
        \mathcal{W}_1\Big(
            \mu
        ,
            \mu^N
        \Big)
    \le &
        \mathcal{W}_1\Big(
                (\bid\times \Gcal^{\star})_{\#}\mu_X
            ,
                (\bid\times \Gcal^{\star})_{\#}\mu_X^N 
            \Big)
    \\
    \le &
        \operatorname{Lip}(I_{\mathcal{X}\times \Gcal^{\star}})
        \,
        \mathcal{W}_1(\mu_X,\mu_X^N)
    \\
    \le &
        \max\{1,L^{\star}\}
        \,
        \mathcal{W}_1(\mu_X,\mu_X^N)
.
\end{align*}
This completes the proof.
\end{proof}

We will assume that our samples of $\Gcal^{\star}$, distributed according to $\mu_X$, are drafted from a Gaussian measure on $\mathcal{X}$. 
%}

\begin{lemma}[General concentration inequality for Lipschitz hypotheses]
\label{lem:COD_Gaussianm}
Assume the setting of Lemma~\ref{lem:structuring} and fix a positive integer $N$.  
Suppose further that $\mu_X$ is a center Gaussian measure on $\mathcal{X}$ with weak variance $\Sigma$ and that the 
%%%
small ball function $\psi$ satisfies:
\begin{enumerate}
    \item[(i)] There exists a constant $c>0$ such that $\psi(\eta)\le c\,\psi(2\eta)$ for every $\eta$ small enough\footnote{I.e.: There exists some $\eta_0>0$ such that (i) holds whenever $0<\eta\le \eta_0$.},
    \item[(ii)] For every $\alpha>0$ and each positive integer $N$, $N^{-\alpha} = o\big( \psi^{-1}\left(\log(N)\right) \big)$.
\end{enumerate}
%%
% Let $\{(A_n,U_n)\}_{n=1}^N$ be a finite set of i.i.d.\ $\mathcal{X}\times \mathcal{Y}$-valued random variables with law $\mu$ defined on a common measurable space $(\Omega,\mathcal{A})$, and consider the random measure $\mu^N := \frac1{N}\, \sum_{n=1}^N\, \delta_{(A_n,U_n)}$.  
There exists a constant $C_{\mu}$, depending only on $\mu_X$, such that: for every $0<\delta\le 1$
\begin{equation}
\label{eq:Concentration_Infinite__LemmaAppliedOmegawise}
    \sup_{\Gcal \in \Gscr}
    \,
            \E_{(a,u)\sim \mathbb{Q}}\left[\ell(\Gcal(a),u) \right]
            -
            \bar{L}\,\E_{(a,u)\sim \mu^N}\left[\ell(\Gcal(a),u) \right] % TB Empirical Risk
    \le 
        \bar{L}
        \,
            \big(
                    \varepsilon %OOD Shift
                +
                    %%% Empirical Concentration Rate
                        C_{\mu}\psi^{-1}\big(\log(N)\big)
                    +
                        \frac{\Sigma \sqrt{-2\log(\delta)}}{\sqrt{N}}
                    %%%
            \big),
\end{equation}
holds with probability at-least $1-\delta$; where $\bar{L}:=L_{\ell}\max\{1,L\}\max\{1,L^{\star}\}$.  
\hfill\\
Furthermore, suppose that $H_{\mu}(r)\in \Theta\big( 
\frac{
\log(1/r)^{2\beta/(2+\alpha)}
}{
r^{2\alpha/(2+\alpha)}
}
\big)$ then the right-hand side of~\eqref{eq:Concentration_Infinite__LemmaAppliedOmegawise}
\begin{equation}
\label{eq:Concentration_Infinite__LemmaAppliedOmegawise__Clarified}
    \sup_{\Gcal \in \Gscr}
    \,
            \E_{(a,u)\sim \mathbb{Q}}\left[\ell(\Gcal(a),u) \right]
            -
            \bar{L}\,\E_{(a,u)\sim \mu^N}\left[\ell(\Gcal(a),u) \right] % TB Empirical Risk
    \le 
        \bar{L}
        \,
            \big(
                    \varepsilon %OOD Shift
                +
                    %%% Empirical Concentration Rate
                        C_{\mu}\tilde{\psi}^{-1}\big(\log(N)\big)
                    +
                        \frac{\Sigma \sqrt{-2\log(\delta)}}{\sqrt{N}}
                    %%%
            \big),
\end{equation}
where $\tilde{\psi}(\eta)= C\frac{\log(1/\eta)^{\beta}}{\eta^{\alpha}}$ and $C>0$ is an absolute constant.
\end{lemma}
\begin{remark}
\label{remark:nonfinitedimensionality}
As remarked on \citep[page 542]{BoissardLeGouic_AIHP_2014_ConcentrationWassersteinInfinite}, condition (ii) in Lemma~\ref{lem:COD_Gaussianm} implies that the centered Gaussian measure $\mu$ is not supported on a finite-dimensional Banach subspace of $\mathcal{X}\times \mathcal{Y}$.
\end{remark}
\begin{proof}[{Proof of Lemma~\ref{lem:COD_Gaussianm}}]
\noindent\textbf{Reduction to estimating the concentration of the empirical Sampling measure $\mu^N_X$ to $\mu_X$:}
By Lemma~\ref{lem:structuring}, we have
\begin{equation}
\label{eq:estimate_A}
        \mathcal{W}_1(\nu,\mu)
    :=
        \mathcal{W}_1(\mu^N,\mu)
    \le 
        \max\{1,L^{\star}\}
        \,
        \mathcal{W}_1(\mu_X,\mu_X^N)
    .
\end{equation}
Set $\nu:=\mu^N$, in the notation of (\ref{eq:Empirial_measure}). Applying Lemma~\ref{lem:COD} yields
\begin{equation}
\label{eq:Concentration_Infinite__LemmaAppliedOmegawise}
\begin{aligned}
            \E_{(a,u)\sim \mathbb{Q}}\left[\ell(f(a),u) \right]
    \le &
        \operatorname{L}_{\ell}\max\{1,L\}\,
            \big(
                    \varepsilon %OOD Shift
                +
                    \mathcal{W}_1(\nu,\mu) % TB Empiriical Measure
                +
                    \E_{(a,u)\sim \mu^N}\left[\ell(f(a),u) \right] % TB Empirical Risk
            \big)
    \\
    \le &
        \operatorname{L}_{\ell}\max\{1,L\}\,
            \big(
                    \varepsilon %OOD Shift
                +
                    \max\{1,L^{\star}\}
                    \,
                    \mathcal{W}_1(\mu_X,\mu_X^N)
                +
                    \E_{(a,u)\sim \mu^N}\left[\ell(f(a),u) \right] % TB Empirical Risk
            \big)
    ,
\end{aligned}
\end{equation}
for each $f\in \mathcal{F}$ (for each $\omega \in \Omega$).

\noindent\textbf{Applying the sampling estimates for $\mu_X$:}
Under our assumptions on the small ball function $\psi$, \citep[Theorem 1.4]{BoissardLeGouic_AIHP_2014_ConcentrationWassersteinInfinite} implies that there exists a constant $C_{\mu}>0$, depending only on $\mu_X$, such that for every $\eta>0$
\begin{equation}
\label{eq:Concentration_Infinite__lambdaform}
        \mathcal{W}_2\big(
            \mu^N_X
        ,   
            \mu_X
        \big)
    \le 
        (C_{\mu} + \eta)
        \,
        \psi^{-1}(\log(N)),
\end{equation}
holds with probability at-least $1-\exp(-N\,(\psi^{-1}(\log(N)))^2\, \frac{\lambda^2}{2\Sigma^2})$.
Here, we have denoted by $\mathcal{W}_2(\mu^N_X, \mu_X)$ the Wasserstein distance of the order two that measures the distance between two distributions $\mu^N_X$ and $\mu_X$.

Set, $\eta := -\log(\delta)^{1/2}\,2^{1/2}\Sigma/(N^{1/2}\psi^{-1}(\log(N)))$, then~\eqref{eq:Concentration_Infinite__lambdaform} implies that
\begin{equation}
\label{eq:Concentration_Infinite__muform}
        \mathcal{W}_1\big(
            \mu^N_X
        ,   
            \mu_X
        \big)
    \le
        \mathcal{W}_2\big(
            \mu^N_X
        ,   
            \mu_X
        \big)
    \le 
        C_{\mu}\psi^{-1}\big(\log(N)\big)
        +
        \Sigma \frac{\sqrt{-2\log(\delta)}}{\sqrt{N}},
\end{equation}
holds with probability at-least $1-\delta$; where we used the fact that $\mathcal{W}_1\le \mathcal{W}_2$ (see e.g.\ \citep[Remark 6.6]{villani2009optimal}) to deduce the left-hand side of~\eqref{eq:Concentration_Infinite__lambdaform}.  
Combining~\eqref{eq:Concentration_Infinite__LemmaAppliedOmegawise} with~\eqref{eq:Concentration_Infinite__muform} implies that: for every $0<\delta \le 1$ and each $\Gcal \in \Gscr$ we have
\begin{equation}
\label{eq:Concentration_Infinite__preF}
            \E_{(a,u)\sim \mathbb{Q}}\left[\ell(\Gcal(a),u) \right]
            -
            \bar{L}\,\E_{(a,u)\sim \mu^N}\left[\ell(\Gcal(a),u) \right] % TB Empirical Risk
    \le 
        \bar{L}
        \,
            \big(
                    \varepsilon %OOD Shift
                +
                    %%% Empirical Concentration Rate
                        C_{\mu}\psi^{-1}\big(\log(N)\big)
                    +
                        \frac{\Sigma \sqrt{-2\log(\delta)}}{\sqrt{N}}
                    %%%
            \big),
\end{equation}
holds with probability at-least $1-\delta$; where $\bar{L}:=L_{\ell}\max\{1,L\}\max\{1,L^{\star}\}$.  Since the right-hand side of~\eqref{eq:Concentration_Infinite__preF} was in-dependant of $\Gcal$, then taking the supremum over the class $\Gscr$ on both sides of~\eqref{eq:Concentration_Infinite__preF} yields the conclusion.

Finally, if $H_{\mu}(r)\in \Theta\big( 
\frac{
\log(1/r)^{2\beta/(2+\alpha)}
}{
r^{2\alpha/(2+\alpha)}
}
\big)$ then \citep[Theorem 1.2]{li1999approximation} implies that $\psi(\eta) \in \Theta\big(\frac{\log(1/\eta)^{\beta}}{\eta^{\alpha}}\big)$.  
% Absorbing the constant obtained from rewriting $\log$ as $\ln$ yields the conclusion.  
\end{proof}

Applying Lemma~\ref{lem:COD_Gaussianm} to the hypothesis classes $\no$ and $\widetilde{\no}$, defined in (\ref{neuraloperatorclass}) and (\ref{neuraloperatortildeclass}), respectively, and yields our main generalization bound for out-of-sample distribution learning; i.e., \ Theorem~\ref{thrm:Risk_Bound_NONEXTv2}.

\begin{proof}[{Proof of Theorem~\ref{thrm:Risk_Bound_NONEXTv2}}]
Set $\mathbb{Q}:=\mu_{OOD}$.  Lemma~\ref{lip-norm-NO} implies that under the respective assumptions:  Assumption~\ref{Assumption2} and~\ref{Assumption3}, the hypothesis classes $\no$ and $\widetilde{\no}$ are Lipschitz and it provides explicit estimates on the Lipschitz constants $L_{\mathcal{F}}$ of these neural operators.  The result then follows from Lemma~\ref{lem:COD_Gaussianm}.
\end{proof}

We have used the following lemma that estimates the Lipschitz norms for two classes $\no$ and $\widetilde{\no}$ corresponding to standard NO and $\NONeXt$, respectively.
\begin{lemma}[Lipschitz stability of the hypothesis classes {($\no$ and $\widetilde{\no}$)}]
\label{lip-norm-NO}
(i) Let Assumption~\ref{Assumption2} hold.
Then, we have that
\[
\left\|\Gcal \right\|_{\mathrm{Lip}}
\leq 
(C_w + C_k)^{L+1}C_{\sigma}^{L}, \ \Gcal \in \no.
\]

(ii) Let Assumption~\ref{Assumption3} hold.
Then, we have that
\[
\left\|\Gcal \right\|_{\mathrm{Lip}}
\leq 
\left[\prod_{\ell=0}^{L}(\boldsymbol{\mathrm{Z}}_{L}+\boldsymbol{\mathrm{X}}_{L}C_{w}^{M+1}C_{\sigma}^{M})(\boldsymbol{\mathrm{Z}}_{L}+\boldsymbol{\mathrm{X}}_{L}C_k C_{\sigma})
\right], \ \Gcal \in \widetilde{\no}.
\]
\end{lemma}
The proof is given by the same argument in the proofs of Corollaries~\ref{GEB for additive NO} and \ref{GEB for multiplicative NO}.
%Analogous observations can be made as in Remark~\ref{GEB Comparison} regarding the upper bound of Lipschitz norms can be done.
%That is, if $(C_{w}+C_{k})C_{\sigma} > 1$, then the upper bound in (i) diverges with depth $L$.
%On the other hands, if $\boldsymbol{\mathrm{Z}}_{\ell}=1$ and $\boldsymbol{\mathrm{X}}_{\ell}$ follows a Bernoulli distribution (which corresponds to ($\NONeXt$)v2) with an appropriate choice of $p_{\ell}$, then upper bound in (ii) remain bounded as $L\to \infty$.

% \begin{definition}[Out of Distributional Risk]
% \label{defn:OOD_Risk}
% Let $\mathbb{Q}$ be a probability measure on $\mathcal{P}\big(\mathcal{Y} \times \mathcal{Y}\big)$.  
% The out of distributional (OOD) risk $\mathcal{L}_{OOD}(\Gcal)$ is defined by
% \[
%     \mathcal{R}_{OOD}(\Gcal):=
%         \E_{(a,u)\sim \mathbb{Q}}\left[\ell(f(a),u) \right]
% \]
% \end{definition}

\paragraph{\textbf{Discussion}}
Theorem~\ref{thrm:Risk_Bound_NONEXTv2} supports our experimental evidence that the risk-bounds for the $(sNO+\varepsilon I)v2$ are much tighter than those for the $sNO+\varepsilon I$ model, precisely since the constant of the former is much tighter than that of the latter.
We expect that comparable lower-bounds could be derived.  However, since lower-bounds with tight constants can take years to perfect, as seen by the time gap between \cite{talagrand1994sharper} and \cite{kontorovich2019exact}, then we will in future research.

\section{Analysis of the generalization error bound of the networks}
\label{GEB}
%\section{Analysis of generalization error bound}\label{GEB}
    
% {\color{red} This section is still under developing ... Takashi}

Through experimental observation, we have found that our proposed network exhibits superior performance compared to standard networks, specifically in terms of lower test errors.
The test error is synonymous with generalization error in the field of statistical learning theory.
This section provides the theoretical analysis of generalization error for both standard networks and our proposed networks. \par 

It is important to mention that \citet{kovachki2021neural} established the standard universal approximation theorem that shows that any continuous operator can be approximated in compact sets by standard neural operators. 
As our network is an extension of the standard network, \emph{universality also holds for our proposed networks}.
Consequently, in the context of universality, we are unable to distinguish differences. 
Therefore, our primary focus on this section will be on the complexity of networks and their corresponding generalization error bounds. %and the mathematical understanding of the out-of-distribution behavior.

\medskip

%We analyze the generalization error bounds for the networks previously described. For the \emph{proposed architectures}, we devide our analysis in the following: (a) when the skip connection is removed, (b) when $\prob\{\mathbf{X_\ell} = 1\} = 1$ for all $\ell$, and (c), \cref{FNO+epsilon v2}, when the previous cases are presented in a non-trivial form. For the notation, we refer to \cref{table of notation}. \

\subsection{Preliminaries}
Let $D \subset \mathbb{R}^d$ be a bounded domain, and $\Lp(D; \mathbb{R}^{h})$ be the $\Lp$ space of $\mathbb{R}^{h}$-value function on $D$. Let $S=\{a_i, u_i: 1\le i \le n\}$ be the sequence of independent samples of $\mu$, i.e. $(a_i, u_i) \overset{\text{i.i.d}}{\sim} \mu$\footnote{Independent identically distributed samples drawn from, $\mu$, on $\Lp(D; \mathbb{R}^{d_a})\times \Lp(D; \mathbb{R}^{d_u})$.}, with marginals $\mu_a$ in $\Lp(D; \mathbb{R}^{d_a})$ and $\mu_u$ in $\Lp(D; \mathbb{R}^{d_u})$. 
Let $\boldsymbol{\mathscr{G}}$ be the class of operators mapping from $\Lp(D; \mathbb{R}^{d_a})$ to $\Lp(D; \mathbb{R}^{d_u})$, and $\ell: \Lp(D; \mathbb{R}^{d_u}) \times \Lp(D; \mathbb{R}^{d_u}) \to \mathbb{R}_{\ge 0}$ be the loss function. 
We denote by the expected risk $\mathcal{L}$ and empirical risk $\widehat{\mathcal{L}}_{\mathrm{S}}$, defined rigorously in Appendix~\ref{EmpiricalExpected}.

%Finally, let $\mathcal{L}$ be the Expected risk, and $\widehat{\mathcal{L}}_{\mathrm{S}}$ be the Empirical risk, defined rigorously in Appendix~\ref{EmpiricalExpected}\footnote{The Expected risk is defined by the Bochner Integral on $\Lp(D; \mathbb{R}^{d_a})\times \Lp(D; \mathbb{R}^{d_u})$, see Appendix~\ref{Bochner} and \citet{yoshida1980functional}.}.

We review the Rademacher complexity, which measures the richness of a class of real-valued functions.
\begin{definition}
(Rademacher complexity) Let $\mathscr{F}$ be the set of real-valued measurable functions on  a measurable space $(S, \mathcal{S})$. Let $\{\rad_{i}\}_{i=1}^{n}$ is a sequence of i.i.d. RV's with Rademacher distribution; i.e., $\prob\{\rad_{i}=1\}=1/2=\prob\{\rad_{i}=-1\}$. The Rademacher complexity of the class $\mathscr{F}$ is defined as 
\[  
\mathfrak{R}^n_{\mathrm{S}}(\mathscr{F}):=\E_{\rad \sim \mathrm{Rad}}\left[\sup_{f \in \mathscr{F}}\frac{1}{n}\left|\, \sum_{i=1}^{n}\rad_{i}f(a_i, u_i) \right| \,\right],
\]
(Cf. \citet[Definition 3.1.19 ]{Gine}).
\end{definition}
%Intuitively, Rademacher  complexity $\mathfrak{R}^n_{\mathrm{S}}(\mathscr{F})$ measures richness of a class $\mathscr{F}$ of real-valued functions.

\iffalse{
The expected error Appendix~\ref{EmpiricalExpected}
\begin{equation}
\mathcal{L}(\boldsymbol{\mathcal{G}}):=\E_{(a,u)\sim \mu}\left[\ell(\boldsymbol{\mathcal{G}}(a),u) \right]=\int \ell(\boldsymbol{\mathcal{G}}(a),u)\,\,d\mu(a,u), \quad \boldsymbol{\mathcal{G}} \in \boldsymbol{\mathscr{G}}, \label{generalization error}
\end{equation}
$\ell: \Lp(D; \mathbb{R}^{d_u}) \times \Lp(D; \mathbb{R}^{d_a}) \to \mathbb{R}_{\ge 0}$ is the Loss Function and the Integral in the right-hand side is a Bochner Integral on $\Lp(D; \mathbb{R}^{d_a})\times \Lp(D; \mathbb{R}^{d_u})$, see Appendix \ref{Bochner} (Cf.\citep{yoshida1980functional} for a throughout study of the Integral and its properties). The empirical error, Appendix~\ref{EmpiricalExpected}
}

\begin{equation}
\widehat{\mathcal{L}}_{\mathrm{S}}(\Gcal):= \frac{1}{n}\sum_{i=1}^{n}\ell(g(a_i),u_i), \quad \Gcal \in \Gscr. \label{empirical error}
\end{equation}
\fi
%%%%%%%%%%%%%%%%%%
\begin{assumption}\label{Assumption1}
There exists positive constants $\rho>0$, $R_u>0$ such that
\begingroup
  \setlength\itemsep{1mm}
\begin{enumerate}[(i).]
%%%%%%%%%%%%%%%%%
\item $\ell$ is $\rho$-Lipschitz continuous, i.e., $|\ell(u_1, v) - \ell(u_2, v) | \leq \rho \left\|u_1 -u_2 \right\|_{\Lp(D; \mathbb{R}^{d_u})}$ for $u_1, u_2, v \in \Lp(D; \mathbb{R}^{d_u})$.
%%%%%%%%%%%%%%%%%
\item
$\ell(\boldsymbol{0}, \cdot)$ is bounded above by $R_u$, i.e., $|\ell(\boldsymbol{0}, u)|\leq R_u$ for $u \in 
\mathrm{supp}(\mu_u)$\footnote{Support of a measure $\mu$ is defined
$\mathrm{supp}(\mu):= \left\{x\in X : 
\mu(\mathcal{U})>0 \text{ for all open neighborhood } \mathcal{U} \text{ of } x\right\}$
(Cf. \citet[Ch. 5]{ambrosio2005gradient}).
}.
%%%%%%%%%%%%%%%%%
\end{enumerate}
\endgroup
\end{assumption}

First, we estimate the generalization error bound for the general setting.
%%%%%%%%%%%%%%%%%%%%%%%%%%%%%%%%%%%%%%%%%%%%%%%%%%%
\begin{lemma}[Generalization error bound]\label{generalization error bound for general setting}
Let \cref{Assumption1} holds and suppose there exists $R>0$ such that $\left\|\Gcal(a)\right\|_{\Lp(D; \mathbb{R}^{d_u})} \leq R$, 
for all $\Gcal \in \Gscr$, and $a \in \mathrm{supp}(\mu_a)$ for the hypothesis class, $\Gscr$. Hence, for any $\delta>\log2$, the following inequality holds with probability greater than $1-2\exp(-\delta)$,
\begin{equation}
\mathcal{L}(\Gcal) \leq \widehat{\mathcal{L}}_{\mathrm{S}}(\Gcal) + 2\mathfrak{R}^n_{\mathrm{S}}(\mathscr{F}_{\Gscr}) + (\rho R + R_u)\sqrt{\frac{2 \delta}{n}}, \quad \forall\Gcal \in \Gscr, \label{Generalization error bound}
\end{equation}
where 
$\mathfrak{R}^n_{\mathrm{S}}(\mathscr{F}_{\Gscr})$ is the {\it Rademacher complexity} of the class $\mathscr{F}_{\Gscr}$, 
%see Definition~\ref{Rademacher}. $\mathcal{L}(\Gcal)$, 
and the class $\mathscr{F}_{\Gscr}$ is defined as
$$\mathscr{F}_{\Gscr}:=\left\{(a,u) \mapsto \ell(\Gcal (a), u)  \, :  \, (a,u) \in \mathrm{supp}(\mu),\,  \Gcal\in \Gscr \right\}.$$
\end{lemma}

See Appendix~\ref{Proof-lamma-1} for the proof.
The idea is to break down the generalization error $\mathcal{L}(\Gcal)$ into two components:  the approximation error $\widehat{\mathcal{L}}_{\mathrm{S}}(\Gcal)$ and the complexity error $\mathcal{L}(\Gcal)-\widehat{\mathcal{L}}_{\mathrm{S}}(\Gcal)$. 
The upper bound of the complexity error $\mathcal{L}(\Gcal)-\widehat{\mathcal{L}}_{\mathrm{S}}(\Gcal)$ can be established using the \emph{Rademacher complexity}, $\mathfrak{R}^n_{\mathrm{S}}(\mathscr{F}_{\Gscr})$, by the \emph{Uniform laws of large numbers} (\cref{Uniform Laws of Large Numbers}).
\par
If the class $\Gscr$ is a universal approximator, the approximation error $\widehat{\mathcal{L}}_{\mathrm{S}}(\Gcal)$ can be made "\textit{small enough}" through training.
In fact, if $\Gscr$ is chosen to be the classes of neural operators \citep{kovachki2021neural} and DeepONets \citep{lanthaler2022error}, both of which are universal approximators.
%\citet{marcati2023exponential} proved exponential expression rates of the approximation error $\widehat{\mathcal{L}}_{\mathrm{S}}(\Gcal)$ in the case of elliptic PDEs.
In the following, we focus on the analysis of the Rademacher complexity for both standard neural operators ($\mathrm{NO}s$) and proposed neural operators ($\MNO$).

\iffalse{
Assumption~\ref{Assumption1} allows us to relate the Rademacher complexity of a hypothesis class to its true and empirical risks.  
The idea is to break down the generalization error $\mathcal{L}(\Gcal)$ into two components:  $\widehat{\mathcal{L}}_{\mathrm{S}}(\Gcal)$ and $\mathcal{L}(\Gcal)-\widehat{\mathcal{L}}_{\mathrm{S}}(\Gcal)$. 
The upper bound limit of $\mathcal{L}(\Gcal)-\widehat{\mathcal{L}}_{\mathrm{S}}(\Gcal)$ can be established using the \emph{Rademacher complexity}, $\mathfrak{R}^n_{\mathrm{S}}(\mathscr{F}_{\Gscr})$, while $\mathcal{O}(n^{-1/2})$, by the \emph{Uniform laws of large numbers} (\cref{Uniform Laws of Large Numbers}).
}\fi

    \subsection{Related Works of generalization error bound (GEB)}
    References such as \citet{bartlett2021deep, jakubovitz2019generalization} have extensively investigated generalization error bounds (GEB) for networks that map between finite-dimensional spaces. However, to the best of our knowledge, there has been limited exploration of GEB for operators on infinite dimensional spaces. \citet{de2022generic} provided the GEB for (general) operator architectures using Hoeffding’s inequality, without involving the analysis of the Rademacher complexity. \citet{gopalani2022capacity} and \citet{kim2022bounding} have provided GEB for DeepOnet and FNOs, respectively, by the Rademacher complexity. However, in these works, the authors assumed that the trainable parameters are \emph{finite-dimensional} (such as matrices), \emph{while our work does not need this assumption}. Our study distinguishes itself from \citet{kim2022bounding} in several key aspects. Firstly, we directly analyze the integral operator under the assumption of Lipschitz continuity of the kernel, whereas \citet{kim2022bounding} assumes a truncated expansion for FNOs and evaluates the Rademacher complexity based on the number of truncations. Secondly, our work not only generalizes the findings of \citet{kim2022bounding} but also provides sharper bounds on the Rademacher complexity with the order $\mathcal{O}(1/n^{\frac{1}{\hat{d}+1}})$, compared to $\mathcal{O}(1)$ in \citet{kim2022bounding}.

\iffalse{
Specifically, our study focuses on analyzing the integral operator directly, assuming that the kernel is a Lipschitz continuous function. In contrast, \citet{kim2022bounding} assumed a truncated expansion for FNOs, and evaluated the Rademacher complexity on the number of the truncation. Furthermore, our work not only generalizes \citet{kim2022bounding}, but also provides to sharper bound of the Rademacher complexity with the order $\mathcal{O}( 1/n^{\frac{1}{\hat{d}+1}})$ (while order $\mathcal{O}(1)$ in \citet{kim2022bounding}).

For generalization error bound for , we refer to \citep{gopalani2022capacity} and \citep{kim2022bounding}, they analyzed Rademacher complexity. 
Furthermore, their analysis is based on parameterizing in finite-dimensional space.
We bound the generalization error for Neural Operator, and essentially evaluate Rademacher complexity via covering number.
Also, \emph{We do not parameterize finite dimensional space}, that is, we parameterize Neural Operator, which maps between infinite dimensional space.
The covering number of Neural Operator is bounded above from that of weights and kernel functions that are assumed to be Lipschitz continuous functions.
The covering number of matrices and Lipschitz continuous function can be evaluated by \citep[Section 5]{wainwright}.
}\fi

    \subsection{Rademacher complexity of neural operators}
    \label{sec:Rademacher Complexity1}
    We analyze the Rademacher Complexity of Neural Operators, \citep{kovachki2021neural}. Let us define the family of standard Neural Operator as follows:
\iffalse{
\begin{equation}\label{neuraloperatorclass}
\begin{split}
\no 
=
&
\biggl\{ 
\Gcal_\theta: \Lp(D; \mathbb{R}^{d_a}) \to \Lp(D; \mathbb{R}^{d_u})\, : \,
\Gcal_\theta = (W_{L}+\bk_{L})\circ \sigma (W_{L-1}+\bk_{L-1})\circ \cdots \circ \sigma (W_{0}+\bk_{0 })
\\
&
\theta = (W_{\ell}, \bk_{\ell})_{\ell=0,...,L}, \,
W_{\ell} \in \mathbb{R}^{d_{\ell +1} \times d_{\ell}},\text{  and  } \bk_{\ell}: \Lp(D; \mathbb{R}^{d_{\ell}}) \to \Lp(D; \mathbb{R}^{d_{\ell +1}}), \, \text{  and  } d_0=d_a,\ d_{L+1}=d_u
\biggr\}.
\end{split}
\end{equation}
}\fi

\begin{equation}\label{neuraloperatorclass}
\begin{split}
\no 
=
&
\biggl\{ 
\Gcal_\theta: \Lp(D; \mathbb{R}^{d_a}) \to \Lp(D; \mathbb{R}^{d_u})\, \colon \,
\Gcal_\theta = (W_{L}+\bk_{L})\circ \sigma (W_{L-1}+\bk_{L-1})\circ \cdots \circ \sigma (W_{0}+\bk_{0 })
\\
&
\theta = (W_{\ell}, \bk_{\ell})_{\ell=0,...,L}, \,
W_{\ell} \in \mathbb{R}^{d_{\ell +1} \times d_{\ell}}, \bk_{\ell}: \Lp(D; \mathbb{R}^{d_{\ell}}) \to \Lp(D; \mathbb{R}^{d_{\ell +1}}), \, \text{  and  } d_0=d_a,\ d_{L+1}=d_u
\biggr\}.
\end{split}
\end{equation}

\noindent $\sigma : \mathbb{R} \to \mathbb{R}$ is an element-wise nonlinear map, and
$\bk_{\ell}$ are linear integral operators with kernel function, $k_{\ell}:D\times D \to \mathbb{R}^{d_{\ell+1}\times d_{\ell}}$, i.e., $
x\mapsto\left(\bk_{\ell}\,u\right)(x):=\int_{D}k_{\ell}(x,y)u(y)\, dy$ and $u \in \Lp(D; \mathbb{R}^{d_{\ell}})$.
We shall write, $w_{\ell, ij} = (W_{\ell})_{i,j} \in \mathbb{R}$ and $k_{\ell, ij} = (k_{\ell})_{i,j}:D\times D \to \mathbb{R}$ as $(i,j)$-element of $W_{\ell}$ and $k_{\ell}$, respectively.
%%%%%%%%%%%%%%%%%%
\begin{assumption}\label{Assumption2}
There exists positive constants $C_w$, $C_k$, $C_d$, $C_a$, $C_{\sigma}$, and $C_{\beta}$ such that
\begingroup
  \setlength\itemsep{1mm}
\begin{enumerate}[(i).]
%%%%%%%%%%%%%%%%%
\item $\left\|W_{\ell} \right\|_{\mathrm{op}} \leq C_w$,  and $d_{\ell} \leq C_{d}$ for all $\ell=0,...,L$, where $\left\|\cdot \right\|_{\mathrm{op}}$ is the operator norm.
%%%%%%%%%%%%%%%%%
\item $\left\|\bk_{\ell} \right\|_{\Lp, \mathrm{F}}:= \left( \sum_{i,j} \left\|k_{\ell,ij}\right\|^{2}_{\Lp(D \times D)} \right)^{1/2} \leq C_k $ for all $\ell=0,...,L$, where $ |D| = \int \boldsymbol{1}_D\, d\mathrm{\lambda}$\footnote{$\lambda$ is the Lebesgue measure.}, and $k_{\ell}:D\times D \to \mathbb{R}^{d_{\ell+1} \times d_{\ell}}$ is the kernel function. \label{assump:kernel}
%%%%%%%%%%%%%%%%%
\iffalse{\item $d_{\ell} \leq C_{\sigma}$ for all $\ell=0,...,L$,}\fi
%%%%%%%%%%%%%%%%%
\item $\left\|a \right\|_{\Lp(D;\mathbb{R}^{d_a})} \leq C_a$ for all $a \in \mathrm{supp}(\mu_a)$.
%%%%%%%%%%%%%%%%%
\item $\sigma$ is $C_{\sigma}$-Lipschitz, i.e., $|\sigma(s)-\sigma(t)|\leq C_{\sigma}|s-t|$ for $s,t \in \mathbb{R}$.
%%%%%%%%%%%%%%%%%
\item $\sup_{x,y \in D}|k_{\ell, ij}(x,y)| \leq C_{\alpha}$ for $\ell=0,...,L$, $i=1,..,d_{\ell}$, and $j=1,...,d_{\ell +1}$.
%%%%%%%%%%%%%%%%%
\item $k_{\ell, ij}:D\times D\to \mathbb{R}$ is $C_{\beta}$-Lipschitz, see Definition~\ref{Lipschitz Kernel}, for $\ell=0,...,L$, $i=1,..,d_{\ell}$, and $j=1,...,d_{\ell +1}$.
%%%%%%%%%%%%%%%%%
\end{enumerate}
\endgroup
\end{assumption}
\par 
\noindent Under these assumptions, we obtain the following upper bound for Rademacher Complexity for $\mathrm{NO}s$.

%----------------------------------
%- Radenacher NOs
%--------------------------------------
\begin{theorem}[Rademacher Complexity for NOs]\label{Rademacher for standard NO}
Let suppose assumptions~\ref{Assumption1} and \ref{Assumption2} hold. Then,
\begin{equation}
\mathfrak{R}^n_{\mathrm{S}}(\mathscr{F}_{\no}) \leq \gamma\, L^{\frac{\hat{d}+2}{\hat{d}+1}}
\{ (C_w + C_k)C_{\sigma} \}^{L} \left(\frac{1}{n} \right)^{\frac{1}{\hat{d}+1}}, \label{RC for standard NO}
\end{equation}
where $\hat{d}:= \mathrm{ddim}(D \times D)$ is the doubling dimension of $D\times D$ (see Definition~\ref{doubling dimension}), and $\gamma$ is the positive constant independent of $L$ and $n$, defined in (\ref{definition of gamma}).
\end{theorem}

%\begin{remark} 
See Appendix~\ref{main-theorem-1} for the proof.
The idea behind the proof is as follows. 
the Rademacher Complexity, $\mathfrak{R}^n_{\mathrm{S}}(\mathscr{F}_{\no})$, is evaluated by using \emph{Dudley’s Theorem} (\cref{Dudley Theorem}, and \citet{Kakade, bartlett2017spectrally}).
The upper bound is then determined by the \emph{covering number} (as defined in \cref{Covering Number}). Since NOs are parameterized by their weight matrices and (kernel) Lipschitz continuous functions, the evaluation of the covering number ultimately involves analyzing these components, by using \citet{wainwright} and \citet{gottlieb2016adaptive}, respectively. 
%\end{remark}
%\noindent See Section~\ref{Proof-main-part-section} for the proof.

\iffalse{
\begin{center}
\begin{tikzpicture}
\node[draw, fill=gray!10, inner sep=0.2cm] (text) at (0,0) {%
\begin{minipage}{1\textwidth}
\textbf{The idea behind the proof}: the Rademacher Complexity, $\mathfrak{R}^n_{\mathrm{S}}(\mathscr{F}_{\no})$, is evaluated by using \emph{Dudley’s Theorem} (\cref{Dudley Theorem}, and \citet{Kakade, bartlett2017spectrally}). The upper bound is then determined by the \emph{covering number} (as defined in \cref{Covering Number}). Since NOs are parameterized by their weight matrices and (kernel) Lipschitz continuous functions, the evaluation of the covering number ultimately involves analyzing these components, by using \citet{wainwright} and \citet{gottlieb2016adaptive}, respectively. 
\iffalse{
In order to evaluate the Rademacher complexity in our setting, we employ Dudley’s Theorem (see e.g., \citet{Kakade}) and evaluate the covering number for the set of Neural Operators, whose way is different from directly evaluating the Rademacher complexity like \citet{gopalani2022capacity, kim2022bounding}.
The covering number of Neural Operators is decomposed into that of weights matrices (for the local operation) and Lipschitz continuous function (for the non-local operation).
The covering number of matrices and Lipschitz continuous function can be evaluated by using \citet{wainwright} and \citet{gottlieb2016adaptive}, respectively.}\fi
\end{minipage}%
};
\end{tikzpicture}
\end{center}
}\fi

\noindent See \cref{remark:finite_basis} for finite basis expansion (applicable integral kernel).

    \subsection{Rademacher of $\MNO$ and intermediate architectures}
    \label{sec:Rademacher Complexity2}
    %Standard Neural operator is additive for local operator $W$ and non-local operator $K$. We define the set of multiplicative Neural operator by
In this section, we analyze the Rademacher Complexity of the proposed networks. We define the family, see Section~\ref{Proposed Method}, as
%Section~\ref{MNO}, \ref{FNO+epsilon}, and \ref{FNO+epsilon v2}, as 
\begin{equation}\label{neuraloperatortildeclass}
\begin{split}
\widetilde{\no} =\biggl\{& \Gcal_\theta: \Lp(D; \mathbb{R}^{d_a}) \to \Lp(D; \mathbb{R}^{d_u}) \colon
\\
&
\Gcal_\theta = (\boldsymbol{\mathrm{Z}}_{L}\bid + \boldsymbol{\mathrm{X}}_{L} f_{L}) \circ (\boldsymbol{\mathrm{Z}}_{L}\bid + \boldsymbol{\mathrm{X}}_{L} \sigma \circ \bk_{L}) \circ  \cdots \circ (\boldsymbol{\mathrm{Z}}_{0}\bid + \boldsymbol{\mathrm{X}}_{0} f_{0}) \circ (\boldsymbol{\mathrm{Z}}_{0}\bid + \boldsymbol{\mathrm{X}}_{0} \sigma \circ \bk_{0})
\\
&
\boldsymbol{\mathrm{Z}}_{\ell}, \boldsymbol{\mathrm{X}}_{\ell} \in \{0,1\}, \, 
f_{\ell} = W_{\ell,M} \circ \sigma(W_{\ell,M-1}) \circ \cdots \circ \sigma(W_{\ell,0})
\text{ is an Mth layer MLP}
\\
&
\theta = (W_{\ell,m}^{\phantom{1}}, \bk_{\ell})_{\substack{\ell=0,\ldots,L\\ \scriptstyle m=0,\ldots,M}}, \,
W_{\ell,m} \in \mathbb{R}^{d^{w}_{\ell, m+1} \times d^{w}_{\ell, m}} \text{ and }
\bk_{\ell}: \Lp(D; \mathbb{R}^{d^{k}_{\ell}}) \to \Lp(D; \mathbb{R}^{d^{k}_{\ell+1}})
\\
&
d_{\ell, 0}^{w}=d^{k}_{\ell + 1},\ d^{w}_{\ell, M}=d^{k}_{\ell + 1},\ d^{k}_{0}=d_a, \ d^{k}_{L+1}=d_u
\biggr\}.
\end{split}
\end{equation}
%%%%%%%%%%%%%%%%%%
\begin{assumption}\label{Assumption3}
There exists positive constants $C_w$, $C_k$, $C_d$, $C_a$, $C_{\sigma}$, and $C_{\beta}$ such that
\begingroup
  \setlength\itemsep{1mm}
\begin{enumerate}[(i).]
%%%%%%%%%%%%%%%%%
\item $\left\|W_{\ell, m} \right\|_{\mathrm{op}} \leq C_w$, and $d_{\ell}^{k}, d_{\ell,m}^{w} \leq C_{d}$, for $\ell=0,...,L$, $m=0,...,M$.
%%%%%%%%%%%%%%%%%
\item $\left\|\bk_{\ell} \right\|_{\Lp, \mathrm{F}} \leq C_k$, for $\ell=0,...,L$.
%%%%%%%%%%%%%%%%%
\item $\left\|a \right\|_{\Lp(D;\mathbb{R}^{d_a})} \leq C_a$, for $a \in \mathrm{supp}(\mu_a)$.
%%%%%%%%%%%%%%%%%
\item $\sigma$ is $C_{\sigma}$-Lipschitz, i.e., $|\sigma(s)-\sigma(t)|\leq C_{\sigma}|s-t|$ for $s,t \in \mathbb{R}$.
%%%%%%%%%%%%%%%%%
\item $\sup_{x,y \in D}|k_{\ell, ij}(x,y)| \leq C_{\alpha}$ for $\ell=0,...,L$, $i=1,..,d_{\ell}$, and $j=1,...,d^{k}_{\ell +1}$.
%%%%%%%%%%%%%%%%%
\item $k_{\ell, ij}:D\times D\to \mathbb{R}$ is $C_{\beta}$-Lipschitz, for $\ell=0,...,L$, $i=1,..,d_{\ell}^{k}$, and $j=1,...,d_{\ell +1}^{k}$.
%%%%%%%%%%%%%%%%%
\end{enumerate}
\endgroup
\end{assumption}
%%%%%%%%%%%%%%%%%%%%%%%%%%%%%%%%%%%%%%%%%%%%%%%%%%%
\noindent Upon these assumptions, we obtain the following upper bound for the Rademacher Complexity of $\widetilde{\no}$.
\begin{theorem}[Rademacher Complexity of proposed network(s)]\label{Rademacher for multiplicative NO}
Let Assumptions~\ref{Assumption1} and \ref{Assumption3} hold. Then, 
\begin{equation}
\begin{split}
&\mathfrak{R}^n_{\mathrm{S}}(\mathscr{F}_{\widetilde{\no}}) 
%\\
%&
\leq
\widetilde{\gamma}  L^{\frac{1}{\hat{d}+1}} 
\left( \sum_{\ell=0}^{L}
\frac{\boldsymbol{\mathrm{X}}_{\ell}C_{w}^{M+1}C_{\sigma}^{M}}{\boldsymbol{\mathrm{Z}}_{\ell}+\boldsymbol{\mathrm{X}}_{\ell}C_{w}^{M+1}C_{\sigma}^{M}}
+\frac{\boldsymbol{\mathrm{X}}_{\ell}}{
\boldsymbol{\mathrm{Z}}_{\ell}+\boldsymbol{\mathrm{X}}_{\ell}C_{k}C_{\sigma}
}
\right)
%\\ 
%&
%\times 
\left[\prod_{\ell=0}^{L}(\boldsymbol{\mathrm{Z}}_{\ell}+\boldsymbol{\mathrm{X}}_{\ell}C_{w}^{M+1}C_{\sigma}^{M})(\boldsymbol{\mathrm{Z}}_{\ell}+\boldsymbol{\mathrm{X}}_{\ell}C_k C_{\sigma})
\right]
\left(\frac{1}{n}\right)^{\frac{1}{\hat{d}+1}}, \label{RC for multiplicative NO}
\end{split}
\end{equation} 
where $\widetilde{\gamma}$ is the positive constant independent of $L$ and $n$, defined in (\ref{definition of gamma tilde}).
\end{theorem}
\noindent
Theorem~\ref{Rademacher for multiplicative NO} can be proved by similar arguments in Theorem~\ref{Rademacher for standard NO}. See Appendix~\ref{main-theorem-2} for the proof.

    \subsection{GEB and comparison among architectures}
    \label{sec:Rademacher Complexity3}
    By Lemma~\ref{generalization error bound for general setting}, and Theorems~\ref{Rademacher for standard NO} and \ref{Rademacher for multiplicative NO}, we get. 
\begin{corollary}\label{GEB for additive NO}
Let Assumptions~\ref{Assumption1} and \ref{Assumption2} hold. Then, for any $\delta>\log2$ and $\Gcal \in \no$, the following inequality holds, with probability greater than $1-2 \exp(-\delta)$:
\begin{equation}
\begin{split}
\mathcal{L}(\Gcal) 
& \leq \widehat{\mathcal{L}}_{\mathrm{S}}(\Gcal) 
+ 2\gamma L^{\frac{\hat{d}+2}{\hat{d}+1}}\{ (C_w + C_k)C_{\sigma} \}^{L} \left(\frac{1}{n} \right)^{\frac{1}{\hat{d}+1}} 
+ \left(\rho \{(C_w + C_k)C_{\sigma}\}^{L} (C_w + C_k) C_a  + R_u \right)\sqrt{\frac{2 \delta}{n}}. 
%\label{additiveGEB}
\label{Generalization error bound for additive NO-main}
\end{split}
\end{equation}
\end{corollary}
See Appendix~\ref{main-coro-1} for the proof.

\begin{corollary}\label{GEB for multiplicative NO}
Let Assumptions~\ref{Assumption1} and \ref{Assumption3} hold. 
Then, for any $\delta>\log2$ and $\Gcal \in \widetilde{\no}$, the following inequality with probability greater than $1-2 \exp(-\delta)$:
\begin{equation}
\begin{split}
&\mathcal{L}(\Gcal) \leq \widehat{\mathcal{L}}_{\mathrm{S}}(\Gcal) 
\\
&
+ 2 \widetilde{\gamma}  L^{\frac{1}{\hat{d}+1}} 
\left( \sum_{\ell=0}^{L}
\frac{\boldsymbol{\mathrm{X}}_{\ell}C_{w}^{M+1}C_{\sigma}^{M}}{\boldsymbol{\mathrm{Z}}_{\ell}+\boldsymbol{\mathrm{X}}_{\ell}C_{w}^{M+1}C_{\sigma}^{M}}
+\frac{\boldsymbol{\mathrm{X}}_{\ell}}{
\boldsymbol{\mathrm{Z}}_{\ell}+\boldsymbol{\mathrm{X}}_{\ell}C_{k}C_{\sigma}
}
\right)
\left[\prod_{\ell=0}^{L}(\boldsymbol{\mathrm{Z}}_{\ell}+\boldsymbol{\mathrm{X}}_{\ell}C_{w}^{M+1}C_{\sigma}^{M})(\boldsymbol{\mathrm{Z}}_{\ell}+\boldsymbol{\mathrm{X}}_{\ell}C_k C_{\sigma})
\right]
\left(\frac{1}{n}\right)^{\frac{1}{\hat{d}+1}} 
\\
&
+ \left(\rho 
\left[\prod_{\ell=0}^{L}(\boldsymbol{\mathrm{Z}}_{\ell}+\boldsymbol{\mathrm{X}}_{\ell}C_{w}^{M+1}C_{\sigma}^{M})(\boldsymbol{\mathrm{Z}}_{\ell}+\boldsymbol{\mathrm{X}}_{\ell}C_k C_{\sigma})
\right]
C_a  + R_u \right)\sqrt{\frac{2 \delta}{n}}.\label{multiplicativeGEB}
%\label{Generalization error bound for multiplicative NO-main}
\end{split}
\end{equation}
\end{corollary}
See Appendix~\ref{main-coro-1} for the proof.

\begin{remark}
When $\boldsymbol{\mathrm{Z}}_{\ell}=0$ and $\boldsymbol{\mathrm{X}}_{\ell}=1$ corresponds to $\MNO$, if $\boldsymbol{\mathrm{Z}}_{\ell}=1$ and $\boldsymbol{\mathrm{X}}_{\ell}=1$ to ($\NONeXt$)v1. Finally, if $\boldsymbol{\mathrm{Z}}_{\ell}=1$ and $\boldsymbol{\mathrm{X}}_{\ell}$ is a Bernoulli RV with $\prob\{\boldsymbol{\mathrm{X}}_{\ell}  =  1 \} = p_\ell$, and $\prob\{\boldsymbol{\mathrm{X}}_{\ell}  = 0\} = 1-p_\ell$ for $p_\ell \in [0,1]$ corresponds to ($\NONeXt$)v2.
\end{remark}

\begin{remark} \label{GEB Comparison}%$\widehat{\mathcal{L}}_{\mathrm{S}}(\Gcal)$.
\iffalse{NO can be decomposed into local and non-local, providing an architecture as $\MNO$, where $f_\ell$ is $1$-layer, the converse is also true (see, \citet[Section 2.5.1]{kovachki2021universal}). Albeit, the difference in the generalization error bound has not been studied.}\fi 
The $2$nd, and $3$rd terms decay as the samples increases, $n \to \infty$, with orders $\mathcal{O}\left(1/n^{\frac{1}{\hat{d}+1}}\right)$ and $\mathcal{O}(1/n^{\frac{1}{2}})$, respectively.
We finally observe the coefficients depending on the number of layers, $L$ (see also Remark~\ref{Comparision-bounds}).
\end{remark}
\begingroup
  \setlength\itemsep{2mm}
   \begin{enumerate}
       \item If $(C_{w}+C_{k})C_{\sigma} < 1$ (or  $C_{w}^{M+1}C_{\sigma}^{M+1}C_{k} < 1$), the upper bounds of standard NO (or $\MNO$) remain bounded as $L$ tends to infinity. On the other hand, if $(C_{w}+C_{k})C_{\sigma} > 1$ (or $C_{w}^{M+1}C_{\sigma}^{M+1}C_{k} > 1$), then, the upper-bounds diverges with depth, similarly than finite-dimensional networks \citet{truong2022rademacher}.
       \item \label{sNO_GEB} If the condition $C_{w}<1$ and $C_{\sigma} \leq 1$ holds true, then $C_{w}^{M+1} C_{\sigma}^{M+1} C_{k}\leq (C_w + C_k)C_{\sigma}$, which implies that the upper bounds of $\MNO$ are lower-bound of standard NOs. See \cref{Comparision-bounds}.
       \item Since $C_{w}^{M+1} C_{\sigma}^{M} C_{k}\leq (1+C_{w}^{M+1}C_{\sigma}^{M}) (1+C_{\sigma} C_{k})$, the upper bounds of standard NOs are lower-bound of ($\NONeXt$)v1, despite the outcomes of our experiments, see Figure~\ref{fig:boxplot}. 
       \item \label{sNOeps_GEB} Finally the RVs can control the GEB. If $\prob\{\boldsymbol{X}_\ell=1\}= p_{\ell}= x_{\ell}/ L^{\frac{1}{\hat{d}+1}}$, where $x_{\ell} \in [0,1]$ satisfies $\sum_{\ell=0}^{\infty}x_{\ell}< \infty$, the upper bound for ($\NONeXt$)v2 does not blow up as $L$ increases, \emph{regardless of $C_{w}$, $C_{k}$, and $C_{\sigma}$}. The expectation with respect to $\boldsymbol{\mathcal{X}}=(\boldsymbol{\mathrm{X}}_{0},..., \boldsymbol{\mathrm{X}}_{L})$ is bounded above by the expression (see Lemma~\ref{BRV-estimate} in Appendix~\ref{Appendix 4})\[
\begin{split}
\mathbb{E}_{\boldsymbol{\mathcal{X}}}[\text{RHS of (\ref{multiplicativeGEB})}]
& 
\lesssim
\widehat{\mathcal{L}}_{\mathrm{S}}(\Gcal) 
+
\left(\sum_{\ell=1}^{L}
x_{\ell}\right) 
\prod_{\ell=0}^{L}
\left[
1+(C_{w}^{M+1}C_{\sigma}^{M}+C_k C_{\sigma}+C_{w}^{M+1}C_kC_{\sigma}^{M+1})
x_{\ell}
\right]
\left(\frac{1}{n}\right)^{\frac{1}{\hat{d}+1}} 
\\
&
+ \left(\rho 
\prod_{\ell=0}^{L}
\left[
1+(C_{w}^{M+1}C_{\sigma}^{M}+C_k C_{\sigma}+C_{w}^{M+1}C_kC_{\sigma}^{M+1})
x_{\ell}
\right]
C_a  + R_u \right)\sqrt{\frac{2 \delta}{n}}, %\label{multiplicativeGEB}
%\label{Generalization error bound for multiplicative NO}
\end{split}
\]
whose coefficients do not blow up as $L \to \infty$ (the infinite products converge because $\sum_{\ell=0}^{\infty}x_{\ell}< \infty$, see, \citet{William}). Here, $\lesssim$ implies that the left-hand side is bounded above by the right-hand side times a constant independent of $n$ and $L$. 
%For example, $x_\ell \in \mathcal{O}(\ell^{-(1+\varepsilon)})$ for $\varepsilon> 0$, satisfies $\sum_{\ell=0}^{\infty}x_{\ell}< \infty$. 
For example, if $x_\ell$ decay with order $\mathcal{O}(\ell^{-(1+\varepsilon)})$ for some $\varepsilon> 0$, then it holds that  $\sum_{\ell=0}^{\infty}x_{\ell}< \infty$. \footnote{
Notice that, the assumption of $\sum_{\ell=0}^{\infty}p_{\ell}< \infty$ by the Borel–Cantelli lemma, implies that the probability that infinitely many of $X_{\ell}=1$ (layers that are active) occur is zero}.  \citep{huang2016deep} proposed linear decay, which does not satisfy $\sum_{\ell=0}^{\infty}x_{\ell}< \infty$. However, it is assumed that the number of layers $L$ is finite (typically around $100$), our analysis on the other hand showed that the upper bound is valid regardless of the number of layers if the Bernoulli RVs satisfied the above-mentioned condition. A less \emph{restrictive} decay on the RVs can be chosen.

\end{enumerate}

\noindent Therefore, our proposed architecture, especially ($\NONeXt$)v2, would have a smaller generalization error than the standard architecture under assumptions of the RVs.

\section{Summary and discussion}\label{Conclusions and Future Work}
    We perform a detailed empirical and theoretical analysis of the generalization capabilities of neural operators and $\NONeXt$ for approximating the parametric form of the Helmholtz equation, as well as a surrogate model for the forward operator associated with the study of the inverse boundary value 
problem for the Helmholtz equation. We work with high-frequency given the documented difficulties of numerical methods, \citep{ernst2011difficult, erlangga2008advances, gander2019class}, and the amount of previous work associated with other PDEs, which traditional neural operators already approximate remarkably well\footnote{Darcy flow}. \par

The $\NONeXt$ family demonstrated improved performance without increasing the number of parameters (in the case without stochastic depth) or compromising the approximation capabilities of traditional neural operators for high-frequency Helmholtz problems. We maintained strict constraints throughout our analysis, including not increasing the size of the training dataset, and testing on datasets of comparable size as those used in the training.\par

We conduct a thorough empirical analysis of the stability of the trained networks to different realizations of the wave speed, and $(\FNONeXt)$v2 demonstrated resilience to these changes. In light of these results, we derive upper bounds for out-of-distribution generalization for Gaussian measures in abstract Banach spaces, and we link the experimental behavior to the presence of the random variables presented in stochastic depth. For the results in-distribution, we also provide an upper bound of the generalization error by estimating the Rademacher complexity of each of the networks. Similarly, showing that the random variables in stochastic depth are effectively controlling the complexity of the hypothesis class for the $(\NONeXt)$v2 family. \par  

%We have made progress in understanding of the theoretical guarantees for neural operators and similar architectures beyond their universal approximation property. However, one of the shortcomings of this work is that our bounds are not tight. While deriving lower bounds remains a challenging task, we anticipate further progress in this area is plausible.\par 
We have made progress in understanding the theoretical guarantees of neural operators and similar architectures, going beyond their approximation property. However, it is worth noting that one of the limitations of our work is that the bounds we derived are not tight. Although deriving lower bounds presents a challenge, we remain optimistic about the possibility of making further advancements in this area.

On the experimental side, our results suggest that it is possible to capture the forward operator effectively (mapping functions to operators), and we expect to apply this surrogate model to solve inverse problems, particularly for Bayesian inversion and for using algorithms that only require multiple evaluations of the costly forward operator, such as the derivative-free ensemble Kalman method \citep{iglesias2016regularizing}.

\section*{Acknowledgments}
    J.A.L.B. is grateful to colleagues at PGS Imaging group for their insightful discussions during their 2022 internship, valuable assistance with resources, and diligent work on applying a previous version of the proposed network to real field data \cite{huang2023deep}. J.A.L.B. also appreciates the support from ChatGPT in optimizing TikZ code. T.F.\ was supported by Research Grant for Young Scholars funded by yamanashi Prefecture. F.F.\ acknowledges the use of the cluster PlaFRIM\footnote{\texttt{https://www.plafrim.fr/}.} for the dataset generation. X.M.T. would like to thank the Community Cluster Program and the Rosen Center for Advanced Computing and Purdue. M.V.\ de~H. gratefully acknowledges support from the Department of Energy under grant DE-SC0020345, the National Science Foundation under grant DMS-2108175, the Simons Foundation under the MATH + X program, and the corporate members of the Geo-Mathematical Imaging Group at Rice University. A.K.\ was funded by the NSERC (grants no.\ RGPIN-2023-04482 and DGECR-2023-00230).

\begin{comment}
M.V.\ de~H. was supported by the Simons Foundation under the MATH + X program, the National Science Foundation under grant DMS-2108175, and the corporate members of the Geo-Mathematical Imaging Group at Rice University. A.K.\ was funded by the NSERC (grants no.\ RGPIN-2023-04482 and DGECR-2023-00230).
\end{comment}

%%%%%%%%%%%%%%%%%%%%%%%%%%%%%%%%%%%%%%%%%%%%%%%%%%%%%%%%%%%%%%%%%
%\bibliographystyle{plain}s

\bibliography{ref.bib}

%%%%%%%%%%%%%%%%%%%%%%%%%%%%%%%%%%%%%%%%%%%%%%%%%%%%%%%%%%%%%%%%%
\newpage

\appendix

\part*{Appendix}

\renewcommand{\thesection}{\Alph{section}}

%\part{Appendix}

%\section{Table of Notations}
%\input{Appendix/tableofNotation.tex}

\section{Preliminaries}

\subsection{Notation}\label{table of notation}

\begin{table}[h!]
\centering
\begin{tabular}{|c|c|}
\hline
\rowcolor[HTML]{EFEFEF}
\textbf{Notation}  & \textbf{Meaning}  \\
\hline 
$d$ & Dimension of spatial domain \\

$d_a$ & Dimension of input function $a(x)$ \\

$d_u$ & Dimension of output function $u(x)$ \\

$d_{\ell}$ & Number of the column for $W_{\ell}$ \\

$d_{\ell,m}^{w}$ & Number of the column for $W_{\ell,m}$ \\

$d_{\ell}^{k}$ & Number of the column for $k_{\ell}(x,y)$ \\

$D \subset \mathbb{R}^{d}$  & Spatial domain \\

$a \in \Lp(D; \mathbb{R}^{d_a})$  & Input function \\

$u \in \Lp(D; \mathbb{R}^{d_u})$  & Output function \\

$n$ & Number of training data  \\

$\mathrm{S}=\{a_i, u_i\}_{i=1}^{n}$  &  Training dataset drawn from probability measure $\mu$ \\

$\mu$  &  Probability measure on $\Lp(D; \mathbb{R}^{d_a}) \times \Lp(D; \mathbb{R}^{d_u})$ \\

$\mu_a$  & Marginals of $\mu$ on $\Lp(D; \mathbb{R}^{d_a})$ \\

$\mu_u$  & Marginals of $\mu$ on $\Lp(D; \mathbb{R}^{d_u})$ \\

$\ell$  & Loss function \\

$\sigma$  & Activation function \\

$\no$ & Space of Neural Operators \\

$W_{\ell}$ & $d_{\ell + 1}\times d_{\ell}$-matrix in $\no$ \\

$\bk_{\ell}$ & Integral operator with kernel $k_{\ell}$ in $\no$ \\

$k_{\ell}$ & $d_{\ell + 1}\times d_{\ell}$-kernel matrix for $\bk_{\ell}$ \\

$L$ & Number of layers \\

$\hat{d}$ & Doubling dimension of $D \times D$ \\

$\widetilde{\no}$ & Space of $\multiplicative$ Neural Operators \\

$f_{\ell}$ & MLPs in $\widetilde{\no}$ \\

$M$ & Number of layers in MLPs $f_{\ell}$ \\

$W_{\ell,m}$ & $d^{w}_{\ell, m + 1}\times d^{w}_{\ell, m}$-matrix in MLPs $f_{\ell}$ \\

$k_{\ell}$ & $d^{k}_{\ell + 1}\times d^{k}_{\ell}$-kernel matrix in $\widetilde{\no}$ \\

$\Lp(D; \mathbb{R}^{h})$ & $L^2$ space of $\mathbb{R}^{h}$-value function on $D$ \\

$\left\| \cdot \right\|_{\Lp(D; \mathbb{R}^{h})}$ & $L^2$-norm \\

$\left\| \cdot \right\|_{2}$ & $\ell_2$-norm \\

$\left\| \cdot \right\|_{\mathrm{F}}$ & Frobenius norm \\

$\left\| \cdot \right\|_{\mathrm{S}}$ & Sampling norm, $\left\| f \right\|_{\mathrm{S}}:= \left(\frac{1}{n}\sum_{i=1}^{n}f(a_i, u_i)^2 \right)^{\frac{1}{2}}$ \\

$\left\| \cdot \right\|_{\mathrm{op}}$ & Operator norm \\ \hline
\end{tabular}
\caption{Table of Notations}
\end{table}

%%%%%%%%%%%%%%%%%%%%%%%%%%
\iffalse{
%%%%%%%%%%%%%%%%%%%%%%%%%%%%%%%%%%%
\subsection{$\mathrm{L}^2$ spaces and Sobolev space.}  \label{appendix:Sobolev_spaces}
\noindent
We denote $\Lp(D)=\Lp(D; \mathbb{R})$ by $\Lp$ space of real-value functions $u$ on $D \subset \mathbb{R}^d$ such that
\[
\left\| u \right\|_{\Lp(D)}^{2}:= \int_{D} |u(x)|^{2}dx < \infty.
\]
For natural number $k \in \mathbb{N}_{0}$, we define Sobolev space $H^{k}(D)$ by
\[
H^{k}(D):=\left\{ u \in \Lp(D) : \partial_{x}^{\alpha} u \in \Lp(D) \ \forall |\alpha| \leq k \right\}.
\]
For positive non-integer $s >0$, we define Sobolev space $H^{s}(D)$ by
\begin{align}\label{positive-Sobolev space}
H^{s}(D):=\left\{ u \in H^{\lfloor s \rfloor}(D) : \sup_{|\alpha|=
\lfloor s \rfloor
} [\partial_{x}^{\alpha} u]_{\theta_s, D} < \infty
\right\},
\end{align}
where $\theta_s:=s-\lfloor s \rfloor \in (0,1)$. 
Here, $[f]_{\theta, D}$ is defined by
\[
[f]_{\theta, D}:= \left(\int_D \int_D \frac{|f(x)-f(y)|^2}{|x-y|^{2\theta + d}}dxdy
\right)^{1/2}.
\]
For further details, we refer to e.g., \citet{adams2003sobolev}.
}\fi
%%%%%%%%%%%%%%%%%%%%%%%%%%%%%%%%%%

%%%%%%%%%%%%%%%%%%%%%%%%%%%%%%%%%%%
\subsection{Vector-Valued $\mathrm{L}^2$ spaces and Sobolev space.} \label{appendix:Sobolev_spaces}
\par
$\Lp(D; \mathbb{R}^{d_a})$ is the $\Lp$ space of $\mathbb{R}^{d_a}$-value functions on $D \subset \mathbb{R}^d$. It is defined as the space of functions such that,
\[
\left\| a \right\|_{\Lp(D; \mathbb{R}^{d_a})}^{2}:= \int_{D} \left\|a(x) \right\|_{2}^{2}dx< \infty, 
\]
where $D\ni x\mapsto \left\|a (x) \right\|_{2}^{2} = \sum_j a_j^2(x)$; notices that, $\left\|\cdot \right\|_{2}^{2}$ is the usual $\ell_2$-norm in $\mathbb{R}^{d_a}$.
\par
For natural number $k \in \mathbb{N}_{0}$, we define Sobolev space $H^{k}(D; \mathbb{R}^{d_a})$ by
\[
H^{k}(D;\mathbb{R}^{d_a}):=\left\{ u \in \Lp(D;\mathbb{R}^{d_a}) : \partial_{x}^{\alpha} u \in \Lp(D;\mathbb{R}^{d_a}) \ \forall |\alpha| \leq k \right\}.
\]
For positive non-integer $s >0$, we define Sobolev space $H^{s}(D;\mathbb{R}^{d_a})$ by
\begin{align}\label{positive-Sobolev space}
H^{s}(D;\mathbb{R}^{d_a}):=\left\{ u \in H^{\lfloor s \rfloor}(D;\mathbb{R}^{d_a}) : \sup_{|\alpha|=
\lfloor s \rfloor
} [\partial_{x}^{\alpha} u]_{\theta_s, D} < \infty
\right\},
\end{align}
where $\theta_s:=s-\lfloor s \rfloor \in (0,1)$. 
Here, $[f]_{\theta, D}$ is defined by
\[
[f]_{\theta, D}:= \left(\int_D \int_D \frac{\left\|f(x)-f(y)\right\|^{2}_{2}}{\left\|x-y\right\|_{2}^{2\theta + d}}dxdy
\right)^{1/2}.
\]
For further details, we refer to, e.g., \citet{adams2003sobolev}.
%%%%%%%%%%%%%%%%%%%%%%%%%%%%%%%%%%%%%%%%%%%%%%%%%%%%%%%%%%%%%
\subsection{Linear Bounded Operator}\label{BO}
\par
\begin{definition}[Linear Bounded Operator] We say that $\boldsymbol{A}: X \to Y$ is the Linear Bounded Operator mapping from a Banach space $X$ to a Banach space $Y$, if it is linear and if there exists a positive constant $C>0$ such that,
\[
\|\boldsymbol{A} x \|_{Y}\leq C\|x\|_{X}, \ x \in X.
\]    
\end{definition}

\begin{definition}[Operator norm]We also recall that the Operator norm $\|A\|_{\mathrm{op}}$ for linear bounded operator $A$ as \[
\|\boldsymbol{A}\|_{\mathrm{op}}:=\inf\left\{C \in \mathbb{R}_{\ge 0} \, : \ \|\boldsymbol{A} x \|_{Y}\leq C\|x\|_{X} \right\}.
\]    
\end{definition}

In particular, Neural Operators \citep{li2020fourier} include the Linear Integral Bounded Operator

\begin{definition}[Linear Integral Bounded Operator] \label{IO} It is an Linear Bounded Operator $\bk:\Lp(D; \mathbb{R}^{n}) \to \Lp(D; \mathbb{R}^{m})$ defined by  
\[
x\mapsto \left(\bk \,g\right)(x):= \int_{D}\,k(x,y)\,g(y)\,dy, \quad x \in D, \ g \in \Lp(D; \mathbb{R}^{n}),  \]
where $k:D \times D \subset \mathbb{R}^{d\times d} \to \mathbb{R}^{m \times n}$ is the Integral Kernel. 
\end{definition}

\begin{definition}[Lipschitz Kernel]\label{Lipschitz Kernel} We say a vector-valued Integral Kernel is Lipschitz continuous if there exists $C>0$ such that \[
|k_{i,j}(x,y) - k_{i,j}(x',y')| \leq C \left\| (x,y)-(x', y') \right\|_{2}, \quad (x,y), (x', y') \in D \times D.
\]
for $i,j\in \{1,\ldots,d\}$.
\end{definition}

%%%%%%%%%%%%%%%%%%%%%%%%%%%%%%%%%%%%%%%%%%%%%%%%%%%%%%%%%%%%%
% NO
%%%%%%%%%%%%%%%%%%%%%%%%%%%%%%%%%%%%%%%%%%%%%%%%%%%%%%%%%%%%%
\subsection{Neural Operator.}\label{NO}
\par

Let $D$ a bounded domain and let $\mathcal{A}(D; \mathbb{R}^{d_a})$, $\mathcal{U}(D; \mathbb{R}^{d_{v_i}})$, and $\mathcal{U}(D; \mathbb{R}^{d_u})$ be abstract (separable) Banach spaces. 

\begin{definition}[Neural Operator] Let define $\Gcal_{\theta}: \mathcal{A}(D; \mathbb{R}^{d_a}) \to \mathcal{U}(D; \mathbb{R}^{d_u})$ such that 
\begin{equation}
    u = \Gcal_{\theta}(a) = \bq \circ \mathcal{L}_k \circ \ldots \circ \mathcal{L}_1 \circ \br(a), 
\end{equation}
in where $\bq: \mathcal{A}(D; \mathbb{R}^{d_a}) \to \mathcal{U}(D; \mathbb{R}^{d_{v_1}})$ (Lifting map), and  $\br: \mathcal{U}(D; \mathbb{R}^{d_{v_{k+1}}}) \to \mathcal{U}(D; \mathbb{R}^{d_u})$ (Projection map), such that 
\begin{subequations}
    \begin{align}
        \br(a)(x) := \left(\,Ra(x) \,\right), \quad R \in \mathbb{R}^{d_{v_1} \times d_a}.\\
        \bq(v)(x) := \left(\,Qv(x) \,\right), \quad Q \in \mathbb{R}^{d_u \times d_{v_{k+1}}}.
    \end{align}
\end{subequations}
and $\mathcal{L}_i$, $(i= 1, \ldots, k)$ is defined as

\begin{equation}
    D\ni x\mapsto \left(\mathcal{L}_iv \right)(x) : = \sigma \left( \, W_iv(x) + (\bk_i v) (x)\ \right), \quad W_i \in \mathbb{R}^{v_{i+1} \times v_{i}}. \tag{Layers}
\end{equation}
\end{definition}

$i=1, \ldots, k$, and $\bk_{i}$ is an integral operator mapping from $\mathcal{U}(D; \mathbb{R}^{d_{v_i}})$ to $\mathcal{U}(D; \mathbb{R}^{d_{v_{i+1}}})$, see definition \ref{IO}. In the definition of \citet[Section 9.1]{kovachki2021neural}, Neural Operators parameterize the integral kernel as neural networks, which satisfies the Lipschitz continuity used in the Assumption~\ref{Assumption3}.
%(Cf. \cite{kovachki2021universal} for details).

\iffalse{
\begin{figure*}[!ht]
    \centering
    \includegraphics[scale= 0.4]{Figures/NO.png}
    \caption{\small{\textbf{NO}. Neural Operator architecture.}}
    \label{fig:NO}
\end{figure*}
}
\fi 

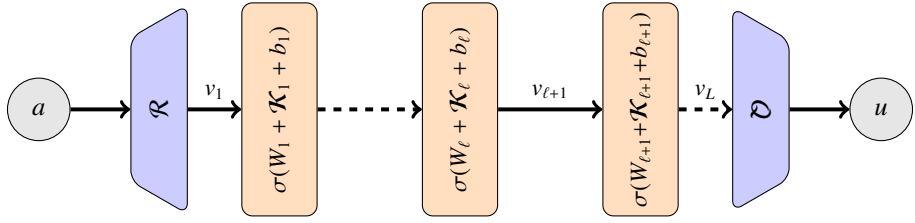
\begin{figure}[!ht]
    \centering
    \begin{tikzpicture}[scale=0.8]
    % Define the nodes
    % Define the nodes
    \node[draw,trapezium,trapezium left angle=60,trapezium right angle=60,fill=blue!20,rotate=90,minimum width=.5cm,minimum height=.75cm,text width=1cm,text centered, rounded corners] (trap1) at (-2,0) {$ \mathcal{R}$};
    \node[draw,trapezium,trapezium left angle=60,trapezium right angle=60,fill=blue!20,rotate=270, minimum width=.5cm,minimum height=.75cm,text width=1cm,text centered, rounded corners] (trap2) at (8,0) {$ \mathcal{Q}$};
    \node[draw,rectangle,fill=orange!25,rotate=90,minimum width=0.5cm,minimum height=1.cm,text width=2.6cm,text centered, rounded corners] (rect1) at (0,0) {\small $ \sigma (W_1 + \bk_1 + b_1)$};
    \node[draw,rectangle,fill=orange!25,rotate=90,minimum width=0.5cm,minimum height=1.cm,text width=2.6cm,text centered, rounded corners] (rect2) at (3,0) {\small $\sigma (W_\ell + \bk_\ell + b_\ell)$};
    \node[draw,rectangle,fill=orange!25,rotate=90,minimum width=0.5cm,minimum height=1.cm,text width=2.6cm,text centered, rounded corners] (rect3) at (6,0) { \small $ \sigma (W_{\ell+1} + \bk_{\ell+1} + b_{\ell+1})$};
    \node[draw,circle,fill=gray!20,rotate=0,minimum width=0.75cm,text width=0.5cm,text centered] (circ0) at (-4,0) {$a$};
    \node[draw,circle,fill=gray!20,rotate=0,minimum width=0.75cm,text width=0.5cm,text centered] (circ1) at (10,0) {$u$};
    % Draw the arrows
    \draw[ ->, line width=1.5pt] (circ0) -- (trap1);
    \draw[ ->, line width=1.5pt] (trap1) -- (rect1) node[midway,above] {\small $v_1$};
    \draw[dashed, ->, line width=1.5pt] (rect1) -- (rect2);
    \draw[->, line width=1.5pt] (rect2) -- (rect3) node[midway,above] { \small $v_{\ell+1}$};
    \draw[dashed, ->, line width=1.5pt] (rect3) -- (trap2) node[midway,above] { \small $v_{L}$};
    \draw[->, line width=1.5pt] (trap2) -- (circ1);
    \end{tikzpicture}
    \caption{\small{\textbf{NO}. Neural Operator architecture.}}
    \label{fig:NO}
\end{figure}

\subsubsection{Fourier Neural Operators (FNOs).}
A natural ansatz in the integral operator is assuming to be convolutional, so that, 
\begin{equation} \label{convthm}
    (k \star v)=\boldsymbol{\mathcal{F}}^{-1}\left( \boldsymbol{\mathcal{F}}(k)\cdot \boldsymbol{\mathcal{F}}(v) \right).
\end{equation}
\footnote{$\boldsymbol{\mathcal{F}}$, and $\boldsymbol{\mathcal{F}}^{-1}$ represents the Fourier and Inverse Fourier transform respectively.} if the kernel function and $v$ lies on the adequate space, say $\Lp$. When Equation~\ref{convthm} is estimated by the FFT algorithm, the Neural Operator is efficiently implemented, leading to the network presented in \citet{li2020fourier}.

\subsubsection{Remark: Universality of $\MNO$.}\label{sNO: universal}
\citet[Theorem 11]{kovachki2021neural} have shown that the compositional operator $( \sigma \circ \bk_L ) \circ \cdots \circ( \sigma \circ \bk_1 ) $ of the linear integral operator $\bk_\ell$ and the element-wise nonlinear activation function $\sigma$, can approximate any nonlinear continuous operator. Therefore, the addition of any local operation in Neural Operators does not affect the universality property, i.e., standard, and $\multiplicative$ NOs have the same universality property.

%%%%%%%%%%%%%%%%%%%%%%%%%%%%%%%%%%%%%%%%%%%%%%%%%%%%%%%%%

\subsection{Bochner integral}\label{Bochner}
\par
In the study of generalization error bounds, the Expected error, see Appendix~\ref{EmpiricalExpected}, is defined through the Bochner Integral. We briefly introduce it, informally, as the natural generalization of the Lebesgue integral on (separable) Banach spaces. \par  
For our purpose, it suffices to define the integral (informally) on $\Lp(D; \mathbb{R}^{d_a})\times \Lp(D; \mathbb{R}^{d_u})$. Assume that a function $(a,u) \mapsto \ f(a, u) \in \mathbb{R}$ is Bochner integrable with respect to the measure $\mu$ on $\Lp(D; \mathbb{R}^{d_a})\times \Lp(D; \mathbb{R}^{d_u})$, i.e., there exists a sequence of integrable simple functions $s_{n}$ (the finite linear combination of indicator functions of measurable sets) such that
\[
\lim_{n \to \infty} \int \left|
f(a, u) - s_{n}(a,u)\right|d\mu(a,u) = 0.
\]
%$\ell (\cdot, \cdot)$ shall be interpreted as the loss function. 
Thus, the Bochner Integral is defined by
\[
\int
\ell(\Gcal(a), u)\,d\mu(a,u) =
\lim_{n \to \infty} \int s_{n}(a,u)\,d\mu(a,u).
\]
For a detailed (formal) definition of the Bochner integral, as well as its properties, see \citet{yoshida1980functional}.

%%%%%%%%%%%%%%%%%%%%%%
\subsection{Gaussian measure} \label{appendix:grf}
\par
The typical choice of the measure $\mu$ in the context of PDEs is the Gaussian Measure, which will be reviewed as follows (refer to, e.g., \citet[Section 6]{stuart2010inverse}):
First, a function $m \in X$ is called the mean of $\mu$ if for all $\ell \in X^{\ast}$, where $X^{\ast}$ denote the dual space of linear functionals on $X$, 
$$
\ell(m)=\int_{X}\ell(x) \mu(dx).
$$
%where the integral in the right hand side is defined by Bocher integral (see Appendix~\ref{Bochner}) on $X$ with respect to $\mu$.
A linear operator $\boldsymbol{\mathcal{C}}: X^{\ast} \to X$ is called the Covariance Operator if for all $k, \ell \in X^{\ast}$,
$$
k(\boldsymbol{\mathcal{C}}\ell)=\int_{X}k(x-m)\ell(x-m) \mu(dx).
$$
We say that $u$ draws from {\it Gaussian Measure} $\mathcal{N}(m, \boldsymbol{\mathcal{C}})$ (write $u \sim \mathcal{N}(m, \boldsymbol{\mathcal{C}})$ ) if for all $\ell \in X^{\ast}$, $\ell(u)$ draws from the one-dimensional Gaussian distribution $\mathcal{N}(\ell(m), \ell(\boldsymbol{\mathcal{C}}\ell))$.
\par
If $X$ is a Hilbert space, then we can characterize random draws from a Gaussian Measure by using the {\it Karhunen-Lo\'{e}ve expansion} as follows (see, e.g., \citet[Theorem 6.19]{stuart2010inverse}):
\begin{theorem}
Let $X$ be a Hilbert space, and let $\boldsymbol{\mathcal{C}}:X \to X$ be a self-adjoint, positive semi-definite, compact operator, and let $m \in X$. 
Let $\{\phi_k, \gamma_k\}_{k=1}^{\infty}$ be an orthonormal set of eigenvectors and eigenvalues for $\boldsymbol{\mathcal{C}}$ ordered so that
$$
\gamma_1 \geq \gamma_2 \geq \cdots.
$$
Take $\{\xi_{k}\}_{k=1}^{\infty}$ to be an i.i.d. sequence with $\xi_1 \sim \mathcal{N}(0, 1)$.
Then, the random variable $u \in X$ given by the Karhunen-Lo\'{e}ve expansion 
\begin{equation}
u = m + \sum_{k=1}^{\infty}\sqrt{\gamma_k}\xi_k \phi_k \label{KL ex}
\end{equation}
draws from $\mathcal{N}(m, \boldsymbol{\mathcal{C}})$.
\end{theorem}
%%%%%%%%%%%%%%%%%%%%%%%%%%
\iffalse{
\begin{definition}[Support of a measure]\label{supp mu} We define the {\it support} $\mathrm{supp}(\mu) \subset X$ of $\mu$ on a Topological space $(X, \mathscr{T})$, as the set of points, $x$, such that for any open neighborhood, $\mathcal{U}_x\in \mathscr{T}$, containing $x$, has positive-valued measure, that is
$$\mathrm{supp}(\mu):= \left\{x\in X : \mu(\mathcal{U}_x)>0 \quad  \forall: \mathcal{U}_x \in \mathscr{T} \text{ and } x\in \mathcal{U}_x \right\}.$$
(Cf. \cite[Ch. 5]{ambrosio2005gradient})
\end{definition}

Let $X$ be a separable Banach space, and let $\mathscr{B}(X)$ be the family of the Borel sets of $X$, and let $\mu$ be a measure on $(X, \mathscr{B}(X))$.
\begin{definition}[Support of a measure]\label{supp mu} We define the {\it support} $\mathrm{supp}(\mu) \subset X$ of $\mu$ as the set of points, $x$, such that for any neighborhood containing $x$, has positive-valued measure, that is
$$\mathrm{supp}(\mu):= \left\{x\in X : 
\mu(\mathcal{U}_x)>0 \quad \text{for all open neighborhood  } \mathcal{U}_x \text{ of } x\right\}.$$
(Cf. \cite[Ch. 5]{ambrosio2005gradient})
\end{definition}
}
\fi
%%%%%%%%%%%%%%%%%%%%%%%%%%%%%%%%%%%%%%%%%%%%%%%%%%%%%%
%%%%%%%%%%%%%%%%%%%%%%
\subsection{Cameron-Martin space}\label{Cameron-Martin space-app}
We briefly review the definition of the Cameron-Martin space (refer to, e.g., \citet[Section 3.2.]{hairer2009introduction}). 
\begin{definition}
Let $\mu$ be a Gaussian Measure on a separable Banach space $X$.
The Cameron-Martin space $\mathcal{H}_{\mu}$ of $\mu$ is the completion of 
the linear subspace
\[
\{
h \in X : \exists h^{*} \in X^{*} \text{ with } 
C_{\mu}(h^{\ast}, \ell)=\ell(h) \ \forall \ell \in X^{*}
\},
\]
under the norm 
\[
\left\| h \right\|_{\mu}^{2}
=\langle h, h\rangle_{\mu}=C_{\mu}(h^{*}, h^{*}),
\]
where $C_{\mu}:X^{*} \times X^{*} \to \mathbb{R}$ is defined by
\[
C_{\mu}(k, \ell):=\int_{X} k(x)\ell(x) \mu(dx), \ \ k, \ell \in X^{*}.
\]
In can be shown that $\mathcal{H}_{\mu}$ is a reproducing kernel Hilbert space with the inner product 
$\langle h, k\rangle_{\mu}=C_{\mu}(h^{*}, k^{*})$.
\end{definition}
When $X$ is a finite-dimensional space, the Cameron-Martin space is given by the range of the covariance matrix  \citep[Exercise 3.28]{hairer2009introduction}. 
\par
\noindent
We now review properties of the Cameron-Martin space (see \citet[Theorem 3.41 and Proposition 3.4.2]{hairer2009introduction}).

\begin{theorem}
For $h \in X$, we define the map $T_{h}: X \to X$ by $T_{h}(x)=x+h$.
Then, the push-forward measure $T_{h \sharp} \mu$ of $\mu$ by $T_{h}$ is absolutely continuous with respect to $\mu$ if and only if $h \in \mathcal{H}_{\mu}$.
\end{theorem}

\begin{proposition}
The space $\mathcal{H}_{\mu} \subset B$ is the intersection of all (measurable) linear subspaces of full measure.
However, if $\mathcal{H}_{\mu}$ is infinite-dimensional, then one has $\mu(\mathcal{H}_{\mu})=0$.
\end{proposition}

That is, the Cameron-Martin space $\mathcal{H}_{\mu}$ of $\mu$ represents the directions in $X$ where translation is invariant, meaning that the translated measure has the same null sets as the original measure.
Furthermore, when $\mathrm{dim}(\mathcal{H}_{\mu})=\infty$,  $\mathcal{H}_{\mu}$ is "smaller" than $X$ in the sense that $\mu(\mathcal{H}_{\mu})=0$. 
In contrast, the finite-dimensional Lebesgue measure is invariant under translations in any direction. 
This is an illustration of the tendency for measures in infinite-dimensional spaces to be mutually singular.

%%%%%%%%%%%%%%%%%%%%%%%%%%%%%%%%%%%%%%%%%%%%%%%%%%%%%%%%
%%%%%%%%%%%%%%%%%%%%%%
\subsection{Gaussian random field}
\par
Let $(\Omega, \mathscr{F}, \mathbb{P})$ be a probability space.
We say that a function $u:D \times \Omega \to \mathbb{R}$ is a Gaussian Random Field (GRF) if $u(x, \cdot) \in L^{2}(\Omega)$, and for any $x_1,...,x_M \in D$ and any $M \in \mathbb{N}$, 
$$\boldsymbol{u}_{M}:=(u(x_1, \cdot),...,u(x_M,\cdot))^{\mathrm{T}}$$
draws from the multivariate Gaussian distribution $\mathcal{N}(\boldsymbol{m}_{M}, \boldsymbol{\mathcal{C}}_{M})$.
Here, $m(x):=\mathbb{E}_{\omega}[u(x, \omega)]$ is the mean function, and $c(x,y)=\mathbb{E}_{\omega}[(u(x, \omega)-m(x))(u(y, \omega)-m(y))^{\ast}]$ is the covariance function. 
We have denoted by $\boldsymbol{m}_{M}:=(m_1,...,m_{M})^{\mathrm{T}}$ and $\boldsymbol{\mathcal{C}}_{M}=(c_{ij})_{i,j=1}^{M}$, where
$m_i:=m(x_i)$, and $c_{ij}:=c(x_i, x_j)$.
The GRF also has the Karhunen-Lo\'{e}ve expansion with (\ref{KL ex}) as $X=L^{2}(D)$, $m$ is the mean function, and $\boldsymbol{\mathcal{C}}$ is the integral operator with the kernel given by the covariance function (see \citet[Theorem 7.52]{lord2014introduction}).
\par
We can construct the GRF drawing from a certain Gaussian Measure. 
We simply consider the Gaussian Measure $\mathcal{N}(0, (-\Delta)^{-\alpha})$ where $\Delta$ is the Laplacian with domain $H^{1}_{0}(D) \cap H^{2}(D)$ where $D=[0,1]^2$ and $\alpha>1$. 
Then, the draw $u$ from $\mathcal{N}(0, (-\Delta)^{-\alpha})$ are almost surely in $C(D)$ (see \citet[Example 6.28]{stuart2010inverse}), which means that the function $u$ can be point-wisely defined, and then, for any $x_1,...,x_M \in D$ and any $M \in \mathbb{N}$, 
$(u(x_1, \cdot),...,u(x_M,\cdot))^{\mathrm{T}}$ draws from the multivariate Gaussian distribution, that is, $u$ is the GRF.
%%%%%%%%%%%%%%%%%%%%%%%%%%%%%%%%%%%%%%%%%%%%%%%%%%%%%%%%
\subsection{Statistical learning} \label{EmpiricalExpected}
\begin{definition}[Expected Risk/Loss]The Expected risk is defined by 
%the Bochner Integral, see Appendix~\ref{Bochner}.
\[
\mathcal{L}(\Gcal):=\E_{(a,u)\sim \mu}\left[\ell(\Gcal(a),u) \right]=\int_{\mathrm{supp}(\mu)} \ell(\Gcal(a),u)\,\mu(d(a,u)),
\]
with respect to $\Gcal\in \Gscr$, where the set $\Gscr$ is the hypothesis class. For the purpose of this paper, the class corresponds to Neural Operators or $\multiplicative$ Neural Operators, and $\ell: \Lp(D; \mathbb{R}^{d_u}) \times \Lp(D; \mathbb{R}^{d_a}) \to [0, \infty)$ is the loss function.    
\end{definition}

\begin{definition}[Empirical Risk/Loss] It is defined as the unbiased estimator of the Expected risk, that is 
\[
\widehat{\mathcal{L}}_{S}(\Gcal):= \frac{1}{n}\sum_{i=1}^{n}\ell(\Gcal(a_i),u_i),
\]
where $(a_i, u_i) \overset{\text{i.i.d}}{\sim} \mu$.
%where $\mu$ is defined on $\Lp(D; \mathbb{R}^{d_a})\times \Lp(D; \mathbb{R}^{d_u})$.  
\end{definition}
\par
The generalization error $\mathcal{L}(\Gcal)$ is decomposed into $\widehat{\mathcal{L}}_{S}(\Gcal)$ and $\mathcal{L}(\Gcal)-\widehat{\mathcal{L}}_{S}(\Gcal)$. 
The difference, $\mathcal{L}(\Gcal)-\widehat{\mathcal{L}}_{S}(\Gcal)$ between the generalization and empirical errors is evaluated using the Uniform Laws of Large Numbers (see, e.g., \citep[Theorem 4.10]{wainwright} or \citep[Theorem 3.4.5]{Gine}).

\begin{lemma}[Uniform Laws of Large Numbers]\label{Uniform Laws of Large Numbers}
Let $\mathscr{F}$ be the set of real-valued measurable functions on  a measurable space $(S, \mathcal{S})$ with
absolute values bounded by $R$, let $X_i$ $(i \in \mathbb{N})$ be i.i.d, $S$-valued random variables with common probability law $\prob$, and let $\rad_i$ $(i \in \mathbb{N})$ be a sequence of  i.i.d Rademacher RVs, i.e., $\rad_i$ are independent, and $\prob\{\rad_{i}=1\}=1/2=\prob\{\rad_{i}=-1\}$.
Then, for all $n \in \mathbb{N}$ and $\delta > 0$, the following inequality holds with probability greater than $1-2\exp(-\delta)$,
\[
\sup_{f \in \mathscr{F}}\left|\dfrac{1}{n} \sum_{i=1}^{n}f(X_i) - \E[f(X)] \right|
\leq 
2 \mathfrak{R}^n_{\mathrm{S}}(\mathscr{F}) + R\sqrt{\dfrac{2\delta}{n}},
\] 
where $\mathfrak{R}^n_{\mathrm{S}}(\mathscr{F})$ is the Rademacher complexity of the class $\mathscr{F}$ defined above.
\end{lemma}
\par
The Rademacher complexity $\mathfrak{R}^n_{\mathrm{S}}(\mathscr{F})$ of the class $\mathscr{F}$ is defined as follows.
\begin{definition}\label{Rademacher}
(Rademacher complexity) Let $\mathscr{F}$ be the set of real-valued measurable functions on a measurable space $(S, \mathcal{S})$. Let $\{\rad_{i}\}_{i=1}^{n}$ is a sequence of i.i.d. RV's with Rademacher distribution; i.e., $\prob\{\rad_{i}=1\}=1/2=\prob\{\rad_{i}=-1\}$. The Rademacher Complexity of the class $\mathscr{F}$ is defined as 
\[  
\mathfrak{R}^n_{\mathrm{S}}(\mathscr{F}):=\E_{\rad \sim \mathrm{Rad}}\left[\sup_{f \in \mathscr{F}}\frac{1}{n}\left|\, \sum_{i=1}^{n}\rad_{i}f(a_i, u_i) \right| \,\right],
\]
(Cf. \citet[Definition 3.1.19 ]{Gine}).
\end{definition}
Intuitively, Rademacher  complexity $\mathfrak{R}^n_{\mathrm{S}}(\mathscr{F})$ measures richness of a class $\mathscr{F}$ of real-valued functions.
\begin{definition}[Covering number]\label{Covering Number}
Let $(\mathscr{F}, \left\| \cdot \right\|)$ be a normed vector space.
We define, $N(\varepsilon, \mathscr{F}, \|\cdot\|)$, the covering number of $\mathscr{F}$ (sometimes known as entropy number) which means 
the minimal cardinality of a subset $\mathcal{C} \subset \mathscr{F}$ that covers $\mathscr{F}$ at scale $\varepsilon$ with respect to the norm $\|\cdot\|$. 
\end{definition}

\begin{figure}[ht!]
    \centering
    \begin{tikzpicture}
  % Draw the domain G
  \filldraw[red!40, opacity=0.4] (0,0,0) -- (2,0,0) -- (2,2,0) -- (0,2,0) -- cycle;
  \filldraw[red!40, opacity=0.4] (0,0,0) -- (0,0,2) -- (0,2,2) -- (0,2,0) -- cycle;
  \filldraw[red!40, opacity=0.4] (0,0,0) -- (2,0,0) -- (2,0,2) -- (0,0,2) -- cycle;
  \filldraw[red!40, opacity=0.4] (0,0,2) -- (0,2,2) -- (2,2,2) -- (2,0,2) -- cycle;
  \filldraw[red!40, opacity=0.4] (0,2,0) -- (0,2,2) -- (2,2,2) -- (2,2,0) -- cycle;
  \filldraw[red!40, opacity=0.4] (2,0,0) -- (2,2,0) -- (2,2,2) -- (2,0,2) -- cycle;

  % Draw the balls
  \foreach \x/\y/\z in {0.5/0.5/0.5, 0.5/0.5/0.5, 1/1/1}{
    \shade[ball color=blue, opacity=0.15] ($(\x,\y,\z) + (0.2,0.2,0.2)$) circle (2) node {\small{$\varepsilon$-ball}};
  }
  % Add labels
  \node at (1.5,1.5,1.5) {$\mathscr{F}$};
\end{tikzpicture}
    \caption{Illustration of the covering number on $\mathscr{F}$. }
    \label{fig:covering_diagram}
\end{figure}

\noindent Roughly speaking, the covering number $N(\varepsilon, \mathscr{F}, \|\cdot\|)$ is the necessary number of $\varepsilon$-balls with respect to norm $ \|\cdot\|$ to completely cover a space $\mathscr{F}$ (see e.g., \citet[Definition 5.1]{wainwright}). 
Furthermore, it is possible to estimate Rademacher Complexity $\mathfrak{R}^n_{\mathrm{S}}(\mathscr{F})$ by using the {\it covering number}.
The following lemma is known as {\it Dudley’s Theorem} (see, e.g., \citet[Lemma A.5]{bartlett2017spectrally}).
\begin{lemma}[Dudley’s Theorem]\label{Dudley Theorem}
Let $\mathscr{F}$ be the set of real-valued functions. Then,
\[
\mathfrak{R}^n_{\mathrm{S}}(\mathscr{F}) \leq \inf_{\alpha \geq 0} \left\{ 4\alpha + \frac{12}{\sqrt{n}}\int_{\alpha}^{\infty}\sqrt{\log N(\varepsilon, \mathscr{F}, \left\| \cdot \right\|_{\mathrm{S}} )\,} \,d\varepsilon \right\}
\]
where $\left\| f \right\|_{\mathrm{S}}:= \left(\frac{1}{n}\sum_{i=1}^{n}f(X_i)^2 \right)^{1/2}$.
\end{lemma}

\par
\noindent The  main result in this paper is to apply these lemmas as $\mathscr{F}$ is the set of loss function $\ell(\Gcal(\cdot), \cdot)$ where $\Gcal$ is the class of Neural Operators or $\multiplicative$ Neural Operators.
Neural Operators $\Gcal$ are parameterized by weight matrices and Lipschitz continuous functions, and finally we will arrive at evaluating their covering number, which are referred to \citep{wainwright}.

\medskip

When we analyze the \emph{covering number} of Lipschitz continuous functions, the doubling dimension of $D\times D$ appears. We will now review the definition of the doubling dimension of a metric space (see, e.g., \citep{gupta2003bounded}). \par

% \citep{gupta2003bounded})

\begin{definition}[Doubling dimension]\label{doubling dimension}
A metric space $(\boldsymbol{\mathrm{X}}, \boldsymbol{\mathrm{d}})$ with metric $\boldsymbol{\mathrm{d}}$ is called doubling, if there exists a constant $M>0$ such that for any $x \in \boldsymbol{\mathrm{X}}$ and $r>0$, it is possible to cover the ball $B_{r}(x):=\{y \in X \ \vert  \ \boldsymbol{\mathrm{d}}(x,y)< r \}$ with the union of at most $M$ balls of radius $\frac{r}{2}$. The doubling dimension of $\boldsymbol{\mathrm{X}}$ is defined by $\mathrm{ddim}(\boldsymbol{\mathrm{X}})=\log_{2}(M)$.
\end{definition}

\begin{figure}[ht!]
    \centering
\begin{tikzpicture}
%Draw a b;ie corc;e 
%\draw[fill=blue!25, opacity=0.25] (0.8,0.) circle (0.5);
\draw[fill=blue, opacity=0.15] (0.,.80) circle (0.5);
\draw[fill=blue, opacity=0.15] (0,-.8) circle (0.5);
%\draw[fill=blue!25, opacity=0.25] (-0.8,0.) circle (0.5);
%%%%%%
\draw[fill=blue, opacity=0.15] (0,0) circle (0.5);
%%%%
\draw[fill=blue, opacity=0.15] (.7,0.40) circle (0.5);
\draw[fill=blue, opacity=0.15] (.7,-0.40) circle (0.5);
\draw[fill=blue, opacity=0.15] (-.7,0.40) circle (0.5);
\draw[fill=blue, opacity=0.15] (-.7,-0.40) circle (0.5);
% Draw a pink circle with 25% opacity
\fill[red!40, opacity=0.4] (0,0) circle (1.);
\end{tikzpicture}
    \caption{Ilustration of the doubling number. }
    \label{fig:doubling_diagram}
\end{figure}

\newpage

%\section{Proof of Lemma \ref{generalization error bound for general setting}}\label{Appendix1}
%\input{Appendix/1.tex}
% \newpage
%%%%%%%%%%%%%%%%%%%%%%%%%%%%%%%%%%%%%%%%%%%%%%%%%%%%%%%%%%%%%%%%%

%\section{Proof of Theorem \ref{Rademacher for standard NO}}
%\label{Appendix2}
%\input{Appendix/2.tex}
%%%%%%%%%%%%%%%%%%%%%%%%%%%%%%%%%%%%%%%%%%%%%%%%%%%%%%%%%%%%%%%%%
% \newpage
\section{Proof in Sections~\ref{GEB}.}
\label{Appendix3}

%\subsection{\textbf{Proofs of the results of this section}.}
%\label{s:GEB__ss:Proofs}

%%%%%%%%%%%%%%%%%%%%%%%%%%%%%%
\subsection{Proof of Lemma~\ref{generalization error bound for general setting}}\label{Proof-lamma-1}
\begin{proof}
By using (\ref{Generalization error bound}), we have for $f=\ell(\Gcal(\cdot), \cdot) \in \mathscr{F}_{\Gscr}$ and $\Gcal\in \Gscr$, 
\begin{align}
\nonumber
& |f(a,u)| \leq |\ell(\Gcal(a), u)-\ell(0, u)|+|\ell(0, u)|
\\
\label{A2used}
& 
\le
\rho \left\|\Gcal(a) \right\|_{\Lp(D; \mathbb{R}^{d_u})} + R_u \leq \rho R + R_u,
\end{align}
for $(a,u) \in \Lp(D;\mathbb{R}^{d_a})\times \Lp(D;\mathbb{R}^{d_u})$, where~\eqref{A2used} followed from Assumption~\ref{Assumption1} (i) and (ii).  
This implies that by employing \citet[Theorem 4.10]{wainwright} or \citet[Theorem 3.4.5]{Gine}, we have the following inequality with probability greater than $1-2e^{-\delta}$,
\[
\begin{split}
|\mathcal{L}(\Gcal)-\widehat{\mathcal{L}}_{\mathrm{S}}(\Gcal)| & \leq \sup_{f \in \mathscr{F}_{\Gscr}} \left|\frac{1}{n}\sum_{i=1}^{n}f(a_i, u_i) - \mathbb{E}_{(a,u)\sim \mu}[f(a,u)] \right|
\\
&
\leq 2 \mathbb{E}_{\rad}\left[ \sup_{f \in \mathscr{F}_{\Gscr}} \frac{1}{n}\left|\sum_{i=1}^{n}\rad_{i} f(a_i, u_i) \right| \right]+(\rho R +R_u) \sqrt{\frac{2\delta}{n}}, \ \Gcal \in \Gscr,
\end{split}
\]
where $\{\rad_{i}\}_{i=1}^{n}$ is a sequence of i.i.d. RV's with Rademacher distribution; i.e., $\prob\{\rad_{i}=1\}=1/2=\prob\{\rad_{i}=-1\}$.
\end{proof}

\subsection{Proof of Theorem~\ref{Rademacher for standard NO}}
\label{main-theorem-1}
\begin{proof}
By employing \citet[Lemma A.5]{bartlett2017spectrally} or \citet[Theorem 1.1]{Kakade}, we have
\begin{equation}
\mathfrak{R}^n_{\mathrm{S}}(\mathscr{F}_{\no}) \leq \inf_{\alpha \geq 0} \left\{ 4\alpha + \frac{12}{\sqrt{n}}\int_{\alpha}^{\infty}\left(\log N(\varepsilon, \mathscr{F}_{\no}, \left\| \cdot \right\|_{\mathrm{S}} )  \right)^{\frac{1}{2}} d\varepsilon \right\},
\label{Rademacher Complexity-covering number}
\end{equation}
where $\left\| f \right\|_{\mathrm{S}}:= \left(\frac{1}{n}\sum_{i=1}^{n}f(a_i, u_i)^2 \right)^{\frac{1}{2}}$.
Here, we denote by $N(\varepsilon, \mathscr{F}, \|\cdot\|)$ the covering number of the function space $\mathscr{F}$ which means 
the minimal cardinality of a subset $C \subset \mathscr{F}$ that covers $\mathscr{F}$ at scale $\varepsilon$ with respect to the norm $\|\cdot\|$. 
In the following, we will estimate the covering number $N(\varepsilon, \mathscr{F}_{\no}, \left\| \cdot \right\|_{\mathrm{S}} )$.
\par
Let $f=\ell(\Gcal(\cdot), \cdot)$ and $f^{\prime}=\ell(\Gcal^{\prime}(\cdot), \cdot)$ where $\Gcal, \Gcal^{\prime} \in \no$.
By (i) of Assumption \ref{Assumption1}, we calculate 
\begin{equation}
\left|f(a,u)-f^{\prime}(a,u)\right| 
= \left|\ell(\Gcal(a),u)- \ell(\Gcal^{\prime}(a),u)\right|
\leq \rho \left\| \Gcal(a) -\Gcal^{\prime}(a) \right\|_{\Lp(D;\mathbb{R}^{d_u})}. \label{covering number-estimation 1}
\end{equation}
Denoting by
\[
\Gcal_{\ell}:=(W_{\ell}+\bk_{\ell})\circ \sigma (W_{\ell-1}+\bk_{\ell-1})\circ \cdots \circ \sigma (W_{0}+\bk_{0}),
\]
\[
\Gcal^{\prime}_{\ell}:=(W^{\prime}_{\ell}+\bk^{\prime}_{\ell})\circ \sigma (W^{\prime}_{\ell-1}+\bk^{\prime}_{\ell-1})\circ \cdots \circ \sigma (W^{\prime}_{0}+\bk^{\prime}_{0}),
\]
the quantity $\left\| \Gcal(a) -\Gcal^{\prime}(a) \right\|_{\Lp(D;\mathbb{R}^{d_u})}$ is evaluated by
\begin{equation}
\begin{split}
&
\left\| \Gcal(a) -\Gcal^{\prime}(a) \right\|_{\Lp(D;\mathbb{R}^{d_u})} = \left\| \Gcal_{L}(a) -\Gcal^{\prime}_{L}(a) \right\|_{\Lp(D;\mathbb{R}^{d_{L+1}})}
\\
&
= \Bigl\|(W_{L}+\bk_{L})\circ \sigma( \Gcal_{L-1}(a)) - (W_{L}+\bk_{L})\circ \sigma( \Gcal_{L-1}^{\prime}(a))
\\
& \hspace{2cm}
+(W_{L}+\bk_{L})\circ \sigma(\Gcal_{L-1}^{\prime}(a)) - (W^{\prime}_{L}+\bk^{\prime}_{L})\circ  \sigma(\Gcal_{L-1}^{\prime}(a)) \Bigr\|_{\Lp(D;\mathbb{R}^{d_{L+1}})}
\\
&
\underset{Assumption\ref{Assumption2}(vi)}{}
\\
&
\leq 
\Biggl(
\underbracket{\left\| W_{L} \right\|_{\mathrm{op}}}_{\underset{(\ref{covering number-estimation 3})}{\leq} C_w}
+ 
\underbracket{\left\| \bk_{L} \right\|_{\mathrm{op}}}_{\underset{(\ref{covering number-estimation 4})}{\leq} C_k}
\Biggr) C_{\sigma} \left\|\Gcal_{L-1}(a)-\Gcal^{\prime}_{L-1}(a)\right\|_{\Lp(D;\mathbb{R}^{d_L})}
\\
& \hspace{2cm}
+\left(\left\| W_{L}-W^{\prime}_{L} \right\|_{\mathrm{op}}
+ \left\| \bk_{L}-\bk^{\prime}_{L} \right\|_{\mathrm{op}} \right) C_{\sigma}
\left\|\Gcal^{\prime}_{L-1}(a)\right\|_{\Lp(D;\mathbb{R}^{d_L})},
\end{split}
\label{covering number-estimation 2}
\end{equation}
where $\left\|\cdot \right\|_{\mathrm{op}}$ is the Operator norm.
Here, we have employed the following estimations:
\begin{equation}
\left\| W_{L}g \right\|^{2}_{\Lp(D;\mathbb{R}^{d_{L+1}})}\leq 
\int_{D}
\underbracket{
\left\|W_{L}g(x) \right\|^{2}_{2}
}_{\leq \left\|W_{L}\right\|^{2}_{\mathrm{F}}\left\|g(x)\right\|^{2}_{2}}
dx 
\underset{\text{Assumption} \ref{Assumption2}(i)}{\leq} C^{2}_{w}
\left\|g \right\|^{2}_{\Lp(D;\mathbb{R}^{d_{L}})},
\label{covering number-estimation 3}
\end{equation}
\begin{equation}
\begin{split}
\left\| \bk_{L}g \right\|^{2}_{\Lp(D;\mathbb{R}^{d_{L+1}})}
\leq 
\int_{D} 
\underbracket{\left\|\int_{D}
\bk_{L}(x,y)g(y)dy \right\|^{2}_{2}
}_{\leq 
\left(\sum_{i,j}
\left\|k_{L,ij}(x,\cdot)\right\|^{2}_{\Lp(D)} \right) \left\|g\right\|^{2}_{\Lp(D;\mathbb{R}^{d_{L}}) }}
dx 
& \leq \left\|\bk_{L}\right\|_{\Lp, \mathrm{F}}^{2}
\left\|g \right\|^{2}_{\Lp(D;\mathbb{R}^{d_{L}}) }
\\
&\hspace{-1cm}
\underset{\text{Assumption} \ref{Assumption2}(ii)}{\leq}
C^{2}_{k}\left\|g \right\|^{2}_{\Lp(D;\mathbb{R}^{d_{L}})},
\label{covering number-estimation 4}
\end{split}
\end{equation}
for $g \in \Lp(D; \mathbb{R}^{d_L})$,
where $\left\| \cdot \right\|_{2}$ is the $\ell_2$-norm.
By the same argument in (\ref{covering number-estimation 2})--(\ref{covering number-estimation 4}), we evaluate
\begin{equation}
\begin{split}
& \left\|\Gcal_{L-1}(a)-\Gcal^{\prime}_{L-1}(a)\right\|_{\Lp(D;\mathbb{R}^{d_L})}
\\
&
\leq (C_w +C_k)C_{\sigma}\left\|\Gcal_{L-2}(a)-\Gcal^{\prime}_{L-2}(a)\right\|_{\Lp(D;\mathbb{R}^{d_{L-1}})}
\\
& \hspace{0cm}
+\left(\left\| W_{L-1}-W^{\prime}_{L-1} \right\|_{\mathrm{op}}
+ \left\| \bk_{L-1}-\bk^{\prime}_{L-1} \right\|_{\mathrm{op}} \right) C_{\sigma}
\left\|\Gcal^{\prime}_{L-2}(a)\right\|_{\Lp(D;\mathbb{R}^{d_{L-1}})}.
\end{split}
\label{covering number-estimation 5}
\end{equation}
By repeatedly evaluating $\left\|\Gcal_{\ell}(a)-\Gcal^{\prime}_{\ell}(a)\right\|_{\Lp(D;\mathbb{R}^{d_{\ell+1}})}$ ($\ell=L,L-1,...,0$), we obtain 
\begin{equation}
\begin{split}
& \left\|\Gcal(a)-\Gcal^{\prime}(a)\right\|_{\Lp(D;\mathbb{R}^{d_u})}
\\
&
\leq \left\{(C_w +C_k)C_{\sigma}\right\}^{L}
\underbracket{
\left\|(W_0 + \bk_0)(a)-(W^{\prime}_0 + \bk^{\prime}_0)(a)\right\|_{\Lp(D;\mathbb{R}^{d_{1}})}}
_{
\underset{\text{Assumption} \ref{Assumption2}(iii)}{}\leq C_a (\left\| W_{0}-W^{\prime}_{0} \right\|_{\mathrm{op}}
+ \left\| \bk_{0}-\bk^{\prime}_{0} \right\|_{\mathrm{op}})}
\\
& \hspace{0cm}
+\sum_{\ell=0}^{L-1}\left(\left\| W_{\ell+1}-W^{\prime}_{\ell+1} \right\|_{\mathrm{op}}
+ \left\| \bk_{\ell+1}-\bk^{\prime}_{\ell+1} \right\|_{\mathrm{op}} \right) 
\\
&\hspace{3cm}
\times \underbracket{(C_w + C_k)^{L-1-\ell} C^{L-\ell}_{\sigma}
\left\|\Gcal^{\prime}_{\ell}(a)\right\|_{\Lp(D;\mathbb{R}^{d_{\ell+1}})}}
_{\underset{(\ref{covering number-estimation 7})}{}\leq \left\{(C_w +C_k)C_{\sigma}\right\}^{L}C_a \ }
\\
&
\leq \left\{(C_w+C_k)C_{\sigma}\right\}^{L}C_a \sum_{\ell=0}^{L}\left(\left\| W_{\ell}-W^{\prime}_{\ell} \right\|_{\mathrm{op}}
+ \left\| \bk_{\ell}-\bk^{\prime}_{\ell} \right\|_{\mathrm{op}} \right).
\end{split}
\label{covering number-estimation 6}
\end{equation}
Here, we have employed the following estimation:
\begin{equation}
\left\| \Gcal^{\prime}_{\ell} \right\|^{2}_{\Lp(D;\mathbb{R}^{d_{\ell+1}})}
\leq 
(C_w+C_k)^{\ell+1}C_{\sigma}^{\ell}C_a.
\label{covering number-estimation 7}
\end{equation}
Furthermore, by using ideas of (\ref{covering number-estimation 3}) and (\ref{covering number-estimation 4}), we estimate 
\begin{equation}
\begin{split}
\left\| W_{\ell}-W^{\prime}_{\ell} \right\|_{\mathrm{op}} & \leq \left\| W_{\ell}-W^{\prime}_{\ell} \right\|_{\mathrm{F}}
\\
&
\leq
\sum_{j=1}^{d_{\ell+1}}\sum_{i=1}^{d_{\ell}} |w_{\ell, ij}-w^{\prime}_{\ell, ij}|
\leq
\sum_{j=1}^{d_{\ell+1}}\sum_{i=1}^{d_{\ell}}C_w \left|\frac{w_{\ell, ij}}{C_w}-\frac{w^{\prime}_{\ell, ij}}{C_w}\right|,
\end{split}
\label{covering number-estimation 8}
\end{equation}
\begin{equation}
\begin{split}
&
\left\| \bk_{\ell}-\bk^{\prime}_{\ell} \right\|_{\mathrm{op}} 
%\leq
%|D|
%\left\| k_{\ell}-k^{\prime}_{\ell} \right\|_{\mathrm{F}, \infty}
%\\
%&
\leq
\sum_{j=1}^{d_{\ell+1}}\sum_{i=1}^{d_{\ell}}
|D| \left\|k_{\ell, ij}-k^{\prime}_{\ell, ij}\right\|_{\boldsymbol{\mathrm{L}}^{\infty}(D\times D;\mathbb{R})}
\leq
\sum_{j=1}^{d_{\ell+1}}\sum_{i=1}^{d_{\ell}} |D| C_{\alpha} \left\|\frac{k_{\ell, ij}}{C_{\alpha}}-\frac{k^{\prime}_{\ell, ij}}{C_{\alpha}}\right\|_{\boldsymbol{\mathrm{L}}^{\infty}(D\times D;\mathbb{R})}.
\end{split}
\label{covering number-estimation 9}
\end{equation}
Combining (\ref{covering number-estimation 1}), (\ref{covering number-estimation 6}), (\ref{covering number-estimation 8}), and (\ref{covering number-estimation 9}), the norm $\left\|f -f^{\prime} \right\|_{\mathrm{S}}$ is estimated by
\[
\begin{split}
&\left\|f -f^{\prime}\right\|_{\mathrm{S}}
=
\left(\frac{1}{n}\sum_{i=1}^{n}|f(a_i,u_i) -f^{\prime}(a_i,u_i)|^2 \right)^{\frac{1}{2}}
\\
&
\leq
\sum_{\ell=0}^{L}
\sum_{j=1}^{d_{\ell+1}}
\sum_{i=1}^{d_{\ell}}
\Biggl[
\rho
\left\{(C_w+C_k)C_{\sigma}
\right\}^{L}C_a C_w 
\left|\frac{w_{\ell, ij}}{C_w}-\frac{w^{\prime}_{\ell, ij}}{C_w}\right|
\\
&
\hspace{3cm}
+
\rho
\left\{(C_w+C_k)C_{\sigma}
\right\}^{L}C_a |D| C_{\alpha} 
\left\|\frac{k_{\ell, ij}}{C_{\alpha}}-\frac{k^{\prime}_{\ell, ij}}{C_{\alpha}}\right\|_{\boldsymbol{\mathrm{L}}^{\infty}(D\times D;\mathbb{R})}
\Biggr],
\end{split}
\]
which implies that we have
\begin{equation}
\begin{split}
&N(\varepsilon, \mathscr{F}_{\no}, \left\| \cdot \right\|_{\mathrm{S}} )
\\
&
\leq 
\prod_{\ell=0}^{L}\prod_{j=1}^{d_{\ell+1}}
\prod_{i=1}^{d_{\ell}}
N\left(\frac{\varepsilon}
{2\left( \sum_{\ell=0}^{L}d_{\ell}d_{\ell+1} \right)
\rho\left\{(C_w+C_k)C_{\sigma}
\right\}^{L}C_a C_w}, [-1,1], |\cdot|
\right)
\\
&
\times
N\left(\frac{\varepsilon}
{2\left( \sum_{\ell=0}^{L}d_{\ell}d_{\ell+1} \right)
\rho\left\{(C_w+C_k)C_{\sigma}
\right\}^{L}C_a |D| C_{\alpha}}, \mathcal{F}_{C_{\beta}}, \left\|\cdot\right\|_{\boldsymbol{\mathrm{L}}^{\infty}(D\times D; \mathbb{R})}
\right),
\end{split}
\label{covering number-estimation 10}
\end{equation}
where $\mathcal{F}_{C_{\beta}}:=\{k:D \times D \to [-1,1]\ |\ k\ \text{is}\ C_{\beta}-Lipschitz \}$ (see (vi) in Assumption \ref{Assumption2}).
\par
By taking logarithmic functions in (\ref{covering number-estimation 10}), we have 
\begin{equation}
\begin{split}
&\log N(\varepsilon, \mathscr{F}_{\no}, \left\| \cdot \right\|_{\mathrm{S}} )
\\
&
\leq 
\left( \sum_{\ell=0}^{L}d_{\ell}d_{\ell+1} \right)
\Biggl\{
\underbracket{
\log
N\left(\frac{\varepsilon}
{2\left( \sum_{\ell=0}^{L}d_{\ell}d_{\ell+1} \right)
\rho\left\{(C_w+C_k)C_{\sigma}
\right\}^{L}C_a C_w}, [-1,1], |\cdot|
\right)
}_{=:H_w(\varepsilon)}
\\
&\hspace{1cm}
+
\underbracket{
\log
N\left(\frac{\varepsilon}
{2\left( \sum_{\ell=0}^{L}d_{\ell}d_{\ell+1} \right)
\rho\left\{(C_w+C_k)C_{\sigma}
\right\}^{L}C_a |D| C_{\alpha}}, \mathcal{F}_{C_{\beta}}, \left\|\cdot\right\|_{\boldsymbol{\mathrm{L}}^{\infty}(D\times D; \mathbb{R})}
\right)}_{=:H_k(\varepsilon)}
\Biggr\}.
\end{split}
\label{log-covering}
\end{equation}
By using \citet[Example 5.3]{wainwright} and \citet[Lemmas 2.1 and 4.2]{gottlieb2016adaptive}, we estimate $H_w$ and $H_k$, respectively as follows:
\begin{equation}
\begin{split}
H_w(\varepsilon) & \leq 
\log \left(1+ \frac{2\left( \sum_{\ell=0}^{L}d_{\ell}d_{\ell+1} \right)
\rho\left\{(C_w+C_k)C_{\sigma}
\right\}^{L}C_a C_w}{\varepsilon}\right)
\\
&
\leq
\left( \frac{I_w}{\varepsilon} \right) \leq
\left( \frac{I_w}{\varepsilon} \right)^{\hat{d}+1}, \ \ \mbox{for} \ 0< \varepsilon < 2\left( \sum_{\ell=0}^{L}d_{\ell}d_{\ell+1} \right)
\rho\left\{(C_w+C_k)C_{\sigma}
\right\}^{L}C_a C_w,
\end{split}
\label{estimation H_w}
\end{equation}
%%%%%%%%%%%%%%%%%%%%%%%%%%%%%%%%%%%%%%%%%
\begin{equation}
\begin{split}
& H_k(\varepsilon) \leq 
\left(\frac{
8C_{\beta}\mathrm{diag}(D \times D)
\left( \sum_{\ell=0}^{L}d_{\ell}d_{\ell+1} \right)
\rho\left\{(C_w+C_k)C_{\sigma}
\right\}^{L}C_a |D| C_{\alpha}
}{\varepsilon} \right)^{\hat{d}}
\\
&
\times
\log \left(\frac{
16 \left( \sum_{\ell=0}^{L}d_{\ell}d_{\ell+1} \right)
\rho\left\{(C_w+C_k)C_{\sigma}
\right\}^{L}C_a |D| C_{\alpha}
}{\varepsilon}\right)
\\
&
\leq
\left(\frac{I_k}{\varepsilon} \right)^{\hat{d}+1},
\ \ \mbox{for} \ 0 < \varepsilon < 2\left( \sum_{\ell=0}^{L}d_{\ell}d_{\ell+1} \right)
\rho\left\{(C_w+C_k)C_{\sigma}
\right\}^{L}C_a |D| C_{\alpha},
\end{split}
\label{estimation H_k}
\end{equation}
where we denoted by
\[
I_w:=2\left( \sum_{\ell=0}^{L}d_{\ell}d_{\ell+1} \right)
\rho\left\{(C_w+C_k)C_{\sigma}
\right\}^{L}C_a C_w,
\]
\[
I_k:=8\max\left\{C_{\beta}\mathrm{diag}(D\times D), 2 \right\}\left( \sum_{\ell=0}^{L}d_{\ell}d_{\ell+1} \right)
\rho\left\{(C_w+C_k)C_{\sigma}
\right\}^{L}C_a |D| C_{\alpha}.
\]
By employing (\ref{log-covering}), (\ref{estimation H_w}), and (\ref{estimation H_k}), we calculate
\[
\int_{\alpha}^{\infty}\left(\log N(\varepsilon, \mathscr{F}_{\no}, \left\| \cdot \right\|_{\mathrm{S}} )  \right)^{\frac{1}{2}} d\varepsilon
\leq
\left( \sum_{\ell=0}^{L}d_{\ell}d_{\ell+1} \right)^{\frac{1}{2}}
\underbracket{
\int_{\alpha}^{\infty}(H_w(\varepsilon)+H_k(\varepsilon))^{\frac{1}{2}}d\varepsilon
}
_{=:(\ast)}
\]
\[
\begin{split}
& (\ast) \leq
\int_{\alpha}^{\infty}
H_w(\varepsilon)^{\frac{1}{2}} d\varepsilon
+
\int_{\alpha}^{\infty}
H_w(\varepsilon)^{\frac{1}{2}} d\varepsilon
\\
&
\leq
\int_{\alpha}^{2\left( \sum_{\ell=0}^{L}d_{\ell}d_{\ell+1} \right)
\rho\left\{(C_w+C_k)C_{\sigma}
\right\}^{L}C_a C_w}\left( \frac{I_w}{\varepsilon} \right)^{\frac{\hat{d}+1}{2}}d\varepsilon
\\
& \hspace{3cm}
+\int_{\alpha}^{2\left( \sum_{\ell=0}^{L}d_{\ell}d_{\ell+1} \right)
\rho\left\{(C_w+C_k)C_{\sigma}
\right\}^{L}C_a |D| C_{\alpha}}\left( \frac{I_k}{\varepsilon} \right)^{\frac{\hat{d}+1}{2}}d\varepsilon
\\
&
\leq
\left(I_{w}^{\frac{\hat{d}+1}{2}}
+
I_{k}^{\frac{\hat{d}+1}{2}}
\right)
\frac{2}{\hat{d}-1}\alpha^{-\frac{\hat{d}-1}{2}} 
\\
&
\leq
\frac{4}{\hat{d}-1}
\left(
\max\left[2C_w, 8|D| C_{\alpha}\max\left\{C_{\beta} \mathrm{diag}(D\times D), 2\right\} \right]\left( \sum_{\ell=0}^{L}d_{\ell}d_{\ell+1} \right)
\rho\left\{(C_w+C_k)C_{\sigma}
\right\}^{L}C_a
\right)^{\frac{\hat{d}+1}{2}}
\alpha^{-\frac{\hat{d}-1}{2}},
\end{split}
\]
that is, we have by (i) of Assumption \ref{Assumption2}.
\[
\begin{split}
&
\int_{\alpha}^{\infty}\left(\log N(\varepsilon, \mathscr{F}_{\no}, \left\| \cdot \right\|_{\mathrm{S}} )  \right)^{\frac{1}{2}} d\varepsilon
\\
&
\leq
\underbracket{
\frac{4}{\hat{d}-1}
\left(
\max\left[2C_w, 8|D| C_{\alpha}\max\left\{C_{\beta} \mathrm{diag}(D\times D), 2\right\} \right]
(LC_{d}^{2})^{\frac{\hat{d}+2}{\hat{d}+1}}
\rho\left\{(C_w+C_k)C_{\sigma}
\right\}^{L}C_a
\right)^{\frac{\hat{d}+1}{2}}
}
_{=:K} \ 
\alpha^{-\frac{\hat{d}-1}{2}}
\end{split}
\]
which implies that we conclude that with (\ref{Rademacher Complexity-covering number})
\[
\begin{split}
&
\mathfrak{R}^n_{\mathrm{S}}(\mathscr{F}_{\no}) \leq 4 \inf_{\alpha \geq 0} \left\{ \alpha
+
\underbracket{
\frac{3K}{\sqrt{n}}
}_{=:K^{\prime}}
\alpha^{-\frac{\hat{d}-1}{2}}\right\}
\\
&
=
4\left( \left(\frac{(\hat{d}-1)K^{\prime}}{2} \right)^{\frac{2}{\hat{d}+1}} 
+
K^{\prime}
\left(\frac{(\hat{d}-1)K^{\prime}}{2} \right)^{\frac{2}{\hat{d}+1} \left(-\frac{\hat{d}-1}{2}\right)}
\right)
=\gamma L^{\frac{\hat{d}+2}{\hat{d}+1}} \left\{(C_w+C_k)C_{\sigma}
\right\}^{L}\left(\frac{1}{n}\right)^{\frac{1}{\hat{d}+1}}
\end{split}
\]
where $\gamma$ is the positive constant defined by
\begin{equation}
\begin{split}
\gamma
&
:=4
\left\{
\left(\frac{\hat{d}-1}{2} \right)^{\frac{2}{\hat{d}+1}}
+
\left(\frac{\hat{d}-1}{2} \right)^{-\frac{\hat{d}-1}{\hat{d}+1}}
\right\}
\left(\frac{12}{\hat{d}-1} \right)^{\frac{2}{\hat{d}+1}}
\\
& \hspace{3cm}
\times
\max\left[2C_w, 8|D| C_{\alpha}\max\left\{C_{\beta} \mathrm{diag}(D\times D), 2\right\} \right]
C_{d}^{\frac{2(\hat{d}+2)}{\hat{d}+1}}
\rho C_a.
\label{definition of gamma}
\end{split}
\end{equation}
\end{proof}

\subsection{Proof of Theorem~\ref{Rademacher for multiplicative NO}}
\label{main-theorem-2}
\begin{proof}
The following argument is almost same with the proof of Theorem \ref{Rademacher for standard NO}.

By employing \citet[Lemma A.5]{bartlett2017spectrally} or \citet[Theorem 1.1]{Kakade}, we have 
\begin{equation}
\mathfrak{R}^n_{\mathrm{S}}(\mathscr{F}_{\widetilde{\no}}) \leq \inf_{\alpha \geq 0} \left\{ 4\alpha + \frac{12}{\sqrt{n}}\int_{\alpha}^{\infty}\left(\log N(\varepsilon, \mathscr{F}_{\widetilde{\no}}, \left\| \cdot \right\|_{\mathrm{S}} )  \right)^{\frac{1}{2}} d\varepsilon \right\}
\label{Rademacher Complexity-covering number tilde}
\end{equation}
In the following, we will estimate the covering number $N(\varepsilon, \mathscr{F}_{\widetilde{\no}}, \left\| \cdot \right\|_{\mathrm{S}} )$.
\par
Let $f=\ell(\Gcal(\cdot), \cdot)$ and $f^{\prime}=\ell(\Gcal^{\prime}(\cdot), \cdot)$ where $\Gcal, \Gcal^{\prime} \in \widetilde{\no}$.
By (i) of Assumption \ref{Assumption1}, we calculate 
\begin{equation}
\left|f(a,u)-f^{\prime}(a,u)\right| 
= \left|\ell(\Gcal(a),u)- \ell(\Gcal^{\prime}(a),u)\right|
\leq \rho \left\| \Gcal(a) -\Gcal^{\prime}(a) \right\|_{\Lp(D;\mathbb{R}^{d_u})}. \label{covering number-estimation 1 tilde}
\end{equation}
Denoting by
\[
\Gcal_{\ell}:=(\boldsymbol{\mathrm{Z}}_{\ell}\bid + \boldsymbol{\mathrm{X}}_{\ell} f_{\ell}) \circ (\boldsymbol{\mathrm{Z}}_{\ell}\bid + \boldsymbol{\mathrm{X}}_{\ell} \sigma \circ \bk_{\ell}) \circ  \cdots \circ (\boldsymbol{\mathrm{Z}}_{0}\bid + \boldsymbol{\mathrm{X}}_{0} f_{0}) \circ (\boldsymbol{\mathrm{Z}}_{0}\bid + \boldsymbol{\mathrm{X}}_{0} \sigma \circ \bk_{0}),
\]
\[
\Gcal^{\prime}_{\ell}:=(\boldsymbol{\mathrm{Z}}_{\ell}\bid + \boldsymbol{\mathrm{X}}_{\ell} f^{\prime}_{\ell}) \circ (Z_{\ell}\bid + \boldsymbol{\mathrm{X}}_{\ell} \sigma \circ \bk^{\prime}_{\ell}) \circ  \cdots \circ (\boldsymbol{\mathrm{Z}}_{0}\bid + \boldsymbol{\mathrm{X}}_{0} f^{\prime}_{0}) \circ (Z_{0}\bid + \boldsymbol{\mathrm{X}}_{0} \sigma \circ \bk^{\prime}_{0}),
\]
the quantity $\left\| \Gcal(a) -\Gcal^{\prime}(a) \right\|_{\Lp(D;\mathbb{R}^{d_u})}$ is evaluated by
\begin{equation}
\begin{split}
&
\left\| \Gcal(a) -\Gcal^{\prime}(a) \right\|_{\Lp(D;\mathbb{R}^{d_u})} = \left\| \Gcal_{L}(a) -\Gcal^{\prime}_{L}(a) \right\|_{\Lp(D;\mathbb{R}^{d_{L+1}})}
\\
&
= \Bigl\|(\boldsymbol{\mathrm{Z}}_{L}\bid + \boldsymbol{\mathrm{X}}_{L} f_{L})\circ (\boldsymbol{\mathrm{Z}}_{L}\bid + \boldsymbol{\mathrm{X}}_{L} \sigma \circ \bk_{L}) (\Gcal_{L-1}(a)) - (\boldsymbol{\mathrm{Z}}_{L}\bid + \boldsymbol{\mathrm{X}}_{L} f_{L})\circ (\boldsymbol{\mathrm{Z}}_{L}\bid + \boldsymbol{\mathrm{X}}_{L} \sigma \circ \bk_{L}) (\Gcal_{L-1}^{\prime}(a))
\\
& \hspace{0cm}
+(\boldsymbol{\mathrm{Z}}_{L}\bid + \boldsymbol{\mathrm{X}}_{L} f_{L}) \circ (\boldsymbol{\mathrm{Z}}_{L}\bid + \boldsymbol{\mathrm{X}}_{L} \sigma \circ \bk_{L}) (\Gcal_{L-1}^{\prime}(a)) 
- (\boldsymbol{\mathrm{Z}}_{L}\bid + \boldsymbol{\mathrm{X}}_{L} f^{\prime}_{L}) \circ (\boldsymbol{\mathrm{Z}}_{L}\bid + \boldsymbol{\mathrm{X}}_{L} \sigma \circ \bk^{\prime}_{L}) (\Gcal_{L-1}^{\prime}(a)) \Bigr\|%_{\Lp(D;\mathbb{R}^{d_{L+1}})}
\\
&
\underset{Assumption\ref{Assumption3}(i)(ii)(iv)}{}
\\
&
\leq 
(\boldsymbol{\mathrm{Z}}_{L}+\boldsymbol{\mathrm{X}}_{L}C_{w}^{M+1}C_{\sigma}^{M})(\boldsymbol{\mathrm{Z}}_{L}+\boldsymbol{\mathrm{X}}_{L}C_k C_{\sigma} )
 \left\|\Gcal_{L-1}(a)-\Gcal^{\prime}_{L-1}(a)\right\|_{\Lp(D;\mathbb{R}^{d_L})}
\\
& 
+ \left( (\boldsymbol{\mathrm{Z}}_{L} + \boldsymbol{\mathrm{X}}_{L} C_{\sigma}C_{k})\boldsymbol{\mathrm{X}}_{L} \left\| f_{L}-f^{\prime}_{L} \right\|_{\mathrm{op}}
+ (\boldsymbol{\mathrm{Z}}_{L}+\boldsymbol{\mathrm{X}}_{L}C_{w}^{M+1}C_{\sigma}^{M})
\boldsymbol{\mathrm{X}}_{L}
\left\| \bk_{L}-\bk^{\prime}_{L} \right\|_{\mathrm{op}} \right) 
\left\|\Gcal^{\prime}_{L-1}(a)\right\|_{\Lp(D;\mathbb{R}^{d_L})}.
\end{split}
\label{covering number-estimation 2 tilde}
\end{equation}
Here, we have employed the following estimation:
\begin{equation}
\left\| f_{L} \right\|_{\mathrm{op}}
= 
\left\| W_{L, M} \circ \sigma(W_{L, M-1}) \circ \cdots \circ \sigma (W_{L, 1}) \circ \sigma  (W_{L, 0}) \right\|_{\mathrm{op}}
\leq
C_{w}^{M+1}C_{\sigma}^{M},
\label{covering number-estimation 3 tilde}
\end{equation}
By the same argument in (\ref{covering number-estimation 2 tilde})--(\ref{covering number-estimation 3 tilde}), we evaluate
\begin{equation}
\begin{split}
& \left\|\Gcal_{L-1}(a)-\Gcal^{\prime}_{L-1}(a)\right\|_{\Lp(D;\mathbb{R}^{d_L})}
\\
&
\leq 
(\boldsymbol{\mathrm{Z}}_{L-1}+\boldsymbol{\mathrm{X}}_{L-1}C_{w}^{M+1}C_{\sigma}^{M})(\boldsymbol{\mathrm{Z}}_{L-1}+\boldsymbol{\mathrm{X}}_{L-1}C_k C_{\sigma} )
 \left\|\Gcal_{L-2}(a)-\Gcal^{\prime}_{L-2}(a)\right\|_{\Lp(D;\mathbb{R}^{d_{L-1}})}
\\
& 
+ \Biggl( (\boldsymbol{\mathrm{Z}}_{L-1} + \boldsymbol{\mathrm{X}}_{L-1} C_{\sigma}C_{k})\boldsymbol{\mathrm{X}}_{L-1} \left\| f_{L-1}-f^{\prime}_{L-1} \right\|_{\mathrm{op}}
\\
&
\hspace{2cm}
+ (\boldsymbol{\mathrm{Z}}_{L-1}+\boldsymbol{\mathrm{X}}_{L-1}C_{w}^{M+1}C_{\sigma}^{M})
\boldsymbol{\mathrm{X}}_{L-1}
\left\| \bk_{L-1}-\bk^{\prime}_{L-1} \right\|_{\mathrm{op}} \Biggr) 
\left\|\Gcal^{\prime}_{L-2}(a)\right\|_{\Lp(D;\mathbb{R}^{d_{L-1}})}.
\end{split}
\label{covering number-estimation 5 tilde}
\end{equation}
By repeatedly evaluating $\left\|\Gcal_{\ell}(a)-\Gcal^{\prime}_{\ell}(a)\right\|_{\Lp(D;\mathbb{R}^{d_{\ell+1}})}$ ($\ell=L,L-1,...,0$), we obtain 
\begin{equation}
\begin{split}
& \left\|\Gcal(a)-\Gcal^{\prime}(a)\right\|_{\Lp(D;\mathbb{R}^{d_u})}
\\
&
\leq 
C_{a}\prod_{\ell=0}^{L}(\boldsymbol{\mathrm{Z}}_{\ell}+\boldsymbol{\mathrm{X}}_{\ell}C_{w}^{M+1}C_{\sigma}^{M})(\boldsymbol{\mathrm{Z}}_{\ell}+\boldsymbol{\mathrm{X}}_{\ell}C_k C_{\sigma} )
\\
& 
\times
\sum_{\ell=0}^{L}
\left(
\frac{\boldsymbol{\mathrm{X}}_{\ell}}{\boldsymbol{\mathrm{Z}}_{\ell}+\boldsymbol{\mathrm{X}}_{\ell}C_{w}^{M+1}C_{\sigma}^{M}}\left\| f_{L-1}-f^{\prime}_{L-1} \right\|_{\mathrm{op}}
+
\frac{\boldsymbol{\mathrm{X}}_{\ell}}{\boldsymbol{\mathrm{Z}}_{\ell}+\boldsymbol{\mathrm{X}}_{\ell}C_{k}C_{\sigma}}
\left\| \bk_{\ell}-\bk^{\prime}_{\ell} \right\|_{\mathrm{op}} \right).
\\
&
\leq 
C_{a}
\underbracket
{\prod_{\ell=0}^{L}(\boldsymbol{\mathrm{Z}}_{\ell}+\boldsymbol{\mathrm{X}}_{\ell}C_{w}^{M+1}C_{\sigma}^{M})(\boldsymbol{\mathrm{Z}}_{\ell}+\boldsymbol{\mathrm{X}}_{\ell}C_k C_{\sigma} )}_{=:T_{L}}
\\
& 
\times
\sum_{\ell=0}^{L}
\left(
\underbracket
{
\frac{\boldsymbol{\mathrm{X}}_{\ell}C_{w}^{M+1}C_{\sigma}^{M}}{\boldsymbol{\mathrm{Z}}_{\ell}+\boldsymbol{\mathrm{X}}_{\ell}C_{w}^{M+1}C_{\sigma}^{M}
}
}_{=:C_{w, \ell}}
\sum_{m=0}^{M}\left\|W_{\ell,m}-W_{\ell,m}^{\prime} \right\|_{\mathrm{op}}
+
\underbracket
{
\frac{\boldsymbol{\mathrm{X}}_{\ell}}{\boldsymbol{\mathrm{Z}}_{\ell}+\boldsymbol{\mathrm{X}}_{\ell}C_{k}C_{\sigma}}
}_{=:C_{k, \ell}}
\left\| \bk_{\ell}-\bk^{\prime}_{\ell} \right\|_{\mathrm{op}} \right).
\end{split}
\label{covering number-estimation 6 tilde}
\end{equation}
%%%%%%%%%%%%%%
%%%%%%%%%%%%%%%%%%%%%%%%%%%%%%%%%%%
Combining (\ref{covering number-estimation 1 tilde}), (\ref{covering number-estimation 6 tilde}), (\ref{covering number-estimation 8}), and (\ref{covering number-estimation 9}), the norm $\left\|f -f^{\prime} \right\|_{\mathrm{S}}$ is estimated by
\[
\begin{split}
&\left\|f -f^{\prime}\right\|_{\mathrm{S}}
=
\left(\frac{1}{n}\sum_{i=1}^{n}|f(a_i,u_i) -f^{\prime}(a_i,u_i)|^2 \right)^{\frac{1}{2}}
\\
&
\leq
\sum_{\ell=0}^{L}
\Biggl[
\sum_{m=0}^{M}
\sum_{j=1}^{d^{w}_{\ell, m+1}}
\sum_{i=1}^{d^{w}_{\ell, m}}
\rho C_a T_{L}C_{w, \ell}C_{w}
\left|\frac{w_{\ell,m, ij}}{C_w}-\frac{w^{\prime}_{\ell,m, ij}}{C_w}\right|
\\
&
\hspace{1cm}
+
\sum_{j=1}^{d^{k}_{\ell}}
\sum_{i=1}^{d^{k}_{\ell}}
\rho C_a T_{L} C_{k, \ell} |D| C_{\alpha}
\left\|\frac{k_{\ell, ij}}{C_{\alpha}}-\frac{k^{\prime}_{\ell, ij}}{C_{\alpha}}\right\|_{\boldsymbol{\mathrm{L}}^{\infty}(D\times D;\mathbb{R})}
\Biggr],
\end{split}
\]
which implies that we have
\begin{equation}
\begin{split}
&N(\varepsilon, \mathscr{F}_{\widetilde{\no}}, \left\| \cdot \right\|_{\mathrm{S}} )
\\
&
\leq 
\prod_{\ell=0}^{L}\prod_{j=1}^{d_{\ell+1}}
\prod_{m=0}^{M}
\prod_{j=1}^{d^{w}_{\ell,m+1}}
\prod_{i=1}^{d^{w}_{\ell,m}}
\prod_{j'=1}^{d^{k}_{\ell+1}}
\prod_{i'=1}^{d^{k}_{\ell}}
\\
&
\times
N\left(\frac{\varepsilon}
{2\left( 
\sum_{\ell=0}^{L}\sum_{m=0}^{M}d^{w}_{\ell,m+1}d^{w}_{\ell,m}C_{w, \ell} \right)
\rho  C_a T_{L}C_{w}}, [-1,1], |\cdot|
\right)
\\
&
\times
N\left(\frac{\varepsilon}
{2\left( 
\sum_{\ell=0}^{L}d^{k}_{\ell+1}d^{k}_{\ell}C_{k, \ell} \right)
\rho C_a T_{L}|D|C_{\alpha}}, \left\|\cdot\right\|_{\boldsymbol{\mathrm{L}}^{\infty}(D\times D; \mathbb{R})}
\right),
\end{split}
\label{covering number-estimation 10 tilde}
\end{equation}
By taking logarithmic functions in (\ref{covering number-estimation 10 tilde}), we have 
\begin{equation}
\begin{split}
&\log N(\varepsilon, \mathscr{F}_{\widetilde{\no}}, \left\| \cdot \right\|_{\mathrm{S}} )
\\
&
\leq 
\left( \sum_{\ell=0}^{L}\sum_{m=0}^{M}d^{w}_{\ell,m+1}d^{w}_{\ell,m}d^{k}_{\ell+1}d^{k}_{\ell} \right)
\\
&
\times
\Biggl\{
\underbracket{
\log
N\left(\frac{\varepsilon}
{2\left( 
\sum_{\ell=0}^{L}\sum_{m=0}^{M}d^{w}_{\ell,m+1}d^{w}_{\ell,m}C_{w, \ell} \right)
\rho  C_a T_{L}C_{w}}, [-1,1], |\cdot|
\right)
}_{=:\widetilde{H}_w(\varepsilon)}
\\
&
+
\underbracket{
\log
N\left(\frac{\varepsilon}
{2\left( 
\sum_{\ell=0}^{L}d^{k}_{\ell+1}d^{k}_{\ell}C_{k, \ell} \right)
\rho C_a T_{L}|D|C_{\alpha}}, \left\|\cdot\right\|_{\boldsymbol{\mathrm{L}}^{\infty}(D\times D; \mathbb{R})}
\right)
}_{=:\widetilde{H}_k(\varepsilon)}
\Biggr\}.
\end{split}
\label{log-covering tilde}
\end{equation}
By same arguments in (\ref{estimation H_w}) and (\ref{estimation H_k}) (using \citet[Example 5.3]{wainwright} and \citet[Lemma 4.2]{gottlieb2016adaptive}), we can estimate $\widetilde{H}_w$ and $\widetilde{H}_k$, respectively as follows:
\begin{equation}
\begin{split}
\widetilde{H}_w(\varepsilon)
& \leq
\left( \frac{\widetilde{I}_{w}}{\varepsilon} \right)^{\hat{d}+1}, 
\\
& \mbox{for} \
0< \varepsilon <
2\left( 
\sum_{\ell=0}^{L}\sum_{m=0}^{M}d^{w}_{\ell,m+1}d^{w}_{\ell,m}C_{w, \ell} \right)
\rho  C_a T_{L}C_{w}, 
\end{split}
\label{estimation H_w tilde}
\end{equation}
%%%%%%%%%%%%%%%%%%%%%%%%%%%%%%%%%%%%%%%%%
\begin{equation}
\begin{split}
\widetilde{H}_k(\varepsilon)
\leq &
\left(\frac{\widetilde{I}_{k}}{\varepsilon} \right)^{\hat{d}+1},
\\
& \hspace{1cm}
\ \mbox{for} \ 0< \varepsilon < 
2\left( 
\sum_{\ell=0}^{L}d^{k}_{\ell+1}d^{k}_{\ell}C_{k, \ell} \right)
\rho C_a T_{L}|D|C_{\alpha}, 
\end{split}
\label{estimation H_k tilde}
\end{equation}
where we denoted by
\[
\widetilde{I}_w:=2\left( 
\sum_{\ell=0}^{L}\sum_{m=0}^{M}d^{w}_{\ell,m+1}d^{w}_{\ell,m}C_{w, \ell} \right)
\rho  C_a T_{L}C_{w},
\]
\[
\widetilde{I}_k:=8\max\left\{C_{\beta}\mathrm{diag}(D\times D), 2 \right\}
\left( 
\sum_{\ell=0}^{L}d^{k}_{\ell+1}d^{k}_{\ell}C_{k, \ell} \right)
\rho C_a T_{L}|D|C_{\alpha}.
\]
By employing (\ref{log-covering tilde}), (\ref{estimation H_w tilde}), and (\ref{estimation H_k tilde}), we calculate
\[
\int_{\alpha}^{\infty}\left(\log N(\varepsilon, \mathscr{F}_{\widetilde{\no}}, \left\| \cdot \right\|_{\mathrm{S}} ) \right)^{\frac{1}{2}} d\varepsilon
\leq
\left( \sum_{\ell=0}^{L}\sum_{m=0}^{M}d^{w}_{\ell,m+1}d^{w}_{\ell,m}d^{k}_{\ell+1}d^{k}_{\ell} \right)^{\frac{1}{2}}
\underbracket{
\int_{\alpha}^{\infty}(\widetilde{H}_w(\varepsilon)+\widetilde{H}_k(\varepsilon))^{\frac{1}{2}}d\varepsilon
}
_{=:(\ast)}
\]
\[
\begin{split}
& (\ast) \leq
\int_{\alpha}^{\infty}
\widetilde{H}_w(\varepsilon)^{\frac{1}{2}} d\varepsilon
+
\int_{\alpha}^{\infty}
\widetilde{H}_w(\varepsilon)^{\frac{1}{2}} d\varepsilon
\\
&
\leq
\left(\widetilde{I}_{w}^{\frac{\hat{d}+1}{2}}
+
\widetilde{I}_{k}^{\frac{\hat{d}+1}{2}}
\right)
\frac{2}{\hat{d}-1}\alpha^{-\frac{\hat{d}-1}{2}} 
\\
&
\leq
\frac{4}{\hat{d}-1}
\Biggl(
\max\left[2C_{w}, 8|D|C_{\alpha}\max\left\{C_{\beta} \mathrm{diag}(D\times D), 2 \right\} \right]
\rho M C_{d}^{2} C_a T_{L}
\Biggr)^{\frac{\hat{d}+1}{2}}
\\
&
\times
\left[ \left( \sum_{\ell=0}^{L}C_{w, \ell} \right)^{\frac{\hat{d}+1}{2}} 
+ \left( \sum_{\ell=0}^{L}C_{k, \ell} \right)^{\frac{\hat{d}+1}{2}}
\right]
\alpha^{-\frac{\hat{d}-1}{2}},
\end{split}
\]
that is, we have %by (iii) of Assumption \ref{Assumption2},
\[
\begin{split}
&
\int_{\alpha}^{\infty}\left(\log N(\varepsilon, \mathscr{F}_{\widetilde{\no}}, \left\| \cdot \right\|_{\mathrm{S}} )  \right)^{\frac{1}{2}} d\varepsilon
\leq
\widetilde{K} 
\alpha^{-\frac{\hat{d}-1}{2}}
\end{split}
\]
where
\[
\begin{split}
\widetilde{K}:= &
\frac{4C_{d}^2M^{1/2}L^{1/2}}{\hat{d}-1}
\Biggl(
\max\left[2C_{w}, 8|D|C_{\alpha}\max\left\{C_{\beta} \mathrm{diag}(D\times D), 2\right\} \right]
\rho M C_{d}^{2} C_a T_{L}
\left(
\sum_{\ell=0}^{L}C_{w, \ell}+C_{k, \ell}
\right)
\Biggr)^{\frac{\hat{d}+1}{2}}
\alpha^{-\frac{\hat{d}-1}{2}}
\end{split}
\]
which implies that we conclude that with (\ref{Rademacher Complexity-covering number tilde})
\[
\begin{split}
&
\mathfrak{R}^n_{\mathrm{S}}(\mathscr{F}_{\widetilde{\no}}) \leq 4 \inf_{\alpha \geq 0} \left\{ \alpha
+
\underbracket{
\frac{3\widetilde{K}}{\sqrt{n}}
}_{=:\widetilde{K}^{\prime}}
\alpha^{-\frac{\hat{d}-1}{2}}\right\}
\\
&
=
4\left( \left(\frac{(\hat{d}-1)\widetilde{K}^{\prime}}{2} \right)^{\frac{2}{\hat{d}+1}} 
+
\widetilde{K}^{\prime}
\left(\frac{(\hat{d}-1)\widetilde{K}^{\prime}}{2} \right)^{\frac{2}{\hat{d}+1} \left(-\frac{\hat{d}-1}{2}\right)}
\right)
\\
&
=\widetilde{\gamma} L^{\frac{1}{\hat{d}+1}} 
\left( \sum_{\ell=0}^{L}
\frac{\boldsymbol{\mathrm{X}}_{\ell}C_{w}^{M+1}C_{\sigma}^{M}}{\boldsymbol{\mathrm{Z}}_{\ell}+\boldsymbol{\mathrm{X}}_{\ell}C_{w}^{M+1}C_{\sigma}^{M}}
+\frac{\boldsymbol{\mathrm{X}}_{\ell}}{
\boldsymbol{\mathrm{Z}}_{\ell}+\boldsymbol{\mathrm{X}}_{\ell}C_{k}C_{\sigma}
}
\right)
\left[\prod_{\ell=0}^{L}(\boldsymbol{\mathrm{Z}}_{\ell}+\boldsymbol{\mathrm{X}}_{\ell}C_{w}^{M+1}C_{\sigma}^{M})(\boldsymbol{\mathrm{Z}}_{\ell}+\boldsymbol{\mathrm{X}}_{\ell}C_k C_{\sigma})
\right]
\left(\frac{1}{n}\right)^{\frac{1}{\hat{d}+1}}
\end{split}
\]
where $\widetilde{\gamma}$ is the positive constant defined by
\begin{equation}
\begin{split}
\widetilde{\gamma}
&
:=4
\left\{
\left(\frac{\hat{d}-1}{2} \right)^{\frac{2}{\hat{d}+1}}
+
\left(\frac{\hat{d}-1}{2} \right)^{-\frac{\hat{d}-1}{\hat{d}+1}}
\right\}
\left(\frac{12}{\hat{d}-1} \right)^{\frac{2}{\hat{d}+1}}
\\
& \hspace{3cm}
\times
\max\left[2C_{w}, 16|D|C_{\alpha}\max\left\{C_{\beta} \mathrm{diag}(D\times D), 2\right\} \right]
(C_{d}^{4}M)^{\frac{\hat{d}+2}{\hat{d}+1}}
\rho C_a 
\label{definition of gamma tilde}
\end{split}
\end{equation}
\end{proof}

%%%%%%%%%%%%%%%%%%%%%%%%%%%%%%%%%%%%
\subsection{Proof of Corollary~\ref{GEB for additive NO}}
\label{main-coro-1}
\begin{proof}
 By using Assumption \ref{Assumption2},
we estimate for $\Gcal \in \no$ and $a \in \mathrm{supp}(\mu_a)$,
\[
\begin{split}
\left\|\Gcal(a) \right\|_{ \Lp(D; \mathbb{R}^{d_u})}&
=
\left\|(W_{L}+\bk_{L})\circ \sigma (W_{L-1}+\bk_{L-1})\circ \cdots \circ \sigma (W_{0}+\bk_{0})(a)
\right\|_{ \Lp(D; \mathbb{R}^{d_u})}
\\
&
\leq 
(C_w + C_k)^{L+1}C_{\sigma}^{L}C_{a}.
\end{split}
\]
Then, by applying Lemma \ref{generalization error bound for general setting} as $R=(C_w + C_k)^{L+1}C_{\sigma}^{L}C_{a}$, and combining with Theorem \ref{Rademacher for standard NO}, we conclude that the inequality (\ref{Generalization error bound for additive NO-main}).
\end{proof}

%%%%%%%%%%%%%%%%%%%%%%%%%%%%%%%%%%%%
\subsection{Proof of Corollary~\ref{GEB for multiplicative NO}}
\label{main-coro-2}
\begin{proof}
By using Assumption \ref{Assumption3},
we estimate for $\Gcal \in \widetilde{\no}$ and $a \in \mathrm{supp}(\mu_a)$,
\[
\begin{split}
&
\left\|\Gcal(a) \right\|_{ \Lp(D; \mathbb{R}^{d_u})}
\\
&
=
\left\|(\boldsymbol{\mathrm{Z}}_{\ell}\bid + \boldsymbol{\mathrm{X}}_{\ell} f_{\ell}) \circ (\boldsymbol{\mathrm{Z}}_{\ell}\bid + \boldsymbol{\mathrm{X}}_{\ell} \sigma \circ \bk_{\ell}) \circ  \cdots \circ (\boldsymbol{\mathrm{Z}}_{0}\bid + \boldsymbol{\mathrm{X}}_{0} f_{0}) \circ (\boldsymbol{\mathrm{Z}}_{0}\bid + \boldsymbol{\mathrm{X}}_{0} \sigma \circ \bk_{0})(a)
\right\|_{ \Lp(D; \mathbb{R}^{d_u})}
\\
&
\leq 
\left[\prod_{\ell=0}^{L}(\boldsymbol{\mathrm{Z}}_{L}+\boldsymbol{\mathrm{X}}_{L}C_{w}^{M+1}C_{\sigma}^{M})(\boldsymbol{\mathrm{Z}}_{L}+\boldsymbol{\mathrm{X}}_{L}C_k C_{\sigma})
\right]C_a.
\end{split}
\]
Then, by applying Lemma \ref{generalization error bound for general setting} as $R=\left[\prod_{\ell=0}^{L}(\boldsymbol{\mathrm{Z}}_{\ell}+\boldsymbol{\mathrm{X}}_{\ell}C_{w}^{M+1}C_{\sigma}^{M})(\boldsymbol{\mathrm{Z}}_{\ell}+\boldsymbol{\mathrm{X}}_{\ell}C_k C_{\sigma})
\right]C_{a}$ and combining with Theorem \ref{Rademacher for multiplicative NO}, we conclude that the inequality (\ref{multiplicativeGEB}).
\end{proof}

%%%%%%%%%%%%%%%%%%%%%%%%%%%%%%%%%%%%%%%%%%%%%%%%%%%%%%%%%%%%%%%%%
% \newpage
%\section{Proof of Corollary \ref{GEB for additive NO} and \ref{GEB for multiplicative NO}}
\section{Remark for Sections \ref{sec:Rademacher Complexity1}, \ref{sec:Rademacher Complexity2}, \ref{sec:Rademacher Complexity3}.}
\label{Appendix 4}

\begin{remark}\label{remark:finite_basis}
\emph{In the implementation of NO}, $\bk_{\ell}$ is projected into a \emph{finite-rank operator} by the chosen basis. 
For clarity's sake, let assume $\bk_{\ell}: \Lp(D) \to \Lp(D)$, i.e., domain and range are the same space, and $\Lp(D) = \Lp(D; \mathbb{R})$. Let $k_{\ell} \in \Lp(D \times D)$ be the kernel of $\bk_{\ell}$, and let $\{\phi_j\}_{j\in \mathbb{N}}$ be an orthonormal basis in $\Lp(D)$\footnote{For FNO, the basis are the Fourier basis.}, so
$\{\phi_i \otimes \phi_j\}_{i,j\in \mathbb{N}}$ is an orthonormal basis of $\Lp(D \times D)$, and thus $k_{\ell}(x,y) = \sum_{j,k\ge 1} k_{\ell,jk}\, \phi_j(x) \otimes \phi_k(y)$, where $k_{\ell, jk}\in \mathbb{R}$, $k_{\ell, jk}= \langle k_{\ell}, \phi_j(x) \otimes \phi_k \rangle_{\Lp(D\times D)} = \langle \phi_j, \mathcal{K} \phi_k \rangle_{\Lp(D)}$. 
By choosing $N$-modes (first $N$ basis), the kernel $k_{\ell}$ is approximated as $k_{\ell}^{(N)}(x,y) = \sum_{j,k \le N} k_{\ell, jk}\, \phi_j(x) \otimes \phi_k(y),$ and so
$$
\|k_{\ell}^{(N)}\|_{\Lp(D\times D)}^2 = \sum_{j,k =1}^N |k_{\ell, jk}|^2 \le \sum_{j,k = 1}^\infty |k_{\ell, jk}|^2  = \|k\|_{\Lp(D\times D)}^2.
$$
Hence, the implementable\footnote{in a computer} kernel $k_{\ell}^{(N)}$ satisfies (ii) \cref{Assumption2}, and the Rademacher Complexity for (\ref{RC for standard NO}) is also an upper-bound.
\end{remark}

\begin{remark}[Summary of generalization error bounds]\label{Comparision-bounds}
%Summing up our bounds as follows:
\[
\begin{split}
&(\mathrm{Bound \ for \ NO}  )
\\
&
\lesssim
L^{\frac{\hat{d}+2}{\hat{d}+1}}
\{ (C_w + C_k)C_{\sigma} \}^{L} \left(\frac{1}{n} \right)^{\frac{1}{\hat{d}+1}} + \left\{(C_w + C_k)C_{\sigma}\right\}^{L} \sqrt{\frac{2 \delta}{n}}.
\\
&
\vspace{1mm}
\\
&(\mathrm{Bound \ for \ } \MNO )
\\
&
\lesssim
L^{\frac{\hat{d}+2}{\hat{d}+1}} 
(C_{w}^{M+1}C_{\sigma}^{M+1}C_k)^{L}
\left(\frac{1}{n}\right)^{\frac{1}{\hat{d}+1}} 
+ (C_{w}^{M+1}C_{\sigma}^{M+1}C_k)^{L} \sqrt{\frac{2 \delta}{n}}.
\\
&
\vspace{1mm}
\\
&(\mathrm{Bound \ for \ } (\NONeXt)v1 )
\\
&
\lesssim
L^{\frac{\hat{d}+2}{\hat{d}+1}} 
\left\{(1+C_{w}^{M+1}C_{\sigma}^{M})(1+C_{\sigma}C_k)\right\}^{L}
\left(\frac{1}{n}\right)^{\frac{1}{\hat{d}+1}} 
+ 
\left\{(1+C_{w}^{M+1}C_{\sigma}^{M})(1+C_{\sigma}C_k)\right\}^{L}
\sqrt{\frac{2 \delta}{n}}.
\\
&
\vspace{1mm}
\\
&(\mathrm{Bound \ for \ } (\NONeXt)v2 )
\\
&
\lesssim L^{\frac{1}{\hat{d}+1}} 
\left( \sum_{\ell=0}^{L}
\frac{\boldsymbol{\mathrm{X}}_{\ell}C_{w}^{M+1}C_{\sigma}^{M}}{1+\boldsymbol{\mathrm{X}}_{\ell}C_{w}^{M+1}C_{\sigma}^{M}}
+\frac{\boldsymbol{\mathrm{X}}_{\ell}}{
1+\boldsymbol{\mathrm{X}}_{\ell}C_{k}C_{\sigma}
}
\right)
\left[\prod_{\ell=0}^{L}(1+\boldsymbol{\mathrm{X}}_{\ell}C_{w}^{M+1}C_{\sigma}^{M})(1+\boldsymbol{\mathrm{X}}_{\ell}C_k C_{\sigma})
\right]
\left(\frac{1}{n}\right)^{\frac{1}{\hat{d}+1}} 
\\
&
+ \left[\prod_{\ell=0}^{L}(\boldsymbol{\mathrm{Z}}_{\ell}+\boldsymbol{\mathrm{X}}_{\ell}C_{w}^{M+1}C_{\sigma}^{M})(\boldsymbol{\mathrm{Z}}_{\ell}+\boldsymbol{\mathrm{X}}_{\ell}C_k C_{\sigma})
\right]
\sqrt{\frac{2 \delta}{n}}.
\end{split}
\]
Here, $\lesssim$ implies that the left-hand side is bounded above by the right-hand side times a constant independent of $n$ and $L$.
Hence, Remark~\ref{GEB Comparison} can be observed.
\end{remark}
%%%%%%%%%%%%%%%%%%%%%%%%%%%%%%%%%%%%%%%%%%%%%%%%%%%%%%%%%%%%%%%%%%%%%%%%%%%%%%%%%%%%%%%%%%%%%%%%%%%%%%%%%%
\begin{lemma}\label{BRV-estimate}
Let $\boldsymbol{\mathrm{Z}}_{\ell}=1$ and $\boldsymbol{\mathrm{X}}_{\ell}$ be a Bernoulli RV with $\prob\{\boldsymbol{\mathrm{X}}_{\ell}  =  1 \} = p_\ell$, and $\prob\{\boldsymbol{\mathrm{X}}_{\ell}  = 0\} = 1-p_\ell$ for $p_\ell \in [0,1]$ in inequality (\ref{multiplicativeGEB}). We assume that $p_{\ell}= x_{\ell}/ L^{\frac{1}{\hat{d}+1}}$ where $x_{\ell} \in [0,1]$ satisfies $\sum_{\ell=0}^{\infty}x_{\ell}< \infty$. 
Then,
\[
\begin{split}
\mathbb{E}_{\boldsymbol{\mathcal{X}}}[\text{RHS of (\ref{multiplicativeGEB})}]
& 
\lesssim
\widehat{\mathcal{L}}_{\mathrm{S}}(\Gcal) 
+
\left(\sum_{\ell=1}^{L}
x_{\ell}\right) 
\prod_{\ell=0}^{L}
\left[
1+(C_{w}^{M+1}C_{\sigma}^{M}+C_k C_{\sigma}+C_{w}^{M+1}C_kC_{\sigma}^{M+1})
x_{\ell}
\right]
\left(\frac{1}{n}\right)^{\frac{1}{\hat{d}+1}} 
\\
&
+ \left(\rho 
\prod_{\ell=0}^{L}
\left[
1+(C_{w}^{M+1}C_{\sigma}^{M}+C_k C_{\sigma}+C_{w}^{M+1}C_kC_{\sigma}^{M+1})
x_{\ell}
\right]
C_a  + R_u \right)\sqrt{\frac{2 \delta}{n}}.
\end{split}
\]
Here, $\lesssim$ implies that the left-hand side is bounded above by the right-hand side times a constant independent of $n$ and $L$.
\end{lemma}
We remark that the upper bound remain bounded as $L$ tends to infinity because $\sum_{\ell=1}^{\infty}
x_{\ell} < \infty$ and 
\[
\sum_{\ell=1}^{\infty}
\bigg(1+(C_{w}^{M+1}C_{\sigma}^{M}+C_k C_{\sigma}+C_{w}^{M+1}C_kC_{\sigma}^{M+1})
x_{\ell} \bigg) < \infty.
\]
As result, infinite products also remain bounded.

\begin{proof}
First, we evaluate that 
\[
\begin{split}
[\text{RHS of (\ref{multiplicativeGEB})}]
&
\leq 
\widehat{\mathcal{L}}_{\mathrm{S}}(\Gcal) 
+
4 \widetilde{\gamma}  (C_{w}^{M+1}C_{\sigma}^{M}+1)L^{\frac{1}{\hat{d}+1}} 
\left(\sum_{\ell=1}^{L}\boldsymbol{\mathrm{X}}_{\ell}\right)
\left[\prod_{\ell=0}^{L}(1+\boldsymbol{\mathrm{X}}_{\ell}C_{w}^{M+1}C_{\sigma}^{M})(1+\boldsymbol{\mathrm{X}}_{\ell}C_k C_{\sigma})
\right]
\left(\frac{1}{n}\right)^{\frac{1}{\hat{d}+1}} 
\\
&
+ \left(\rho 
\left[\prod_{\ell=0}^{L}(1+\boldsymbol{\mathrm{X}}_{\ell}C_{w}^{M+1}C_{\sigma}^{M})(1+\boldsymbol{\mathrm{X}}_{\ell}C_k C_{\sigma})
\right]
C_a  + R_u \right)\sqrt{\frac{2 \delta}{n}},
\end{split}
\]
which implies that 
\[
\begin{split}
\mathbb{E}_{\boldsymbol{\mathcal{X}}}
[\text{RHS of (\ref{multiplicativeGEB})}]
&
\lesssim
\widehat{\mathcal{L}}_{\mathrm{S}}(\Gcal) 
+
L^{\frac{1}{\hat{d}+1}}
\sum_{\ell=1}^{L}
\mathbb{E}_{\boldsymbol{\mathrm{X}}_{\ell}}
\left[
\boldsymbol{\mathrm{X}}_{\ell}
(1+\boldsymbol{\mathrm{X}}_{\ell}C_{w}^{M+1}C_{\sigma}^{M})(1+\boldsymbol{\mathrm{X}}_{\ell}C_k C_{\sigma})
\right]
\\
&
\times
\mathbb{E}_{\boldsymbol{\mathcal{X}} \setminus \boldsymbol{\mathrm{X}}_{\ell}}
\left[\prod_{\ell^{\prime}=0}^{L}
(1+\boldsymbol{\mathrm{X}}_{\ell^{\prime}}C_{w}^{M+1}C_{\sigma}^{M})(1+\boldsymbol{\mathrm{X}}_{\ell^{\prime}}C_k C_{\sigma})
\right]
\left(\frac{1}{n}\right)^{\frac{1}{\hat{d}+1}} 
\\
&
+ \left(\rho 
\mathbb{E}_{\boldsymbol{\mathcal{X}}}
\left[\prod_{\ell=0}^{L}(\boldsymbol{\mathrm{Z}}_{\ell}+\boldsymbol{\mathrm{X}}_{\ell}C_{w}^{M+1}C_{\sigma}^{M})(\boldsymbol{\mathrm{Z}}_{\ell}+\boldsymbol{\mathrm{X}}_{\ell}C_k C_{\sigma})
\right]
C_a  + R_u \right)\sqrt{\frac{2 \delta}{n}}.
\end{split}
\]
Since we have 
\[
\mathbb{E}_{\boldsymbol{\mathrm{X}}_{\ell}}
\left[
\boldsymbol{\mathrm{X}}_{\ell}
(1+\boldsymbol{\mathrm{X}}_{\ell}C_{w}^{M+1}C_{\sigma}^{M})(1+\boldsymbol{\mathrm{X}}_{\ell}C_k C_{\sigma})
\right]
=(1+C_{w}^{M+1}C_{\sigma}^{M}+C_k C_{\sigma}+C_{w}^{M+1}C_kC_{\sigma}^{M+1})p_{\ell},
\]
and 
\[
\mathbb{E}_{\boldsymbol{\mathrm{X}}_{\ell}}
\left[
(1+\boldsymbol{\mathrm{X}}_{\ell}C_{w}^{M+1}C_{\sigma}^{M})(1+\boldsymbol{\mathrm{X}}_{\ell}C_k C_{\sigma})
\right]
=1+(C_{w}^{M+1}C_{\sigma}^{M}+C_k C_{\sigma}+C_{w}^{M+1}C_kC_{\sigma}^{M+1})p_{\ell},
\]
we conclude that by using $p_{\ell}= x_{\ell}/ L^{\frac{1}{\hat{d}+1}}$,
\[
\begin{split}
&
\mathbb{E}_{\boldsymbol{\mathcal{X}}}
[\text{RHS of (\ref{multiplicativeGEB})}]
\\
&
\lesssim
\widehat{\mathcal{L}}_{\mathrm{S}}(\Gcal) 
+
L^{\frac{1}{\hat{d}+1}}
\left(\sum_{\ell=1}^{L}
p_{\ell}\right) 
\prod_{\ell=0}^{L}
\left[
1+(C_{w}^{M+1}C_{\sigma}^{M}+C_k C_{\sigma}+C_{w}^{M+1}C_kC_{\sigma}^{M+1})p_{\ell}
\right]
\left(\frac{1}{n}\right)^{\frac{1}{\hat{d}+1}} 
\\
&
+ \left(\rho 
\prod_{\ell=0}^{L}
\left[
1+(C_{w}^{M+1}C_{\sigma}^{M}+C_k C_{\sigma}+C_{w}^{M+1}C_kC_{\sigma}^{M+1})p_{\ell}
\right]
C_a  + R_u \right)\sqrt{\frac{2 \delta}{n}}
\\
&
\lesssim
\widehat{\mathcal{L}}_{\mathrm{S}}(\Gcal) 
+
\left(\sum_{\ell=1}^{L}
x_{\ell}\right) 
\prod_{\ell=0}^{L}
\left[
1+\frac{(C_{w}^{M+1}C_{\sigma}^{M}+C_k C_{\sigma}+C_{w}^{M+1}C_kC_{\sigma}^{M+1})}{L^{\frac{1}{\hat{d}+1}}}
x_{\ell}
\right]
\left(\frac{1}{n}\right)^{\frac{1}{\hat{d}+1}} 
\\
&
+ \left(\rho 
\prod_{\ell=0}^{L}
\left[
1+\frac{(C_{w}^{M+1}C_{\sigma}^{M}+C_k C_{\sigma}+C_{w}^{M+1}C_kC_{\sigma}^{M+1})}{L^{\frac{1}{\hat{d}+1}}}
x_{\ell}
\right]
C_a  + R_u \right)\sqrt{\frac{2 \delta}{n}}
\end{split}
\]

\end{proof}

\newpage
\section{Experiments.}
\subsection{Loss Landscape Visualization} \label{appendix:loss landscape visualization}

We include here additional views of the training loss landscape of the considered architectures that were created using the method discussed in Section~\ref{section:loss landscape}. In particular, the images below offer a closer view of the landscape in the immediate vicinity of the found minimizer, to allow for a better comparison. In addition, a color-based planar view of the landscapes is provided for a better view of their respective topological features. 

\begin{figure}[hbt]
\includegraphics[width=\linewidth]{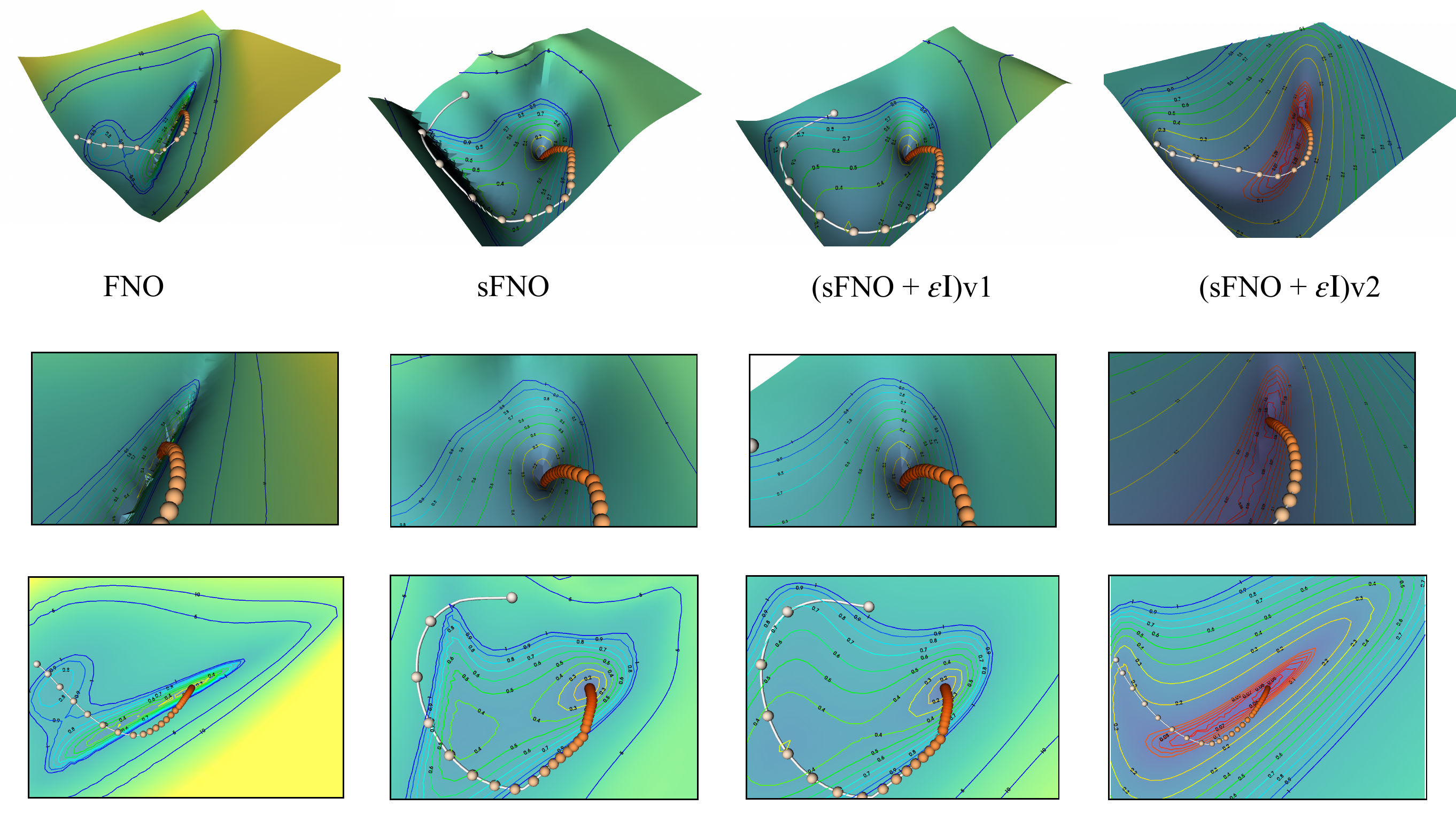}
\caption{Visualization of the training landscapes associated with FNO, $\MFNO$, $\FNONeXt$ v1 and $\FNONeXt$ v2. \label{figure:loss landscape vis 2}}
\end{figure}

\subsection{Out-of-distribution}\label{Out-of-distribution}
%%%%%%%%%%%%%%%%%%%%%%%%%%%%%%%%%%%%%%%%%%%%%%%%%%%%%%%%%%%%%%%%%%%%%%%%%%%%%%%%%%%%%%%%%%%%%%%%%%%%%%%%

In this section, we present the wavefield reconstruction of the other families described in \cref{subsection:OOD}. The values of the parameters are established in \cref{table:ood}, and the relative test loss error is presented in \crefrange{OOD:table1}{OOD:table6}. In our analysis, we selected three realizations from the previously trained neural networks. These networks were trained using a dataset at a frequency of 15 Hz and with the parameters of the random field generating the wave speed set as $\boldsymbol{\lambda} = (1, 1)$ and a wave speed range of $[1500, 5000]$. Specifically, we chose the first three networks documented in \cref{fig:boxplot}.

To test the performance of these networks on a different random field, we kept the smoothness coefficient constant and varied the correlation range of the Whittle-Matérn field. The reconstructed wave fields are presented in \crefrange{OOD:set1}{OOD:set6}. Please note that the imaginary part of the wave field is also recovered, but it is not shown in the figures.

%%%%%%%%%%%%%%%%%%%%%%%%%%%%%%%%%%%%%%%%%%%%%%%%%%%%%%%%%%%%%%%%%%%%%%%%%%%%%%%%
%OOD1
%%%%%%%%%%%%%%%%%%%%%%%%%%%%%%%%%%%%%%%%%%%%%%%%%%%%%%%%%%%%%%%%%%%%%%%%%%%%%%%%
\subsubsection*{OOD 1}
In this set family, we keep the isotropic behaviour of the original data, however we move the value to ${\boldsymbol{\lambda}}_{\mathrm{OOD1}} = (0.20, 0.20)$. The range is kept in $[1500,5000]$. We see that this scenario is the \textit{easier} for the networks. However, FNO still struggles to capture the correct wave propagation. 
\setlength{\modelwidth} {3.65cm}
\setlength{\modelheight}{3.65cm}

\begin{figure}[ht!]\centering
\input{Figures/OOD_main_paper/set0/skeleton_online_set0_appendix} 
    \caption{\small{ 
            %\textbf{(Real part) of the pressure field reconstructed by
             Pressure field reconstructed
             at \num{15} \si{\Hz} trained with \emph{isotropic Whittle–Mat\'ern covariance} 
             $\boldsymbol{\lambda} = (1, 1)$, and wavespeed range of $(1500,5000)$ \cref{eq:lambda_cov} and tested with \cref{OOD:table1}
             ${\boldsymbol{\lambda}}_{\mathrm{OOD1}} = (0.20, 0.20)$, and wavespeed range of $(1500,5000)$
             with the different architectures for multiple realizations of the new GRF
             \emph{out-of distribution},
             realizations of the wave speed. 
             \textit{Left column} shows independent GRF realization of the wave
             speed (see \cref{eq:Helmholtz}).  
             \textit{Second column} shows the real part % and imaginary part
             of the pressure field solution to the wave PDE at frequency
             15 \si{\Hz}, 
             obtained with software \texttt{hawen} \citep{faucher2021hawen}, 
             which we consider as the \emph{reference solution},
             \textit{Other columns} show the approximated reconstructions 
             using the different architectures.
             In each case, we show the real parts of the pressure fields, and 
             the relative error with the reference 
             solution on a logarithmic scale. 
             }} \label{OOD:set1}
\end{figure}
%%%%%%%%%%%%%%%%%%%%%%%%%%%%%%%%%%%%%%%%%%%%%%%%%%%%%%%%%%%%%%%%%%%%%%%%%%%%%%%%
%OOD2
%%%%%%%%%%%%%%%%%%%%%%%%%%%%%%%%%%%%%%%%%%%%%%%%%%%%%%%%%%%%%%%%%%%%%%%%%%%%%%%%
\subsubsection*{OOD 2}
In this set family, we generate an anisotropic random field, different to the original trained data  ${\boldsymbol{\lambda}}_{\mathrm{OOD2}} = (0.10, 0.20)$, however the range was kept similar than the original set.
\begin{figure}[ht!]\centering
\input{Figures/OOD_main_paper/set1/skeleton_online_set1_appendix} 
    \caption{\small{ 
            %\textbf{(Real part) of the pressure field reconstructed by
             Pressure field reconstructed 
             at \num{15} \si{\Hz} trained with \emph{isotropic Whittle–Mat\'ern covariance}
             $\boldsymbol{\lambda} = (1, 1)$, and wavespeed range of $(1500,5000)$ \cref{eq:lambda_cov} and tested with \cref{OOD:table2}
             ${\boldsymbol{\lambda}}_{\mathrm{OOD2}} = (0.10, 0.20)$, and wavespeed range of $(1500,5000)$
           with the different architectures for multiple realizations of the new GRF
             \emph{out-of distribution},
             realizations of the wave speed. 
             \textit{Left column} shows independent GRF realization of the wave
             speed (see \cref{eq:Helmholtz}).  
             \textit{Second column} shows the real part % and imaginary part
             of the pressure field solution to the wave PDE at frequency
             15 \si{\Hz}, 
             obtained with software \texttt{hawen} \citep{faucher2021hawen}, 
             which we consider as the \emph{reference solution},
             \textit{Other columns} show the approximated reconstructions 
             using the different architectures.
             In each case, we show the real parts of the pressure fields, and 
             the relative error with the reference 
             solution on a logarithmic scale. 
             }}
             \label{OOD:set2}
\end{figure}

\iffalse{            Pressure field reconstructed for 
             (\cref{experiment2}) 
             at \num{15} \si{\Hz} trained with \emph{isotropic Whittle–Mat\'ern covariance}
             $\boldsymbol{\lambda} = (1, 1)$, and wavespeed range of $(1500,5000)$ \cref{eq:lambda_cov} and tested with \cref{OOD:table2}
             ${\boldsymbol{\lambda}}_{\mathrm{OOD2}} = (0.10, 0.20)$, and wavespeed range of $(1500,5000)$
           with the different architectures for multiple realizations of the new GRF
             \emph{out-of distribution},
             realizations of the wave speed. 
             \textit{Left column} shows independent GRF realization of the wave
             speed (see \cref{eq:Helmholtz}).  
             \textit{Second column} shows the real part % and imaginary part
             of the pressure field solution to the wave PDE at frequency
             15 \si{\Hz}, 
             obtained with software \texttt{hawen} \citep{faucher2021hawen}, 
             which we consider as the \emph{reference solution},
             \textit{Other columns} show the approximated reconstructions 
             using the different architectures.
             In each case, we show the real parts of the pressure fields, and 
             the relative error with the reference 
             solution on a logarithmic scale. }\fi
%%%%%%%%%%%%%%%%%%%%%%%%%%%%%%%%%%%%%%%%%%%%%%%%%%%%%%%%%%%%%%%%%%%%%%%%%%%%%%%%
%OOD3
%%%%%%%%%%%%%%%%%%%%%%%%%%%%%%%%%%%%%%%%%%%%%%%%%%%%%%%%%%%%%%%%%%%%%%%%%%%%%%%%
\subsubsection*{OOD 3}
In this set family, we generate an isotropic random field, different to the original trained data   ${\boldsymbol{\lambda}}_{\mathrm{OOD3}} = (0.20, 0.20)$, however the range was moved to $[2000,3500]$.
\begin{figure}[ht!]\centering
\input{Figures/OOD_main_paper/set2/skeleton_online_set2_appendix} 
    \caption{\small{ 
            %\textbf{(Real part) of the pressure field reconstructed by
            %\textbf{(Real part) of the pressure field reconstructed by
            Pressure field 
             at \num{15} \si{\Hz} trained with \emph{isotropic Whittle–Mat\'ern covariance}
             $\boldsymbol{\lambda} = (1, 1)$, and wavespeed range of $(1500,5000)$ \cref{eq:lambda_cov} and tested with \cref{OOD:table4}
             ${\boldsymbol{\lambda}}_{\mathrm{OOD3}} = (0.20, 0.20)$, and wavespeed range of $(2000,3500)$
           with the different architectures for multiple realizations of the new GRF
             \emph{out-of distribution},
             realizations of the wave speed. 
             \textit{Left column} shows independent GRF realization of the wave
             speed (see \cref{eq:Helmholtz}).  
             \textit{Second column} shows the real part % and imaginary part
             of the pressure field solution to the wave PDE at frequency
             15 \si{\Hz}, 
             obtained with software \texttt{hawen} \citep{faucher2021hawen}, 
             which we consider as the \emph{reference solution},
             \textit{Other columns} show the approximated reconstructions 
             using the different architectures.
             In each case, we show the real parts of the pressure fields, and 
             the relative error with the reference 
             solution on a logarithmic scale. 
             }}
             \label{OOD:set3}
\end{figure}

\iffalse{   Pressure field reconstructed for 
             (\cref{experiment2}) 
             at \num{15} \si{\Hz} trained with \emph{isotropic Whittle–Mat\'ern covariance}
             $\boldsymbol{\lambda} = (1, 1)$, and wavespeed range of $(1500,5000)$ \cref{eq:lambda_cov} and tested with \cref{OOD:table4}
             ${\boldsymbol{\lambda}}_{\mathrm{OOD3}} = (0.20, 0.20)$, and wavespeed range of $(2000,3500)$
           with the different architectures for multiple realizations of the new GRF
             \emph{out-of distribution},
             realizations of the wave speed. 
             \textit{Left column} shows independent GRF realization of the wave
             speed (see \cref{eq:Helmholtz}).  
             \textit{Second column} shows the real part % and imaginary part
             of the pressure field solution to the wave PDE at frequency
             15 \si{\Hz}, 
             obtained with software \texttt{hawen} \citep{faucher2021hawen}, 
             which we consider as the \emph{reference solution},
             \textit{Other columns} show the approximated reconstructions 
             using the different architectures.
             In each case, we show the real parts of the pressure fields, and 
             the relative error with the reference 
             solution on a logarithmic scale. }\fi
%%%%%%%%%%%%%%%%%%%%%%%%%%%%%%%%%%%%%%%%%%%%%%%%%%%%%%%%%%%%%%%%%%%%%%%%%%%%%%%%
%OOD4
%%%%%%%%%%%%%%%%%%%%%%%%%%%%%%%%%%%%%%%%%%%%%%%%%%%%%%%%%%%%%%%%%%%%%%%%%%%%%%%%
\subsubsection*{OOD 4}
In this set family, we generate an anisotropic random field, different to the original trained data   ${\boldsymbol{\lambda}}_{\mathrm{OOD4}} = (0.10, 0.20)$, however the range was kept to $[2000,3500]$ the same as the original set. 

\begin{figure}[ht!]\centering
\input{Figures/OOD_main_paper/set3/skeleton_online_set3_appendix} 
    \caption{\small{ 
            %\textbf{(Real part) of the pressure field reconstructed by
            %\textbf{(Real part) of the pressure field reconstructed by
            Pressure field reconstructed trained with \emph{isotropic Whittle–Mat\'ern covariance}
             $\boldsymbol{\lambda} = (1, 1)$, and wavespeed range of $(1500,5000)$ \cref{eq:lambda_cov} and tested with \cref{OOD:table4}
             ${\boldsymbol{\lambda}}_{\mathrm{OOD4}} = (0.10, 0.20)$, and wavespeed range of $(2000,3500)$
           with the different architectures for multiple realizations of the new GRF
             \emph{out-of distribution},
             realizations of the wave speed. 
             \textit{Left column} shows independent GRF realization of the wave
             speed (see \cref{eq:Helmholtz}).  
             \textit{Second column} shows the real part % and imaginary part
             of the pressure field solution to the wave PDE at frequency
             15 \si{\Hz}, 
             obtained with software \texttt{hawen} \citep{faucher2021hawen}, 
             which we consider as the \emph{reference solution},
             \textit{Other columns} show the approximated reconstructions 
             using the different architectures.
             In each case, we show the real parts of the pressure fields, and 
             the relative error with the reference 
             solution on a logarithmic scale. 
             }}
             \label{OOD:set4}
\end{figure}
%%%%%%%%%%%%%%%%%%%%%%%%%%%%%%%%%%%%%%%%%%%%%%%%%%%%%%%%%%%%%%%%%%%%%%%%%%%%%%%%
%OOD5
%%%%%%%%%%%%%%%%%%%%%%%%%%%%%%%%%%%%%%%%%%%%%%%%%%%%%%%%%%%%%%%%%%%%%%%%%%%%%%%%
\subsubsection*{OOD 5}
In this set family, we generate an isotropic random field, different to the original trained data   ${\boldsymbol{\lambda}}_{\mathrm{OOD5}} = (0.10, 0.30)$, however the range was also moved to $[2000,6000]$ different than the original set.

\begin{figure}[ht!]\centering
\input{Figures/OOD_main_paper/set4/skeleton_online_set4_appendix} 
    \caption{\small{ 
            %\textbf{(Real part) of the pressure field reconstructed by
            Pressure field reconstructed 
             at \num{15} \si{\Hz} trained with \emph{isotropic Whittle–Mat\'ern covariance}
             $\boldsymbol{\lambda} = (1, 1)$, and wavespeed range of $(1500,5000)$ \cref{eq:lambda_cov} and tested with \cref{OOD:table5}
             ${\boldsymbol{\lambda}}_{\mathrm{OOD5}} = (0.10, 0.30)$, and wavespeed range of $(2000,6000)$
           with the different architectures for multiple realizations of the new GRF
             \emph{out-of distribution},
             realizations of the wave speed. 
             \textit{Left column} shows independent GRF realization of the wave
             speed (see \cref{eq:Helmholtz}).  
             \textit{Second column} shows the real part % and imaginary part
             of the pressure field solution to the wave PDE at frequency
             15 \si{\Hz}, 
             obtained with software \texttt{hawen} \citep{faucher2021hawen}, 
             which we consider as the \emph{reference solution},
             \textit{Other columns} show the approximated reconstructions 
             using the different architectures.
             In each case, we show the real parts of the pressure fields, and 
             the relative error with the reference 
             solution on a logarithmic scale. 
             }}
             \label{OOD:set5}
\end{figure}

%%%%%%%%%%%%%%%%%%%%%%%%%%%%%%%%%%%%%%%%%%%%%%%%%%%%%%%%%%%%%%%%%%%%%%%%%%%%%%%%
%OOD6
%%%%%%%%%%%%%%%%%%%%%%%%%%%%%%%%%%%%%%%%%%%%%%%%%%%%%%%%%%%%%%%%%%%%%%%%%%%%%%%%
\subsubsection*{OOD 6}\label{OOD:6}
In this set family, we generate an anisotropic random field, \emph{significantly} different to the original trained data   ${\boldsymbol{\lambda}}_{\mathrm{OOD6}} = (0.25, 0.75)$, however the range was moved to $[2000,6000]$ different than the original set. 
\begin{figure}[ht!]\centering
\input{Figures/OOD_main_paper/set5/skeleton_online_set5_appendix} 
    \caption{\small{ 
            %\textbf{(Real part) of the pressure field reconstructed by
            Pressure field reconstructed 
             at \num{15} \si{\Hz} trained with \emph{isotropic Whittle–Mat\'ern covariance}
             $\boldsymbol{\lambda} = (1, 1)$, and wavespeed range of $(1500,5000)$ \cref{eq:lambda_cov} and tested with \cref{OOD:table6}
             ${\boldsymbol{\lambda}}_{\mathrm{OOD6}} = (0.25, 0.75)$, and wavespeed range of $(2000,6000)$
           with the different architectures for multiple realizations of the new GRF
             \emph{out-of distribution},
             realizations of the wave speed. 
             \textit{Left column} shows independent GRF realization of the wave
             speed (see \cref{eq:Helmholtz}).  
             \textit{Second column} shows the real part % and imaginary part
             of the pressure field solution to the wave PDE at frequency
             15 \si{\Hz}, 
             obtained with software \texttt{hawen} \citep{faucher2021hawen}, 
             which we consider as the \emph{reference solution},
             \textit{Other columns} show the approximated reconstructions 
             using the different architectures.
             In each case, we show the real parts of the pressure fields, and 
             the relative error with the reference 
             solution on a logarithmic scale. 
             }}
             \label{OOD:set6}
\end{figure}

%-----------------------------------------------------------------------------
\setlength{\modelwidth} {3.8cm}
\setlength{\modelheight}{3.8cm}

\subsection{OOD of the velocity BP 2004.} \label{BP_model}
%-----------------------------------------
To assess the network's ability to handle wave speed that are significantly different from the input distribution (particularly those that deviate from Gaussian measures), we conducted additional tests using the trained networks on a scale version of the velocity model known as the "2004-BP velocity benchmark" \citep{billette20052004}. The source was positioned similarly to the previous experiments, maintaining a frequency of \num{15} Hz, while adjusting the wavespeed's size to accommodate the capabilities of the GPU device. The generated approximations by each network are visualized in \cref{fig:BP}.  
\setlength{\modelwidth} {3.7cm}
\setlength{\modelheight}{3.7cm}
\begin{figure}[!ht] \centering
   \input{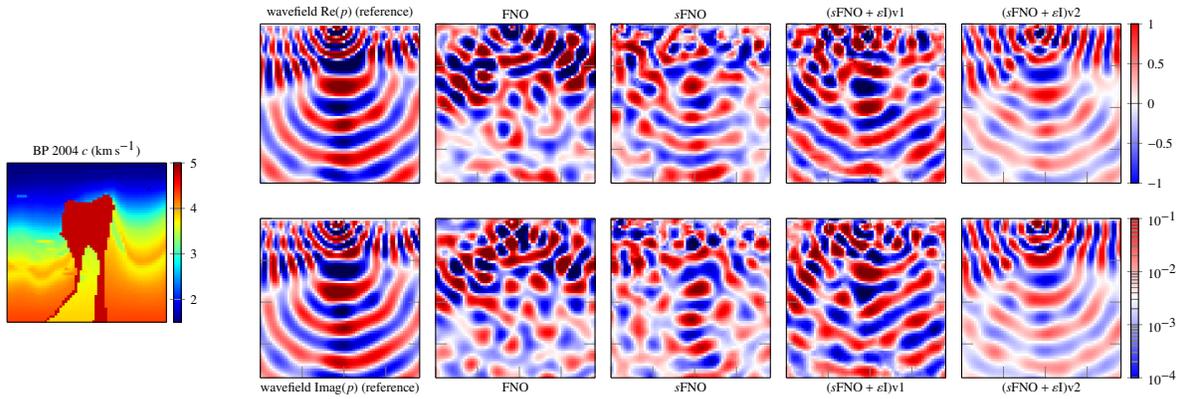}
    \caption{BP 2004 \citep{billette20052004}. Using the networks trained in row $1$ of \cref{table:experiment1:l2error}.}
    \label{fig:BP}
\end{figure}

\newpage
\subsection{Experiments at $7$, $12$ and $15$ Hz.} \label{appendix:more_experiments}
We consider two further datasets, lower frequency with a similar configuration as in \cref{experiment2} at $12$Hz, and an unrealistic case with the source beneath the surface, at $7$Hz, but we increase the size of the domain.

\begin{remark}
  \bourbaki
    Similarly as in \cref{{Experiments}} \emph{we deliberately avoid increasing the epochs of the training algorithm or the size of the training dataset to compensate the network}. 
    
\end{remark}
\newpage
\subsection*{Experiments of $7$ Hz (different configuration).}
\begin{equation}\label{experiment1}
\text{Experiment $7$Hz} \,\left\lbrace\quad\begin{aligned}
 & \text{2D domain of size \num{3.81}$\times$\num{3.81}\si{\km\squared}} \\
 & \text{\num{40000} GRF wave speeds generated, imposing \,
         $\num{1.5}\si{\km\per\second}\,\leq c(x) \leq \num{3}\si{\km\per\second}$} \\
 & \text{The data are $p$ that solve 
         \cref{eq:Euler} at frequency $\omega/(2\pi)=$ \num{7}\si{\Hz}}.
\end{aligned}\right.\end{equation}

Both the wave speeds and the pressure field solution are 
represented on a Cartesian grid of size \num{128}$\times$\num{128}
pixels, that is, using a grid step of \num{30}\si{\meter}.
We illustrate in \cref{figure:cp-GRF:experiment1} a realization of 
the wave speed model and the corresponding pressure field.

% ---------------------------
\setlength{\modelwidth} {3.8cm}
\setlength{\modelheight}{3.8cm}
% ---------------------------
\begin{figure}[h!]\centering
\subfigure[GRF realization of wave speed model $c$ ]{
\renewcommand{\modelfile}{Figures/dataset/GRF/experiment1_cp_scale1500to3000}
  \pgfmathsetmacro{\cmin}{1.50}
  \pgfmathsetmacro{\cmax}{3.00}
  \pgfmathsetmacro{\xmin}{0}
  \pgfmathsetmacro{\xmax}{3.810}
  \pgfmathsetmacro{\zmin}{0}
  \pgfmathsetmacro{\zmax}{3.810}
{\begin{tikzpicture}

\begin{axis}[%  
width=\modelwidth, height=\modelheight,
axis on top, separate axis lines,
xmin=\xmin, xmax=\xmax, xlabel={$x$ (\si{\km})},
ymin=\zmin, ymax=\zmax, ylabel={$z$ (\si{\km})}, y dir=reverse,
xticklabel pos=right, xlabel near ticks,
colormap/jet,colorbar,
colorbar style={title={{\scriptsize{$c$ (\si{\km\per\second})}}},
width=2mm,xshift=-1mm,
},point meta min=\cmin,point meta max=\cmax,
% style of the axis and size -------------------------------------
  label style={font=\scriptsize},
  tick label style={font=\scriptsize},
  legend style={font=\scriptsize\selectfont},
]
\addplot [forget plot] graphics [xmin=\xmin,xmax=\xmax,
                                 ymin=\zmin,ymax=\zmax] {{\modelfile}.png};
\end{axis}
\end{tikzpicture}%
}
}
\begin{tikzpicture}[scale=1]
\draw[line width=2,black,->]  (1.50,0) to node[draw,rectangle,black,anchor=south,text width=7em,align=center,line width=1pt,yshift=0.50em]{from $c$ to $\pressure$ solving \eqref{eq:Euler}} (4.4,0);  
\end{tikzpicture}
\subfigure[Real part of the pressure field at 7 \si{\Hz} frequency.]{
\renewcommand{\modelfile}{Figures/dataset/GRF/experiment1_p_scale1e-1_real_bluered}
  \pgfmathsetmacro{\xmin}{0}
  \pgfmathsetmacro{\xmax}{3.810}
  \pgfmathsetmacro{\zmin}{0}
  \pgfmathsetmacro{\zmax}{3.810}
  \pgfmathsetmacro{\cmin}{-0.1}
  \pgfmathsetmacro{\cmax}{0.1}
{\begin{tikzpicture}

%  \pgfmathsetmacro{\xmin}{0}
%  \pgfmathsetmacro{\xmax}{1270}
%  \pgfmathsetmacro{\zmin}{0}
%  \pgfmathsetmacro{\zmax}{1270}
%  \pgfmathsetmacro{\cmin}{-0.05}
%  \pgfmathsetmacro{\cmax}{0.05}
  
\begin{axis}[%
width=\modelwidth, height=\modelheight,
axis on top, separate axis lines,
xmin=\xmin, xmax=\xmax, xlabel={$x$ (\si{\km})},
ymin=\zmin, ymax=\zmax, ylabel={$z$ (\si{\km})}, y dir=reverse,
xticklabel pos=right, xlabel near ticks,
colormap/jet,colorbar,
colorbar style={title={{\scriptsize{$\pressure$ (\si{\pascal})}}},
width=2mm,xshift=-1mm,
},point meta min=\cmin,point meta max=\cmax,
colormap={gb}{color=(blue) color=(white) color=(red)},
% style of the axis and size -------------------------------------
  label style={font=\scriptsize},
  tick label style={font=\scriptsize},
  legend style={font=\scriptsize\selectfont},
]
\addplot [forget plot] graphics [xmin=\xmin,xmax=\xmax,
                                 ymin=\zmin,ymax=\zmax] {{\modelfile}.png};
\end{axis}
\end{tikzpicture}%
}
}
\caption{Illustration of the full-wave dataset for Experiment~1
         that considers a computational domain of size 
         \num{3.81}$\times$\num{3.81}\si{\km\squared} 
         with a source buried in the domain.
         The wave speed and pressure field are represented 
         on a Cartesian grid of size \num{128}$\times$\num{128}
         with a grid step of \num{30}\si{\meter}.
         The complete dataset corresponds to \num{40000} couples
         made up of a wave speed model and associated acoustic 
         wave.}
\label{figure:cp-GRF:experiment1}
\end{figure}

%\subsection{Multiple realizations $7$Hz}
\begin{figure}[!ht]
    \centering
    \includegraphics[scale= 0.25]{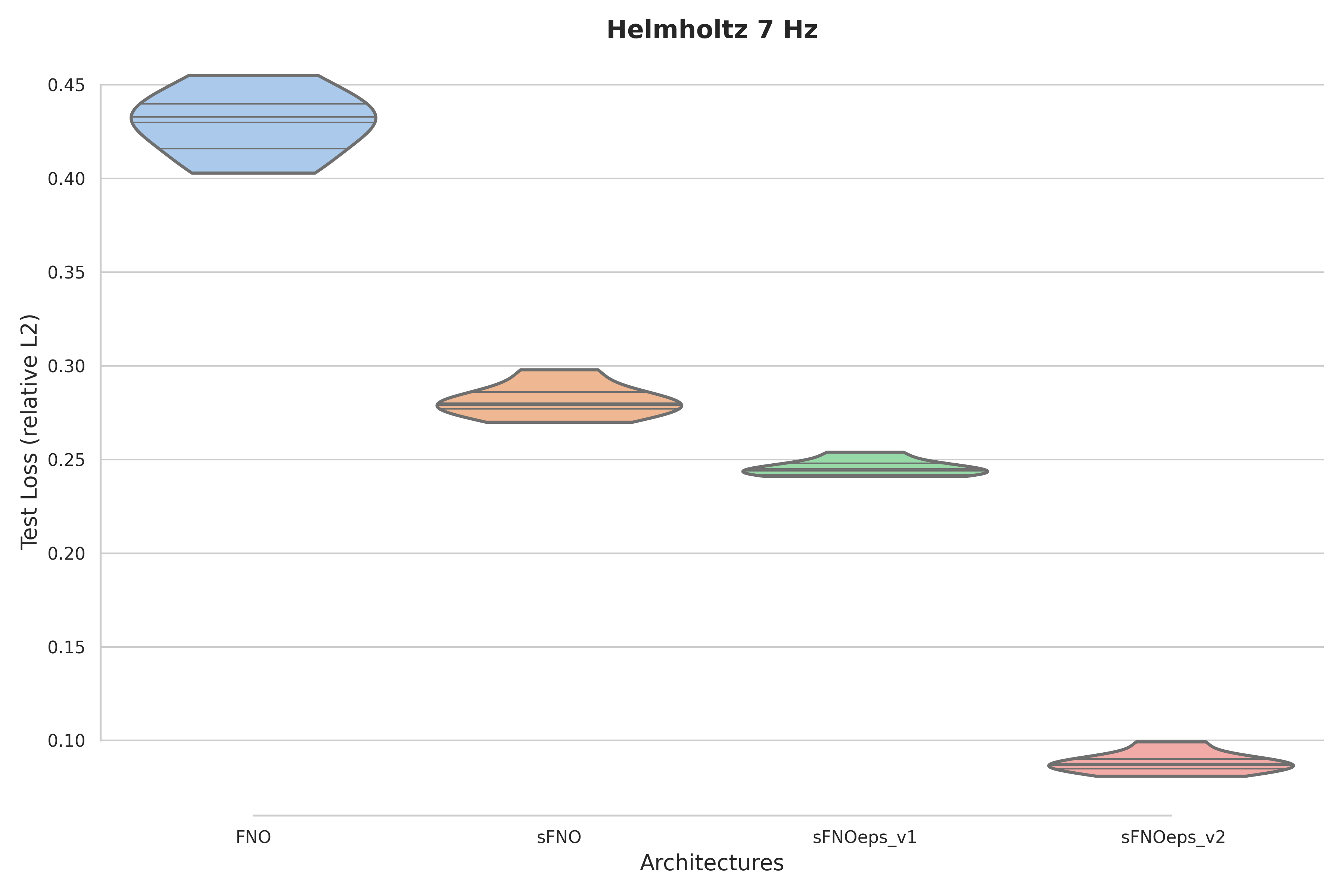}
    \caption{\small{Comparison of test-loss for $\omega/(2\pi) = 7$ Hz. Each architecture is trained $9$ times, the relative $\Lp$-loss, $\|\Gcal^{\mathrm{ref}} - \Gcal^{\mathrm{approx}}\|_{\Lp}/ \|\Gcal^{\mathrm{ref}}\|_{\Lp}$, on the test set is shown in the diagram.}}
    \label{fig:boxplot_12}
\end{figure}

% ---------------------------
\setlength{\modelwidth} {3.7cm}
\setlength{\modelheight}{3.7cm}
\graphicspath{{Figures/results_experiment1/}}
% ---------------------------
%\subsection{Wavefield reconstruction $7$Hz}
\begin{figure}[h!]
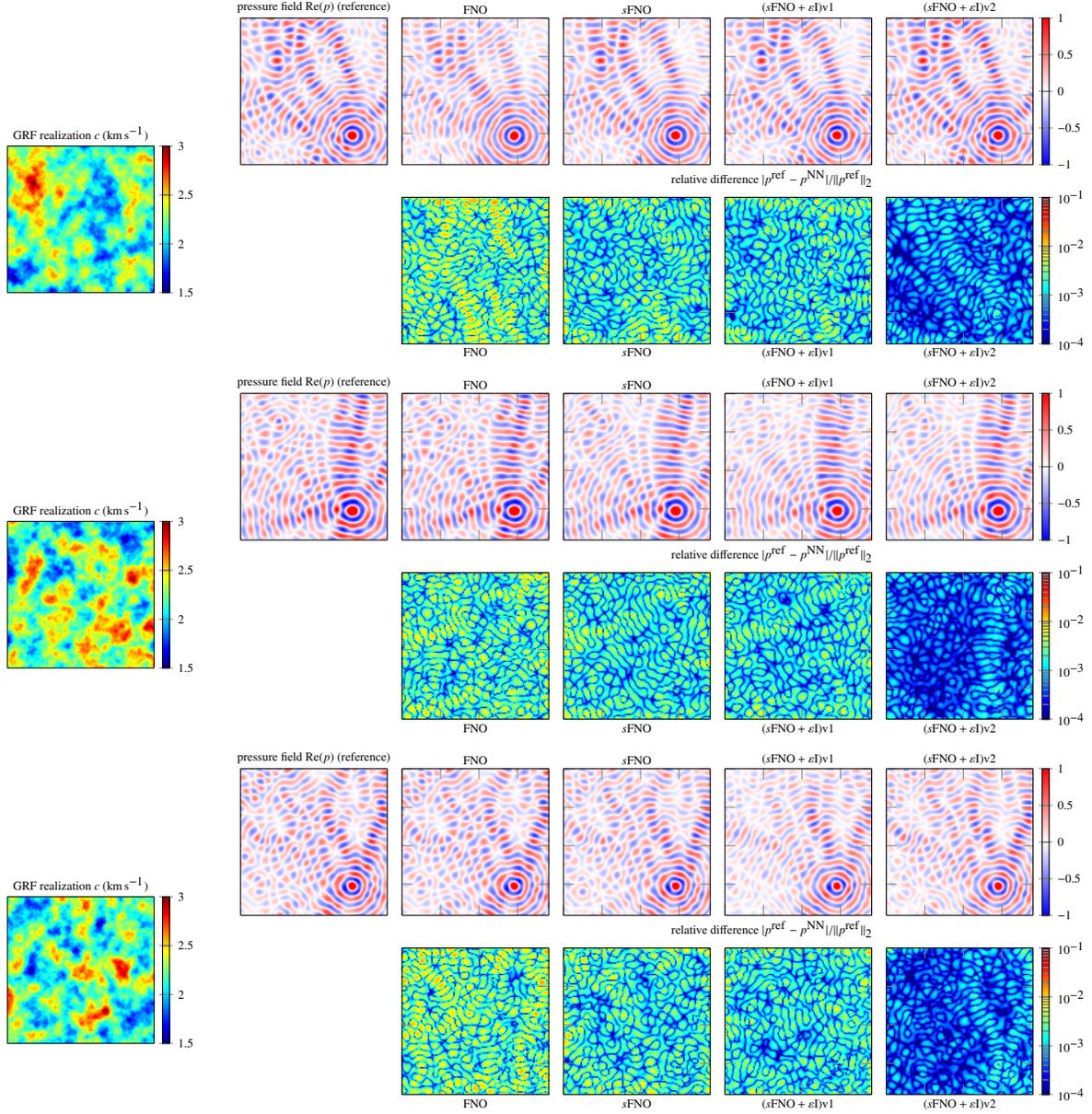
\centering
\input{Figures/results_experiment1/skeleton_oneline_model1} 
\input{Figures/results_experiment1/skeleton_oneline_model2} 
\input{Figures/results_experiment1/skeleton_oneline_model3}

    \caption{\small{ 
            %\textbf{(Real part) of the pressure field reconstructed by
             Pressure field reconstructed  
             at frequency \num{7} \si{\Hz}
             with the different architectures for three test-cases. 
             \textit{First column} shows independent GRF realization of the wave
             speed (see \cref{eq:Helmholtz}).  
             \textit{Second column} shows the solution of the wave PDE
             obtained with software \texttt{hawen} \citep{faucher2021hawen}, 
             which we consider as the \emph{reference solution},
             see \cref{eq:Helmholtz}. 
             \textit{Other columns} show the approximated reconstruction 
             using the different architectures:
             \textit{FNO}, see \citet{kovachki2021neural}; 
             $\multiplicative$ structure ($\MFNO$, see~\cref{MNO}); 
             and the solutions provided by $\FNONeXt$, \cref{FNO+epsilon}.
             In each case, we show the real part of the pressure field, and 
             the relative error with the reference solution using a logarithmic
             scale.
             }}

\label{figure:results-experiment1}
\end{figure}
\newpage

\subsection{Experiments at $12$ Hz.}
\begin{equation}\label{experiment2:12Hz}
\text{Experiment 2} \,\left\lbrace\quad\begin{aligned}
 & \text{2D domain of size \num{1.27}$\times$\num{1.27}\si{\km\squared}} \\
 & \text{\num{40000} GRF wave speeds generated, imposing \,
         $\num{1.5}\si{\km\per\second}\,\leq c(x) \leq \num{5}\si{\km\per\second}$} \\
 & \text{The data are $p$ that solve 
         \cref{eq:Euler} at frequency $\omega/(2\pi)=$ \num{12} \si{\Hz}}.
\end{aligned}\right.\end{equation}

%\subsection{Multiple realizations $12$Hz}
\begin{figure}[!htb]
    \centering
    \includegraphics[scale= 0.25]{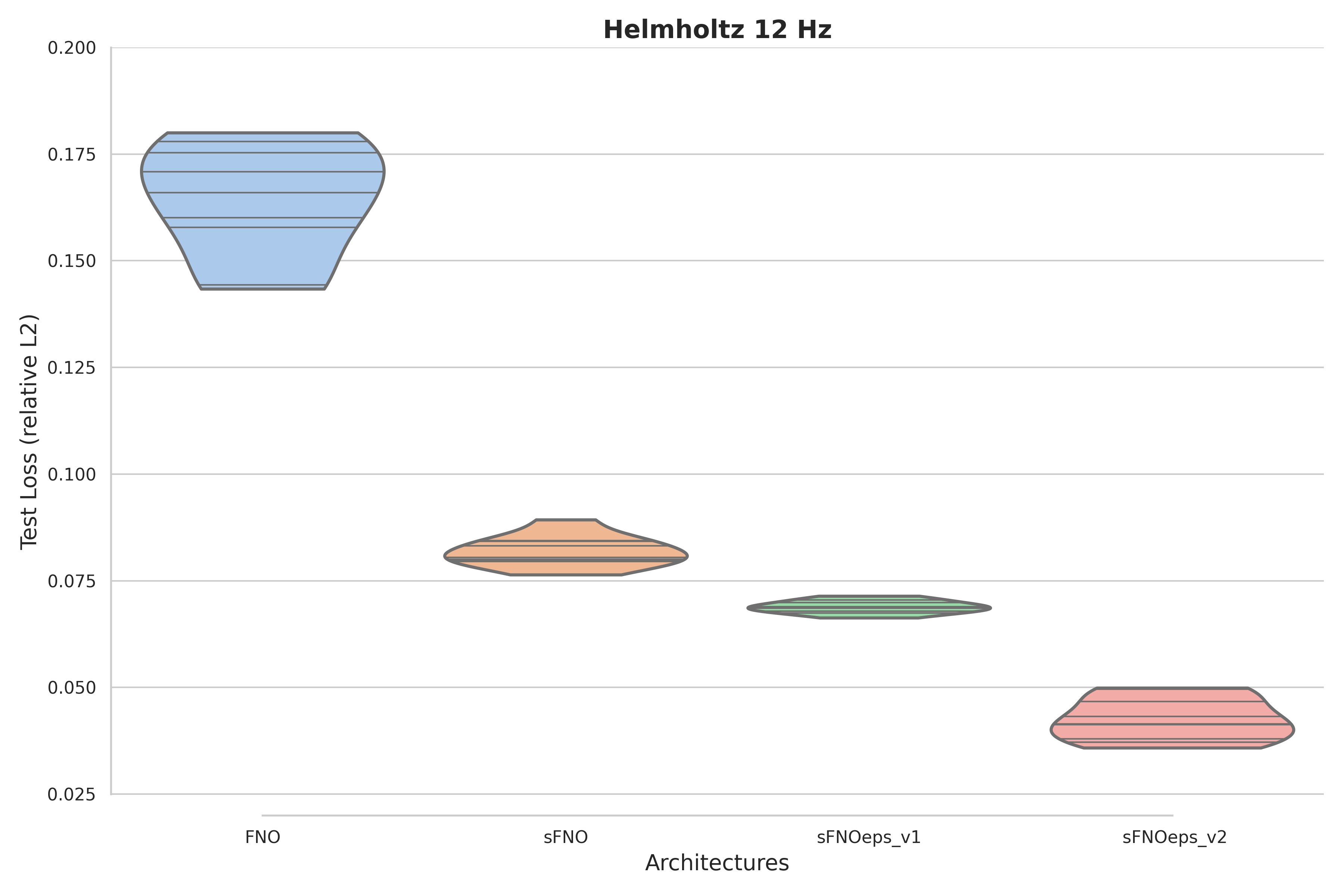}
    \caption{\small{Comparison of test-loss for $\omega/(2\pi) = 12$ Hz. Each architecture is trained $9$ times, the relative $\Lp$-loss, $\|\Gcal^{\mathrm{ref}} - \Gcal^{\mathrm{approx}}\|_{\Lp}/ \|\Gcal^{\mathrm{ref}}\|_{\Lp}$, on the test set is shown in the diagram.}}
    \label{fig:boxplot_12}
    \vspace{2em}
\end{figure}

%\subsection{Wavefield reconstruction $12$Hz}
\graphicspath{{Figures/results_experiment2_12Hz/}}
% ---------------------------
\begin{figure}[ht!]\centering
\input{Figures/results_experiment2_12Hz/skeleton_oneline_model1b} 
\vspace*{2em}

    \caption{\small{ 
            %\textbf{(Real part) of the pressure field reconstructed by
             Pressure field reconstructed  
             at frequency \num{12} \si{\Hz}
             with the different architectures for two GRF 
             realizations of the wave speed. 
             \textit{Left column} shows independent GRF realization of the wave
             speed (see \cref{eq:Helmholtz}).  
             \textit{Second column} shows the real and imaginary parts
             of the pressure field solution to the wave PDE at frequency
             12 \si{\Hz}, 
             obtained with software \texttt{hawen} \citep{faucher2021hawen}, 
             which we consider as the \emph{reference solution},
             see \cref{eq:Helmholtz}. 
             \textit{Other columns} show the approximated reconstructions 
             using the different architectures:
             \textit{FNO}, see \citet{kovachki2021neural}; 
             $\multiplicative$ structure ($\MFNO$, see~\cref{MNO}); 
             and the solutions provided by $\FNONeXt$, \cref{FNO+epsilon}.
             In each case, we show the real and imaginary 
             parts of the pressure fields, and 
             the relative error with the reference 
             solution on a logarithmic scale.
             }}
\label{figure:results-experiment2_12Hz}
\end{figure}

\subsection{Wavefield reconstruction at $15Hz$}\label{Appendux:15Hz_experiment}
\graphicspath{{Figures/results_experiment2_15Hz/}}
% ---------------------------
\begin{figure}[h!]\centering
\input{Figures/results_experiment2_15Hz/skeleton_oneline_model2b}
\vspace*{2em}

    \caption{\small{ 
            %\textbf{(Real part) of the pressure field reconstructed by
             Pressure field reconstructed  
             at frequency \num{15} \si{\Hz}
             with the different architectures for two GRF 
             realizations of the wave speed. 
             \textit{Left column} shows independent GRF realization of the wave
             speed (see \cref{eq:Helmholtz}).  
             \textit{Second column} shows the real and imaginary parts
             of the pressure field solution to the wave PDE at frequency
             12 \si{\Hz}, 
             obtained with software \texttt{hawen} \citep{faucher2021hawen}, 
             which we consider as the \emph{reference solution},
             see \cref{eq:Helmholtz}. 
             \textit{Other columns} show the approximated reconstructions 
             using the different architectures:
             \textit{FNO}, see \citet{kovachki2021neural}; 
             $\multiplicative$ structure ($\MFNO$, see~\cref{MNO}); 
             and the solutions provided by $\FNONeXt$, \cref{FNO+epsilon}.
             In each case, we show the real and imaginary 
             parts of the pressure fields, and 
             the relative error with the reference 
             solution on a logarithmic scale.
             }}
\label{figure:results-experiment2_15Hz}
\end{figure}

\end{document}